\documentclass[a4paper,11pt,twoside]{ThesisStyle}

\usepackage{amsmath,amssymb,amsthm}             % AMS Math
\usepackage{enumerate}
\usepackage[utf8]{inputenc}
\usepackage[T1]{fontenc}
\usepackage[left=1.5in,right=1.3in,top=1.1in,bottom=1.1in,includefoot,includehead,headheight=13.6pt]{geometry}

\usepackage[nottoc, notlof, notlot]{tocbibind}
\usepackage{minitoc}
\usepackage{aecompl}

\usepackage{ifpdf}

\ifpdf
  \usepackage[pdftex]{graphicx}
  \DeclareGraphicsExtensions{.jpg,.png,.pdf}
  \usepackage[a4paper,pagebackref,hyperindex=true]{hyperref}
\else
  \usepackage{graphicx}
  \DeclareGraphicsExtensions{.ps,.eps}
  \usepackage[a4paper,dvipdfm,pagebackref,hyperindex=true]{hyperref}
\fi

\graphicspath{{.}{figures/}}

\renewcommand*{\backref}[1]{}
\renewcommand*{\backrefalt}[4]{%
\ifcase #1 %
(Not cited.)%
\or
(Cited on page~#2.)%
\else
(Cited on pages~#2.)%
\fi}

\usepackage{color}
\definecolor{linkcol}{rgb}{0,0,0.4} 
\definecolor{citecol}{rgb}{0.5,0,0} 

\hypersetup
{
bookmarksopen=true,
pdftitle="Design and Use of Anatomical Atlases for Radiotherapy",
pdfauthor="Olivier COMMOWICK", 
pdfsubject="Creation of atlases and atlas based segmentation", %subject of the document
pdfmenubar=true, %menubar shown
pdfhighlight=/O, %effect of clicking on a link
colorlinks=true, %couleurs sur les liens hypertextes
pdfpagemode=None, %aucun mode de page
pdfpagelayout=SinglePage, %ouverture en simple page
pdffitwindow=true, %pages ouvertes entierement dans toute la fenetre
linkcolor=linkcol, %couleur des liens hypertextes internes
citecolor=citecol, %couleur des liens pour les citations
urlcolor=linkcol %couleur des liens pour les url
}

\usepackage{rotating}                    % Sideways of figures & tables

\usepackage{fancyhdr}                    % Fancy Header and Footer

\let\headruleORIG\headrule
\renewcommand{\headrule}{\color{black} \headruleORIG}

\usepackage{colortbl}
\arrayrulecolor{black}

\fancypagestyle{plain}{
  \fancyhead{}
  \fancyfoot{}
  
}

\makeatletter

\def\cleardoublepage{\clearpage\if@twoside \ifodd\c@page\else%
  \hbox{}%
  \thispagestyle{empty}%              % Empty header styles
  \newpage%
  \if@twocolumn\hbox{}\newpage\fi\fi\fi}

\makeatother

\let\minitocORIG\minitoc
\renewcommand{\minitoc}{\minitocORIG \vspace{1.5em}}

\usepackage{multirow}
\usepackage{slashbox}

\usepackage{marginnote}
\usepackage{etoolbox}

\makeatletter
\patchcmd{\@mn@margintest}{\@tempswafalse}{\@tempswatrue}{}{}
\patchcmd{\@mn@margintest}{\@tempswafalse}{\@tempswatrue}{}{}
\reversemarginpar 
\makeatother

\usepackage{xcolor}
\colorlet{lightblue}{blue!5}
\usepackage{mdframed}
\AtBeginEnvironment{theorem}{%
\begin{mdframed}[backgroundcolor=lightblue, linewidth=1.5pt, nobreak=true]
}
\AtEndEnvironment{theorem}{%
\end{mdframed}
}

\AtBeginEnvironment{property}{%
\begin{mdframed}[backgroundcolor=lightblue, linewidth=0pt, nobreak=true]
}
\AtEndEnvironment{property}{%
\end{mdframed}
}

\AtBeginEnvironment{lemma}{%
\begin{mdframed}[backgroundcolor=lightblue, linewidth=0pt, nobreak=true]
}
\AtEndEnvironment{lemma}{%
\end{mdframed}
}

\usepackage[font=it]{caption}

\newtheorem{theorem}{Theorem}[section]
\newtheorem{definition}{Definition}[section]
\newtheorem{lemma}{Lemma}[section]

\newtheorem{property}{Property}[section]

\theoremstyle{definition}
\newtheorem{example}{Example}[section]

\usepackage[Algorithm]{algorithm}
\usepackage{algpseudocode}

\renewcommand{\epsilon}{\varepsilon}
\newcommand{\vmax}{{\overline{v}}}
\newcommand{\uL}{{\underline{u}}}
\newcommand{\uH}{{\overline{u}}}
\newcommand{\RR}{\mathbb{R}}
\newcommand{\NN}{\mathbb{N}}
\newcommand{\dt}{\Delta t}

\newcommand{\robots}{\mathcal{R}}
\newcommand{\im}[1]{{\mathrm{Im}\left(#1\right)}}
\newcommand{\xobsmax}{\overline{x}^\mathrm{\hspace{0.2mm}obs}}
\newcommand{\xobsmin}{\underline{x}^\mathrm{\hspace{0.2mm}obs}}
\newcommand{\free}{\mathrm{free}}
\newcommand{\chifree}{\chi^{\mathrm{free}}}
\newcommand{\phifree}{\paths(\chifree)}
\newcommand{\phifreeG}{\paths(\chifree_G)}
\newcommand{\chiobs}{\chi^{\mathrm{obs}}}
\newcommand{\chigoal}{\chi^{\mathrm{goal}}}

\newcommand{\controls}{\mathbf{U}}
\newcommand{\vcontrols}{\mathbf{V}}
\newcommand{\e}{\mathbf{e}}

\newcommand{\C}{\mathcal{C}}
\newcommand{\imp}{\mathrm{impulse}}
\newcommand{\entry}{{\mathrm{entry}}}
\newcommand{\exit}{{\mathrm{exit}}}
\renewcommand{\stop}{{\mathrm{stop}}}
\newcommand{\cl}[1]{\mathrm{cl}\left(#1\right)}
\newcommand{\uLbf}{\underline{\mathbf{u}}}
\newcommand{\uHbf}{\overline{\mathbf{u}}}
\newcommand{\vmaxbf}{\mathbf{\vmax}}
\newcommand{\ubf}{\mathbf{u}}
\newcommand{\vbf}{\mathbf{v}}
\renewcommand{\stop}{{\mathrm{stop}}}
\newcommand{\graphs}{\mathcal{G}}

\newcommand{\homotopy}{\mathcal{H}^\mathrm{free}}
\renewcommand{\path}{\varphi}
\newcommand{\paths}{\Phi}
\newcommand{\dbf}{\mathbf{d}}
\newcommand{\dsignals}{\mathbf{D}}
\newcommand{\dmin}{\underline{d}}
\newcommand{\dmax}{\overline{d}}
\newcommand{\dbfmin}{\mathbf{\underline{d}}}
\newcommand{\dbfmax}{\mathbf{\overline{d}}}
\newcommand{\ybf}{\mathbf{y}}

\newcommand{\Imin}{\underline{I}}
\newcommand{\supi}{\widetilde{\sup}^i}
\newcommand{\Ybf}{\mathbf{Y}}
\newcommand{\cycles}{\mathrm{cycles}}
\newcommand{\phibounded}{\underline{\overline{\paths}}}
\newcommand{\Cfree}{C^{\mathrm{free}}}
\newcommand{\Cobs}{C^{\mathrm{obs}}}
\newcommand{\rmax}{\overline{r}}
\newcommand{\dql}{\Delta Q^\mathrm{lim}}
\newcommand{\orthant}{\mathcal{O}}
\newcommand{\ball}{\mathcal{B}}

\usepackage{appendix}
\usepackage{chngcntr}
\usepackage{etoolbox}
\AtBeginEnvironment{subappendices}{%
\counterwithin{lemma}{section}
\counterwithin{equation}{section}
}

\usepackage{framed}
\usepackage{xcolor}

\usepackage{comment}

\usepackage{lmodern}
\usepackage{ae,aecompl}
\usepackage{graphicx}
\usepackage{eso-pic}	% Nécessaire pour mettre des images en arrière plan
\usepackage{array}		% Permet d'écrite 'THESE' de haut en bas
\input{cover-french}
\author{Jean \textsc{Gregoire}}
\title{Coordination de robots mobiles par affectation de priorités\\\ \\Priority-based coordination of mobile robots}
\ED{\'{E}cole doctorale n\textsuperscript{o} 432 : Sciences des M\'{e}tiers de l'Ing\'{e}nieur}
\doctorat{Doctorat ParisTech}
\specialite{Informatique temps réel, \\robotique et automatique}
\directeura{Arnaud \textsc{de La Fortelle}}
\directeurb{Silvère \textsc{Bonnabel}}
\president{Denis \textsc{Gillet}}
\date{29 septembre 2014}

\jurya{Mme Domitilla Del Vecchio}{Professeur assistante, MIT}{Rapporteur}
\juryb{M. Thierry Fraichard}{Chargé de recherche, INRIA Grenoble Rhône-Alpes}{Rapporteur}
\juryc{M. Silvère Bonnabel}{Maître Assistant, MINES ParisTech}{Examinateur}
\juryd{M. Emilio Frazzoli}{Professeur, MIT}{Examinateur}
\jurye{M. Denis Gillet}{Maître d'enseignement et de recherche, EPFL}{Examinateur}
\juryf{M. Arnaud de La Fortelle}{Professeur, MINES ParisTech}{Examinateur}
\juryg{M. Jean-Paul Laumond}{Directeur de Recherche, LAAS}{Examinateur}
\ecole{l'\'{E}cole Nationale Supérieure des Mines de Paris}
\adresse{
	\textbf{MINES ParisTech\\ Centre de Robotique (CAOR)}\\	60 Bd Saint Michel, 75006 Paris, France
}

\usepackage{bookmark}

\usepackage{pdfpages}

\begin{document} 

%TITLE PAGE
\pagedegarde

\dominitoc
\doparttoc

\renewcommand{\mtcgapbeforeheads}{20pt}
\renewcommand{\mtcgapafterheads}{10pt}

\pagenumbering{roman}

\chapter*{\centering Abstract}
\addcontentsline{toc}{chapter}{Abstract}\mtcaddchapter      

\fancyhead[RE]{\bfseries\nouppercase{Abstract}}      % Chapter in the right on even pages
\fancyhead[LO]{\bfseries\nouppercase{Abstract}}     % Section in the left on odd pages 

Since the end of the 1980’s, the development of self-driven autonomous vehicles is an intensive research area in most major industrial countries. Positive socio-economic potential impacts include a decrease of crashes, a reduction of travel times, energy efficiency improvements, and a reduced need of costly physical infrastructure. Some form of vehicle-to-vehicle and/or vehicle-to-infrastructure cooperation is required to ensure a safe and efficient global transportation system. This thesis deals with a particular form of cooperation by studying the problem of coordinating multiple mobile robots at an intersection area. Most of coordination systems proposed in previous work consist in planning a trajectory and to control the robots along the planned trajectory: that is the plan-as-program paradigm where planning is considered as a generative mechanism of action. The approach of the thesis is to plan priorities -- the relative order of robots to go through the intersection -- which is much weaker as many trajectories respect the same priorities. More precisely, priorities encode the homotopy classes of solutions to the coordination problem. Priority assignment is equivalent to the choice of some homotopy class to solve the coordination problem instead of a particular trajectory. Once priorities are assigned, robots are controlled through a control law preserving the assigned priorities, i.e., ensuring the described trajectory belongs to the chosen homotopy class. It results in a more robust coordination system – able to handle a large class of unexpected events in a reactive manner – particularly well adapted for an application to the coordination of autonomous vehicles at intersections where cars, public transport and pedestrians share the road.

\vspace{0.5cm}

\noindent\textbf{Keywords:} mobile robots, multi robot systems, motion planning, coordination space, priority graph, homotopy class, safety, robustness, hybrid architecture

\cleardoublepage
\fancyhead[RE]{\bfseries\nouppercase{\leftmark}}      % Chapter in the right on even pages
\fancyhead[LO]{\bfseries\nouppercase{\rightmark}}     % Section in the left on odd pages

\chapter*{\centering Acknowledgment}
\addcontentsline{toc}{chapter}{Acknowledgment}\mtcaddchapter  

Special thanks to my thesis directors Arnaud de La Fortelle -- who hired me even though I criticized the city he lives in with his wife during the job interview -- and Silvère Bonnabel -- who supposedly accepted to work with me because we are both from Marseille. Discussions with you around drawings and equations have been dreamy.\\\\

Dreamy has also been the celebratory cocktail after my defense organized by the woman of my life, Mélodie, so many thanks for this, so many thanks as well for our intense and illuminating debates around perfomativity and for all your support and the fairy you have brought me all along the thesis.\\\\

Fairy has not necessarily been the exact thing that I have brought my officemates the past three years. You -- literally -- gave me reasons to think however that you have appreciated my coordinated smarties at their fair value. Vincent, Raoul, Amaury, Eva, Edgar, Tao Jin, Philippe, Fernando, Jean-François, Fabien, Sébastien, thanks a lot for your happiness, your savoir vivre, and our -- more or less -- deep discussions.\\\\

All my thanks also for the two incredible ski trips that I had with the lab. Thanks in particular to Vincent and Sylvain, ski and after-ski will be remembered. Thanks also to Arnaud, director of the lab as well as of my thesis, for accepting to come with us despite our insane cruelty and his defeat during snow battles.\\\\

I complained a few times for not having the right to go to the canteen of the Ministry of Research for administrative reasons. It allowed me however -- in addition to slightly enlarge by eating sandwiches instead of healthy meals  -- to spend a very nice lunch time in the garden of the school all along the thesis. Thanks in particular to Raoul, Axel, Martin, Xiangjun, Victorin, Tony, Zhuowei, Jorge, Laure, Olivier, Victor, Martyna, Bruno.\\\\

I must also thank Christine and Christophe, who learned me how to deal with administrative stuff like a superhero.\\\\

Having a coffee break has been so nice, always different, and let's say surprising, thanks to the art of Thérèse in managing the cafeteria and connecting people.\\\\

So many thanks to my family for their support in my choice of doing research, and in particular to my mother~\cite{Gregoire1978} who inspired me doing research when I was just a kid, and to my father who constantly recalls me to work on what I love to do.\\\\

I also want to thank all the members of my thesis jury. It has been a pleasure that you have felt and shared my enthusiasm around my research work.\\\\

Many thanks also to newly arrived PhD students Xiangjun and Sébastien, for sharing this enthusiasm, discussing, continuing and going much ahead my research work.\\\\

So many thanks to my friends who took half a day off to listen to what I have loved to do the three past years and who transformed Mines ParisTech into the best cocktail bar in Paris ever just for the celebratory cocktail of my defense.\\\\

Finally, thanks to Mines ParisTech. After three years as an engineer student and three years as a PhD student, I need to cut the cord for some time. I believe you brought me as much knowledge, people, happiness, friendship and love as you could have.

\cleardoublepage
\thispagestyle{plain}
\par\vspace*{.35\textheight}{\centering \textit{To Mélodie and in honor of Fairy}\par}

\tableofcontents

\mainmatter

%%%%%%% BEGIN - chapters

\chapter{Introduction}
\label{chap:intro}
\vspace{-1.2cm}
\includegraphics[height=1.0\linewidth,angle=-90,trim=160 60 160 60, clip]{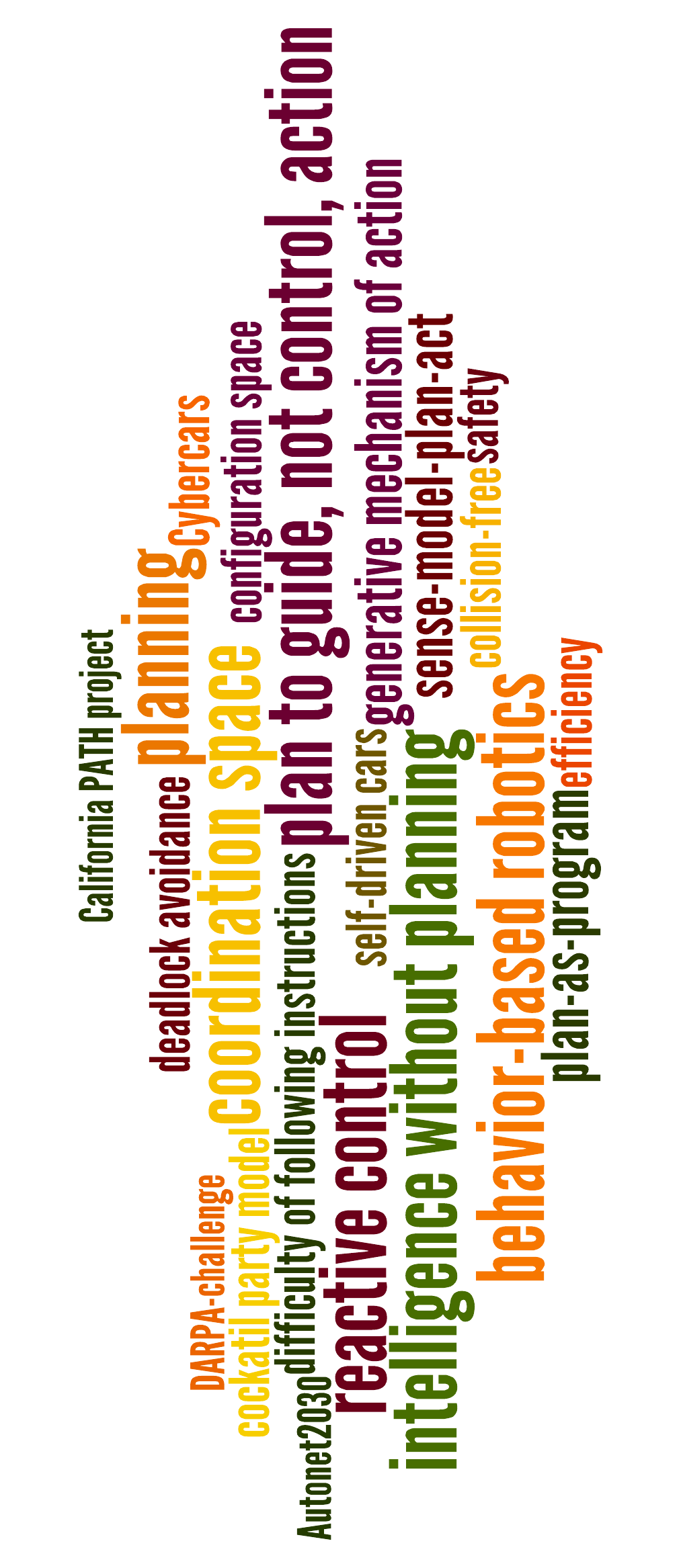}\hfill
\minitoc

\section{Industrial motivation}

The last decades have seen a number of projects addressing the automation of vehicles. The California PATH AHS project (1988-2003) was interested in making progress in automated highway systems~\cite{Alvarez1997,Horowitz2000} with about 600 person-years of effort invested~\cite{Shladover2008}. The European CityMobil project, finalized in 2011, addressed the integration of automated transport systems in the urban environment~\cite{CityMobil2011,Dijke2011} and the cooperation is continuing through CityMobil2 started in September 2012 for 4 years and involving 45 partners from system suppliers, city authorities, the research community and networking organizations~\cite{CityMobil2-2012}. The European interactIVe project, finalized in 2012, focused more on  advanced driver assistance systems (ADAS) for collision avoidance by active intervention in intelligent vehicles~\cite{Etemad2013}. The DARPA challenge, a prize competition for American autonomous vehicles, funded by the Defense Advanced Research Projects Agency, has also stimulated innovation and research in vehicles automation (see, e.g.,~\cite{Thrun2006}). 

All these research projects are funded thanks to high expectations in terms of economic and social impacts. A recent report~\cite{Wallace2012}, jointly written by a consulting company together with an automotive research center, presents self-driving car as the next revolution in the automotive industry. Car automation is expected to decrease crashes, to reduce the need for physical costly infrastructures, to create new models of shared mobility, to reduce and anticipate travel times, to improve productivity, to improve energy efficiency: a combination of social and economic positive impacts. To the authors of the report, it is clear that such disruptive change in the automotive industry opens opportunities for new players and requires all companies of the sector to embrace innovation or to be left behind. 

As autonomous vehicles are now starting to be deployed, cooperation among autonomous vehicles and also between autonomous and human-driven vehicles is necessary. This is the topic of the European project Autonet2030 \cite{Autonet2030}, just started in 2014. Many use cases require cooperation: lane change negotiation, overtaking, cooperative routing, or cooperative speed control. This thesis focuses on the coordination of autonomous vehicles at intersections. Two main goals motivate the research in this topic. The first one is to avoid crashes due to collisions that occur particularly at intersections and because of human error (the leading factor in most of road accidents~\cite{Treat1977,NCSA2004}). The second one is to enhance road traffic efficiency, given that intersections represent bottlenecks in the traffic network resulting in congestion, one of the major problems in today's metropolitan transportation networks. As the results provided in this thesis can be applied to multiple domains including self-driving cars, we will use the more generic term robot instead of vehicle. We consider the problem of coordinating a collection of cooperative mobile robots at an intersection area, that is a region of space with a high concentration of potential collisions. According to the taxonomy proposed in~\cite{Farinelli2004}, we propose to build a strongly coordinated multi robot control system aiming at ensuring safety and efficiency at intersection areas.

\section{Plan or react ?}

Since the 1980's, there is strong debate in the research community on the place of planning and reactive control in the design of autonomous robots~\cite{Elsaesser1994}. For a while, the dominant view in the Artificial Intelligence community was that all the intelligence of an autonomous robot lies in its planning capabilities. On the other hand, Brooks, with the introduction of Subsumption architecture~\cite{Brooks1986}, gave birth to a departure from the traditional planning approach, repudiating plans, convinced that intelligent autonomous robots can be designed through simple interconnected primitive reactive behaviors. Research work involving researchers from Robotics, Artificial Intelligence and Sociology, attempts to conciliate both camps by providing a new view on what planning is. Planning is proposedly considered as the generation of resources to guide action~\cite{Agre1990}, not as a generative mechanism of action. In the sequel, the three approaches introduced above are presented in more details and a literature review of the coordination of multiple robots is provided through the prism of the long-standing debate around the relative place of planning and reactive control.

\subsection{Planning as a generative mechanism of action}

\paragraph{The Sense-Model-Plan-Act paradigm} is the traditional approach of robot control in Artificial Intelligence with four components executed in a serial fashion~\cite{Medeiros1998}. The sensing system receives raw sensor input. The sensing data is turned into a world model by the modeling system. Provided a world model, the planner is in charge of computing a sequence of actions in order to achieve some goal: this step is time consuming and requires reasoning about the future. Finally, a low-level controller executes the plan. This traditional approach to planning, referred as plan-as-program~\cite{Agre1990}, considers planning as "a generative mechanism of action"~\cite{Suchman1986} as the planner dictates the actions to take in the future. 

\paragraph{Reservation-based autonomous intersection management} One of the most known autonomous intersection management system, proposed by Dresner and Stone~\cite{Dresner2008-multiagent-approach,Dresner2004} (see a screen-shot of the simulator in Figure~\ref{fig:dresner-stone-aim}), espouses the plan-as-program paradigm.
\begin{figure}[ht]
\begin{center}
\includegraphics[width=0.5\linewidth]{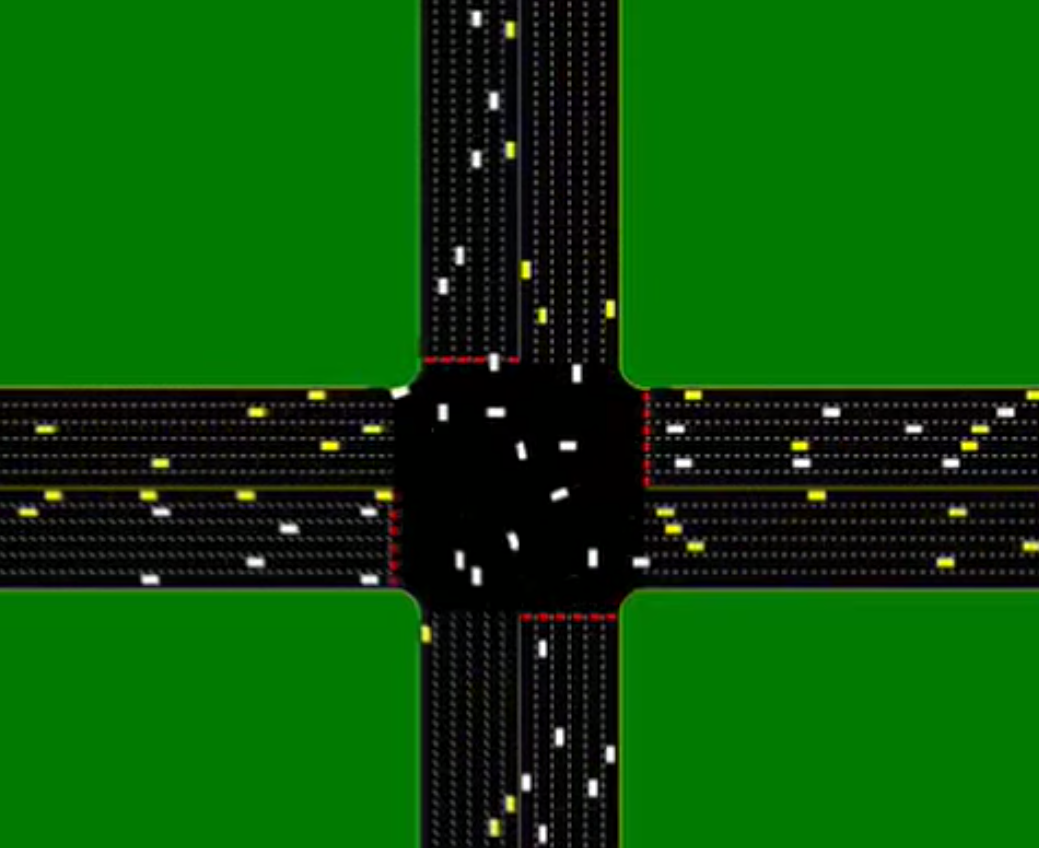}\hfill
\end{center}
\caption{A screen-shot of the "Autonomous Intersection Management" simulator developed by the team of the Department of Computer Sciences at the University of Texas at Austin, USA.}
\label{fig:dresner-stone-aim}
\end{figure}
The system is based on requests sent by vehicles to a central agent in order to reserve space-time regions. Basically, a region of the space where lies the intersection is reserved for a given vehicle during a certain time interval. The central agent ensures that accepted reservations are compatible with each other. Safety is ensured as long as all vehicles respect the specification of their accepted requests, i.e., the specified starting time and velocity profile through the intersection. Hence, requests are the representation of the plan-as-program that robots must execute. The approach has been widely studied with several incremental improvements (see, e.g.,~\cite{Dresner2008-mitigating,Hausknecht2011,Au2011}). Variants of the approach have been developed by several authors. In particular,~\cite{AsisGarciaCollado2010,Mehani2007} introduce critical points in order to improve the precision of the reservation system without increasing its complexity. In~\cite{I.Zohdy2012}, environment variables are taken into account at the planning phase. Reference~\cite{JingmanFan2011} proves that a reservation request can be processed in constant time provided the velocity profiles of vehicles is fixed. Experiments with real vehicles (Cybercars~\cite{Cybercars2006}) using a very simple reservation system are presented in~\cite{Kolodko2003}.

\paragraph{Motion planning in the configuration space} In Robotics, motion planning using the configuration space approach also espouses the plan-as-program paradigm. Reference~\cite{Lozano-Perez1980} introduced the notion of configuration space in order to formalize the traditional motion planning problem. Basically, each dimension of the configuration space represents a degree of freedom, and there is an obstacle region in the configuration space which is the set of forbidden configurations for the robot, to model the presence of a static obstacle or some constraints due to the geometry of the robot (think of a robot with multiple arms). The traditional motion planning problem consists in finding a collision-free path in the configuration space from specified initial/goal configurations. A multi robot system can be considered as a generic robot whose configuration space is the Cartesian product of the configuration space of each robot~\cite{Barraquand1991,Schwartz1983}. The obstacle region then contains  forbidden configurations of each robot, plus forbidden composite configurations to account for possible inter-robot collisions. In this framework, the multi robot motion planning problem consists in finding a collision-free path from a composite start configuration (the start configuration of all robots) to a composite goal configuration (the goal configuration of all robots). Many methods have been devised in order to find a path in a constrained configuration space. For a system of two robots, a shortest path algorithm using the concept of visibility graph is proposed in~\cite{Sheng2006}. In~\cite{Suh1988,Sutton1991}, the authors show how to use dynamic programming to solve motion planning problems. Sampling based methods have also demonstrated their efficiency when the number of degrees of freedom is reasonable. Partial motion planning (see, e.g.,~\cite{Petti2005-partial-motion-planning,Benenson2008}) samples the action space and chooses the control to apply considering only a finite horizon, guaranteeing a bounded computation time. Other sampling methods include probabilistic roadmaps~\cite{Kavraki1996,Kim2003}, which have been applied to multiple robot motion planning in, e.g.,~\cite{Svestka1995}. An improvement of the probabilistic roadmaps, particularly useful for nonholonomic robots is the Rapidly-Exploring Random Trees~\cite{LaValle1998-RRT} (an enhanced provably "optimal" version is proposed in~\cite{Karaman2010}). All the above methods do not scale with the number of robots. Finding a path in the composite configuration space is of high computational complexity and becomes unfeasible in practice for a large number of robots~\cite{Hopcroft1984}.

In~\cite{Kant1986}, a path-velocity decomposition allowing to reduce the problem's complexity is proposed. In this setting, each robot is assumed to move along a predefined path in its own configuration space and then the velocity profiles of the robots along their assigned paths are optimized. The configuration of each robot boils down to its curvilinear position on its path and the configuration space of the whole system is called the coordination space. It is a $n$-dimensional space where $n$ denotes the number of robots going through the intersection. To prevent collisions between robots, some configurations of the coordination space must be excluded: they constitute the so-called obstacle region. Such approaches based on the coordination space turn the coordination problem into the geometric problem of searching a collision-free path for a composite robot in a $n$-dimensional space where the obstacle region has a cylindrical shape~\cite{LaValle2006,Leroy1999} (see Figure~\ref{fig:guo-cylindrical-structure}).
\begin{figure}[!htbp]
\begin{center}
\includegraphics[width=0.7\linewidth]{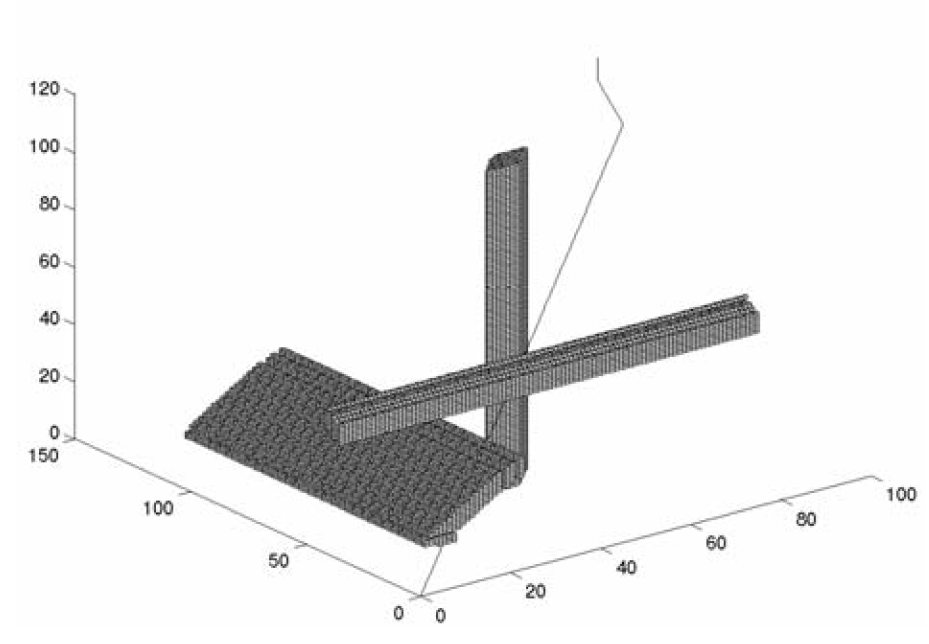}\hfill
\end{center}
\caption{The cylindrical obstacle region in the coordination space for the coordination of three robots and a collision-free path (courtesy of~\cite{Guo2002}).}
\label{fig:guo-cylindrical-structure}
\end{figure}
Even though some authors highlight some cases where the approach fails~\cite{Saha2006}, the approach has become standard in motion planning~\cite{Latombe1991,LaValle1996-optimal,Leroy1999,Lozano-Perez1980, Fraichard1989, Guo2002}. For two-robot systems, optimal solutions have been proposed~\cite{Shin1992,Bien1992,Chang1994}. Reference~\cite{Sharma2005} studies the time complexity of the coordination problem, defined as the completion time, i.e., the time for the last robot to reach its destination: lower and upper bounds are provided. The papers~\cite{Ghrist2005,Ghrist2006} study the problem of finding Pareto-optimal trajectories, i.e., each robot tries to optimize its own particular objective function. They  propose to first discretize the coordination space, and then to take advantage of the cylindrical structure to turn the coordination space into a negatively curved discrete space. Uniqueness of locally Pareto-optimal trajectories in each homotopy class of trajectories appears then as a mere consequence of the uniqueness of geodesics linking two points in a hyperbolic space. However, enumerating all locally optima in each homotopy class to find a globally optimal trajectory is a problem of high combinatorial complexity, and the authors point out the solution proposed  is of interest only with a few robots and a low degree of intersection. The complexity of searching a trajectory in the configuration space led researchers to develop the so-called prioritized motion planning method.

First introduced in~\cite{Erdmann1987}, prioritized motion planning avoids the complexity of searching a trajectory in the $n$-dimensional coordination space. Instead of directly searching a trajectory for the composite robot in the coordination space, it consists in planning the trajectory of each robot sequentially. Each robot is mapped to a real number called the priority of the robot, and the trajectory of each robot is planned, in order of decreasing priorities: robots for which motion has already been planned are considered as dynamic obstacles~\cite{VanDenBerg2005}. The approach has been widely and successfully utilized. The assignment of the priorities is key to the quality of the planned trajectory, e.g., with regards to the delay due to coordination. For $n$ robots, since priorities are sequential, there are $n!$ possible priority schedules. In~\cite{Bennewitz2001, Bennewitz2002}, a (randomized) search is proposed to optimize the prioritization scheme,~\cite{Regele2006} and~\cite{VanDenBerg2005} define simple heuristics for priority adjustment, and the heuristics of~\cite{Clark2002} dynamically updates the priorities of robots. Even if prioritized motion planning is not explicitly mentioned, the approach of~\cite{Dresner2008-multiagent-approach, Mehani2007} for autonomous intersection management also belongs to the family of prioritized motion planning, because the trajectory of robots are planned sequentially. In~\cite{Bekris2009}, a prioritized planning is implemented using a multiagent system approach and taking into account communication aspects.

In~\cite{Akella2002,Peng2003,Peng2005}, a collision-time formulation is proposed, it formulates the motion planning problem as a mixed linear programming (MILP) problem. Every robot is assumed to follow a path with a fixed velocity profile. Hence, the motion planning problem boils down to decide the starting time of each robot along the assigned trajectory (fixed path and velocity profile). The MILP formulation enables to solve the motion planning problem with efficient standard tools for MILP problems, so that as many as 20 robots can be coordinated, according to the authors. 

\paragraph{On the difficulty of following instructions} As noticed in~\cite{Suchman1986}, the main weakness of the plan-as-program approach resides in the inherent "difficulty of following instructions" in the face of environmental uncertainty and unpredictability. It has rapidly become admitted in the Artificial Intelligence community that unpredictability makes open-loop plan execution inefficient and leads to undesired behaviors. Replanning through time to account for new information is an attempt to treat this issue (see, e.g.,~\cite{Bekris2007,Bekris2009} for dynamic replanning of a multi robot system). However, planning is a time consuming task and constantly replanning makes difficult to respect real time execution constraints. These difficulties are at the origin of a completely opposite approach. Instead of considering that intelligence lies in the planning phase, a community of researchers initiated by the seminal work of Brooks~\cite{Brooks1986} tried to design intelligent robots that do not rely on planning at all.

\subsection{Intelligence without planning} 
\label{subsec:intelligence-without-planning}

\paragraph{Reactive and behavior-based robotics} In~\cite{Brooks1986}, Brooks proposed the foundation of what became behavior-based robotics. The main source of novelty is to abandon a centralized and centrally manipulated representation of the world~\cite{Mataric1999}. Instead, the robot control system is layered with several behaviors, each one achieving and/or maintaining a specified goal, e.g., "avoid-obstacles", "go-home". A behavior has either absolutely no world model and no internal state in which case the behavior is purely reactive mapping sensor data to actions, or it has its own minimal internal state maintained only in order to achieve its own goal. Behaviors are simple enough to run in real-time, they run in parallel and are layered so that the capability of the system increases as new behaviors are introduced. The intelligence of the autonomous system is not necessarily obvious when looking at every individual behavior. However, so-called emergent behaviors -- intelligence "in the eye of the beholder"~\cite{Werger1999} -- originate from the large amount of interactions between behaviors in the environment. The belief of the research community advocating for behavior-based design is that their approach can scale to higher level of complexity than many researchers of the plan-as-program school assume and result in more efficient and robust systems than through traditional planning~\cite{Werger1999}. 

\paragraph{The "cocktail party" model} is a simple example of such a reactive approach to the coordination of multiple robots. In this setting proposed in~\cite{Lumelsky1997}, a robot can only sense the surrounding objects, it knows its current and its target position, it can distinguish between static obstacle and robots and can sense the instantaneous motion of other robots. A reactive coordination of robots is proposed with only these capabilities and without mutual communication. The authors claim the obtained system demonstrates good performance and a remarkable robustness. The idea of the proposed algorithm is based on maze-searching techniques. Robots follow the boundary of static obstacles. For moving obstacles (including other robots), a collision front is build considering the maximal motion of the moving obstacles. Then, as for static obstacles, the boundary of the collision front is followed, ensuring collision avoidance. The term "cocktail party" is justified by the analogy with the behavior of a guest willing to talk to someone in a crowded place, such as a cocktail party. The guest travels between tables, chairs and other guests, planning his/her motion "on the fly". In~\cite{Lumelsky1997}, only the translation of robots in the plane is considered and robots are assumed to be able to stop instantly. These assumptions are relaxed in~\cite{Pallottino2007} where nonholonomic constraints are considered. 

\paragraph{On the difficulty of deadlock avoidance under reactive schemes}

As noticed in~\cite{Lumelsky1997,Pallottino2007}, deadlock avoidance is difficult to ensure in such a reactive control scheme. Reference~\cite{Lumelsky1997} provides an example itself drawn from a previous work~\cite{Schwartz1983} and depicted in Figure~\ref{fig:deadlock-example-schwartz}.
\begin{figure}[ht]
\begin{center}
\includegraphics[width=0.5\linewidth]{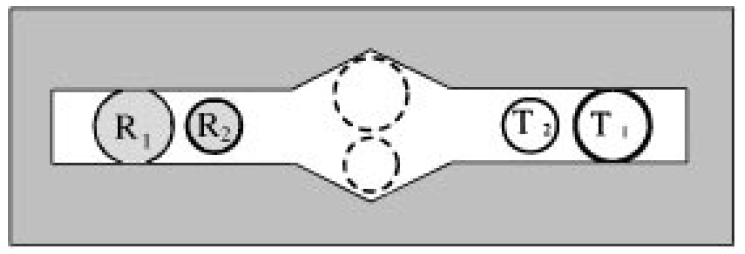}\hfill
\end{center}
\caption{An example where planning is necessary to achieve a task. "In order for the circular robots $R_1$ and $R_2$ to reach their respective targets, their relative position need to be switched. The only way to do so is to move $R_2$ into one of the 'wedges' and then move $R_1$ through the other wedge. The task is clearly impossible unless the motion of both robots is closely coordinated in a centralized manner." (courtesy of~\cite{Schwartz1983})}
\label{fig:deadlock-example-schwartz}
\end{figure}
As noticed by Lumelsky, the task seems to be "impossible" unless some planning is carried out to coordinate the motion of robots (see comments in Figure~\ref{fig:deadlock-example-schwartz}). Reference~\cite{Coffman1971} proposes a general characterization of system deadlocks as a situation where the following conditions hold:
\begin{itemize}
\item tasks claim exclusive control of the resources they require ("mutual exclusion" condition);
\item tasks hold resources already allocated to them while waiting for additional resources ("wait for" condition);
\item resources cannot be forcibly removed from the tasks holding them until the resources are used to completion ("no preemption" condition);
\item a circular chain of tasks exists, such that each task holds one or more resources that are being requested by the next task in the chain ("circular wait" condition).
\end{itemize}
To illustrate the notion, the authors give an example of traffic deadlock (see Figure~\ref{fig:traffic-deadlock}). As noticed by the authors, in this example, resources are the space occupied by cars. The "mutual exclusion" condition holds as two cars cannot occupy the same region without colliding. The "wait for" condition also holds as cars need to move forward (to get the next the space region) before releasing the current space region. The "no preemption" condition holds as cars cannot disappear from the real space, and finally the "circular wait" condition is clearly visible in Figure~\ref{fig:traffic-deadlock}.
\begin{figure}[!htbp]
\begin{center}
\includegraphics[width=0.7\linewidth]{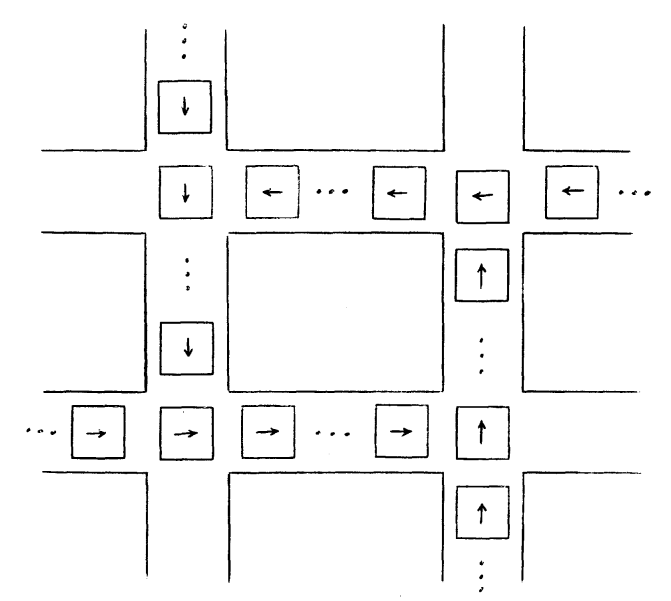}\hfill
\end{center}
\caption{An example of traffic deadlock with a circular chain of vehicles blocking each other (courtesy of~\cite{Coffman1971})}
\label{fig:traffic-deadlock}
\end{figure}
Later work refined the notion of deadlock in resource allocation systems and studied the complexity of deadlock avoidance/resolution (see, e.g.,~\cite{Araki1977,Lawley2001} and~\cite{Reveliotis2010,Reveliotis2011,Roszkowska2008} for a specific study focused on multiple robot systems). 

The deadlock avoidance problem mirrors the inability (by design) of reactive systems to carry out look-ahead to make better choices to accomplish actions. In the late 1980's, a research movement initiated by Agre and Chapman attempted to reconcile the camp of plan-as-program with the school of reactive control by asking the question: "what are plans for ?"~\cite{Agre1990}.

\subsection{Plan as a resource to guide action}

\paragraph{Plans should be used to guide, not control, action}

First of all, inspired by previous work of social scientists (see references therein~\cite{Agre1990}), Agre and Chapman proposed to retire the term "plan execution" advocated by the plan-as-program camp, considered as "prejudicial", and to use a more neutral term: "using a plan". They noticed that this terminology consideration raises new questions. First, "what can one do with a plan besides executing it ?". Second, if plan users are able to use plans sensibly rather than simply executing them, what are the implications on the representation of plans (how a plan looks like) and on the generation of plans (how to devise a plan easy to use in a sensible way by plan users ?). Conceptually, they propose to consider planning as the devising of resources to guide action. Plans are not executed but they are interpreted. Plans are a resource among others to decide the action to execute. Planning tasks are executed in parallel and asynchronously in order to retain a reactive quality. Plans are here to let the system be more goal-directed, to enhance performance, not to dictate action. 

\paragraph{Gradient fields as a guide to action}

As first noticed in~\cite{Payton1990}, gradient fields are an example of the new kind of plan -- a resource to guide action -- proposed in~\cite{Agre1990}.
\begin{figure}[!htbp]
\begin{center}
\includegraphics[width=0.7\linewidth]{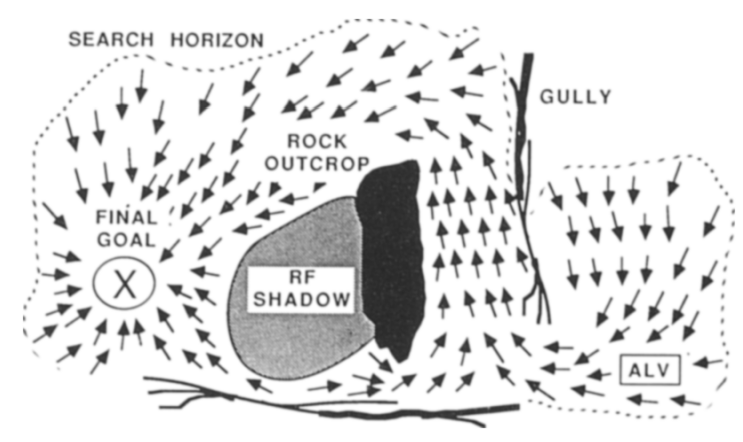}\hfill
\end{center}
\caption{A gradient field to guide a mobile robot. The mission of the robot is to get from its current location to the "final goal" location, bypassing a gully and avoiding a large rock. Moreover, along the maneuver, the robot needs to maintain communication with a remote radio tower. To this purpose, the "shadow" of the rock should also be avoided  (courtesy of~\cite{Payton1990}).}
\label{fig:gradient-field-payton}
\end{figure}
As depicted in Figure~\ref{fig:gradient-field-payton}, there is no explicit "traditional" plan. This plan can, instead, be interpreted as: "follow the arrows to reach the goal". Such plan is much more flexible than a traditional plan-as-program: at any point of time, a robot may decide not to follow the arrows for some reason (the sensors of the robot suddenly detect an obstacle that seems have not been detected at the moment of the gradient field computation). Nevertheless, the robot can still use the gradient field in the future no matter where it is currently located. In many scenarios, the robot will reach the goal without replanning, i.e., without a new time consuming computation of the vector field. Vector fields approaches for single robot (see, e.g.,~\cite{Rimon1992}) have been adapted for the coordination of holonomic robots in~\cite{Loizou2002}. In~\cite{Dimarogonas2006} a totally distributed version is proposed, the approach of~\cite{Dimarogonas2004} also applies to nonholonomic robots, and in~\cite{Pereira2003} sensing and communication constraints are taken into account.

\paragraph{Traffic signals as a guide to action}

Traffic signals are an effective way to coordinate competing traffic flows at intersections. To this purpose, they alternate the right of way of users (cars, buses, pedestrians).
\begin{figure}[ht]
\begin{center}
\includegraphics[width=0.5\linewidth]{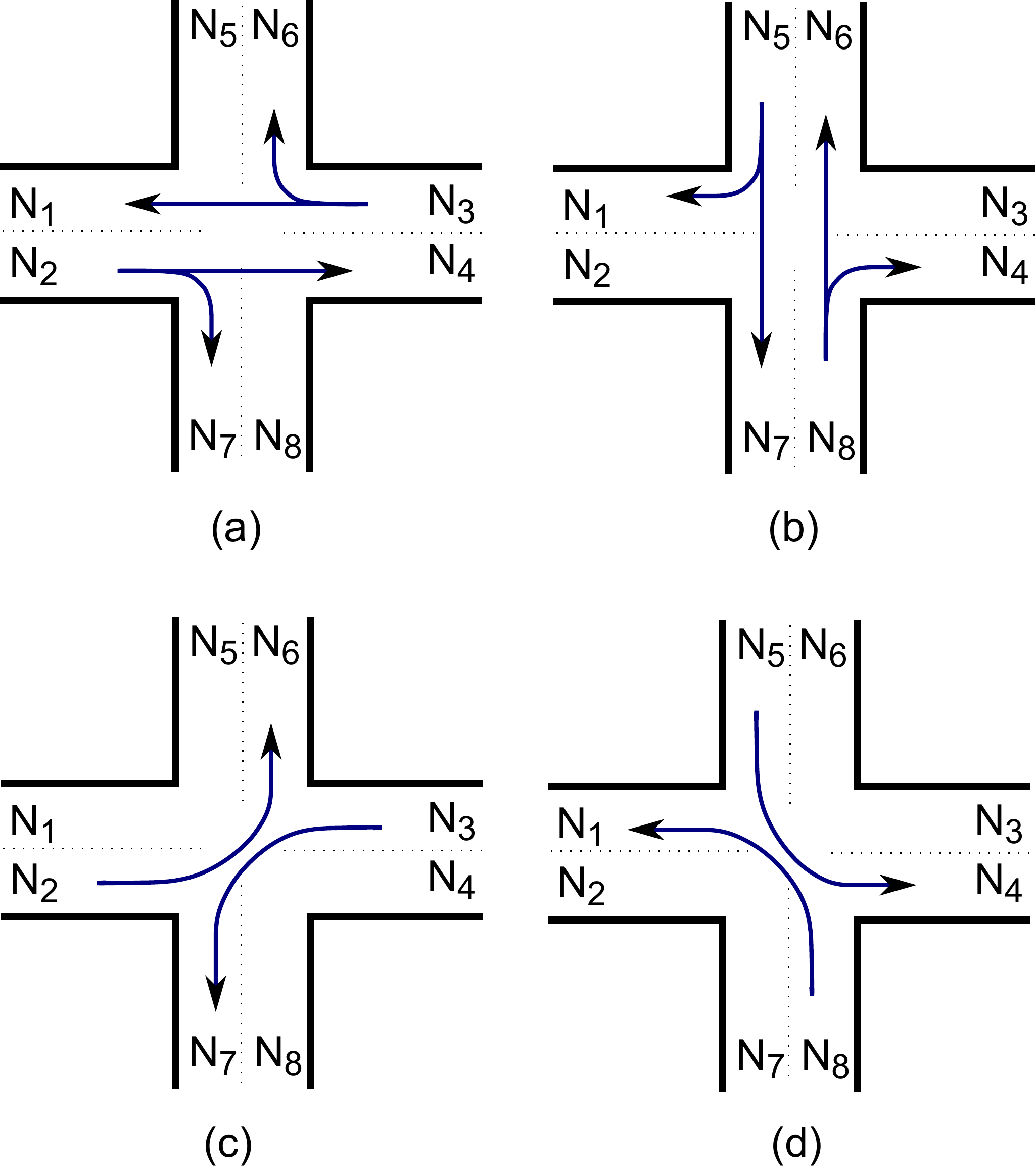}\hfill
\end{center}
\caption{A typical set of feasible phases at a junction.}
\label{fig:phases}
\end{figure}
A particular set of feasible simultaneous rights of way, called a phase, is decided for a certain period of time~\cite{Papageorgiou2003} (see Figure~\ref{fig:phases}). In an intersection ruled by a traffic signal, the traffic signal does not dictate actions to users. A pedestrian is not obliged to cross the road even if he/she is given the right of way. Vehicles are invited to cross the intersection when they have the right of way. However, they do not have to follow a precise assigned velocity profile, and if some unexpected event occurs like a pedestrian crossing the road without the right of way, the vehicle should stop as it is much more important to avoid pedestrians than to cross the intersection rapidly. Hence, traffic signals are a good example of a planning process that consists in providing resources to guide actions. Without such a resource, it would be difficult for vehicles to coordinate. Rules can also be decided in advance or be displayed using traffic signs, and again, they constitute the resources to guide action just like the traffic signal, and the design of rules can be considered as a planning process.

A valuable property of traffic signals compared to static rules is that the traffic signal can be controlled in order to enhance efficiency. Controlling a traffic light consists in designing rules to decide which phase to apply over time. It is interesting to see that just like any planning process, traffic signal optimization is a complex time consuming task. Pre-timed policies activate phases according to a time-periodic pre-defined schedule. There is much previous work on designing optimal pre-timed policies, e.g.,~\cite{Miller1963}. However, such policies are not efficient under changing arrival rates which require adaptive control. Most used adaptive traffic signal control systems include SCOOT~\cite{Hunt1982}, SCATS~\cite{Lowrie1990}, PRODYN~\cite{Henry1984}, RHODES~\cite{Mirchandani2001}, OPAC~\cite{Gartner1983} or TUC~\cite{Diakaki2002}. These systems update some control variables of a configurable pre-timed policy on middle term, based on traffic measures, and apply it on short term. Control variables may include phases, splits, cycle times and offsets~\cite{Papageorgiou2003}. Such algorithms may differ in the way optimization is carried out (e.g., mixed-integer linear programming~\cite{Gartner1975}, dynamic programming, exhaustive enumeration) and in the modeling approach (e.g., queuing network model~\cite{Osorio2009-analytic-finite-capacity,Osorio2009-surrogate-model}, cell transmission model~\cite{Lo2001}, store-and-forward~\cite{Aboudolas2009}, petri nets~\cite{DiFebbraro2002}).  Many major cities currently employ these systems which proved to be able to yield various benefits, including travel time and fuel consumption reduction, as well as safety improvements~\cite{Shepherd1992}. More recently, based on the seminal paper~\cite{Tassiulas1992}, feedback controls have been proposed both in the case of deterministic arrivals~\cite{Varaiya2013}, or stochastic arrivals~\cite{Varaiya2009, Wongpiromsarn2012,Le2013}. Time is slotted and at every time slot, a feedback controller decides the phase to apply based on current queue lengths estimation. This requires real-time queue length measures, but it enables to be much more reactive than other traffic controllers and to have stability guarantees. Reference~\cite{Tassiulas1992}  introduced the so-called back-pressure control which computes the control to apply based on queue lengths, and can achieve provably maximum stability. This algorithm was originally applied to wireless communication networks~\cite{Neely2005}, and some effort has been required to apply the approach in the context of a network of intersections~\cite{Varaiya2009, Wongpiromsarn2012,Gregoire2013-capacity,Gregoire2014-unknown-routing}. A key feature of this algorithm is that it can be completely distributed over intersections, in the sense that it can be implemented by running an algorithm of complexity $\mathcal{O}(1)$, requiring only local information, at each intersection.

\section{Contributions}

In this thesis, we advocate for the use of planning as the devising of resources to guide action. Such approach is more than designing a system with both planning and reactive abilities. It requires rethinking what a plan for our system is. In~\cite{Dresner2004}, a plan is a set of granted reservation requests (composed of the starting time and the velocity profile), that vehicles must execute. In~\cite{VanDenBerg2005}, the plan is simply the trajectory for each robot in the configuration space. A recent work~\cite{Kowshik2011} attempted to devise a coordination system for vehicles at intersections allowing some "freedom of action to cars" yet ensuring safety. Hence, the motivation is clearly to have a coordination with both planning and reactive abilities. However, the plan has quite the same representation as in traditional reservation-based systems. We believe that thinking plans as a resource to guide action should lead to a new representation of plans that should differ from traditional systems espousing the plan-as-program paradigm. What kind of plan to guide action can be designed for a multi robot system ?

The main contribution of this thesis is to propose a novel tool in multi robot motion planning: the priority graph. Roughly speaking, priorities describe a high-level coordination strategy: the relative order of robots. More precisely, they uniquely encode the homotopy classes of solutions to the multi robot coordination problem. This powerful theoretical tool is actionable to design a low complexity and robust priority-based coordination system. The planning process consists in assigning priorities. Under assigned priorities, robots can safely travel through the intersection in a reactive way provided all robots respect the assigned priorities. As planning priorities does not dictate a precise trajectory for robots through the intersection, but only provides useful concise resources for safe and efficient coordination, the system demonstrates valuable robustness properties in the face of environmental uncertainty and unpredictability. The thesis is organized as follows.

The first part presents the geometrical foundation of the priority-based approach using the standard coordination space framework. Priorities are formally defined and assigning priorities is provably equivalent to constrain the trajectory of robots to remain in a homotopy class of collision-free trajectories continuously deformable into each other. Assigning priorities does not plan a particular trajectory that robots must execute, yet it plans a higher-level coordination strategy describing the relative order of robots through the intersection: the priority graph. The priority graph can be considered as a unique meaningful representative of a homotopy class of trajectories. Planning priorities is a task of high combinatorial complexity as the set of possible priorities grows exponentially with the number of robots. The most important feature that the priority assignment policy must demonstrate is that the assigned priorities are feasible, i.e., that robots respecting the assigned priorities will eventually go through the intersection. Roughly speaking, the assigned priorities should encode a "non empty homotopy class" of solutions to the coordination problem. Interestingly, given assigned priorities, there are two exclusive options: either the assigned priorities will inevitably lead robots to a deadlock configuration where a circular chain of robots block each other; or the assigned priorities are feasible and robots will provably never reach a deadlock configuration provided priorities are respected. As a consequence, deadlock avoidance can be completely solved at the priority assignment level. It is a valuable property motivating the use of priorities as a coordination resource to guide robots through the intersection as deadlock avoidance is difficult to guarantee in a reactive manner. 

In contrast with the first part which has a quite mathematical -- more precisely, geometrical -- flavor with little care about control issues, the second part shows how to use priorities to guide robots through the intersection with control laws configured by the priority graph and ensuring priority preservation. Most importantly, under assigned priorities, for each pair of robots, there is not two but only one strategy to avoid collisions: the robot with lower priority must decelerate in favor of the robot with higher priority. As a consequence, the combinatorial complexity of multi robot control~\cite{Colombo2012} is avoided, and priority preserving control is of polynomial complexity, thus allowing real-time implementation. Moreover, the proposed control law demonstrates a quite novel robustness property in the presence of inertia. Robots may indeed safely brake at any point of time without violating priorities, in particular without colliding, which is referred as brake safety. This can be useful to handle unexpected events requiring braking like a pedestrian crossing the road. It is a quite novel property with regards to previous work as the standard plan-as-program approach constrains robots to track precisely a planned reference trajectory and thus does not allow a robot to brake if necessary to handle some unexpected event. Finally, the control scheme proposed in Chapter~\ref{chap:control-acceleration} and~\ref{chap:control-uncertainty} is decentralized. Each robot can compute the output of the control law independently without agreement with other robots through communication links. This benefit results from the prior agreement on the priority graph carried out at the planning level and requiring of course some form of communication.

The final part of the thesis proposes a priority-based coordination system adopting a three-layer control architecture. It has a more engineering flavor, specifying how priorities can be assigned dynamically as new robots arrive at the intersection and how to integrate priority preserving control proposed in the second part. A central agent, the intersection controller, constitutes the deliberative layer and assigns priorities. Robots implement several behaviors executed in parallel each one achieving/maintaining a specified goal, they constitute the reactive behavior-based layer. Behaviors include path following, moving forward, not entering the intersection before being accepted by the intersection controller, avoiding pedestrians, and of course respecting priorities which implements the priority preserving control law proposed in the second part of the thesis. Robots communicate asynchronously with the intersection controller through the sequencing layer to request the right of way and be assigned assigned priorities. The sequencing layer interfaces with the behavior-based layer by activating/deactivating/configuring behaviors. The behavior-based layer takes full benefit of the brake safety property. Some behavior, e.g., the behavior ensuring pedestrian avoidance, may indeed require a robot to brake at any point of time with the guarantee that it will not result in a priority violation. Therefore, the coordination system demonstrates significant robustness as it can handle a large class of unexpected events -- all events requiring braking -- without changing priorities, i.e., without replanning. Priority-based coordination combines the efficiency of traditional planning approaches as complex scheduling can be encoded by the priority graph -- much more complex scheduling than using traffic signals -- as well as the ability to handle a large class of unexpected events in a reactive manner.

\part{Priorities: definition and properties}
\label{part:priority-framework}

\chapter*{Introduction}

The present part constitutes the geometrical foundation of the priority-based approach proposed in this thesis. Priorities at road intersections are a well known concept aiming at organizing traffic. Signs, signals, markings are used to inform the users about who has the right to go first, or equivalently who has "priority"~\cite{wiki:traffic}. Convinced that the coordination space approach~\cite{ODonnell1989} is a convenient mathematical formulation of the coordination problem, we propose here a formal definition of priorities as a new concept in the coordination space. In the coordination space approach, a multi robot system composed of $n$ robots traveling along fixed paths is considered as a composite robot evolving in a $n$-dimensional space called the coordination space~\cite{ODonnell1989,LaValle2006}. Potential inter-robot collisions requires the composite robot to avoid an obstacle region in the coordination space. The obstacle region has a cylindrical structure (see Figure~\ref{fig:guo-cylindrical-structure}). In traditional motion planning, the coordination problem is reduced to finding a feasible path in the coordination space (see the collision-free path in Figure~\ref{fig:guo-cylindrical-structure}). It looks like the notion of priorities is completely lost. In this part, we provide theoretical tools in the coordination space in order to endow the coordination space approach with a concept of priority. The idea is that a collision-free path in the coordination space necessarily lies on one side or on the other side with respect to each collision cylinder. Deciding on which side to pass with respect to each collision cylinder is equivalent to deciding the relative order of robots to go through the intersection and constitutes the discrete part of the coordination problem that we refer as priority assignment. Respecting assigned priorities does not require robots to follow a precise path in the coordination space as many collision-free paths respect the same priorities, or equivalently, lie on the same side with respect to collision cylinders. Hence, it is possible to assign priorities, yet retaining some individual freedom of action to robots. More precisely, the result of this part enable to go one step ahead in the understanding of the structure of the solutions to the coordination problem. Previous work noticed the existence of homotopy classes of feasible paths in the coordination space~\cite{Ghrist2005}, and this part demonstrates that priorities are a unique meaningful representative of homotopy classes -- they uniquely encode homotopy classes.

\paragraph{Sketch of the part} Chapter~\ref{chap:priority-graph} starts by exposing the coordination space approach, introducing assumptions and notations. Priorities are defined as a binary relation between robots induced by a collision-free path in the coordination space. As the coordination space is thus endowed with a priority concept, Chapter~\ref{chap:priorities-homotopy} studies the structure of the coordination space under assigned priorities. It is proved that all paths respecting the same priorities are continuously deformable into each other, forming a homotopy class. Finally, the deadlock avoidance problem is shown to be solved by assigning so-called feasible priorities which are characterized. This part motivates the use of priorities as a plan to guide robots through the intersection.

{\pagestyle{plain}
\clearpage
\topskip0pt
\vspace*{\fill}
\includegraphics[height=1.0\linewidth,angle=-90,trim=160 60 160 60, clip]{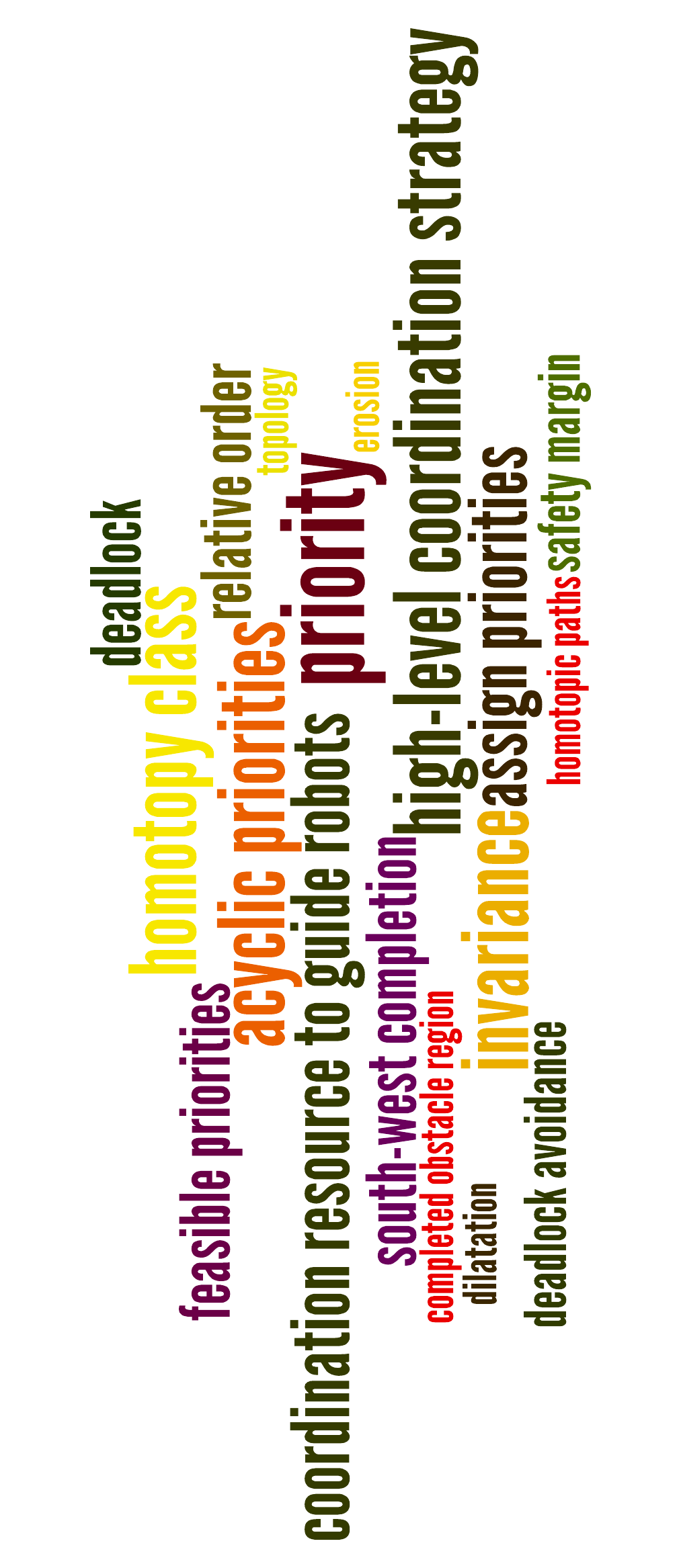}
\vspace*{\fill}

\parttoc}

\chapter[Priorities: a geometric concept in the coordination space]{Priorities: a geometric concept\\ in the coordination space}

\label{chap:priority-graph}
\minitoc

\paragraph{Sketch of the chapter} Section~\ref{sec:coordination-space} is quite expository, it recalls the basics of the coordination space approach providing the main assumptions and notations. Section~\ref{sec:priority-relation} endows the coordination space approach with a priority concept, defining the priority relation as well as the priority graph induced by a feasible path in the coordination space.

\section{The coordination space approach}
\label{sec:coordination-space}

Consider the problem of coordinating the motion of a collection of robots $\robots$ in a two-dimensional space. Every robot $i\in\robots$ follows a particular path $\gamma_i \subset \RR^2$ and we let $x_i \in \RR$ denote its curvilinear coordinate along the path (see Figure~\ref{fig-paths}).
\begin{figure}[!htbp]
\begin{center}
\includegraphics[width=0.7\linewidth]{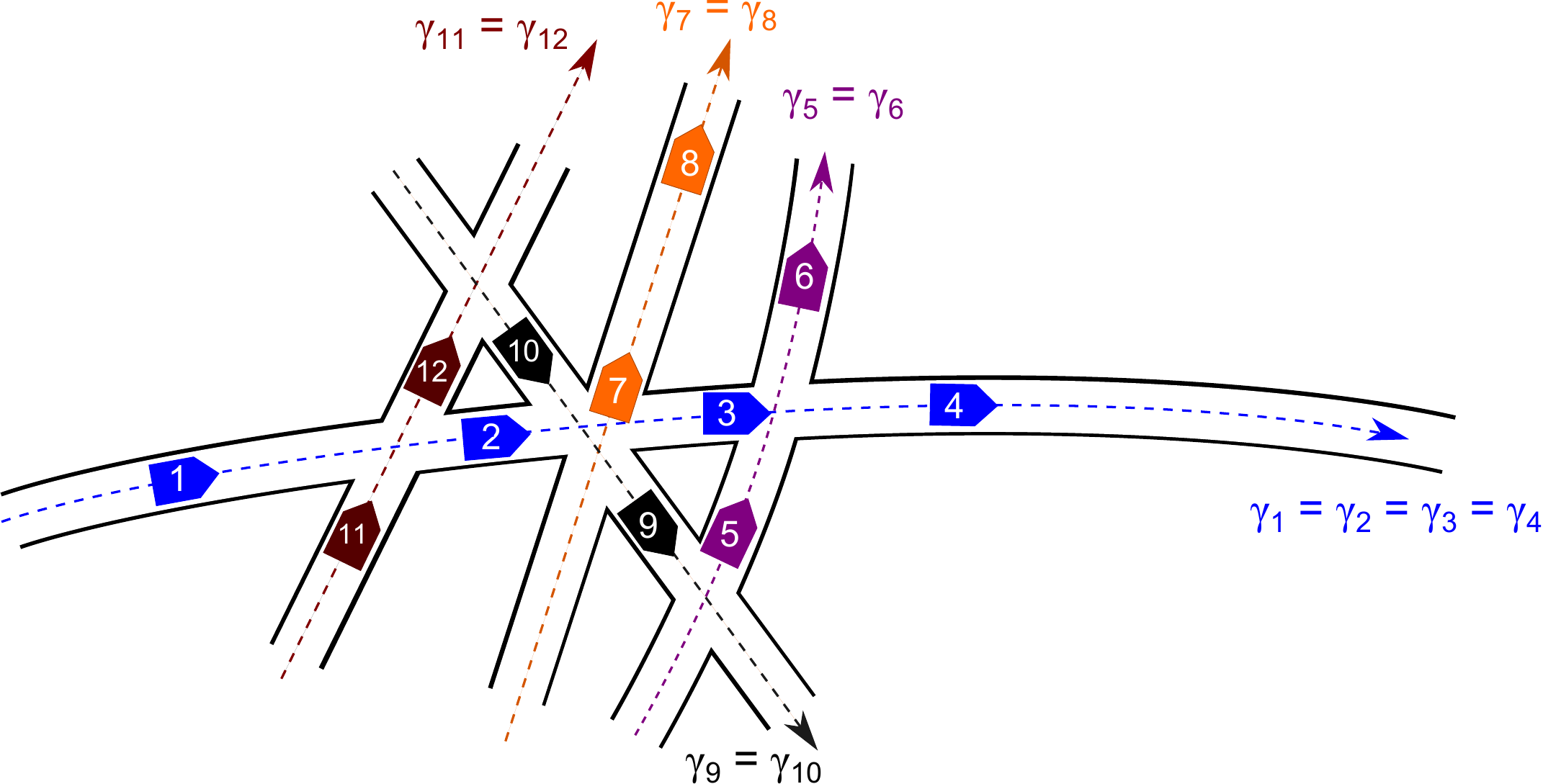}\hfill
\end{center}
\caption{The fixed paths assumption. Every robot travels along an assigned path.}
\label{fig-paths}
\end{figure}
$x:=(x_i)_{i\in\robots}$ indicates the configuration of all robots; $x\in \chi:=\RR^n$ where $n$ denotes the number of robots going through the intersection. The configuration space $\chi$ is known as the coordination space, first introduced in~\cite{ODonnell1989} and which has become a standard tool~\cite{LaValle2006}. This approach is often referred as path-velocity decomposition. It reduces the problem's complexity as each robot has now only one degree of freedom. For an application to autonomous vehicles at road intersections, this additional constraint seems particularly well adapted as the road network is strongly spatially organized (roads and lanes with markings). In the rest of the manuscript,  $\{\e_i\}_{1\le i\le n}$ denotes the canonical basis of $\chi$. Given a subset $A$ of the topological space $\chi$, $\partial A$ refers to the boundary of $A$. We define the Minkowski sum as follows:
\begin{equation}
\forall x^0\in\chi, \forall A\subset\chi,~x^0+A=\left\{x^0+x:x\in A\right\}
\end{equation}
\begin{equation}
\forall A,B\subset\chi,~A+B=\{x+y:x\in A,y\in B\}
\end{equation}
We will use the topology of infinity norm on $\chi \equiv \RR^n$, so the parallelepiped $x^0+(-r,r)^n$ is the open ball of radius $r>0$ centered in $x^0\in\chi$. 

Some configurations must be excluded to avoid collisions between robots (see Figure~\ref{fig-collision-region}). The obstacle region $\chiobs\subset\chi$ is the open set of all collision configurations. Let $\kappa_{ij}\subset \RR^2$ denote the set of configurations $x$ where $i$ and $j$ collide. Let $\chiobs_{ij} \subset\chi$ denote the set of (global) configurations $x$ where $i$ and $j$ collide, we have:
\begin{equation}
\chiobs_{ij}:=\left\{x\in\chi: (x_i,x_j)\in\kappa_{ij}\right\}
\end{equation}
We obviously take $\chiobs_{ii}:=\emptyset$.
\begin{definition}[Obstacle region, Obstacle-free region]
The obstacle region is the set $\chiobs\subset\chi$ of configurations where a collision occurs for some $i,j\in\robots$, i.e., 
\begin{equation}
\chiobs := \cup_{\{i,j\}} \chiobs_{ij}
\end{equation}
$\chifree:=\chi\setminus \chiobs$ denotes the obstacle-free space.
\end{definition}
By construction, $\chiobs_{ij}$ is a cylinder (based on the plane generated by $\e_i$ and $\e_j$), and the obstacle region merely appears as the union of $n(n-1)/2$ cylinders~\cite{LaValle2006} corresponding to as many collision pairs. Every cylinder $\chiobs_{ij}$ is assumed to have an open bounded convex cross-section, i.e., $\kappa_{ij}$ is open and bounded. The boundedness condition on $\chiobs$ is rather technical but ensures the whole intersection lies in a bounded region. In particular, it implies that there exists a lower bound $\xobsmin\in\chifree$ and an upper bound $\xobsmax\in\chifree$ satisfying:
\begin{equation}
\forall i,j\in\robots,\forall x\in\chiobs_{ij},~\xobsmin_i < x_i < \xobsmax_i\text{ and } \xobsmin_j < x_j < \xobsmax_j
\end{equation}
\begin{figure}[p]
\begin{center}
\includegraphics[width=0.7\linewidth]{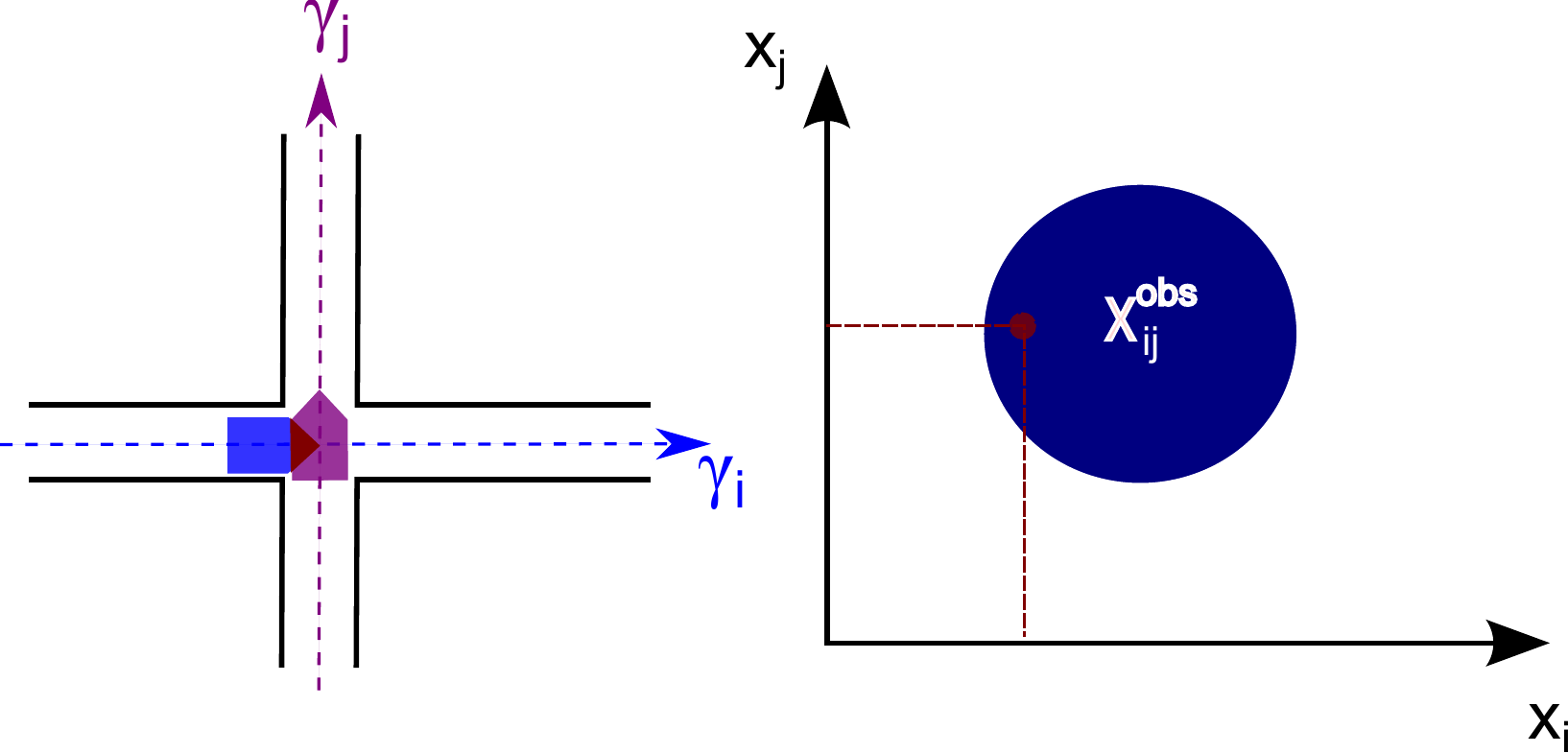}\hfill
\end{center}
\caption{The left drawing depicts two paths with two robots in collision in the current configuration. The right drawing shows the obstacle region $\chiobs_{ij}$ associated to the two paths (more precisely its cross-section along the plane generated by $\e_i$ and $\e_j$) and the collision configuration ${x\in\chiobs_{ij}}$ corresponding to the collision of the left drawing.}
\label{fig-collision-region}
\end{figure}
\begin{figure}[p]
\begin{center}
\raisebox{-0.5\height}{\includegraphics[width=0.35\linewidth]{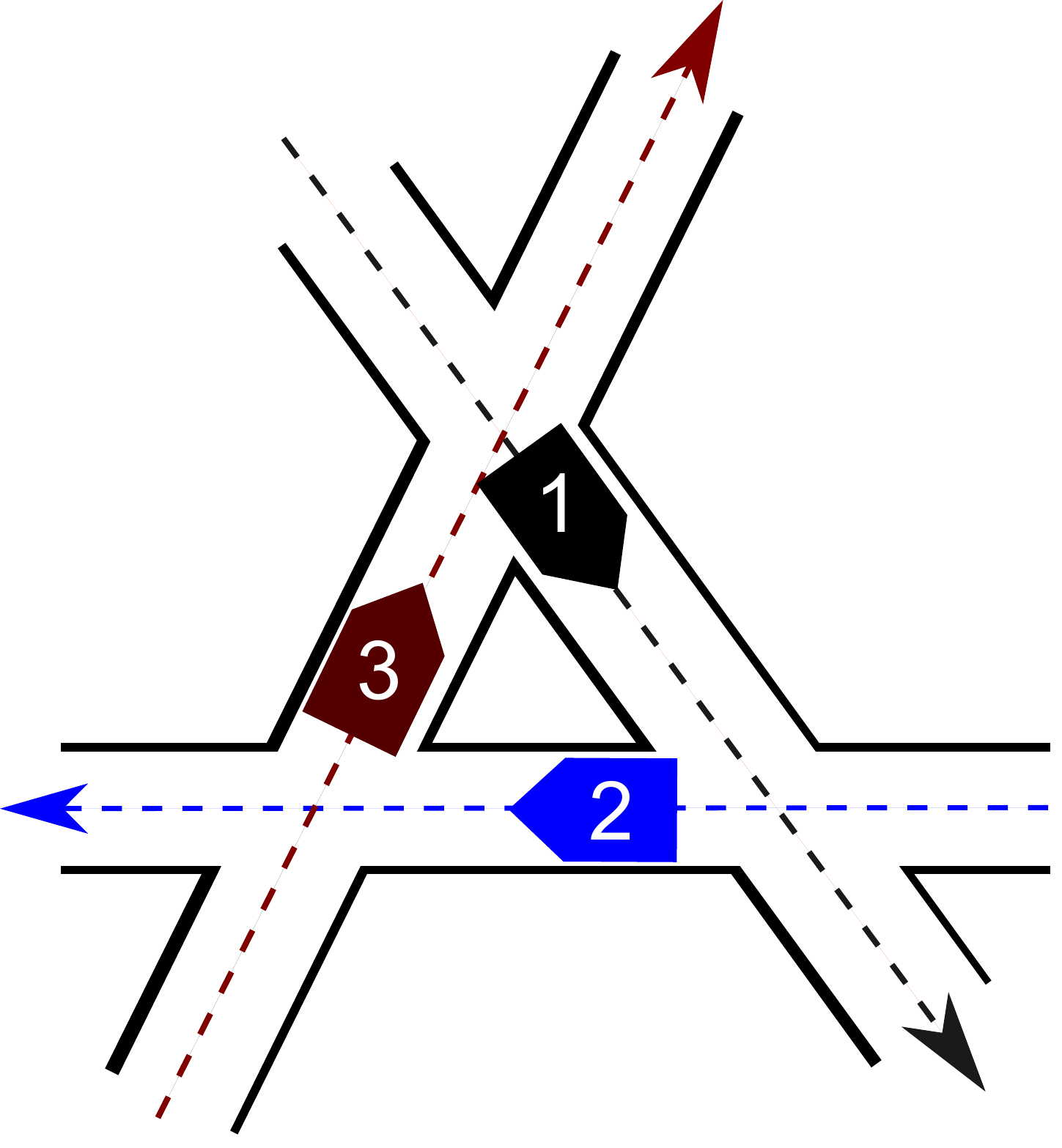}}\hfill
\raisebox{-0.5\height}{\includegraphics[width=0.65\linewidth]{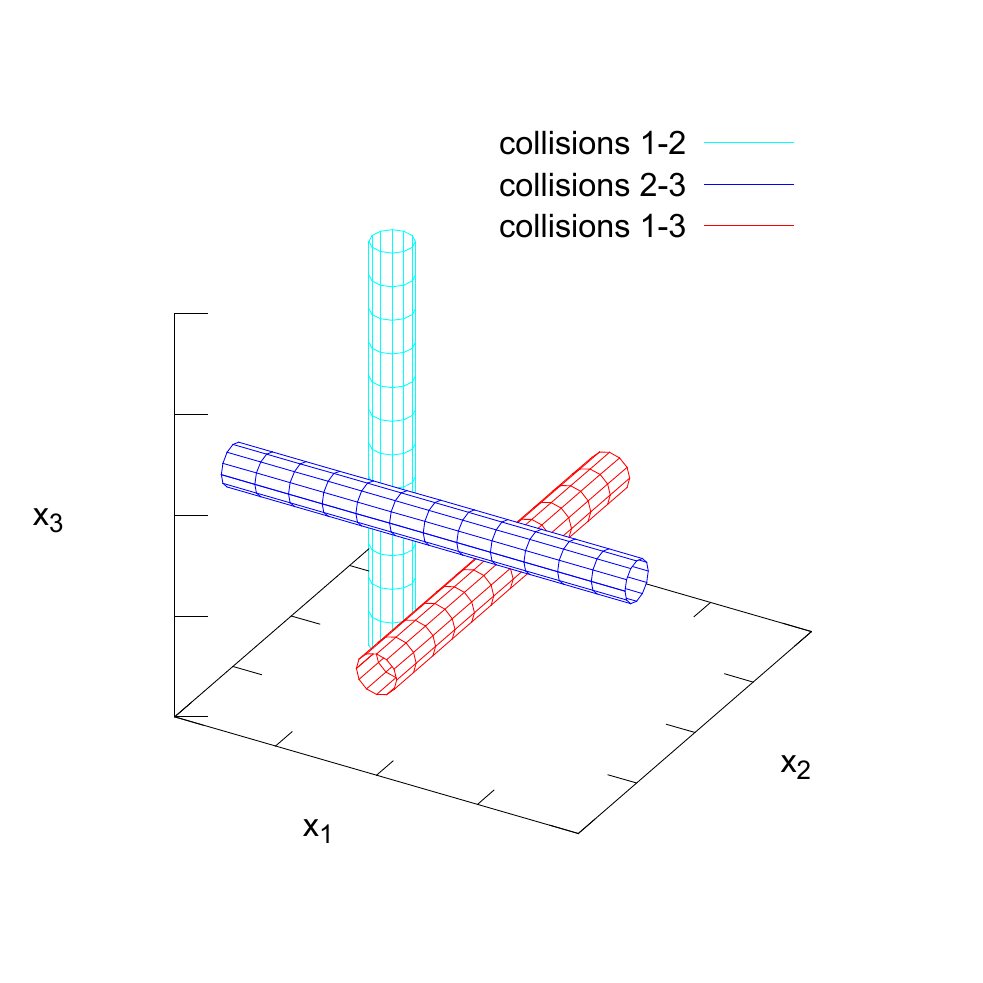}}\hfill
\end{center}
\caption{The right drawing shows the cylindrical structure of the obstacle region for the three-robot system of the left drawing. Each cylinder accounts for the possible collisions between each couple of robots. The right drawing of Figure~\ref{fig-collision-region} depicts the base of such cylinders.}
\label{fig:collision-region-3D}
\end{figure}
A continuous application $\path:[0,1]\to\chi$ will be called a path and we let $\im{\path}$ denote the set of values taken by $\path$:
\begin{equation}
\im{\path}:=\left\{\path(t):t\in[0,1]\right\}
\end{equation}
A partial order $\leq$ for configurations is defined as the product order of $\RR^n$:
\begin{equation}
\forall x,y\in\chi, x\leq y\text{ if } \forall i\in\robots, x_i \leq y_i\\
\end{equation}

\begin{definition}[Feasible path]
A feasible path is a non-decreasing collision-free path $\path:[0,1]\to\chifree$ requiring no coordination beyond its endpoints, i.e., a path satisfying the following conditions:
\begin{enumerate}[(a)]
\item $\path$ is non-decreasing: 
\begin{equation}
\forall t^1,t^2\in[0,1],~ t^1\leq t^2 \Rightarrow \path(t^1)\leq \path(t^2)
\end{equation}
\item $\path$ is collision-free:
\begin{equation}
\im{\path} \subset \chifree
\end{equation}
\item No coordination is required beyond its start point:
\begin{equation}
\left(\path(0)-\RR_+^n\right) \subset \chifree
\end{equation}
\item No coordination is required beyond its endpoint:
\begin{equation}
\left(\path(1)+\RR_+^n\right) \subset \chifree
\end{equation}
\end{enumerate}
\end{definition}
We let $\paths(\chifree)$ denote the set of feasible paths. Note that the two last conditions hold in particular for $\path(0)\equiv\xobsmin$ and $\path(1)\equiv\xobsmax$. The conditions of the above definition are more flexible and do not fix the endpoints. More importantly, we will only consider as feasible motions where robots never move backwards in the intersection area. It is a standard assumption as neither efficiency nor safety can be expected from robots moving backwards at an intersection area. 

More generally, given a subset $C\subset\chi$, we let $\paths(C)$ denote the set of non-decreasing paths satisfying $\im{\path}\subset C$, $(\path(0)-\RR_+^n)\subset C$ and $(\path(1)+\RR_+^n)\subset C$. This notation is coherent with the definition of $\paths(\chifree)$ as the set of feasible paths. Using this notation, $\paths(\chi)$ merely refers to the set of non-decreasing paths as the additional conditions obviously hold for $C\equiv \chi$. In the following, we provide three examples where the obstacle region can be computed analytically.
\begin{figure}[p]
\begin{center}
\includegraphics[width=0.7\linewidth]{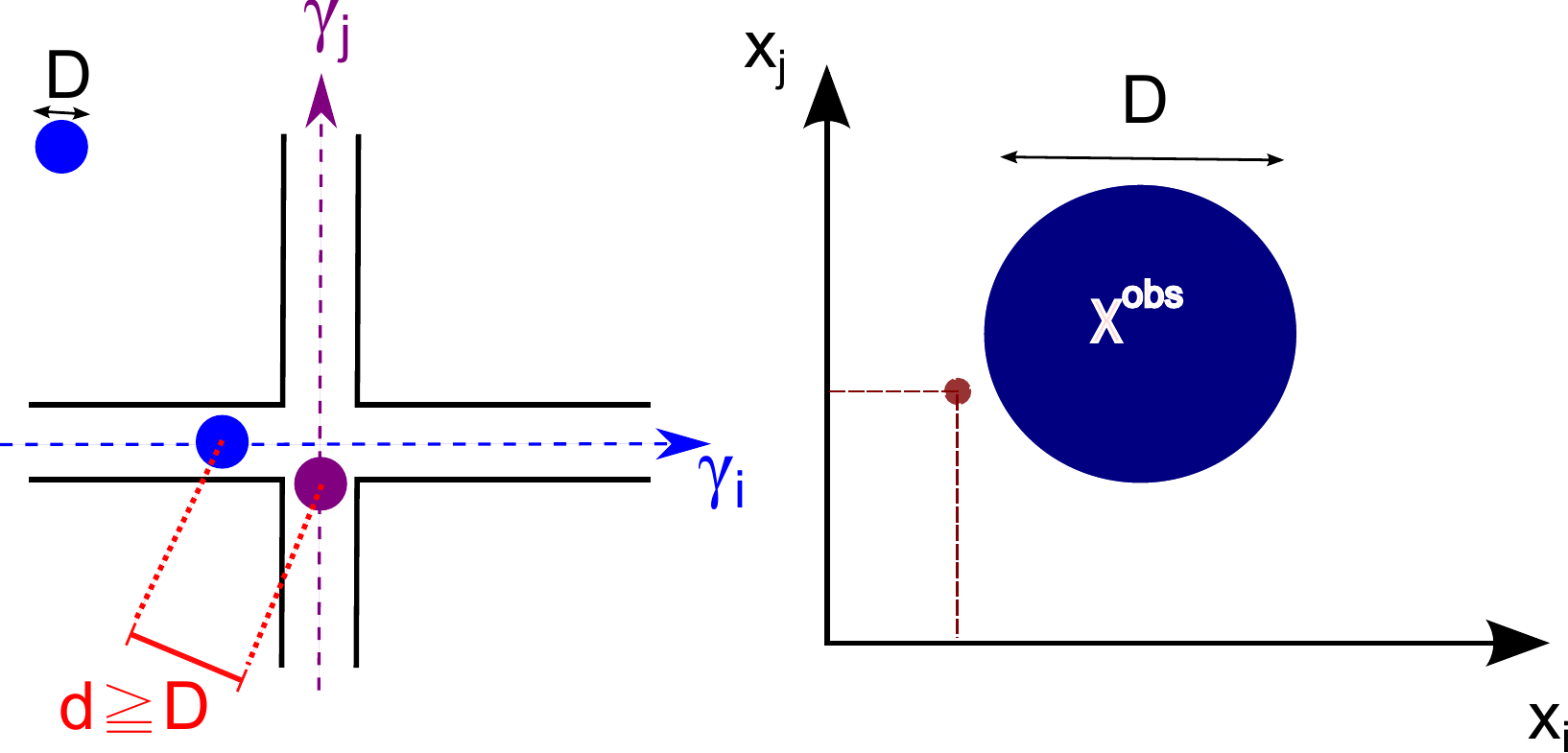}\hfill
\end{center}
\caption{The obstacle region for two circle-shaped robots along straight perpendicular paths.}
\label{fig-example-collision-circle-shaped}
\end{figure}
\begin{figure}[p]
\begin{center}
\includegraphics[width=0.7\linewidth]{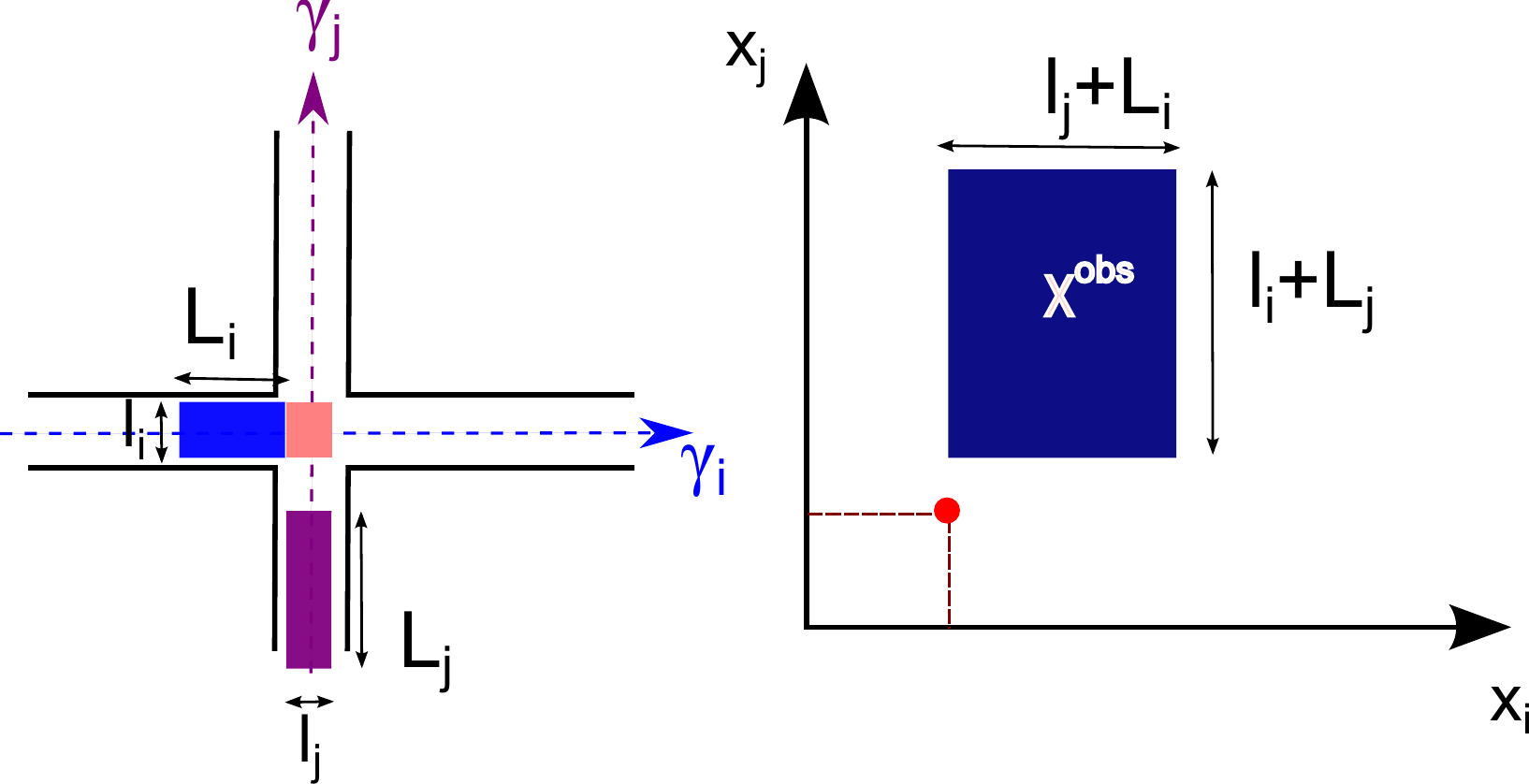}\hfill
\end{center}
\caption{The obstacle region for two rectangular robots along straight perpendicular paths.}
\label{fig-example-collision-rectangular-shaped}
\end{figure}
\begin{figure}[p]
\begin{center}
\includegraphics[width=0.7\linewidth]{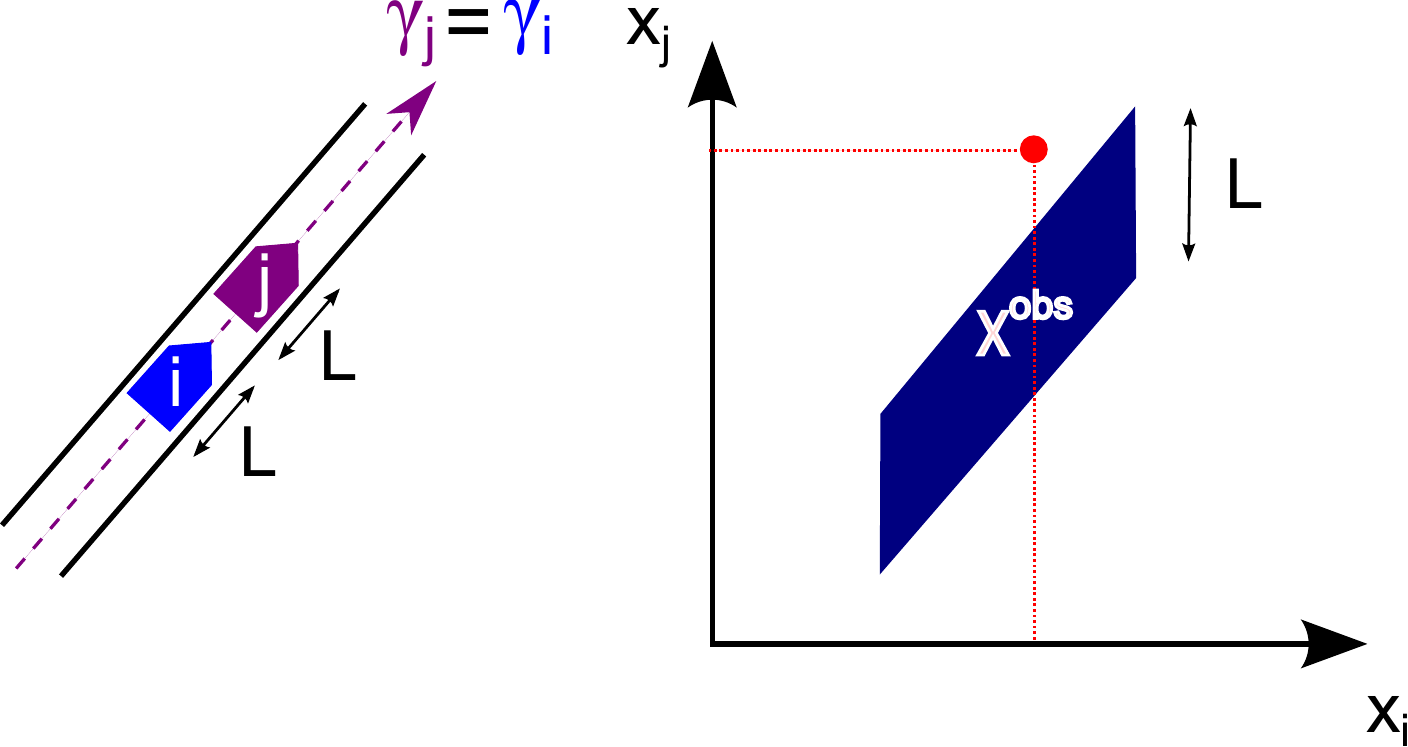}\hfill
\end{center}
\caption{The obstacle region for two robots that follow each other.}
\label{fig-example-car-following}
\end{figure}

\begin{example}[Two straight paths with circle-shaped robots]
Consider two circle-shaped robots of diameter $D$ moving along straight and perpendicular paths. Assume that  the curvilinear coordinate of each robot is $0$ when the center of the robot is exactly at the paths intersection point. Then, the distance between the centers of robots is $d=x_i^2+x_j^2$. As the diameter of robots is $D$, the configuration $(x_i,x_j)$ is collision-free if and only if $d \geq D$, i.e., $x_i^2+x_j^2 \geq D$. As a result, the obstacle region is $\chiobs = \left\{x\in\chi: x_i^2+x_j^2 < D\right\}$ as depicted in Figure~\ref{fig-example-collision-circle-shaped}.
\end{example}

\begin{example}[Two perpendicular paths with rectangular robots]
Consider two rectangular robots $i,j$ of lengths $L_i,L_j$ and widths $l_i,l_j$ along straight perpendicular paths. In the real space, there is a rectangular region of area $l_i \times l_j$ that can be occupied by only one robot, exclusively (see the red box in the left drawing of Figure~\ref{fig-example-collision-rectangular-shaped}). When a robot is at the the entry of this region (robot $i$ in the left drawing of Figure~\ref{fig-example-collision-rectangular-shaped}), it needs to travel the length of the region plus its own length in order to exit this region (robot $i$ needs to travel distance $l_j+L_i$ in order to exit this region). It follows that in the coordination space, the obstacle region is a rectangular region of length $l_j+L_i$ along axis $i$ and $l_i+L_j$ along axis $j$ (see the right drawing of Figure~\ref{fig-example-collision-rectangular-shaped}).
\end{example}

\begin{example}[Two robots along the same straight path]
Finally, consider two robots of length $L$ traveling along the same straight paths as depicted in Figure~\ref{fig-example-car-following} and assume that the same origin is used for the curvilinear coordinate of both robots. There are two options: either robot $i$ follows robot $j$ and collision avoidance requires $x_j\geq x_i+L$, or robot $j$ follows robot $i$ and collision avoidance requires $x_i \geq x_j+L$. Hence, the collision avoidance requirement including both cases is: $\vert x_i-x_j\vert \geq L$, and the obstacle region should be the band $\left\{ x\in\chi: \vert x_i-x_j\vert < L\right\}$. However, we do not aim  to model the collisions in an infinite spatial region. Hence, the band is truncated as depicted in Figure~\ref{fig-example-car-following}. 
\end{example}

\section{The priority relation}
\label{sec:priority-relation}

\subsection{The completed obstacle region}

This subsection shows that the intuitive notion of "assigning priorities" is equivalent to a completion of the obstacle region. It is indeed equivalent to consider as forbidden configurations both collision configurations and configurations that do not respect the assigned priorities, resulting in a completed obstacle region. 

Let $\chiobs_{i\succ j} $ and $\chifree_{i \succ j}$ denote the subsets of $\chi$ defined below:
\begin{eqnarray}
\chiobs_{i\succ j} &:= &\chiobs_{ij} - \RR_+ \e_i + \RR_+ \e_j
\label{eq-fixed-priority-collision-cylinder}\\
\chifree_{i \succ j}&:=&\chi \setminus \chiobs_{i\succ j}
\end{eqnarray}
\begin{figure}[!htbp]
\begin{center}
\includegraphics[width=0.7\linewidth]{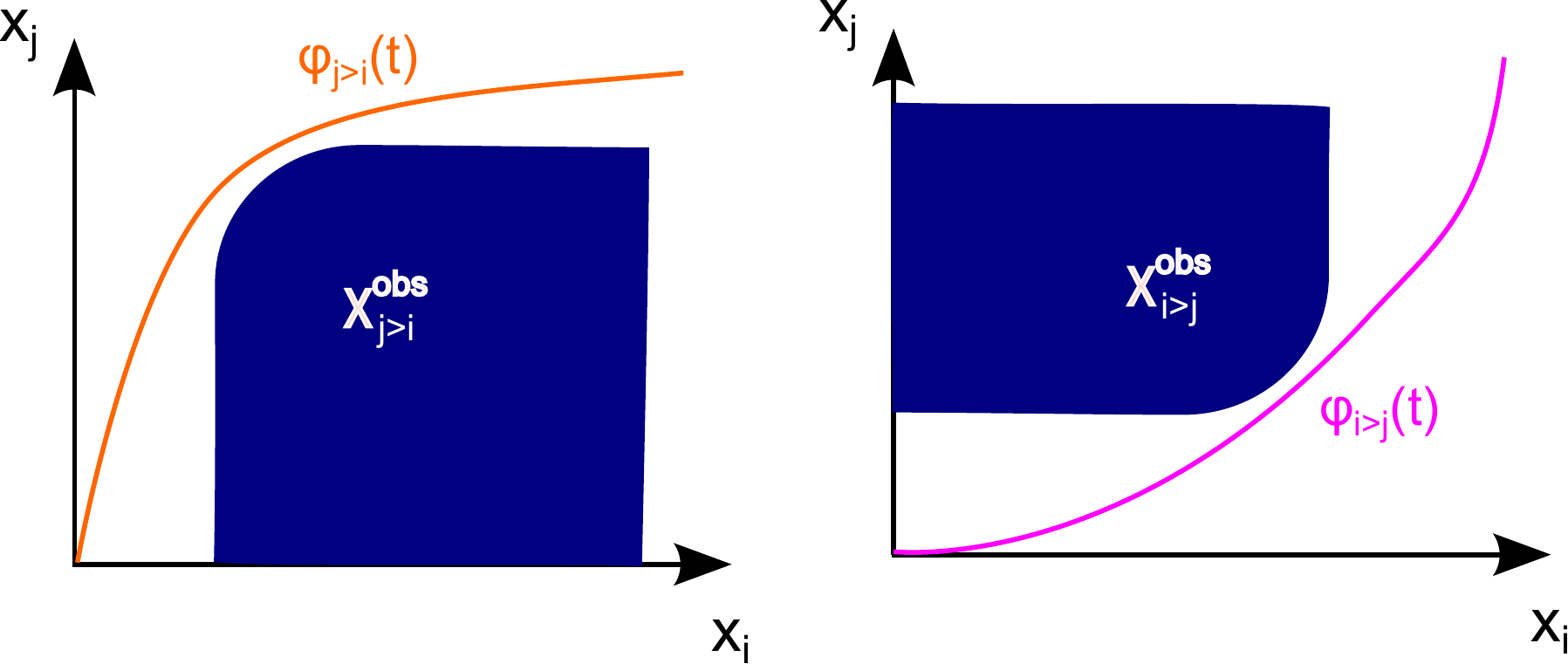}\hfill
\end{center}
\caption{Projection of the completed collision cylinders $\chiobs_{i\succ j}$ and $\chiobs_{j\succ i}$.}
\label{fig-fixed-priority-obstacle-region}
\end{figure}
We also define $\kappa_{i\succ j}\subset\RR^2$ as follows:
\begin{equation}
\kappa_{i\succ j}:=\kappa_{ij} + \RR_-\times\RR_+
\end{equation}
which is the cross-section of $\chiobs_{i\succ j}$, i.e., 
\begin{equation}
\chiobs_{i\succ j}=\left\{ x\in\chi: (x_i,x_j)\in\kappa_{i\succ j} \right\}
\end{equation}

Figure~\ref{fig-fixed-priority-obstacle-region} displays the sets $\chiobs_{i\succ j}$ and $\chiobs_{j \succ i}$. The rationale behind the definition of these sets is that as a feasible path is non-decreasing, it necessarily lies below or above each collision cylinder as depicted in Figure~\ref{fig-fixed-priority-obstacle-region}. This reflects the intuitive notion of priority at intersections. Deciding on which side to pass with respect to each collision cylinder is equivalent to deciding the relative order of robots to go through the intersection. In the sequel, we are going to prove that the definition of the sets $\chiobs_{i\succ j}$ enables to define rigorously the so-called priority relation induced by a feasible path. We start with some geometric properties that will be used in the proofs of the presented results.

\begin{property}[Geometric invariances of $\chiobs_{i\succ j}$ and $\chifree_{i\succ j}$ illustrated in Figure~\ref{fig:fixed-priority-obstacle-region-geometric-invariance}] For all $i,j\in\robots$, the following identities hold:
\begin{eqnarray}
\chiobs_{i\succ j}-\RR_+\e_i+\RR_+\e_j &=& \chiobs_{i \succ j}\label{eq:geometric-invariance-obstacle}\\
\chifree_{i\succ j}+\RR_+\e_i-\RR_+\e_j &=& \chifree_{i \succ j}\label{eq:geometric-invariance-collision-free}
\end{eqnarray}
\label{property:geometric-invariance}
\end{property}
\begin{property}[Invariance through $\min$ and $\max$ operators illustrated in Figure~\ref{fig:fixed-priority-obstacle-region-geometric-invariance}]
\label{property:min-max}
Given $x,y\in\chi$, for all $i,j\in\robots$, the following implications hold:
\begin{eqnarray}
x,y\in\chifree_{i\succ j} &\Rightarrow&  \max\{x,y\} \in \chifree_{i\succ j} \label{eq:property-max}\\
x,y\in\chifree_{i\succ j} &\Rightarrow&  \min\{x,y\} \in \chifree_{i\succ j}\label{eq:property-min}
\end{eqnarray}
\end{property}
\begin{property}[Illustrated in Figure~\ref{fig:property-union-fixed-priority-cylinder}]
For all $i,j\in\robots$ and $y\in\chiobs_{ij}$, we have:
\begin{equation}
\left\{x\in\chi: x_i=y_i\right\} \subset \left( \chiobs_{i\succ j} \cup \chiobs_{j\succ i} \right)
\end{equation}
\label{property-union-fixed-priority-cylinder-1}
\end{property}
\begin{property}[Illustrated in Figure~\ref{fig:property-union-fixed-priority-cylinder}]
For all $i,j\in\robots$, given $x^1\in\chiobs_{j\succ i}$ and $x^2\in\chiobs_{i\succ j}$, we have:
\begin{equation}
\left\{x\in\chi: x^1_i\leq x_i\leq x^2_i \right\} \subset \left(\chiobs_{i\succ j}\cup \chiobs_{j\succ i}\right)
\end{equation}
\label{property-union-fixed-priority-cylinder-2}
\end{property}

\begin{figure}[p]
\begin{center}
\includegraphics[width=1.0\linewidth]{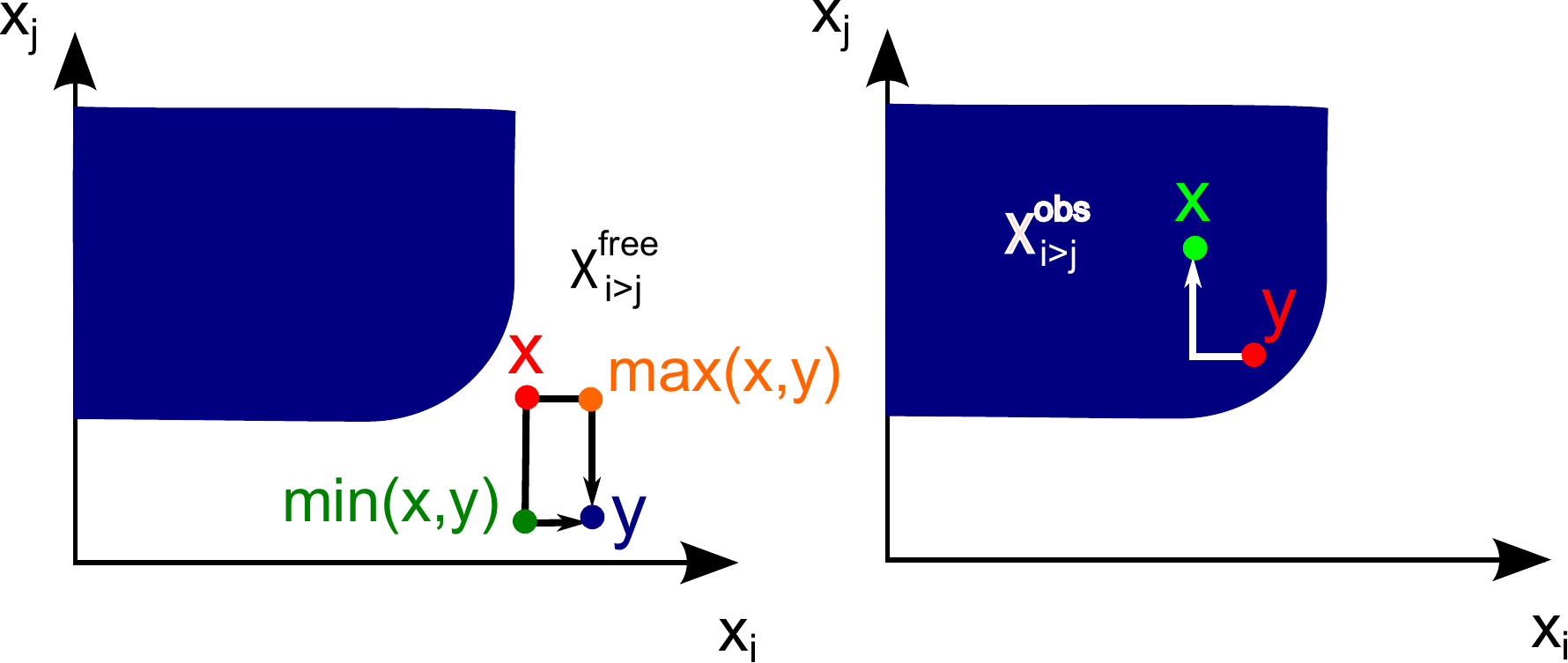}\hfill
\end{center}
\caption{Illustration of Properties~\ref{property:geometric-invariance} and~\ref{property:min-max}.}
\label{fig:fixed-priority-obstacle-region-geometric-invariance}
\end{figure}

\begin{figure}[p]
\begin{center}
\includegraphics[width=1.0\linewidth]{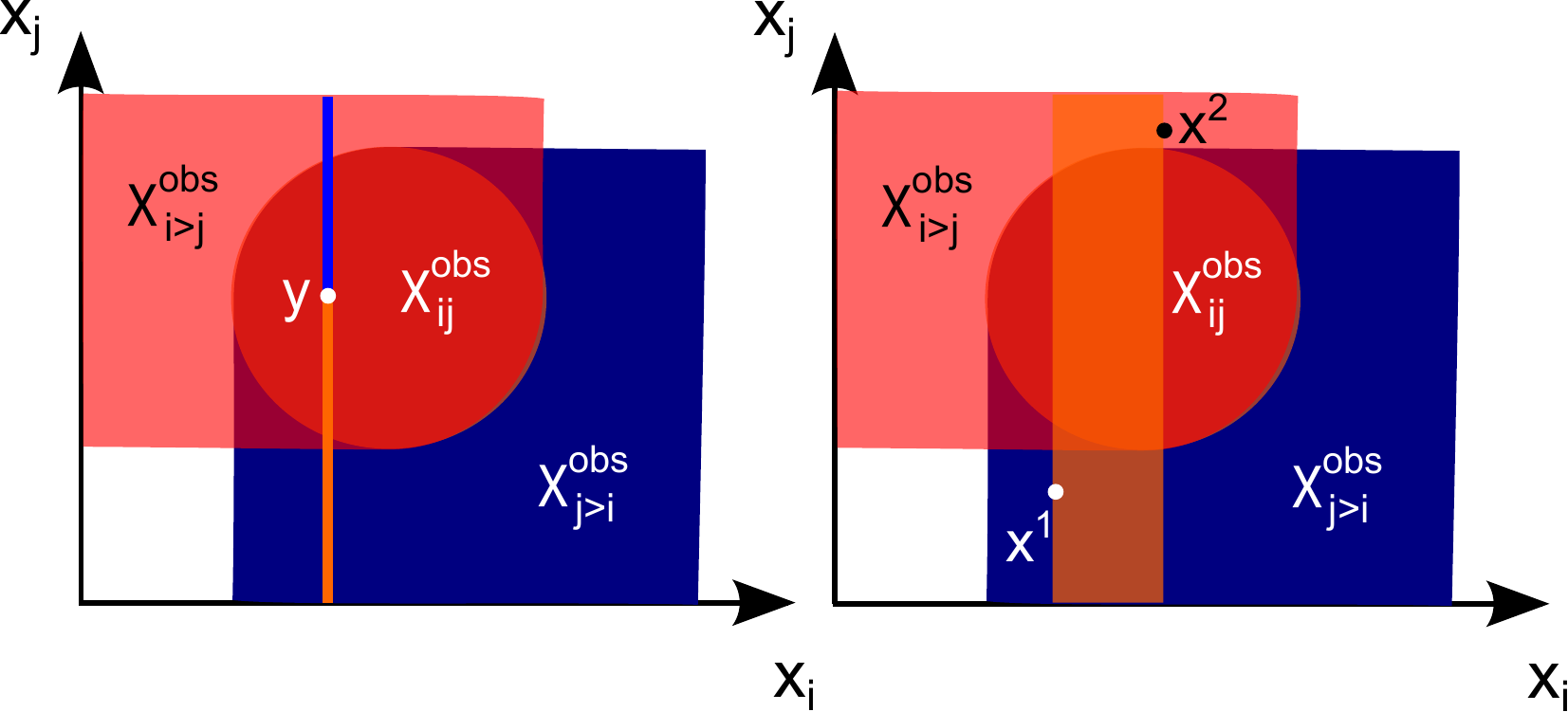}\hfill
\end{center}
\caption{Illustration of Properties~\ref{property-union-fixed-priority-cylinder-1} and~\ref{property-union-fixed-priority-cylinder-2}.}
\label{fig:property-union-fixed-priority-cylinder}
\end{figure}

\begin{proof}[Proof of Property~\ref{property:geometric-invariance}]
Take $i,j\in\robots$. By simple manipulations,
\begin{multline}
\chiobs_{i\succ j}-\RR_+\e_i+\RR_+\e_j =
(\chiobs_{ij}-\RR_+\e_i+\RR_+\e_j)-\RR_+\e_i+\RR_+\e_j=\\
\chiobs_{ij}-(\RR_+ +\RR_+)\e_i+(\RR_+ +\RR_+)\e_j= 
\chiobs_{ij}-\RR_+\e_i+\RR_+\e_j = \chiobs_{i\succ j}
\end{multline}
We have obtained~\eqref{eq:geometric-invariance-obstacle}. Moreover, using the latter result, we have:
\begin{equation}
x\in \chifree_{i\succ j} \Leftrightarrow x\notin \chiobs_{i\succ j} \Leftrightarrow 
x\notin \left( \chiobs_{i\succ j}-\RR_+\e_i+\RR_+\e_j \right)
\end{equation}
Hence, we have:
\begin{multline}
x\in \chifree_{i\succ j} \Leftrightarrow \forall \alpha,\beta \geq 0, x\notin \left(\chiobs_{i\succ j} -\alpha\e_i+\beta \e_j\right) \Leftrightarrow\\
\forall \alpha,\beta \geq 0, \left(x+\alpha\e_i-\beta \e_j\right)\notin \chiobs_{i\succ j} 
 \Leftrightarrow 
 \forall \alpha,\beta \geq 0, \left(x+\alpha\e_i-\beta \e_j\right) \in \chifree_{i\succ j}
\end{multline}
We have obtained~\eqref{eq:geometric-invariance-collision-free}.
\end{proof}

\begin{proof}[Proof of~\eqref{eq:property-max} in Property~\ref{property:min-max}]
Take $i,j\in\robots$, $x,y\in\chifree_{i\succ j}$ and let $z:=\max\{x,y\}$. By definition of $\max$, there are four options:
\begin{itemize}
\item $z_i=x_i$ and $z_j=x_j$: in this case, $(z_i,z_j)=(x_i,x_j)\notin \kappa_{i\succ j}$, so that $z\in\chifree_{i \succ j}$.
\item $z_i=y_i$ and $z_j=y_j$: this is the symmetric case and $y\in\chifree_{i\succ j}$ implies that $z\in\chifree_{i \succ j}$.
\item $z_i=x_i \geq y_i$ and $z_j = y_j$: in this case, we have $z_i \geq y_i$ and $z_j=y_j$. By Property~\ref{property:geometric-invariance}, $y\in\chifree_{i\succ j}$ implies that $z \in \chifree_{i\succ j}$
\item $z_i=y_i \geq x_i$ and $z_j = x_j$: this is the symmetric case and $x\in\chifree_{i\succ j}$ implies that $z \in \chifree_{i\succ j}$
\end{itemize}
\end{proof}

\begin{proof}[Proof of~\eqref{eq:property-min} in Property~\ref{property:min-max}]
Take $i,j\in\robots$, $x,y\in\chifree_{i\succ j}$ and let $z:=\min\{x,y\}$. By definition of $\min$, there are four options:
\begin{itemize}
\item $z_i=x_i$ and $z_j=x_j$: in this case, $(z_i,z_j)=(x_i,x_j)\notin \kappa_{i\succ j}$, so that $z\in\chifree_{i \succ j}$.
\item $z_i=y_i$ and $z_j=y_j$: this is the symmetric case and $y\in\chifree_{i\succ j}$ implies that $z\in\chifree_{i \succ j}$.
\item $z_i=x_i$ and $z_j = y_j \leq x_j$: in this case, we have $z_i = x_i$ and $z_j \leq x_j$. By Property~\ref{property:geometric-invariance}, $x\in\chifree_{i\succ j}$ implies that $z \in \chifree_{i\succ j}$
\item $z_i=y_i$ and $z_j = x_j \leq y_j$: this is the symmetric case and $y\in\chifree_{i\succ j}$ implies that $z \in \chifree_{i\succ j}$
\end{itemize}
\end{proof}

\begin{proof}[Proof of Property~\ref{property-union-fixed-priority-cylinder-1}]
Take $i,j\in\robots$, $y\in\chiobs_{ij}$ and $x\in\left\{x\in\chi: x_i=y_i\right\}$. Either $x_j\geq y_j$ and by Property~\ref{property:geometric-invariance} $y\in\chiobs_{ij}\subset\chiobs_{i\succ j}$ implies that $x\in\chiobs_{i\succ j}$; or, $x_j\leq y_j$ and by Property~\ref{property:geometric-invariance} $y\in\chiobs_{ij}\subset\chiobs_{j\succ i}$ implies that $x\in\chiobs_{j\succ i}$. In both cases, we have $x\in\chiobs_{i\succ j} \cup \chiobs_{j\succ i}$.
\end{proof}

\begin{proof}[Proof of Property~\ref{property-union-fixed-priority-cylinder-2}]
Take $x^1\in\chiobs_{j\succ i}$, $x^2\in\chiobs_{i\succ j}$, and $x\in\chi$ satisfying $x^1_i\leq x_i\leq x^2_i$.  As $x^1\in\chiobs_{j\succ i}$ (which is non-empty open and lower-bounded along axis $i$ with the same bound as $\chiobs_{ij}$), $x^1_i > \inf\{y_i:y\in\chiobs_{ji}\}$ and as $x^2\in\chiobs_{i\succ j}$ (which is non-empty open and upper-bounded along axis $i$ with the same bound as $\chiobs_{ij}$), $x^2_i < \sup\{y_i:y\in\chiobs_{ij}\}$. Hence, we obtain:
\begin{equation}
\inf\{y_i:y\in\chiobs_{ij}\}<x_i<\sup\{y_i:y\in\chiobs_{ij}\}
\end{equation}
As a consequence, there exists $x^0\in\chiobs_{ij}$ such that $x^0_i=x_i$. By Property~\ref{property-union-fixed-priority-cylinder-1}, we obtain $x\in\left(\chiobs_{i\succ j}\cup \chiobs_{j\succ i}\right)$. 
\end{proof}

\subsection{The priority relation}

The definition of the completed obstacle region enables to easily define a priority relation for feasible paths. The fact that a feasible path necessarily and exclusively lies on one side or on the other side of each collision cylinder $\chiobs_{ij}$ is indeed equivalent to intersect, necessarily and exclusively, one of the completed cylinders $\chiobs_{i\succ j}$, or $\chiobs_{j\succ i}$.

\begin{definition}[Priority relation]
The priority relation $\succ$ is a binary relation on the set of robots $\robots$. For all $i,j\in\robots$, $i\succ j$ if $\im{\path} \cap \chiobs_{j\succ i} \neq \emptyset$.
\end{definition}
We say $\succ$ is the priority relation induced by $\path$. The theorem below shows that the relation $\succ$ satisfies basic properties that one can expect from a "priority relation". More precisely, $\succ$ does not define a priority relation between two robots that cannot collide ($\chiobs_{ij}=\emptyset$) and if two robots can potentially collide, a priority relation exists and we have $i\succ j$ or $j\succ i$ exclusively, i.e., if robot $i$ has priority over robot $j$ then robot $j$ does not have priority over robot $i$.

\begin{theorem}[Priority relation properties]
Let $\path\in\phifree$ denote a feasible path and $\succ$ the priority relation induced by $\path$. For all $i,j\in\robots$ such that $\chiobs_{ij}\neq\emptyset$, we have necessarily and exclusively $i \succ j$ or $j \succ i$. For all $i,j\in\robots$ such that $\chiobs_{ij}=\emptyset$, we have $i \not\succ j$.
\label{thm:priority-relation}
\end{theorem}

Note that the first statement of the above theorem can be formulated synthetically as: the binary relation $\succ$ is asymmetric. To prove Theorem~\ref{thm:priority-relation}, we start with the following lemma illustrated in Figures~\ref{fig:SW-NE-completion} and~\ref{fig:south-west-real-space} and proved in Appendix~\ref{app:south-west-completion}:
\begin{lemma}[South-West and North-East completion~\cite{ODonnell1989}]
For all feasible paths $\path\in\phifree$,
\begin{equation}
\forall i,j\in\robots, \im{\path}\cap\left(\chiobs_{i\succ j}\cap\chiobs_{j\succ i}\right) = \emptyset
\end{equation}
\label{lemma:south-west-north-east-completion}
\end{lemma}
\begin{figure}[p]
\begin{center}
\includegraphics[width=1.0\linewidth]{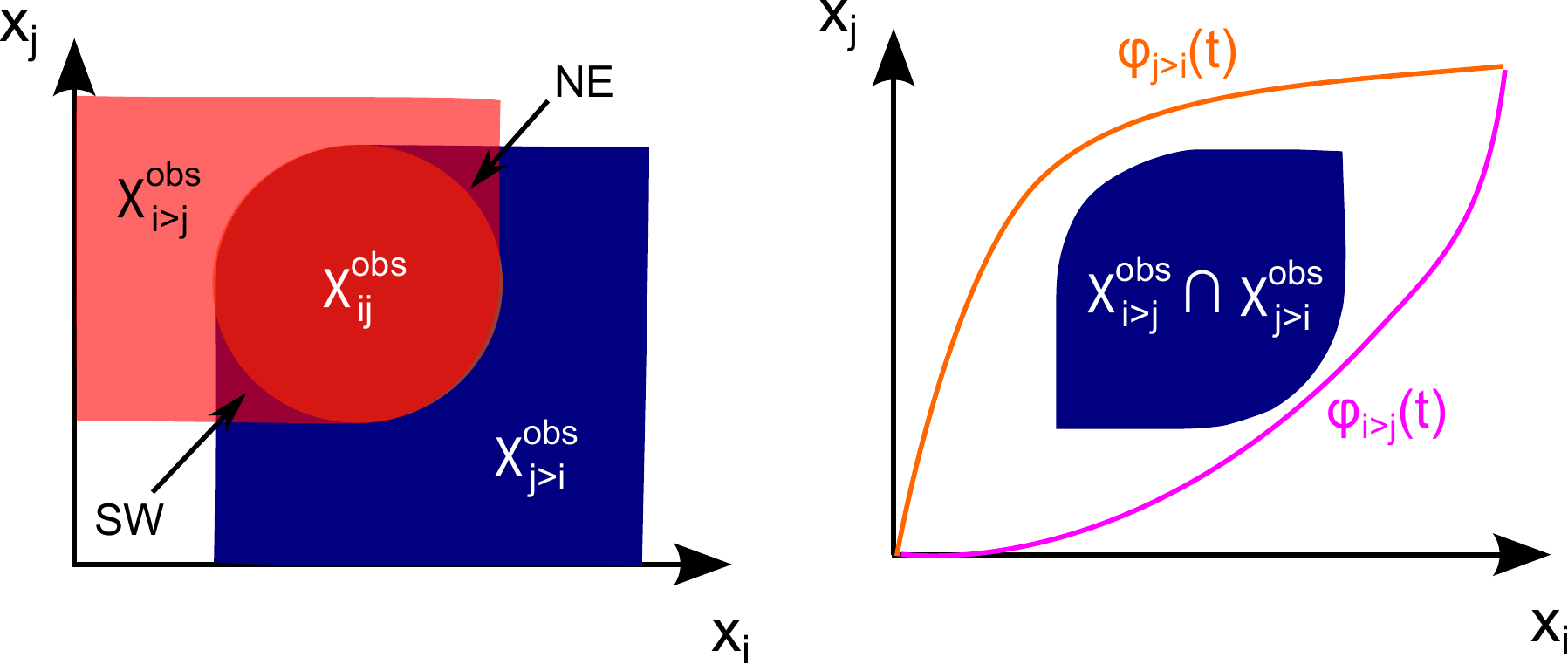}\hfill
\end{center}
\caption{Illustration of Lemma~\ref{lemma:south-west-north-east-completion}. Both $\path_{i\succ j}$ and $\path_{j\succ i}$ are collision-free with regards to $\chiobs_{i\succ j}\cap\chiobs_{j\succ i}$. Compared to $\chiobs_{ij}$, $\chiobs_{i\succ j}\cap\chiobs_{j\succ i}$ additionally contains the south-west (SW) region of the obstacle region and the north-east (NE) region of the obstacle region. Feasible paths do not go through the south-west region, as it necessarily leads to a "deadlock" between robots $i$ and $j$. The north-east region cannot be reached by a feasible (non-decreasing) path.}
\label{fig:SW-NE-completion}
\end{figure}
\begin{figure}[p]
\begin{center}
\includegraphics[width=0.5\linewidth]{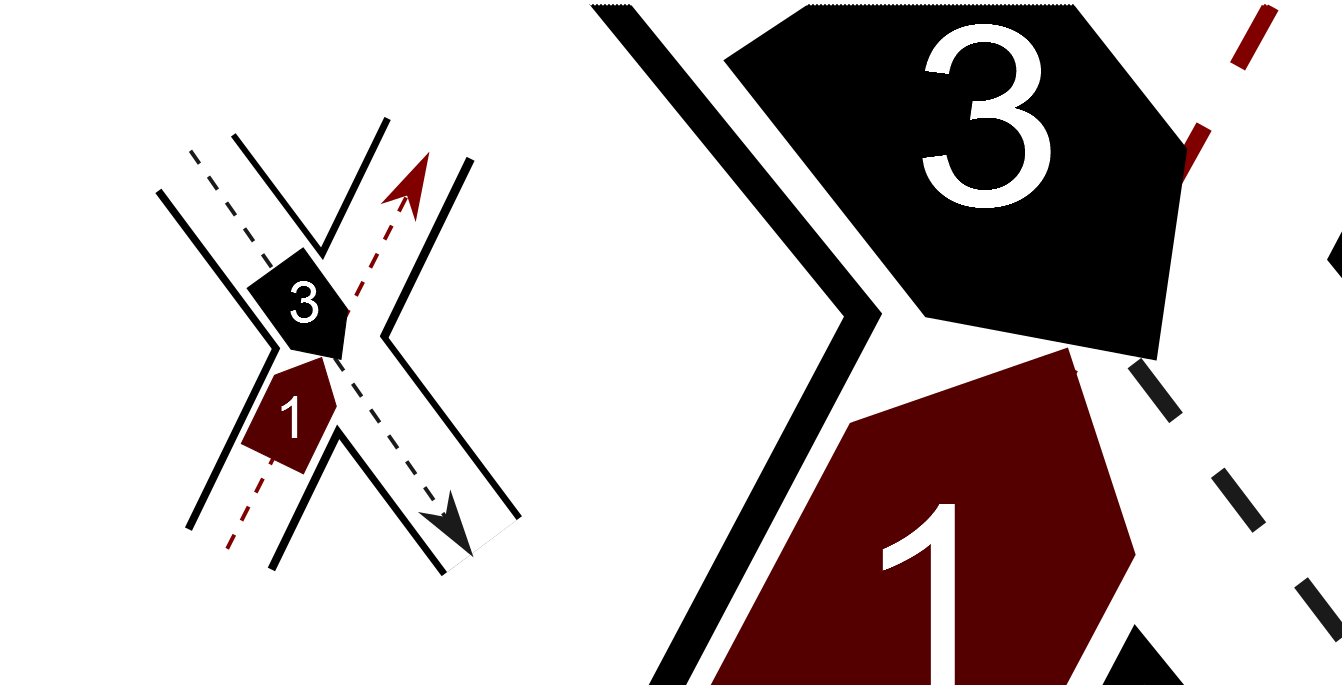}\hfill
\end{center}
\caption{Two robots at a deadlock configuration in the south-west region.}
\label{fig:south-west-real-space}
\end{figure}
Note that it was already noticed in~\cite{ODonnell1989} that south-west completion enables to avoid deadlocks in two-robot systems.

\begin{proof}[Proof of Theorem~\ref{thm:priority-relation}] 
Take a feasible path $\path\in\phifree$ and let $\succ$ denote the priority relation induced by $\path$. Take $i,j\in\robots$ such that $\chiobs_{ij}=\emptyset$. Then, we have $\chiobs_{j\succ i}=\emptyset$, so that $\im{\path}\cap\chiobs_{j\succ i}=\emptyset$, that is $i \not\succ j$. Take $i,j\in\robots$ such that $\chiobs_{ij}\neq\emptyset$ and take $y\in\chiobs_{ij}$. Remember that we have $(\path(0)-\RR_+^n) \subset \chifree$ and $(\path(1)+\RR_+^n) \subset \chifree$. As a consequence, there are two options as depicted in Figure~\ref{fig:priority-relation-property}:
\begin{enumerate}[(a)]
\item either $y\in\im{\path}-\RR_+\e_i+\RR_+\e_j$: it implies that $\im{\path} \cap \chiobs_{j \succ i} \neq \emptyset$;
\item or $y\in\im{\path}-\RR_+\e_j+\RR_+\e_i$: it implies that $\im{\path} \cap \chiobs_{i \succ j} \neq \emptyset$.
\end{enumerate}

\begin{minipage}{\linewidth}
\begin{center}
\includegraphics[width=0.8\linewidth]{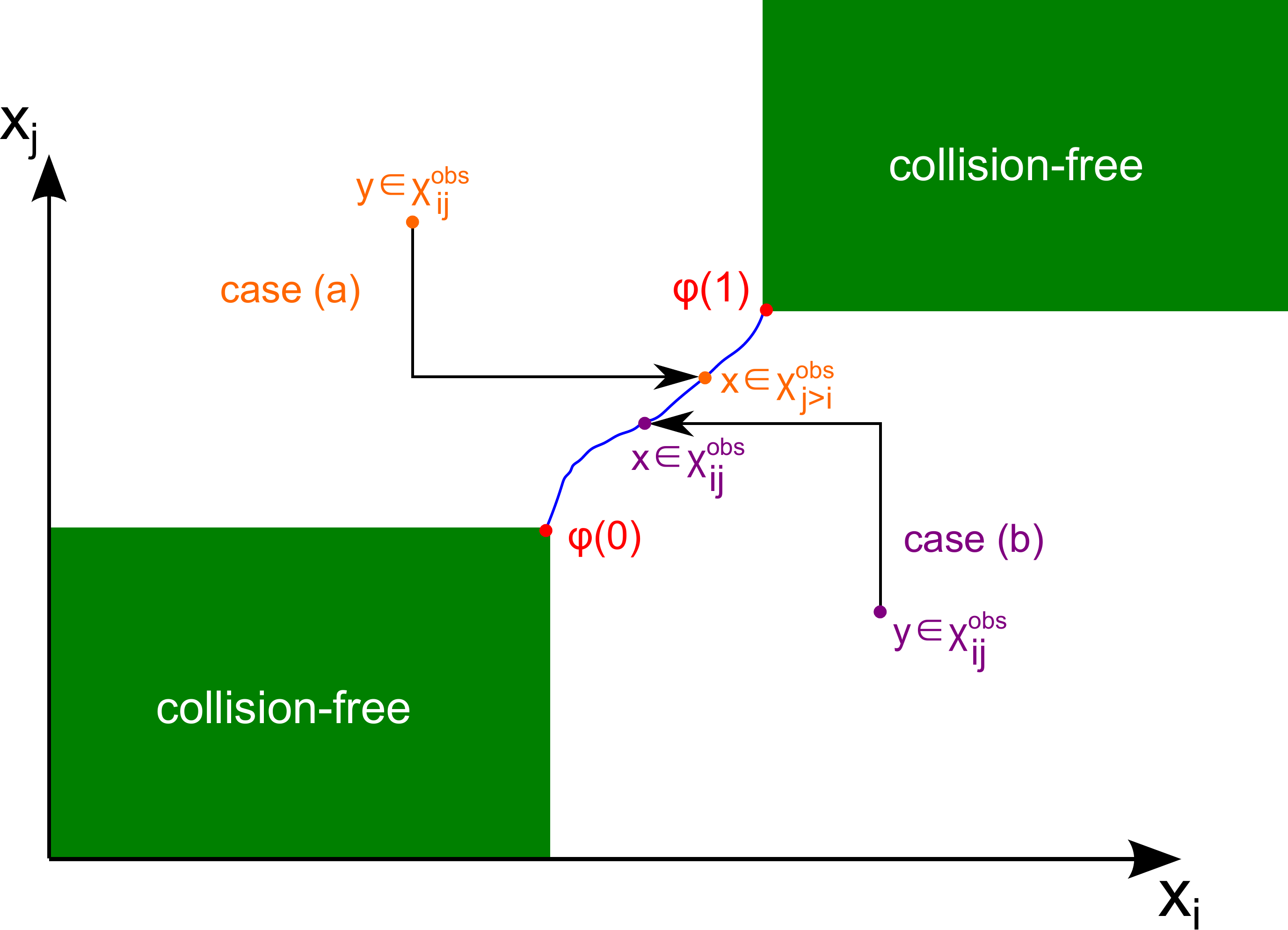}\hfill
\end{center}
\captionof{figure}{The two cases that appear to prove that any path $\path\in\phifree$ necessarily intersects $\chiobs_{i\succ j}$ or $\chiobs_{j\succ i}$ provided $\chiobs_{ij}\neq\emptyset$.}
\label{fig:priority-relation-property}
\end{minipage}
\vspace{0.2cm}

Hence, a feasible path necessarily intersects $\chiobs_{i\succ j}$ or $\chiobs_{j \succ i}$, so we have necessarily $i \succ j$ or $j \succ i$.

Now, we will prove that it is exclusive by contradiction. Take a feasible path $\path$ and assume that for some $t^1\in[0,1]$, $\path(t^1)\in\chiobs_{i \succ j}$ and for some $t^2\in[0,1]$, $\path(t^2)\in\chiobs_{j \succ i}$. Assume arbitrarily that $t^1\leq t^2$ (otherwise, exchange the roles of $i$ and $j$), which implies that $\path(t^1)\leq\path(t^2)$. Consider the subset of $\chi$ defined below:
\begin{equation}
K:=\left\{x\in\chi: x^1_i\leq x_i \leq x^2_i \text{ and } x^1_j\leq x_j \leq x^2_j\right\}
\end{equation}
By Property~\ref{property-union-fixed-priority-cylinder-2}, we have:
\begin{equation}
K \subset \left( \chiobs_{i\succ j} \cup \chiobs_{j\succ i} \right)
\label{eq:inclusion-K}
\end{equation}

As $\path$ is non-decreasing, for all $t\in[t^1,t^2]$, $\path(t)\in K$. If $\path(t)\in\chiobs_{i\succ j}\cap\chiobs_{j \succ i}$ for some $t\in[t^1,t^2]$, $\path$ would not be feasible by Lemma~\ref{lemma:south-west-north-east-completion}. Hence, we have:
\begin{eqnarray}
\path(t^1)&\in&\chiobs_{i\succ j}\setminus\chiobs_{j \succ i} \label{eq:phi-t1}\\
\path(t^2)&\in&\chiobs_{j\succ i}\setminus\chiobs_{i \succ j}\label{eq:phi-t2}
\end{eqnarray}
and for all $t\in[t^1,t^2]$,
\begin{equation}
\path(t)\in  \left(\chiobs_{i\succ j}\setminus\chiobs_{j \succ i}\right) \cup \left(\chiobs_{j\succ i}\setminus\chiobs_{i \succ j}\right)  
\label{eq:phi-in-union-two-cylinders}
\end{equation}
As $\left(\chiobs_{i\succ j}\setminus\chiobs_{j \succ i}\right)\cap\left(\chiobs_{j\succ i}\setminus\chiobs_{i \succ j}\right)=\emptyset$, by continuity of $\path$ (see Lemma~\ref{lemma:two-subsets-frontier} in Appendix~\ref{app:topology-properties}), there exists some $t^0\in[t^1,t^2]$ such that:
\begin{equation}
\path(t^0)\in \partial \left(\chiobs_{i\succ j}\setminus\chiobs_{j \succ i}\right) \cap \partial \left(\chiobs_{j\succ i}\setminus\chiobs_{i \succ j}\right)
\label{eq:phi-at-frontier-two-cylinders}
\end{equation}
As $\chiobs_{i\succ j}$ and  $\chiobs_{j\succ i}$ are open, by Lemma~\ref{lemma:frontier-A-minus-B-cap-B} (see Appendix~\ref{app:topology-properties}), we have $\chiobs_{j\succ i}\cap \partial (\chiobs_{i\succ j}\setminus\chiobs_{j \succ i})=\emptyset$ and $\chiobs_{i \succ j} \cap \partial (\chiobs_{j\succ i}\setminus\chiobs_{i \succ j})= \emptyset$. Hence, we obtain:
\begin{equation}
\partial \left(\chiobs_{i\succ j}\setminus\chiobs_{j \succ i}\right) \cap \partial \left(\chiobs_{j\succ i}\setminus\chiobs_{i \succ j}\right)  \cap \left( \chiobs_{i \succ j} \cup \chiobs_{j \succ i} \right) = \emptyset
\end{equation}
Equations~\eqref{eq:phi-in-union-two-cylinders} and~\eqref{eq:phi-at-frontier-two-cylinders} are therefore absurd as disjoint sets have no element in common.
\end{proof}

\subsection{The priority graph}

As any binary relation, the priority relation admits a graph representation. 
\begin{definition}[Priority graph]
The priority graph induced by a feasible path $\path$ is the oriented graph $G$ whose vertices are $V(G):=\robots$ and such that there is an edge from $i$ to $j$ if $i \succ j$ where $\succ$ denotes the priority relation induced by $\path$. We write $(i,j)\in E(G)$ where $E(G)$ denotes the edge set of the priority graph.
\end{definition}
Two representations of the priority graph are depicted in Figure~\ref{fig:priority-graph}. We let $\Gamma$ denote the application that returns the priority graph $\Gamma(\path)$ induced by a feasible path $\path\in\phifree$. $\Gamma(\path)$ is the graph of the priority relation $\succ$ induced by $\path$. Theorem~\ref{thm:priority-relation} can be rewritten as follows:
\begin{equation}
\forall \path\in\phifree, \Gamma(\path)\in\graphs
\end{equation}
where $\graphs$ is the set of oriented graphs $G$ with vertices $V(G):=\robots$, whose edge set $E(G)$ satisfies:
\begin{equation}
\forall i,j\in\robots,\quad (i,j)\in E(G) \Leftrightarrow
\left\{\begin{matrix}
 \chiobs_{ij} &\neq& \emptyset\\
(j,i)&\notin& E(G)
\end{matrix}\right.
\end{equation}
\begin{figure}[!htbp]
\begin{center}
\includegraphics[width=0.35\linewidth]{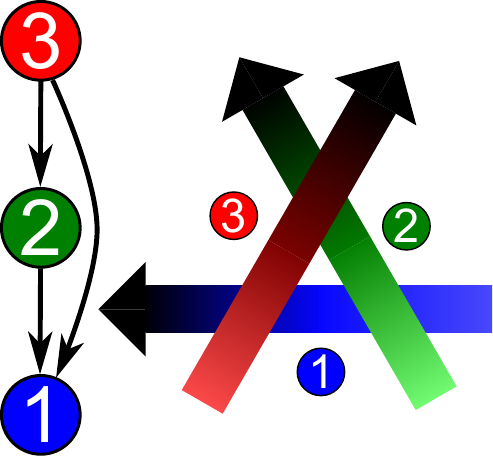} \hspace{2em}
\includegraphics[width=0.35\linewidth]{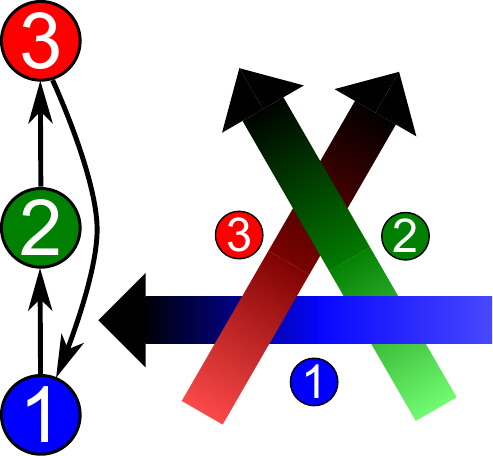}
\end{center}
\caption{Two representations of priority relations. In each drawing relation is represented in two ways: as a complete oriented graph, where orientation yields the priority; and as trajectories over time, foreground being first, background later. The left drawing represents a relation that  is an order (even a total order). The right drawing shows a relation that is not an order.}
\label{fig:priority-graph}
\end{figure}

We say a graph $G$ is a priority graph if $G\in\graphs$. It is natural as a graph $G\in\graphs$ defines a binary relation between robots whose paths intersect, i.e., it defines a priority between all and only robots that need to coordinate. Then, a natural question is: given a priority graph $G$, does a feasible path exist whose induced priority graph is $G$ ? Let $\Gamma^{-1}(G)$ denote the set of feasible paths whose induced priority graph is $G$. The question can then be rephrased as: given a priority graph $G$, do we have $\Gamma^{-1}(G)\neq\emptyset$ ? If there exists some path $\path\in\Gamma^{-1}(G)$, $\path$ should be collision-free with regards to each completed cylinder $\chiobs_{i\succ j}$ for all $(i,j)\in E(G)$. Hence, it is natural to define the completed obstacle region and the collision-free region with regards to a given priority graph $G\in\graphs$ as follows:
\begin{eqnarray}
\chiobs_{G} &:=& \bigcup_{(i,j)\in E(G)} \chiobs_{i \succ j}\\
\chifree_G &:=& \chi \setminus \chiobs_G
\end{eqnarray}
$\{\chiobs_G,\chifree_G\}$ form a partition of $\chi$. By construction, we have $\chiobs\subset\chiobs_G$: respecting the assigned priorities requires remaining in a more constrained space. For all feasible paths $\path\in\phifree$, we have the equivalences:
\begin{equation}
\path\in\Gamma^{-1}(G) \Leftrightarrow \im{\path} \subset \chifree_G \Leftrightarrow \im{\path} \cap \chiobs_G = \emptyset \Leftrightarrow \path\in\phifreeG
\end{equation}
It results that the set of feasible paths is the union of the sets of feasible paths respecting a certain priority graph over all possible priority graphs:
\begin{equation}
\phifree=\bigcup_{G\in \graphs}\phifreeG
\end{equation}
However, $\left\{\phifreeG: G\in\graphs\right\}$ do not form a partition of $\phifree$ as some $\phifreeG$ may be empty. 

The next chapter studies the coordination under assigned priorities, i.e., when the obstacle region is completed with configurations not respecting the assigned priorities, forming the completed obstacle region $\chiobs_G$. In Section~\ref{sec:homotopy}, we will see that each non-empty set of feasible paths respecting a certain priority graph $G$, i.e., each non-empty $\phifreeG$, is a homotopy class of feasible paths continuously deformable into each other. Section~\ref{sec:feasibility} provides a necessary and sufficient condition on $G$ for $\phifreeG$ not to be empty, that is a necessary and sufficient condition on priorities to guarantee that respecting these priorities, all robots can eventually go through the intersection (no deadlock).

\chapter[The coordination space under assigned priorities]{The coordination space\\ under assigned priorities}
\label{chap:priorities-homotopy}
\minitoc

Preliminaries of this work can be found in our conference paper~\cite{Gregoire2012-optimal}.

\paragraph{Sketch of the chapter} The first section of the present chapter demonstrates that respecting assigned priorities does not require robots to follow a particular feasible path in the coordination space. However, the path described by robots in the coordination space needs to remain in a quite large homotopy class of feasible paths continuously deformable into each other. This homotopy class is uniquely encoded by the priority graph. The second section proves that deadlock avoidance can be guaranteed at the priority assignment phase. Either priorities are "feasible" and ensure all robots will eventually go through the intersection provided they respect the assigned priorities; or, the multi robot system will inevitably reach a deadlock configuration.

\section{Priorities: a homotopy invariant}
\label{sec:homotopy}

\subsection{Homotopy classes}

$\paths(\chi)$ is equipped with the topology of pointwise convergence and the notion of homotopic feasible paths is defined as follows. 

\begin{definition}[Homotopic paths]
Given two feasible paths $\path^1$ and $\path^2$, $\path^1$ is homotopic to $\path^2$ if there exists a continuous application $H$ defined on $[0,1]$ such that $H(0)=\path^1$, $H(1)=\path^2$ and for all $\alpha\in[0,1]$, the path $H(\alpha)$ is a feasible path.
\end{definition}
\begin{figure}[!htbp]
\begin{center}
\includegraphics[width=0.7\linewidth]{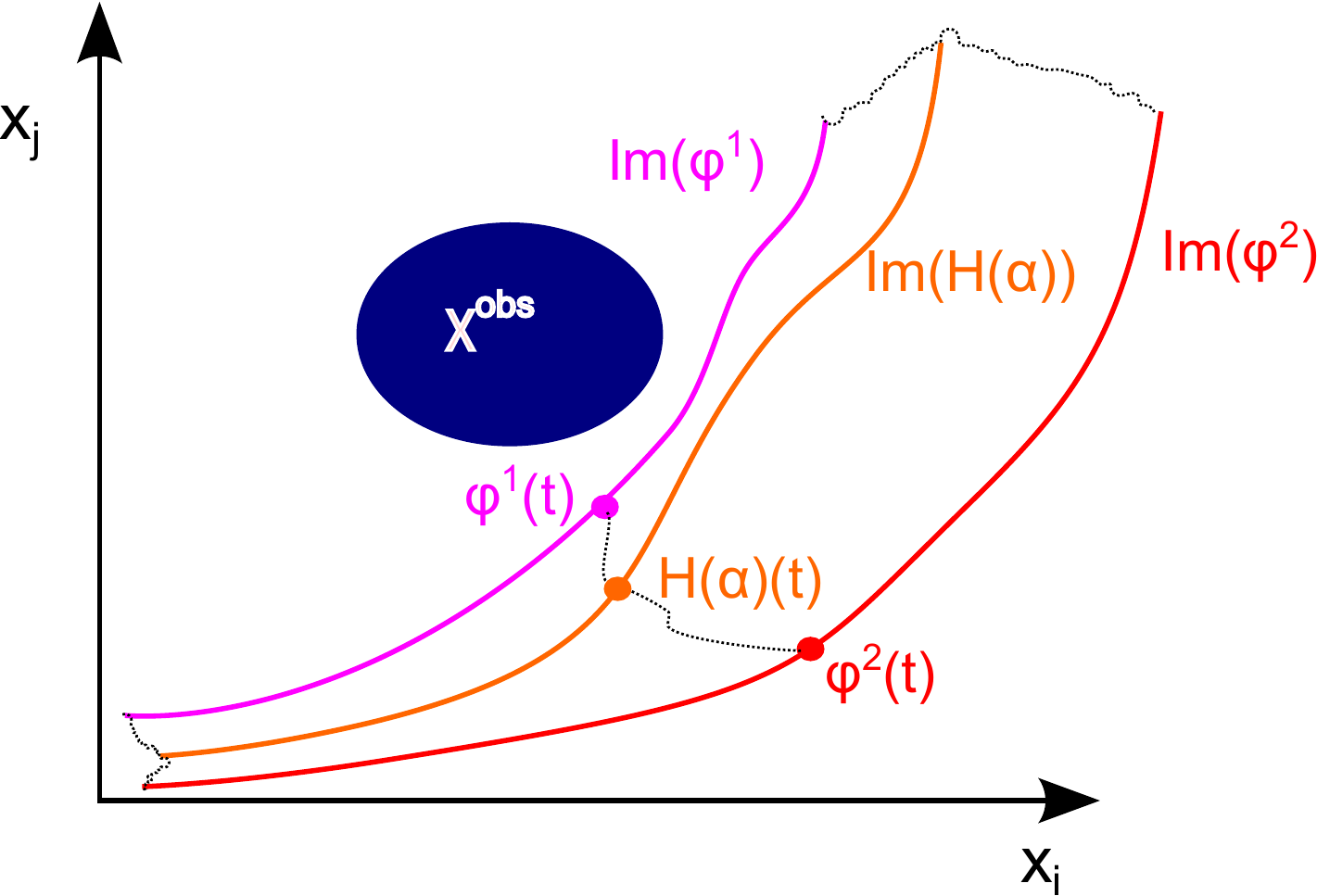}
\end{center}
\caption{Two homotopic paths. As both paths lie below the obstacle region, they can be continuously transformed into each other remaining collision-free along the transformation.}
\label{fig:homotopy}
\end{figure}
We write $\path^1 \sim \path^2$. Said differently, two feasible paths are homotopic if they can be continuously transformed into each other remaining feasible (in particular collision-free and non-decreasing) along the transformation as depicted in Figure~\ref{fig:homotopy}. Homotopy defines an equivalence relation on feasible paths:
\begin{property}[Homotopy: an equivalence relation]
\label{property:homotopy-equivalence-relation}
The homotopy relation~$\sim$ is an equivalence relation on $\phifree$.
\end{property}
\begin{proof}
We have to prove that $\sim$ is an equivalence relation, i.e., that it is (\ref{item:sim-reflexive}) reflexive, (\ref{item:sim-symmetric}) symmetric and  (\ref{item:sim-transitive}) transitive.
\begin{enumerate}[(a)]
\item Take a feasible path $\path\in\phifree$ and consider the constant application $H:\alpha\in[0,1] \mapsto \path$. $H(0)=\path$, $H(1)=\path$ and for all $\alpha\in[0,1]$, the path $H(\alpha)\equiv \path$ is a feasible path. Hence, $\path \sim \path$ and $\sim$ is reflexive. \label{item:sim-reflexive}
\item Take feasible paths $\path,\psi\in\phifree$ and assume that $\path \sim \psi$. Then, there exists $H$ defined on $[0,1]$ such that $H(0)=\path$, $H(1)=\psi$ and for all $\alpha\in[0,1]$, the path $H(\alpha)$ is a feasible path. Consider $G:\alpha\in[0,1]\mapsto H(1-\alpha)$. We have $G(0)=\psi$, $G(1)=\path$ and for all $\alpha\in[0,1]$, the path $G(\alpha)\equiv H(1-\alpha)$ is a feasible path. Hence, $\psi \sim \path$ and $\sim$ is symmetric.
\label{item:sim-symmetric}
\item Take feasible paths $\path^1,\path^2,\path^3\in\phifree$ and assume that $\path^1 \sim \path^2$ and $\path^2 \sim \path^3$. Then, there exists $H^{12}$ defined on $[0,1]$ such that $H^{12}(0)=\path^1$, $H^{12}(1)=\path^2$ and for all $\alpha\in[0,1]$, the path $H^{12}(\alpha)$ is a feasible path and there exists $H^{23}$ defined on $[0,1]$ such that $H^{23}(0)=\path^2$, $H^{23}(1)=\path^3$ and for all $\alpha\in[0,1]$, the path $H^{23}(\alpha)$ is a feasible path. Consider $H$ defined on $[0,1]$ as follows:
\begin{eqnarray}
\forall \alpha \in[0,1/2], H(\alpha)&:=&H^{12}(2\alpha)\\
\forall \alpha \in(1/2,1], H(\alpha)&:=&H^{23}(2(\alpha-1/2))
\end{eqnarray}
$H$ is continuous as $\lim_{\alpha \underset{<}{\to} 1/2}H(\alpha)=H^{12}(1)=\path^2$ and $\lim_{\alpha \underset{>}{\to} 1/2} H(\alpha) = H^{23}(0)=\path^2$. Moreover, $H(0)=\path^1$, $H(1)=\path^3$ and for all $\alpha\in[0,1]$, the path $H(\alpha)$ is a feasible path as it satisfies $H(\alpha)\equiv H^{12}(2\alpha)$ or $H(\alpha)\equiv H^{23}(2(\alpha-1/2))$ which are both feasible paths. Hence, $\path^1 \sim \path^3$ and $\sim$ is transitive. \label{item:sim-transitive}
\end{enumerate}
\end{proof}

As a direct consequence of Property~\ref{property:homotopy-equivalence-relation}, we can define homotopy classes as the equivalence classes induced by this equivalence relation. Let $\homotopy:=\phifree/\sim$ denote the homotopy classes of feasible paths, that is the quotient set of $\phifree$ by the equivalence relation $\sim$. Homotopy classes form a partition of $\phifree$~\cite{Hu1959}.

\begin{figure}[!htbp]
\begin{center}
\includegraphics[width=0.49\linewidth]{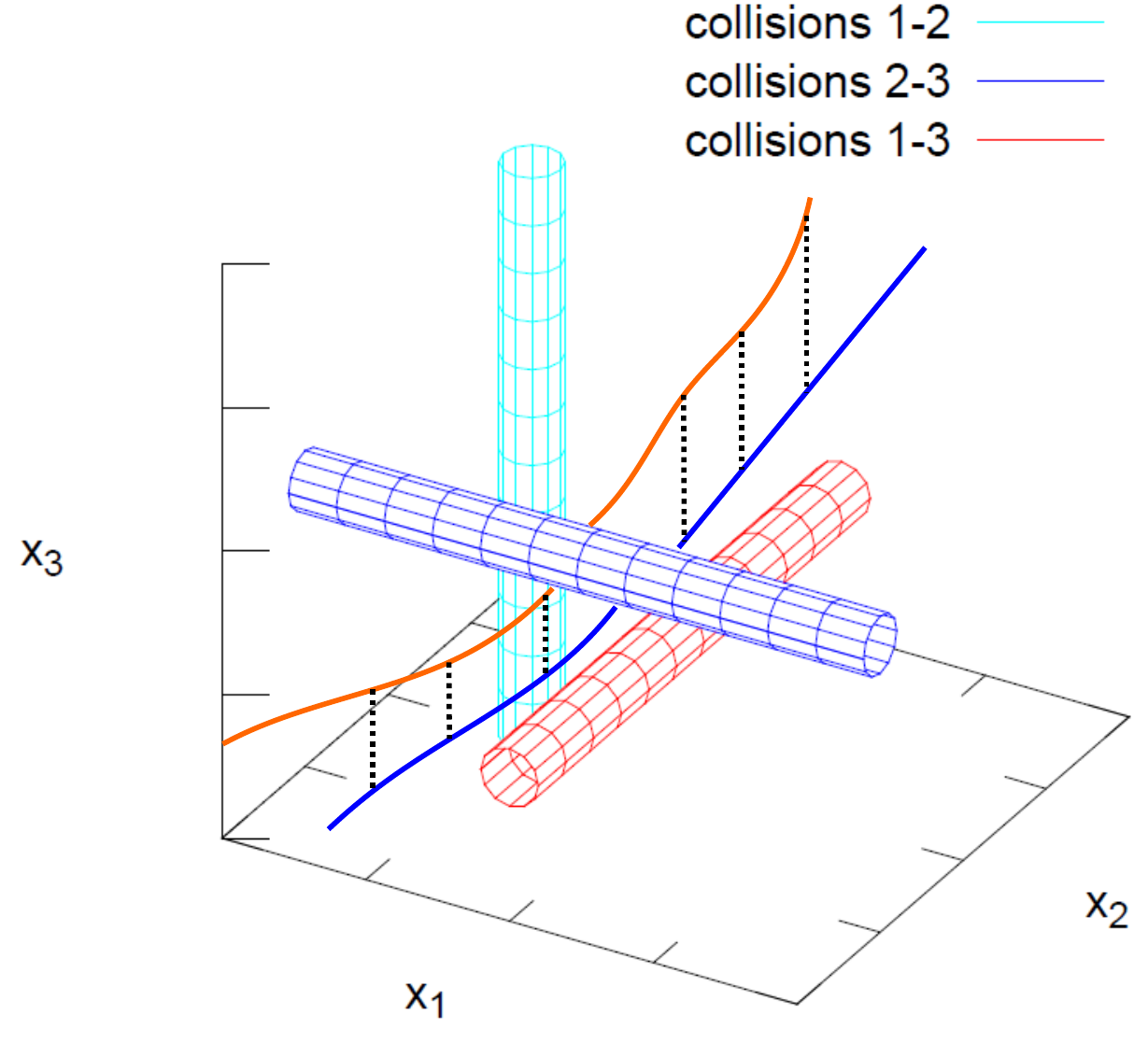}
\includegraphics[width=0.49\linewidth]{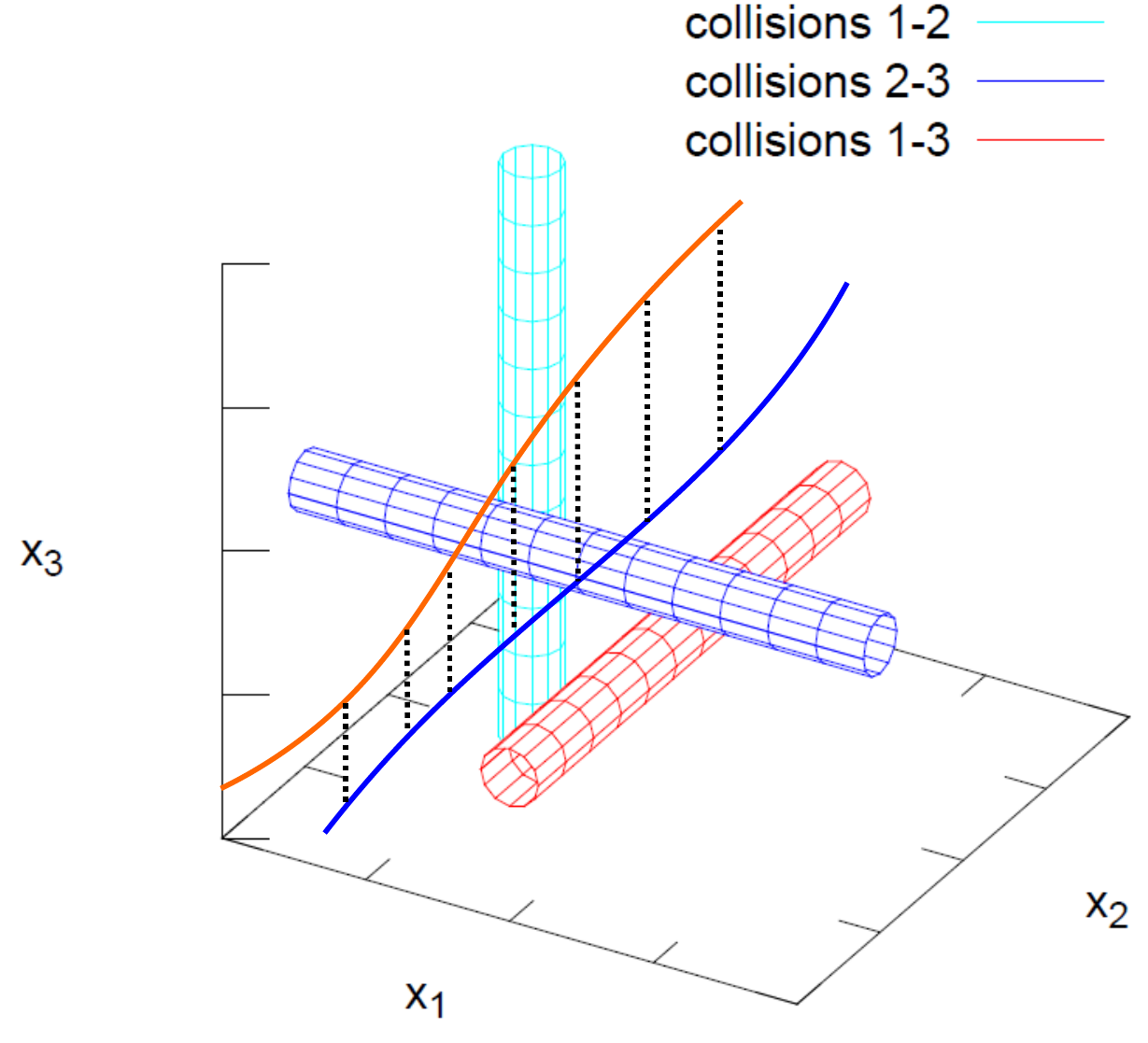}
\end{center}
\caption{Two homotopy classes of feasible paths (and two elements of each class) in a three-dimensional coordination space.}
\label{fig:homotopy-classes-3D}
\end{figure}

The existence of homotopy classes of solutions to the coordination problem was already noticed, e.g., in~\cite{Ghrist2005} (see Figure~\ref{fig:homotopy-classes-3D}). In~\cite{Hu2003}, it is also noticed that there exist a finite number of homotopy classes of solutions to the coordination of multiple agents moving on a plane between fixed points using the concept of braids~\cite{Birman1974}. However, in that work, the geometric paths of agents is not fixed, and optimization is precisely carried out in order to find an optimal trajectory, both spatially and timely. This is not adapted for an application to the coordination of robots on roadways as the two-dimensional space is very constrained and robots have a quite low degree of freedom to choose a geometric path to go through the intersection. It thus appears much more suitable to study the homotopy classes of feasible paths in the coordination space instead of studying the homotopy classes of braids.

\subsection{Invariance of the priority graph}

If previous work already noticed the existence of homotopy classes in multi robot coordination, to our knowledge, no meaningful representative is proposed to encode homotopy classes. In the following, we present the main result of this part: priorities uniquely encode homotopy classes of feasible paths in the coordination space. The existence of a finite number of homotopy classes thus merely appears as the consequence of the finite number of possible priority graphs.

We let $\Gamma(\phifree):=\{\Gamma(\path):\path\in\phifree\}$ denote the set of values taken by the priority graph over all feasible paths. $\Gamma(\phifree)$ is a subset of $\graphs$ containing graphs $G$ such that there exists a feasible path $\path\in\phifree$ satisfying $\Gamma(\path)=G$. The following theorem (illustrated in Figures~\ref{fig:homotopy-priority-graph-unique-representative} and~\ref{fig:homotopy-classes}) shows that priorities and homotopy classes are strongly linked: more precisely, there is a bijective relationship between homotopy classes and "feasible priority graphs" (this term will be precisely defined in Section~\ref{sec:feasibility}). We say the priority graph encodes the homotopy class.

\begin{figure}[!htbp]
\begin{center}
\includegraphics[width=0.8\linewidth]{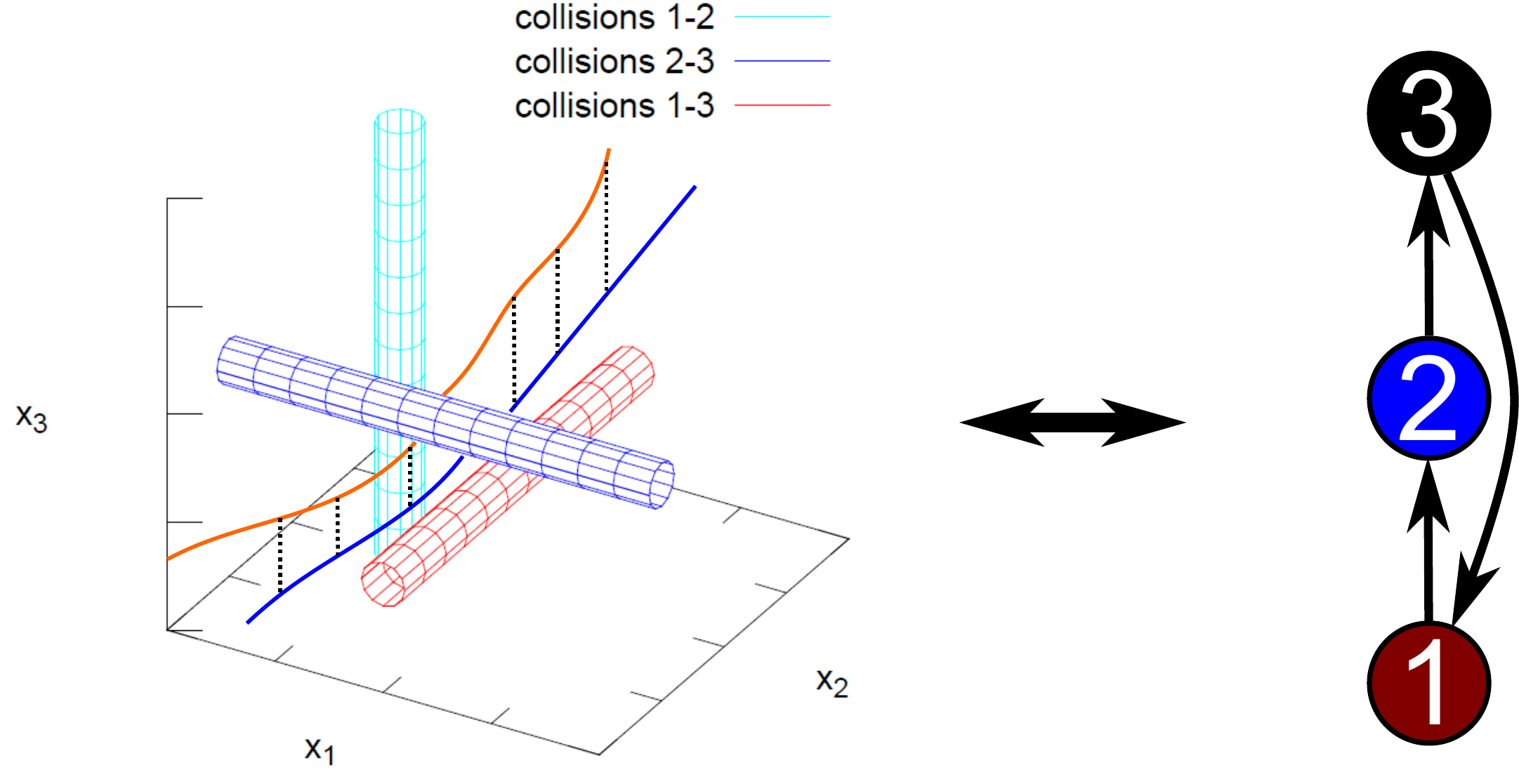}
\end{center}
\caption{A homotopy class of feasible paths in a three-dimensional coordination space and its corresponding unique representative as a priority graph.}
\label{fig:homotopy-priority-graph-unique-representative}
\end{figure}
\begin{figure}[!htbp]
\begin{center}
\includegraphics[width=0.8\linewidth]{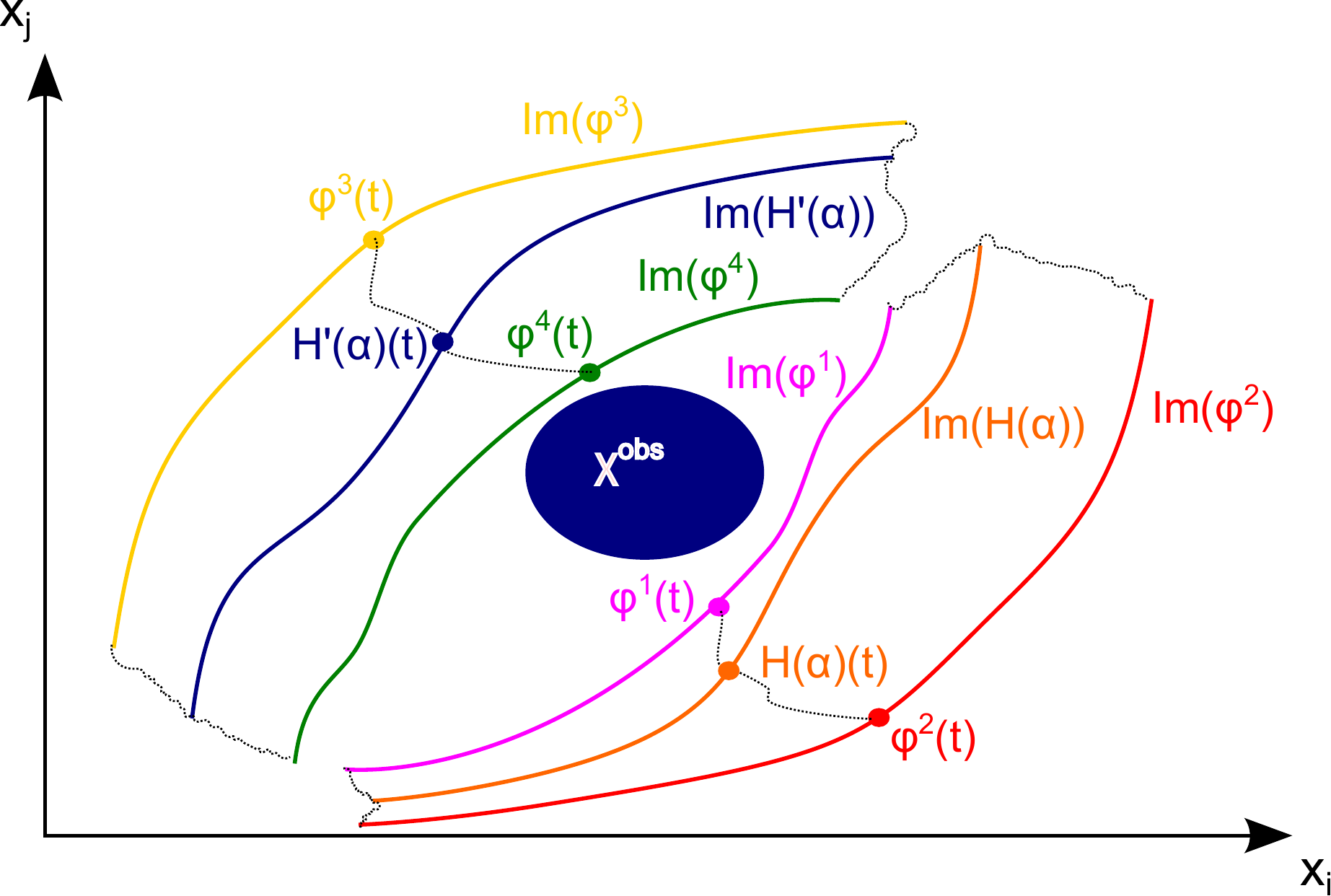}
\end{center}
\caption{In a two-dimensional scenario ($\chi=\RR^2$), provided $\chiobs_{ij}\neq\emptyset$, there are exactly two homotopy classes: all feasible paths lying above the obstacle region form the first homotopy class and all feasible paths lying below the obstacle region form the second homotopy class.}
\label{fig:homotopy-classes}
\end{figure}

\begin{theorem}[Invariance of the priority graph]
The priority graph is an invariant of the homotopy classes of feasible paths that it is distinct for each class: $\homotopy$ is in bijection with $\Gamma(\phifree)$.
\label{thm:invariance-priority-graph}
\end{theorem}
\begin{proof}[Proof of invariance]
First we will prove that the priority graph is an invariant of the homotopy classes of feasible paths. Consider a feasible path $\path\in\phifree$. For all $i,j\in\robots$, $(i,j)\in E(\Gamma(\path))$ if $\path$ intersects $\chiobs_{j \succ i}$ and the set $\chiobs_{j \succ i}$ is open. If a feasible path $\path$ intersects an open set, any feasible path $\psi\in\phifree$ close enough to $\path$ (in the topology of pointwise convergence) also intersects this open set. Hence, we have:
\begin{equation}
\forall i,j\in\robots,~ (i,j)\in E(\Gamma(\path)) \Leftrightarrow (i,j)\in E(\Gamma(\psi))
\end{equation}
provided $\psi$ is close enough to $\path$. Therefore, $\Gamma$ is continuous and since it takes discrete values, it is thus constant in homotopy classes of feasible paths. (We identify $\Gamma$ with the set of applications $g_{ij}:\paths(\chifree)\to\{-1,0,1\}$ satisfying $g_{ij}(\path)=1$ if $i\succ j$, $-1$ if $j\succ i$, and $0$ otherwise.) In conclusion, the priority graph is an invariant of the homotopy classes of feasible paths.
\end{proof}
\begin{proof}[Proof of uniqueness]
To prove uniqueness, consider two feasible paths $\path^1$ and $\path^2$ with the same induced priority graph $G$: $\path^1,\path^2\in\phifreeG$. We have to prove that $\path^1$ and $\path^2$ are homotopic. Consider the following continuous transformation: 
\begin{equation}
H: \alpha\in[0,1] \mapsto \min\left\{ \alpha \path^2(1)+(1-\alpha) \path^1(1), \path^1(\bullet+\alpha), \max\left\{ \path^1, \path^2 \right\} \right\}
\end{equation}
where by convention $\path^1(t+\alpha)\equiv \path^1(1)$ if $t+\alpha\geq 1$. Figure~\ref{fig:homotopy-example} illustrates the proposed transformation in the particular case where the two paths have the same endpoints.

\begin{minipage}{\linewidth}
\begin{center}
\includegraphics[width=1.0\linewidth]{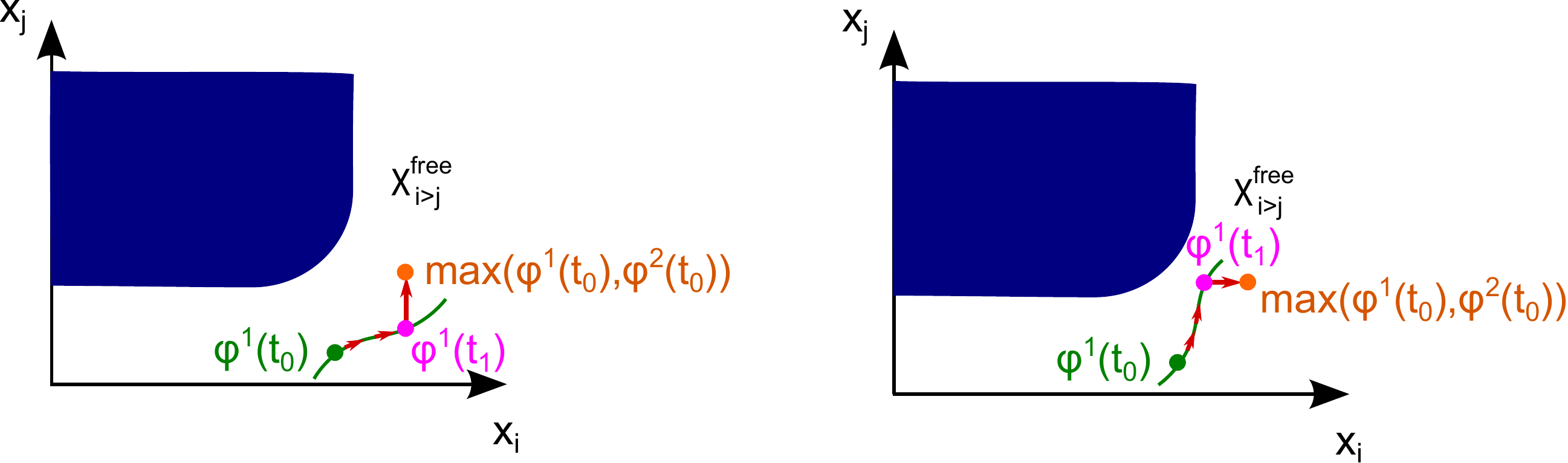}
\end{center}
\captionof{figure}{Illustration of the transformation of $\path^1$ into $\max(\path^1,\path^2)$. At any point of time $t_0$, $\max(\path^1(t_0),\path^2(t_0))$ necessarily lies on the north-east with respect to $\path^1(t_0)$. As a consequence, the two above cases may appear, and in each case, $\path^1(t_0)$ can be continuously transformed into $\max(\path^1(t_0),\path^2(t_0))$ without collision by following the red arrows.}
\label{fig:homotopy-example}
\end{minipage}
\vspace{0.2cm}

$H$ is continuous,
\begin{eqnarray}
H(0) &=& \min\left\{ \path^1(1), \path^1, \max\left\{ \path^1, \path^2 \right\} \right\} = \path^1\\
H(1) &=& \min\left\{ \path^2(1), \path^1(1), \max\left\{ \path^1, \path^2 \right\} \right\}
\end{eqnarray}
Hence, $H$  continuously transforms $\path^1$ into $\min\{ \path^2(1), \path^1(1), \max\{ \path^1, \path^2 \} \}$. Now, we prove that for all $\alpha\in[0,1]$, $H(\alpha)$ is a feasible path. We need to prove that for all $\alpha\in[0,1]$, (\ref{item:continuous}) $H(\alpha)$ is continuous, (\ref{item:start-configuration}) satisfies $(H(\alpha)(0)-\RR_+^n)\subset\chifree$ and (\ref{item:end-configuration}) $(H(\alpha)(1)+\RR_+^n)\subset\chifree$, (\ref{item:non-decreasing}) is non-decreasing, and (\ref{item:collision-free}) is collision-free. 
\begin{enumerate}[(a)]
\item $H(\alpha)$ is continuous as the result of the application of continuous operators $\min$, $\max$ and delay on continuous paths. \label{item:continuous}
\item $\path^1$ and $\path^2$ being feasible, we have $(\path^1(0)-\RR_+^n)\subset\chifree_G$ and $(\path^2(0)-\RR_+^n)\subset\chifree_G$. Hence, we also have $(\max(\path^1(0),\path^2(0))-\RR_+^n)\subset\chifree_G$ by Property~\ref{property:min-max}, which implies that $(H(\alpha)(0)-\RR_+^n)\subset\chifree_G\subset\chifree$. \label{item:start-configuration}
\item $\path^1$ and $\path^2$ being feasible, we have $(\path^1(1)+\RR_+^n)\subset\chifree_G$ and $(\path^2(1)+\RR_+^n)\subset\chifree_G$. Hence, we also have $(\max(\path^1(1),\path^2(1))+\RR_+^n)\subset\chifree_G$ by Property~\ref{property:min-max}, which implies that $(H(\alpha)(1)+\RR_+^n)\subset\chifree_G \subset \chifree$. \label{item:end-configuration}
\item $H(\alpha)$ is non-decreasing as the result of the application of non-decreasing operators $\min$ and $\max$ on non-decreasing paths.
\label{item:non-decreasing}
\item Take $(i,j)\in E(G)$ and $\alpha,t\in[0,1]$. We have $\path^1(t+\alpha)\in\chifree_{i\succ j}$ as $\path^1\in\phifreeG$ and we have also $\max\{\path^1(t),\path^2(t)\}\in\chifree_{i\succ j}$ as $\path^1,\path^2\in\phifreeG$ and using Property~\ref{property:min-max}. Moreover, $(\path^1(1)+\RR_+^n)\subset\chifree_{i\succ j}$ and $(\path^2(1)+\RR_+^n)\subset\chifree_{i\succ j}$ imply that $\alpha \path^2(1)+(1-\alpha) \path^1(1) \in\chifree_{i\succ j}$ (using Property~\ref{property:geometric-invariance}). By Property~\ref{property:min-max}, applying the $\min$ operator on three configurations in $\chifree_{i\succ j}$ returns a configuration in  $\chifree_{i \succ j}$. In conclusion, we have $H(\alpha)(t)\in\chifree_G\subset\chifree$
\label{item:collision-free}
\end{enumerate}
As a result, $\path^1$ is homotopic to $\min\{ \path^2(1), \path^1(1), \max\{ \path^1, \path^2 \} \}$. As $\path^1$ and $\path^2$ have symmetric roles, $\path^2$ is homotopic to $\min\{ \path^1(1), \path^2(1), \max\{ \path^2, \path^1 \} \}$, that is $\min\{ \path^2(1), \path^1(1), \max\{ \path^1, \path^2 \} \}$. Homotopy defining an equivalence relation, $\path^1$ and $\path^2$ are homotopic.
\end{proof}
\begin{proof}[Proof of bijective correspondence]
For each $h\in\homotopy$, take an arbitrary $\path^h\in h$. As the priority graph is invariant in homotopy classes, we have:
\begin{equation}
\Gamma(\phifree)=\left\{ \Gamma(\path^h): h\in\homotopy \right\}
\end{equation}
As the priority graph $\Gamma(\path^h)$ is distinct for each class $h\in\homotopy$, the application $\Psi:h\in\homotopy \mapsto \Gamma(\path^h)$ is a bijection from $\homotopy$ to $\Gamma(\phifree)$. In conclusion, $\homotopy$ is in bijection with $\Gamma(\phifree)$.
\end{proof}

We have proved that all feasible paths sharing the same priorities are continuously deformable into each other. A direct consequence of the above theorem is that there exists a finite number of homotopy classes of feasible paths. When assigning the priority between each pair of robots, there is indeed two possibilities. As there is at most $n(n-1)/2$ collision pairs $i,j$ satisfying $\chiobs_{ij}\neq\emptyset$, there is at most $2^{n(n-1)/2}$ priority graphs. There is thus a finite number of homotopy classes -- at most $2^{n(n-1)/2}$ -- and each homotopy class of feasible paths is uniquely encoded by a priority graph $G\in\graphs$. A natural question is: does any priority graph $G\in\graphs$ encode a (non-empty) homotopy class of feasible paths ? This mathematical question is equivalent to: given assigned priorities, is it possible for robots to go through the intersection, respecting the assigned priorities ?

\section{Feasible priority graphs}
\label{sec:feasibility}

Here, we propose to give a characterization of the set of feasible priority graphs, that we define as graphs $G\in\graphs$ such that there exists a feasible path whose induced priority graph is $G$:
\begin{definition}[Feasible priority graph]
A priority graph $G\in\graphs$ is feasible if and only if $\phifreeG\neq\emptyset$.
\end{definition}
Using the application $\Gamma$, the set of feasible priority graphs can be denoted as $\Gamma(\phifree)$. We start with some examples that show that the existence of a feasible path respecting given priorities is strongly related to the notion of deadlock, and we highlight the role of priority cycles in the formation of deadlocks.

\paragraph{Deadlock examples}

First of all, consider the example on the left drawing of Figure~\ref{fig-deadlock-configurations}. The assigned priorities are $1\succ 2$, $2 \succ 3$ and $3\succ 1$. Hence, the priority graph is cyclic. It is clear that respecting the assigned priorities leads to the deadlock configuration represented in Figure~\ref{fig-deadlock-configurations}. None of the robots can move without colliding. The right drawing of Figure~\ref{fig-deadlock-configurations} gives a similar example with a larger number of robots involved in the priority cycle.

\begin{figure}[p]
\begin{center}
\includegraphics[width=0.7\linewidth]{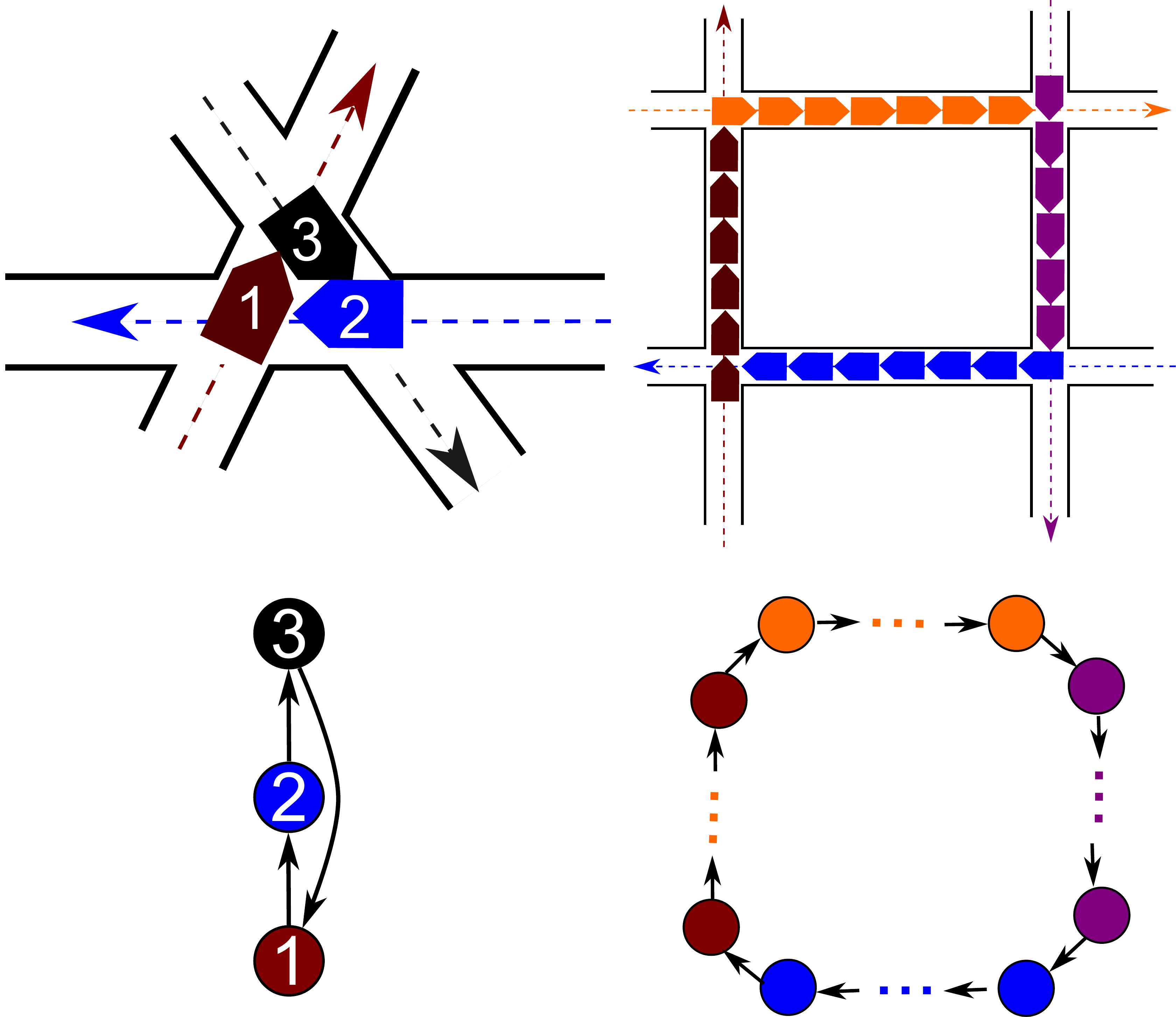}\hfill
\end{center}
\caption{Two examples of deadlock configurations. On the left side, 3 robots are involved in the deadlock. On the right side, the deadlock is caused by a priority cycle involving much more robots. In both examples, none of the robots can move without colliding. It is a deadlock configuration.}
\label{fig-deadlock-configurations}
\end{figure}
\begin{figure}[p]
\begin{center}
\includegraphics[width=0.7\linewidth]{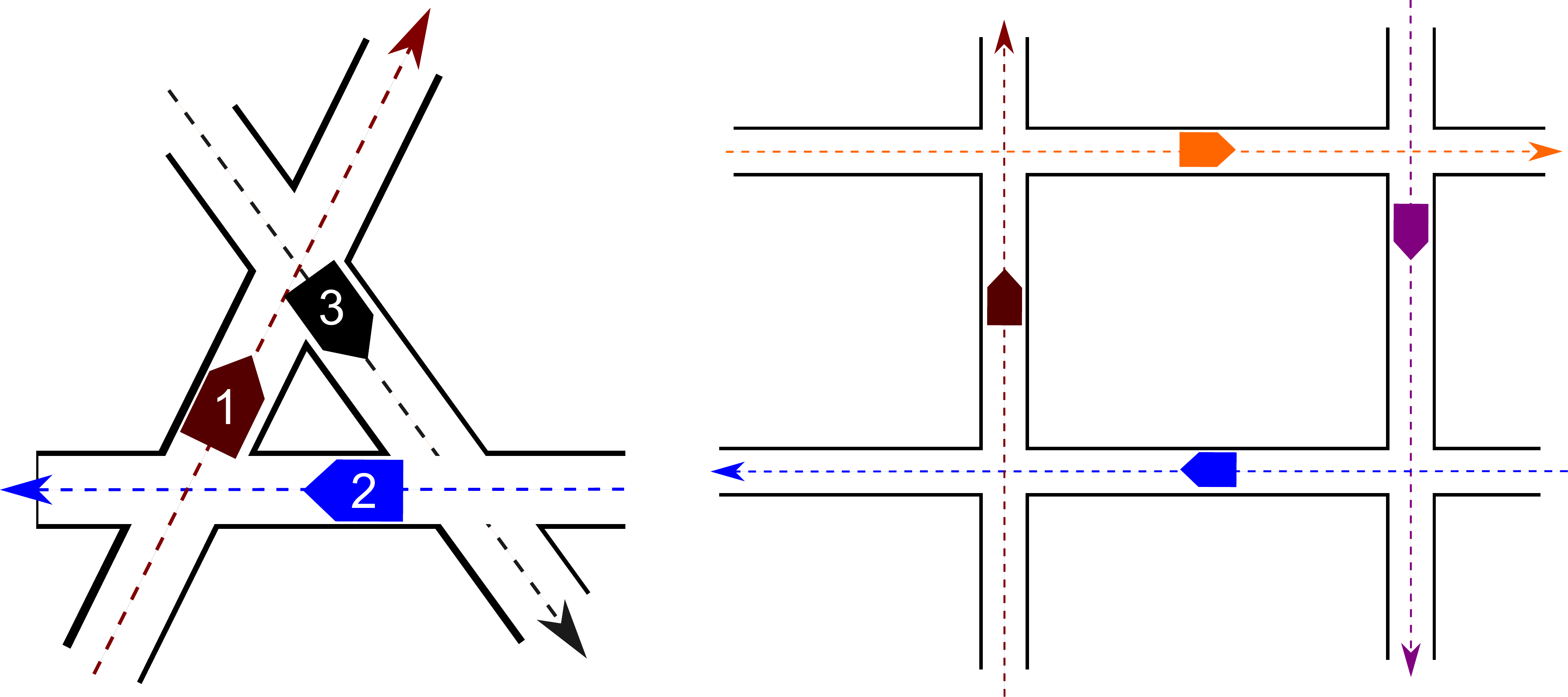}\hfill
\end{center}
\caption{Cyclic deadlock-free examples}
\label{fig-cyclic-deadlock-free-example}
\end{figure}

\paragraph{Cyclic deadlock-free examples}

According to the above example, it is clear that cycles in the priority graph have a strong role in the formation of deadlocks. Now, consider the example in the left drawing of Figure~\ref{fig-cyclic-deadlock-free-example}. Again, the assigned priorities, $1\succ 3$, $2 \succ 1$ and $3\succ 1$, are cyclic. However, it is clear that there exists a feasible path respecting the priorities and all robots will eventually go through the intersection. The right drawing of Figure~\ref{fig-cyclic-deadlock-free-example} provides a similar cyclic deadlock-free example involving four robots. 

The above examples justify the motivation to obtain a characterization of priority graphs such that there exists a feasible path respecting the given priorities. This characterization refines the role of cycles in the formation of deadlocks.

\paragraph{A singular deadlock-free priority graph}

Before providing such a characterization, we expose a last example where the priority graph $G$ is feasible in that there exists a feasible path whose priority graph is $G$; however, all feasible paths respecting these priorities are in contact with the boundary of $\chiobs_G$. Figure~\ref{fig:singular-example} depicts such an example. It is very likely that such priorities should not be considered as feasible in practice as they require a very precise control. Note also that it is a singularity caused by the (arbitrary) openness of the obstacle region. 

\begin{figure}[!htbp]
\begin{center}
\includegraphics[width=0.33\linewidth]{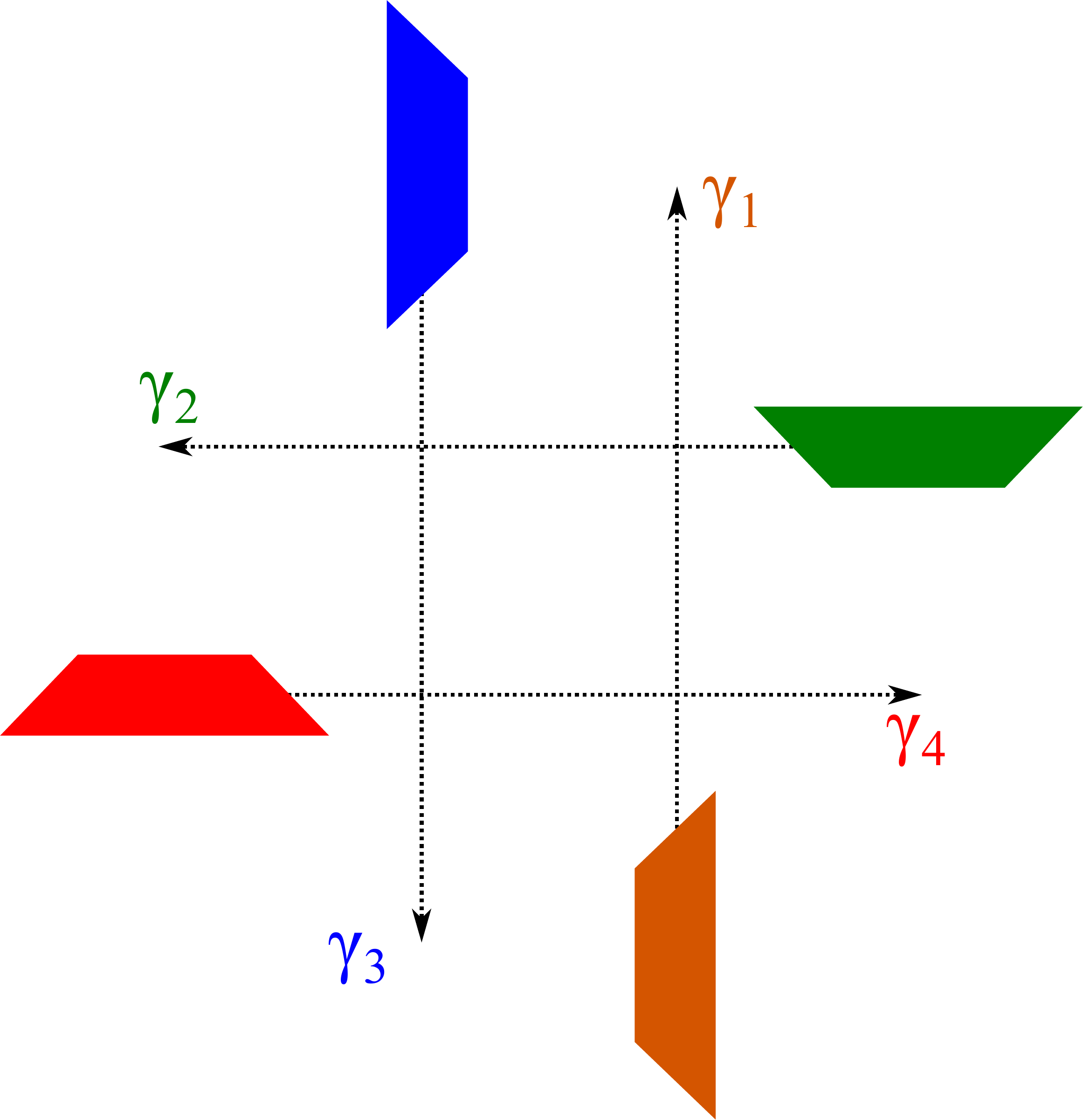}\hfill
\includegraphics[width=0.33\linewidth]{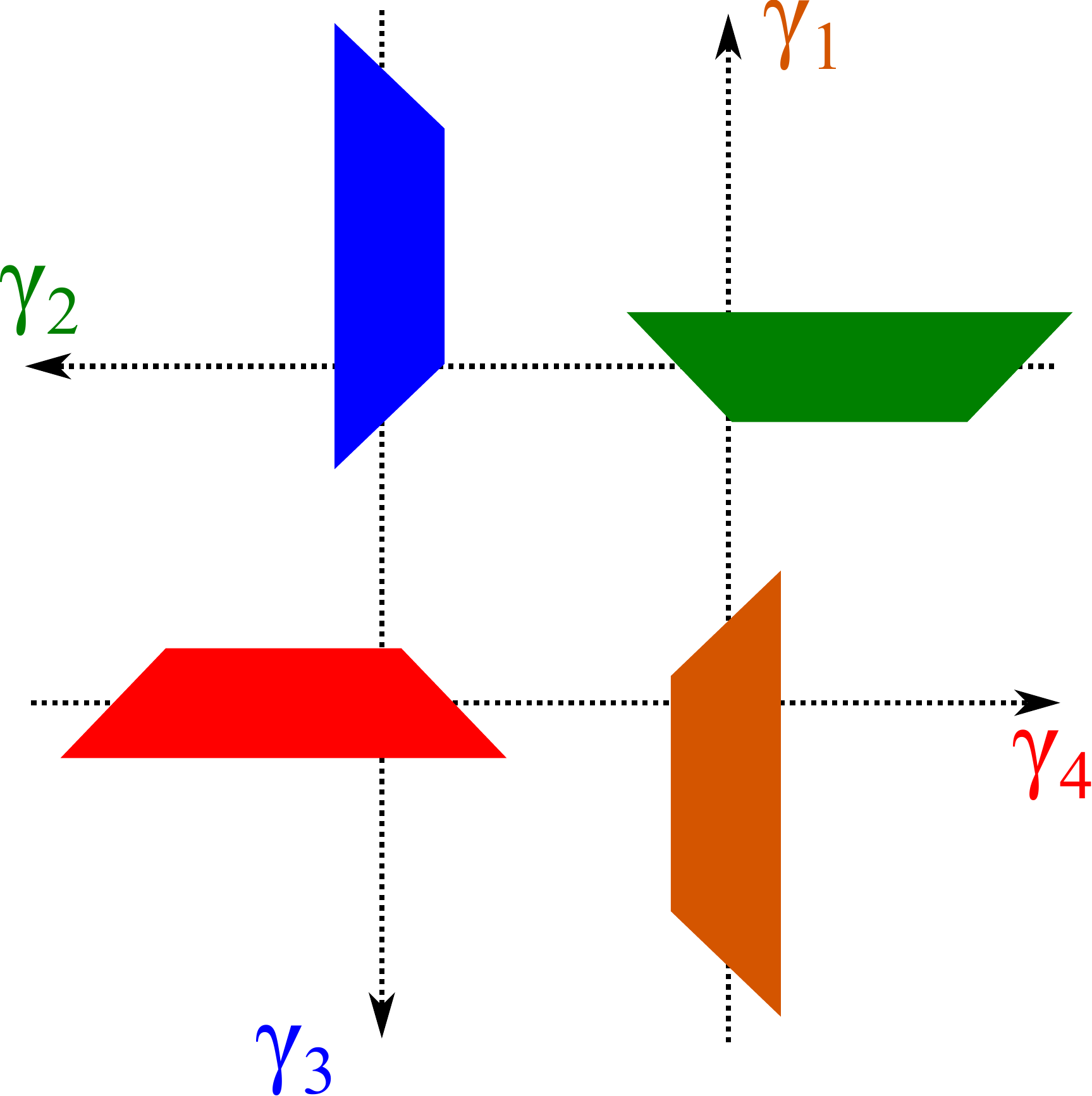}\hfill
\includegraphics[width=0.33\linewidth]{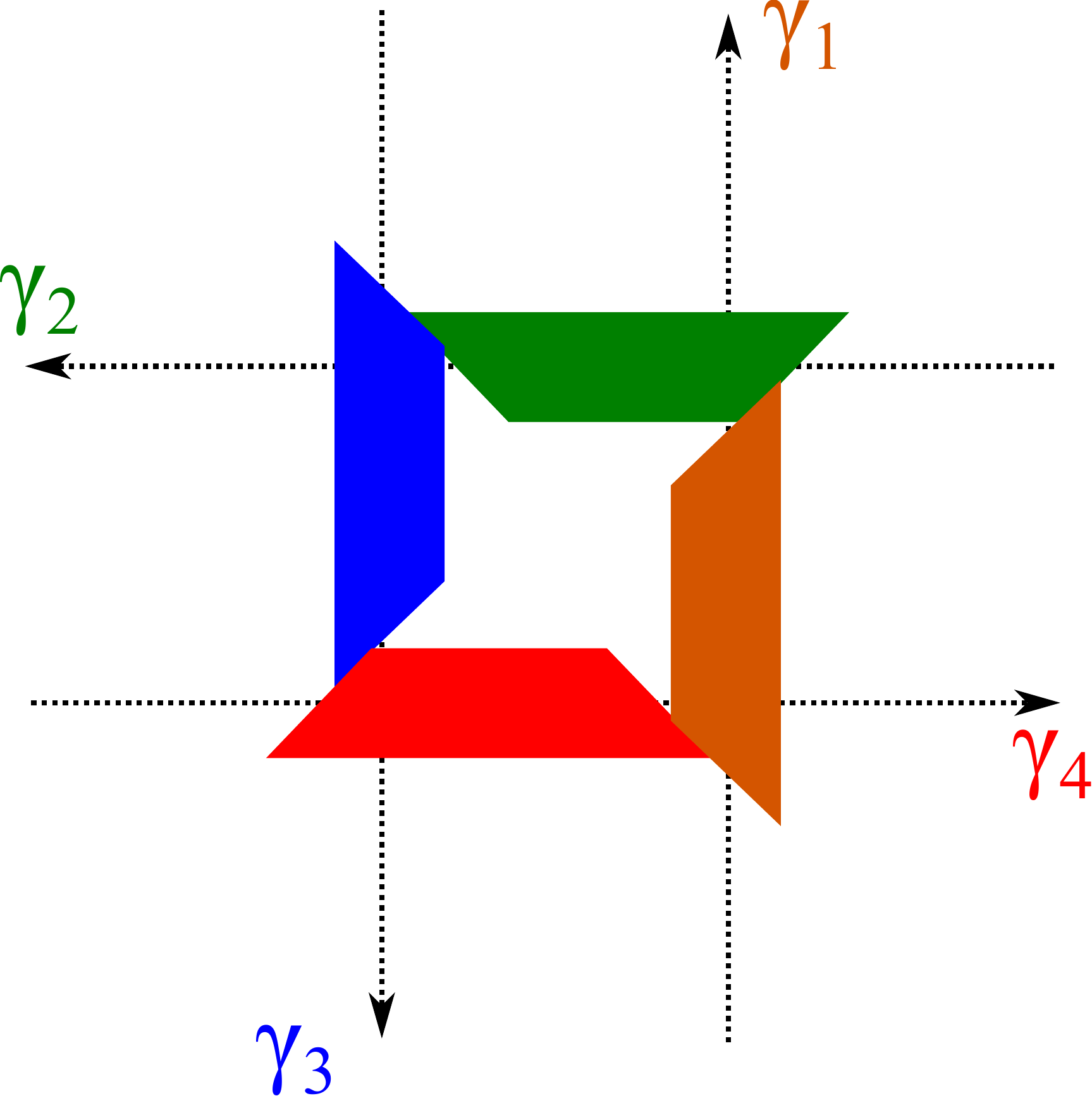}\hfill
\includegraphics[width=0.33\linewidth]{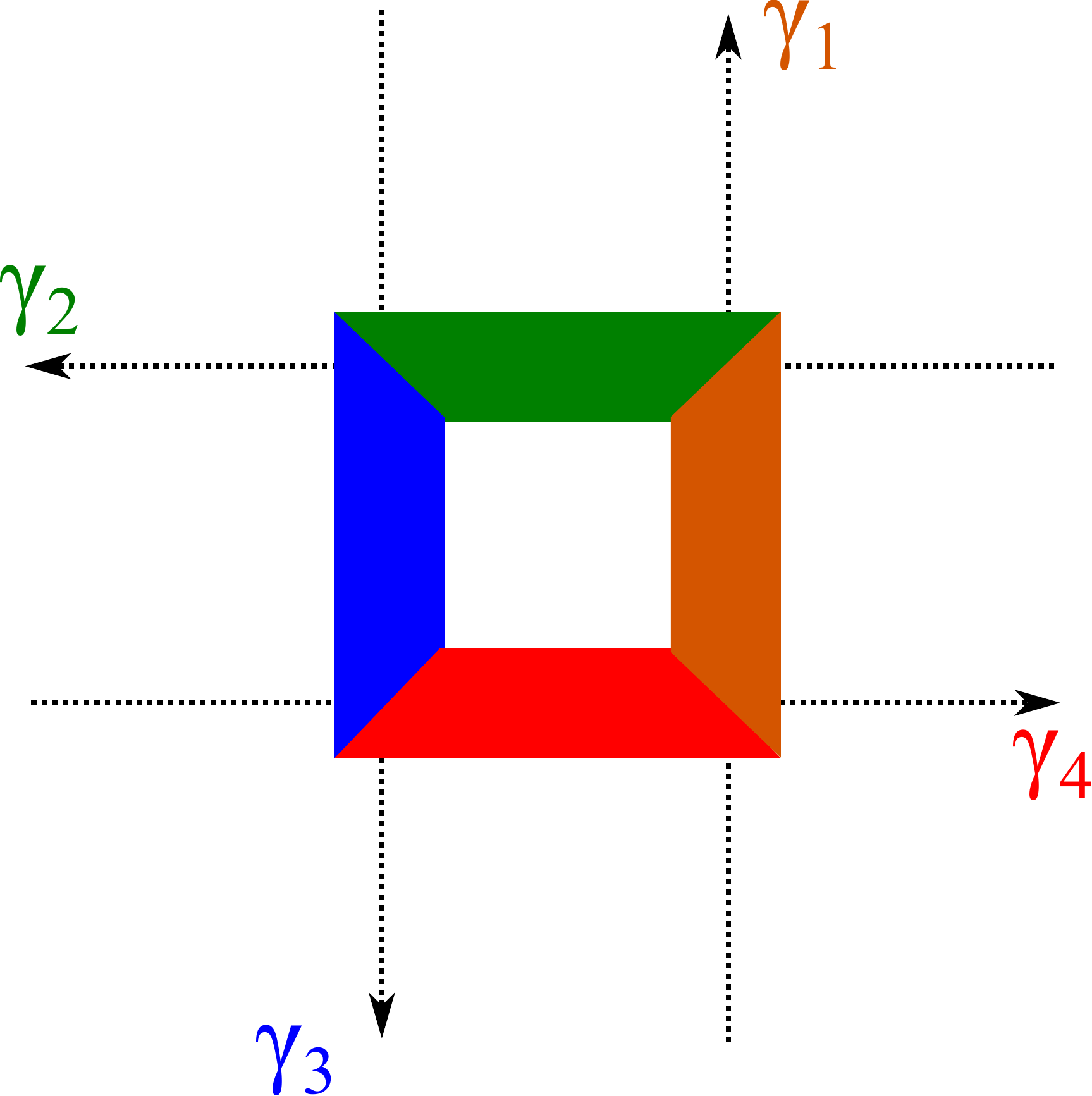}\hfill
\includegraphics[width=0.33\linewidth]{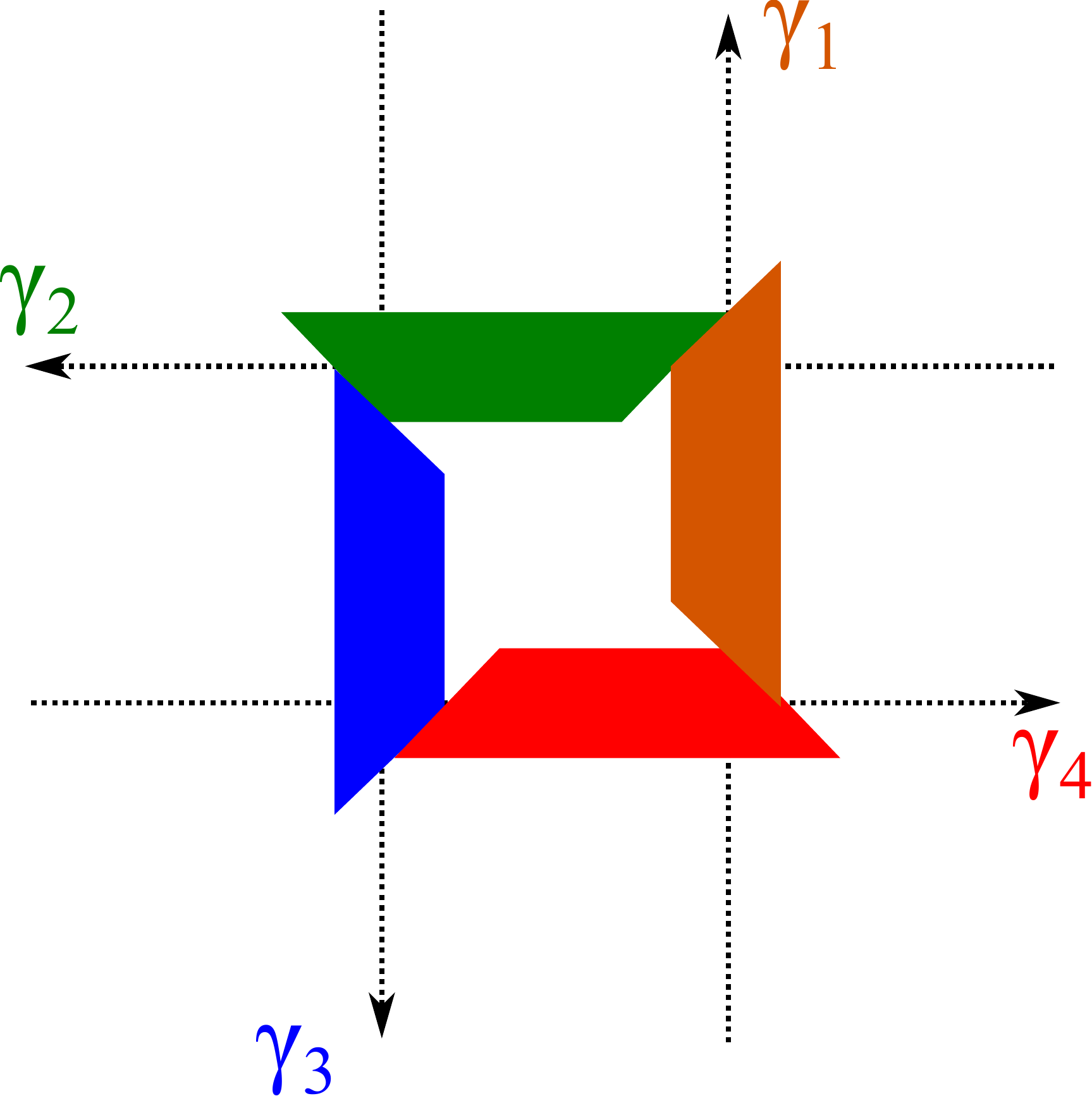}\hfill
\includegraphics[width=0.33\linewidth]{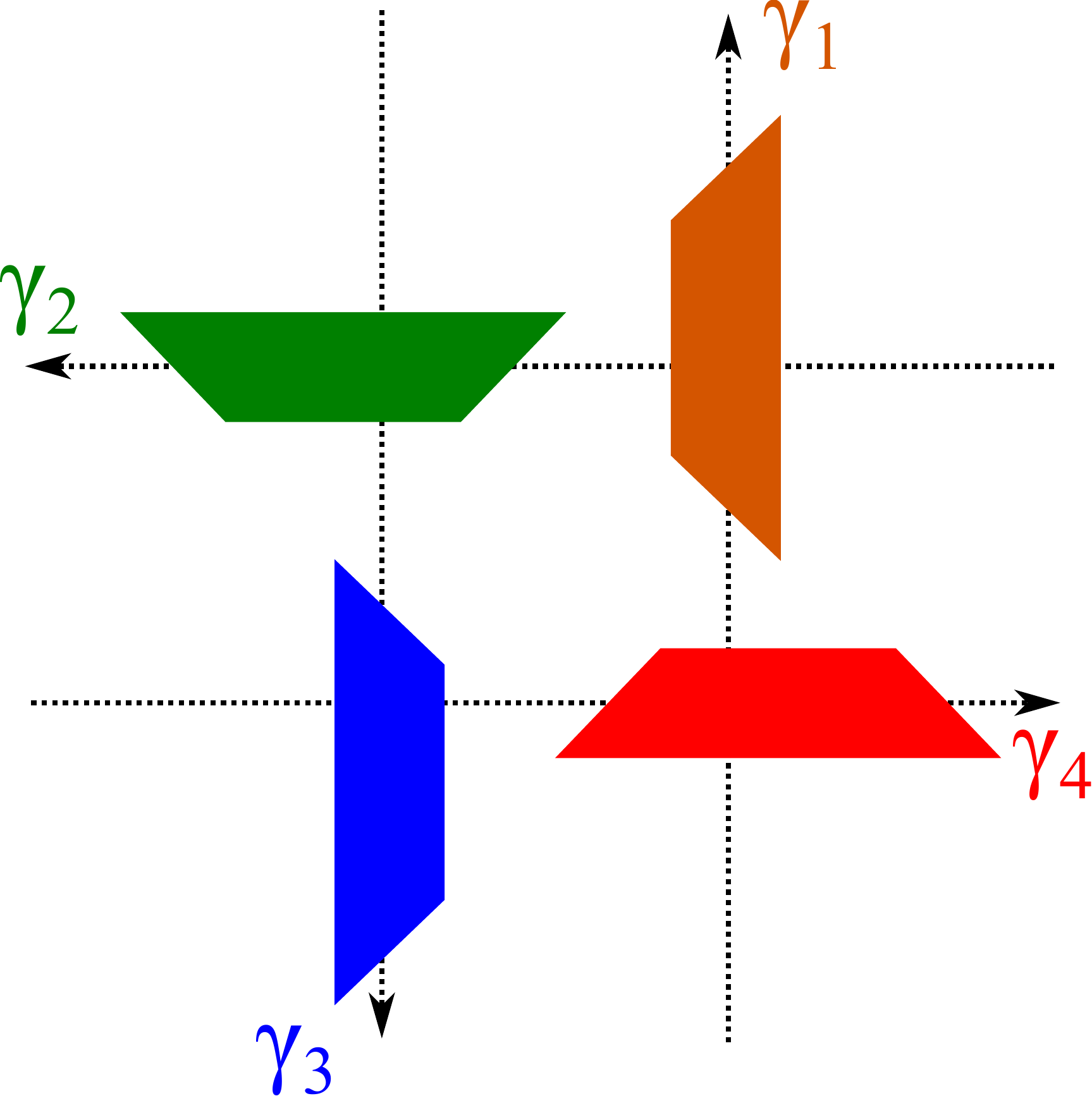}\hfill
\end{center}
\caption{In this (quite imaginative) scenario, robots maintaining constant velocity do not collide and slide on each other to go through the intersection. It is a very singular scenario. There exists a feasible path satisfying the cyclic priorities $2\succ 1$, $3 \succ 2$, $4 \succ 3$ and $1\succ 4$, but with absolutely no "safety margin" as robots need to slide on each other.}
\label{fig:singular-example}
\end{figure}

\subsection{Sufficient condition for priorities feasibility}

The above examples tend to indicate that no deadlock can occur under acyclic priorities. It is not surprising as deadlocks usually involve a "circular wait"~\cite{Coffman1971}. In many circumstances, imposing acylic priorities is not a problematic constraint and demonstrates some benefits including deadlock avoidance (see Part~\ref{part:priority-based-coordination}). That is why we start by providing a sufficient condition for priorities feasibility stating that acyclic priorities ensure deadlock avoidance. 

\begin{theorem}[Sufficient condition for priorities feasibility]
All acyclic priority graphs are feasible. 
\end{theorem}
The proof of the above theorem relies on the fact that under acyclic priorities, a simple feasible path respecting the acyclic priorities can be constructed by letting robots go through the intersection one by one. 
\begin{proof}
Take an acyclic priority graph $G\in\graphs$. To prove that $G$ is feasible, we are going to exhibit a particular feasible path whose induced priority graph is $G$. As $G$ is acylic, it admits a topological ordering of its nodes $\robots$. Consider a relabeling of robots along this topological ordering, i.e., robot $1$ is the maximal element of this topological ordering, ... robot $i$ is the $i$th element of the topological ordering, ... and robot $n$ is the minimal element of the topological ordering. Consider the path $\path$ constructed as follows. $\path(0):=\xobsmin$ and for all $i\in\{1\cdots n\}$, within time interval $[(i-1)/n, i/n]$, robot $i$ moves forward from $\xobsmin_i$ to $\xobsmax_i$ (for example $\path_i$ is linear in that time interval and takes values $[\xobsmin_i,\xobsmax_i]$) while other robots $j\neq i$ do not move ($\path_j$ constant in that time interval). This path is feasible and takes values in $\chifree_G$.
\end{proof}

\paragraph{Note to the reader} The two following subsections intend to treat the case of cyclic yet feasible priorities. They are quite technical and the reader having no particular interest in this case can go directly to Subsection~\ref{subsec:absence-deadlocks}. It will not affect the understanding of the rest of the thesis.

\subsection{Safety margin}

As shown in the examples presented previously, only considering feasibility as a binary question is quite insufficient in practice as some priority graphs are feasible but require robots to slide on each other, i.e., to travel through very risky configurations. That is why we propose a notion of feasibility endowed with the notion of safety margin. 

We say that the priority graph is feasible with a (safety) margin $r\geq 0$ if there exists a feasible path $\path\in\phifreeG$ keeping a distance $r$ from the obstacle region $\chiobs_G$ (in infinity norm). Given a path $\path\in\paths(\chi)$, $d(\path,\chiobs_G)$ is defined as follows:
\begin{eqnarray}
d(\path,\chiobs_G)&:=&\sup\left\{ r\geq 0: \forall t\in[0,1], \forall x\in\chiobs_G, \Vert  \path(t)-x\Vert_\infty \geq r  \right\} 
\end{eqnarray}

When $d(\path,\chiobs_G)\geq 0$, we say that $\path$ is safe with regards to $\chiobs_G$ with a margin $d(\path,\chiobs_G)$. The use of the distance of infinity norm makes sense since it means that a path is safe with regards to $\chiobs_G$ with a margin $r\geq 0$ if robots traveling along this path with an individual precision of $r$ will not collide (with regards to $\chiobs_G$). We have indeed the following equivalence:
\begin{equation}
\left[\forall x\in\chiobs_G,~\Vert\path(t)-x\Vert_\infty \geq r \right] \Leftrightarrow \left[\left(\path(t)+[-r,r]^n\right) \subset \chifree_G\right]
\end{equation}

It is direct that the set of paths $\path\in\paths(\chifree_G)$ satisfying $d(\path,\chiobs_G)\geq r$ is precisely $\paths(\chifree_G \ominus [-r,r]^n)$ where $\ominus$ denotes the erosion operator, i.e.,
\begin{equation}
\chifree_G \ominus [-r,r]^n := \left\{ x\in\chifree_G : x+[-r,r]^n \subset \chifree_G \right\}
\end{equation}
$\chifree_G\ominus[-r,r]^n$ is the erosion of $\chifree_G$ with the structuring element $[-r,r]^n$~\cite{Serra1983} (see Figure~\ref{fig:obstacle-region-dilatation-erosion}).
\begin{figure}[!htbp]
\begin{center}
\includegraphics[width=1.0\linewidth]{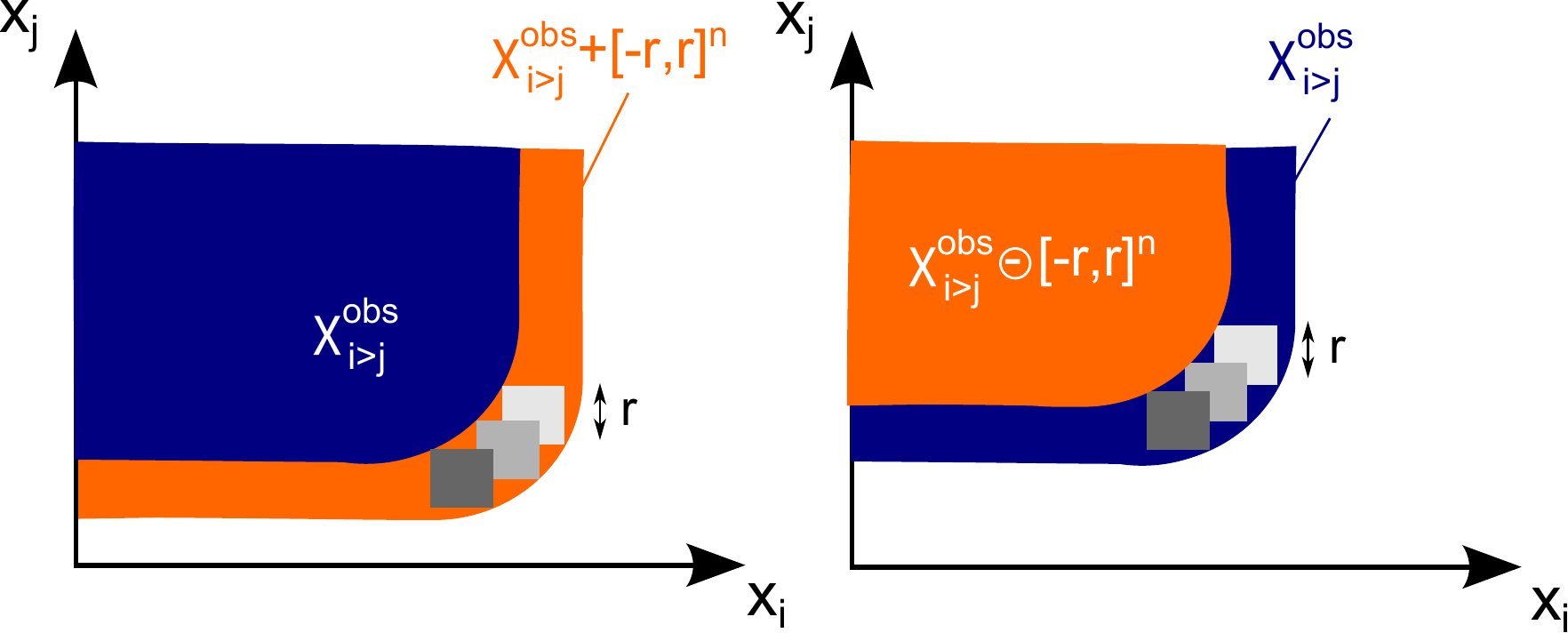}\hfill
\end{center}
\caption{The left drawing represents the dilatation of $\chiobs_{i\succ j}$ with the structuring element $[-r,r]^n$. The right drawing show the erosion of $\chiobs_{i\succ j}$ with the structuring element $[-r,r]^n$.}
\label{fig:obstacle-region-dilatation-erosion}
\end{figure}
The form of the structuring element is due to the use of the infinity norm (it is the closed ball of radius $r$ with regards to the infinity norm). $\paths(\chifree_G\ominus[-r,r]^n)$ denotes the set of feasible paths whose priority graph is $G$ and keeping a distance $r$ from $\chiobs_G$. It is natural to define a safety  margin associated to the priority graph $G$ as follows:
\begin{equation}
\rho_G:=\begin{cases}
~~\max\left\{r\geq 0:\paths(\chifree_G\ominus[-r,r]^n)\neq\emptyset\right\}&\text{ if } \paths(\chifree_G)\neq\emptyset\\
-\min\left\{r> 0:\paths(\chifree_G+[-r,r]^n)\neq\emptyset\right\} & \text{ else.}
\end{cases}
\end{equation}
\begin{itemize}
\item When $\paths(\chifree_G)\neq\emptyset$, $\rho_G\in\RR_+\cup\{+\infty\}$ denotes the maximal distance between $\chiobs_G$ and feasible paths whose priority graph is $G$.
\item When $\paths(\chifree_G)=\emptyset$, $\rho_G<0$. The value taken by $\rho_G<0$ can be interpreted as an indicator of how far the priority graph $G$ is from being feasible.
\end{itemize}
The use of the maximal (resp. minimal) element is justified by Lemma~\ref{lemma:existence-paths-with-maximal-margin} proved in Appendix~\ref{app:existence-paths-with-maximal-margin} that stipulates that the upper (resp. lower) bound is attained, i.e., there exists a path with maximal margin. We refer to $\rho_G$ as the safety margin of $G$. This definition is coherent as a feasible priority graph has necessarily a non-negative safety margin since $\paths(\chifree_G) \neq \emptyset$.

\subsection{Characterization of feasible priorities}

Before providing a formal characterization of feasible priority graphs, a geometric interpretation is provided about why in certain circumstance cyclic priorities are not feasible.
\begin{figure}[!htbp]
\begin{center}
\includegraphics[width=0.5\linewidth]{collision-region-3D}\hfill
\includegraphics[width=0.5\linewidth]{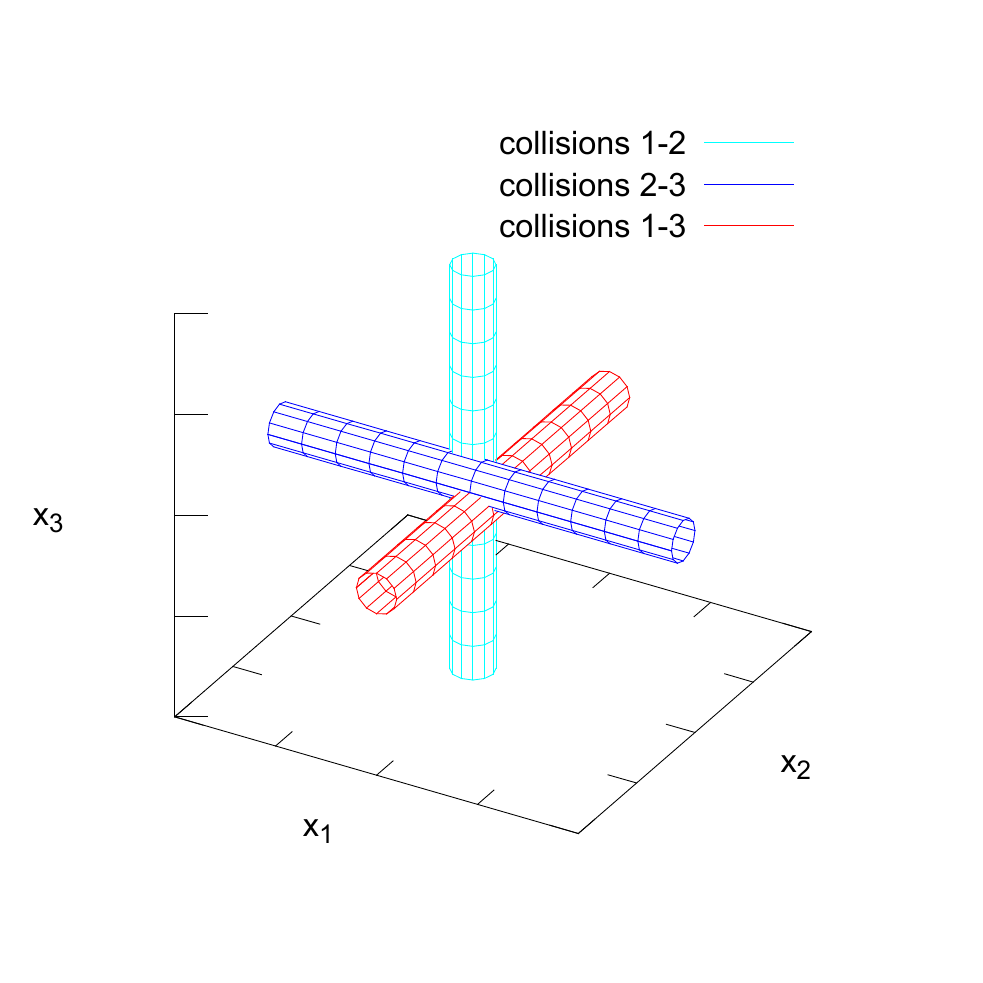}\hfill
\end{center}
\caption{The left drawing depicts the obstacle cylinders in the coordination space for the three-robot system of Figure~\ref{fig-cyclic-deadlock-free-example} in a deadlock-free configuration. The right drawing depicts the obstacle cylinders in the coordination space for the three-robot system of Figure~\ref{fig-deadlock-configurations} stuck in a deadlock. }
\label{fig:interpretation-priorities-feasibility}
\end{figure}
In Figure~\ref{fig:interpretation-priorities-feasibility}, the obstacle cylinders in the coordination space are depicted for both the cyclic deadlock-free example of Figure~\ref{fig-cyclic-deadlock-free-example} and the deadlock example of Figure~\ref{fig-deadlock-configurations}. The main difference is that in the deadlock case, cylinders intersect with each other. In contrast, in the deadlock-free case, cylinders do not intersect each other, there is a certain distance between each cylinder. Thanks to this distance between cylinders, the multi robot system can decide, independently for each collision cylinder, on which side to travel. On the contrary, if there is not a sufficient distance between cylinders, theses decisions are not independent. It is very clear on the right drawing of Figure~\ref{fig:interpretation-priorities-feasibility} that if a feasible path lies above the obstacle cylinder $\chiobs_{23}$ (the blue one) and below $\chiobs_{13}$ (the red one), then it also must lie on the right relative to the obstacle cylinder $\chiobs_{12}$ (the cyan one). In terms of priorities, it means that if $3\succ 2$ and $1 \succ 3$, then we must have $1 \succ 2$, i.e., the cycle $2 \succ 1 \succ 3 \succ 2$ is forbidden.

The definition of the completed obstacle regions $\chiobs_{i\succ j}$ enables to provide a very synthetic characterization of feasible priority graphs. It confirms and refines the role of priority cycles in the formation of deadlocks. In particular, Condition~\eqref{eq:condition-thm-feasibility} gives a necessary and sufficient condition for a priority cycle to be feasible. We let $\cycles(G)$ denote the elementary cycles of a priority graph $G\in\graphs$.

\begin{theorem}[Characterization of feasible priority graphs]
\label{thm:feasible}
A priority graph $G\in\graphs$ is feasible if and only if for all elementary cycles $\C$ in $G$, we have:
\begin{equation}
\bigcap_{(i,j)\in E(\C)}\chiobs_{i\succ j}=\emptyset
\label{eq:condition-thm-feasibility}
\end{equation}
If Condition~\eqref{eq:condition-thm-feasibility} holds, the safety margin is given by:
\begin{equation}
\rho_G=\max\left\{ r\geq 0: \forall \C\in\cycles(G), \bigcap_{(i,j)\in E(\C)} \left( \chiobs_{i\succ j} + [-r,r]^n\right)=\emptyset \right\}
\end{equation}
\end{theorem}

A complete proof of the above theorem is provided in Appendix~\ref{app:thm-feasible}. In the following, we prove that~\eqref{eq:condition-thm-feasibility} is a necessary condition for priority graph feasibility, and we also provide a slightly stronger sufficient condition for priority graph feasibility.
\renewcommand{\Cfree}{\chifree}
\renewcommand{\Cobs}{\chiobs}
\begin{proof}[Proof of the necessary condition]
We will prove the necessary condition by contraposition. Take $G\in\graphs$ and assume that there is an elementary cycle $\C$ of $G$ such that the subset $\bigcap_{(i,j)\in E(\C)} \Cobs_{i\succ j}$ is not empty, and let $x^1$ be an element of this set. We have to prove that $\paths(\Cfree_G)=\emptyset$. To this end, we are going to build a hyper-orthant in the coordination space capturing any path of $\paths(\Cfree_G)$.  For each $j\in V(\C)$, consider $K_j := \{x\in\chi: x_j=x_j^1\text{ and } x_i \leq x_i^1\text{ for } i \neq j\}$. By Property~\ref{property:geometric-invariance} we have:
\begin{equation}
x^1\in \Cobs_{i\succ j}
\Longrightarrow
K_j \subset \Cobs_{i\succ j}
\end{equation}
As $\forall (i,j)\in E(\C), x^1\in\Cobs_{i\succ j}$, applying the latter result yields: 
\begin{equation}
\forall (i,j)\in E(\C), K_j \subset \Cobs_{i\succ j}
\end{equation}
As a consequence, $\bigcup_{(i,j)\in E(\C)} K_j \subset \Cobs_G$, and as $\C$ is a cycle, every vertex $j\in V(\C)$ is involved in some edge $(i,j)\in E(\C)$, so that we have: 
\begin{equation}
\bigcup_{j \in V(\C)} K_j \subset \Cobs_G \label{eq:faces-parallelepiped-included-in-chiobs-corpus}
\end{equation}
In the coordination space restricted to the coordinates which appear in $\C$, $\bigcup_{j \in V(\C)} K_j$ is the set of upper faces (that is the boundary) of the orthant (depicted in Figure~\ref{fig-closed-hyper-orthant}):
\begin{equation}
\orthant:=\{x\in\chi: \forall j\in V(\C), x_j\leq x_j^1\}
\end{equation}

And we have by Equation~\eqref{eq:faces-parallelepiped-included-in-chiobs-corpus}:
\begin{equation}
\partial \orthant = \bigcup_{j \in V(\C)} K_j \subset \Cobs_G
\end{equation}

Now, we will prove that $\paths(\Cfree_G)$ is empty by contradiction. Assume it is not empty and take an element $\path$ of it. We must have $\left(\path(0)-\RR_+^n\right) \subset \Cfree_G $, which implies that: 
\begin{equation}
\path(0) \in \orthant
\end{equation}
and $\left(\path(1)+\RR_+^n\right) \subset \Cfree_G$, which implies:
\begin{equation}
\path(1) \notin \orthant
\end{equation}
Since $\path$ is continuous, by Lemma~\ref{lemma:two-subsets-frontier} (see Appendix~\ref{app:topology-properties}), there exists $t\in[0,1]$ such that $\path(t)\in\partial \orthant= \bigcup_{j \in V(\C)} K_j \subset\Cobs_G$ and $\path$ intersects $\Cobs_G$. It is is contradiction with $\path\in\paths(\Cfree_G)$.

\begin{minipage}{\linewidth}
\begin{center}
\includegraphics[width=0.7\linewidth]{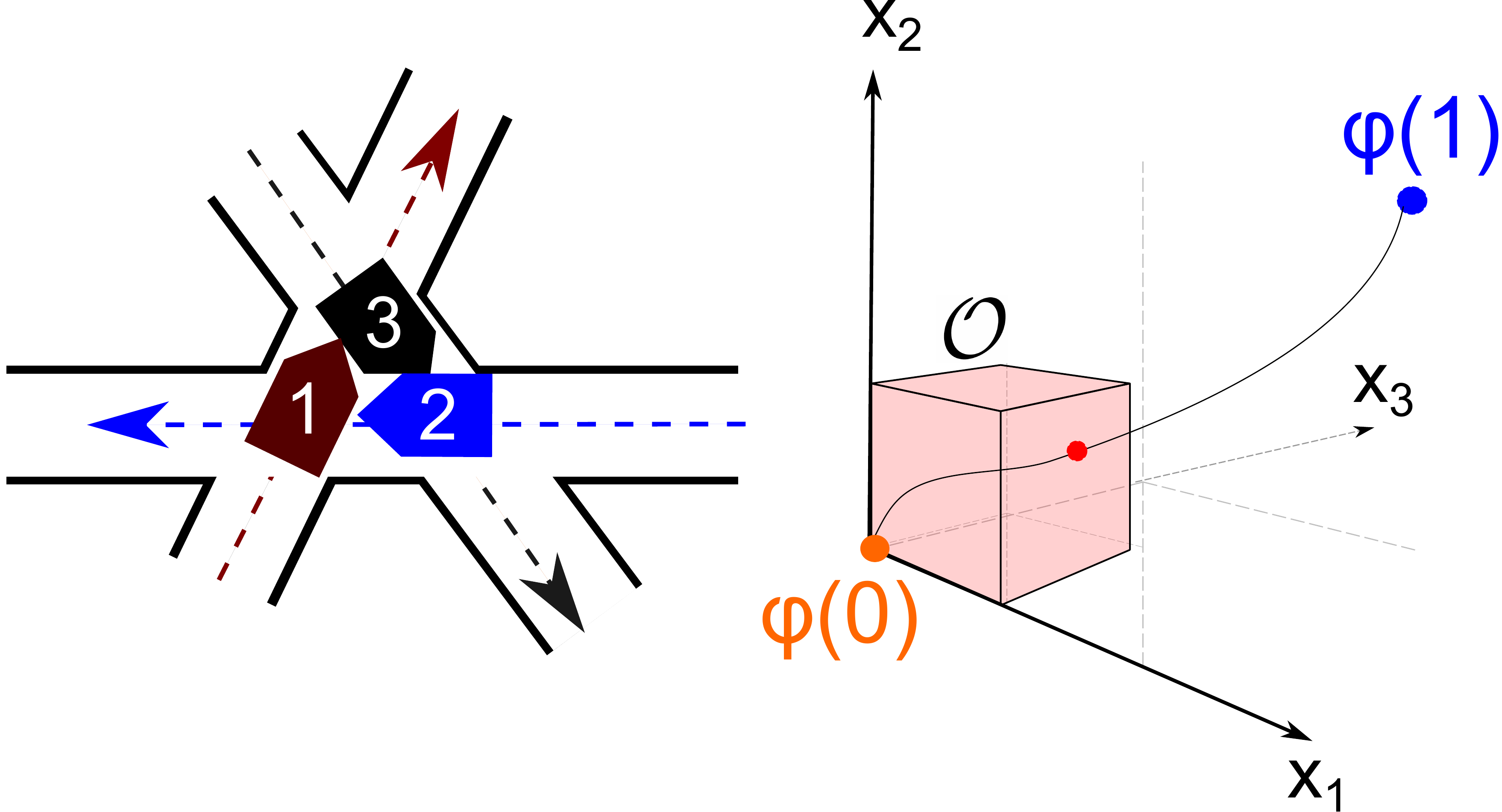}\hfill
\end{center}
\captionof{figure}{The hyper-orthant used in the proof of the necessary condition of Theorem~\ref{thm:feasible} (in a three-dimensional scenario). }
\label{fig-closed-hyper-orthant}
\end{minipage}
\end{proof}

In order to provide a constructive proof of the existence of feasible paths taking values in $\Cfree_G$ under certain conditions, we first introduce the concept of local priority graph. Given a radius $r \geq 0$ and a configuration $x\in\chi$, the local priority graph at configuration $x$ with radius $r \geq 0$ is the sub-graph $G_{|x,r}$ of $G$ with the same vertices and whose edge set is defined below:
\begin{equation}
E(G_{|x,r}):=\left\{(i,j)\in E(G): x \in \left(\Cobs_{i \succ j} + [-r,r]^n \right) \right\}
\end{equation}
\begin{figure}[p]
\begin{center}
\includegraphics[width=0.7\linewidth]{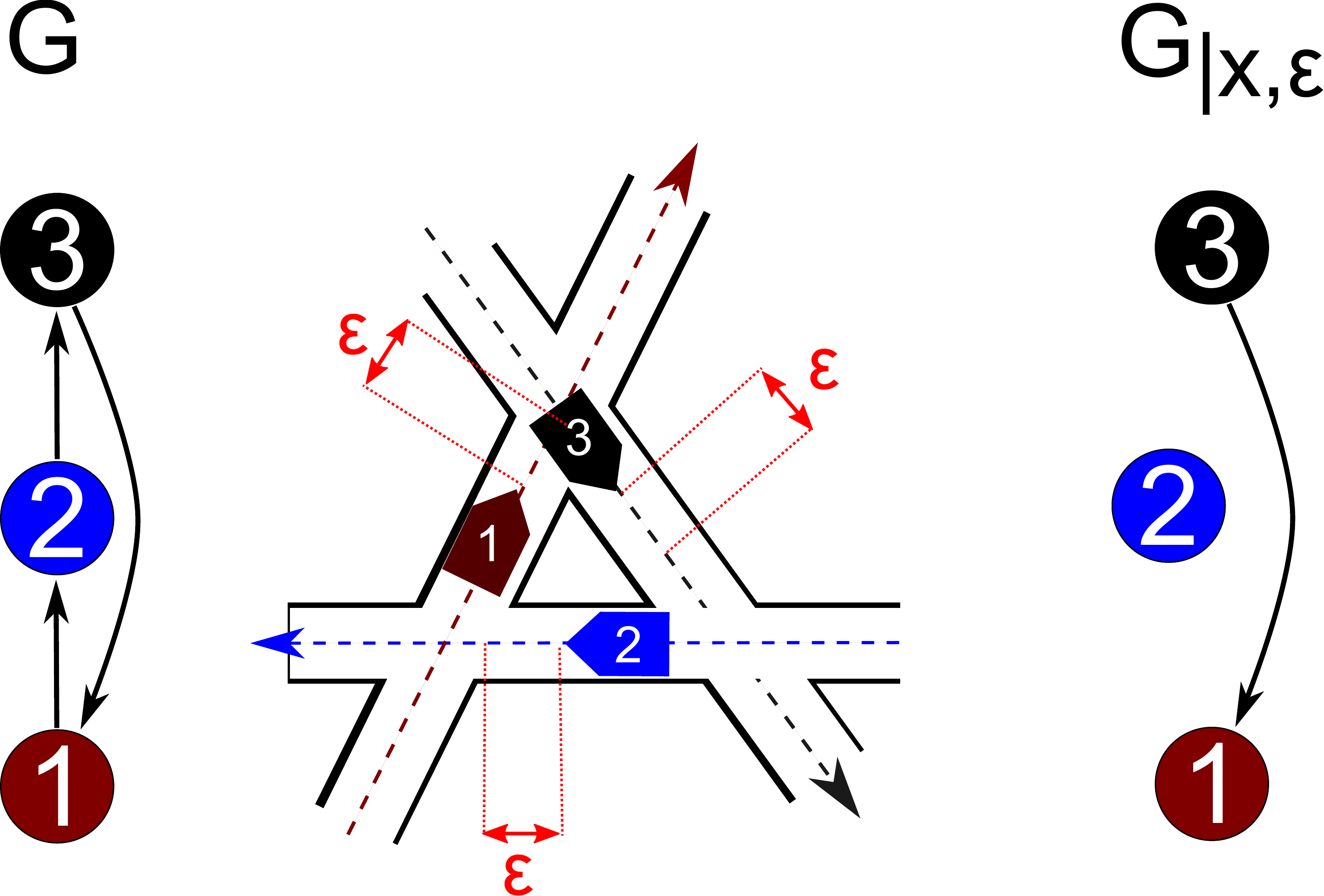}\hfill
\end{center}
\caption{Computation of the local priority graph for a three-robot system. Note that due to the geometry of paths (in particular their relative position), for small enough radius $\epsilon>0$, the local priority graph is acycle at all configurations.}
\label{fig:local-priority-graph}
\end{figure}
\begin{figure}[p]
\begin{center}
\includegraphics[width=0.7\linewidth]{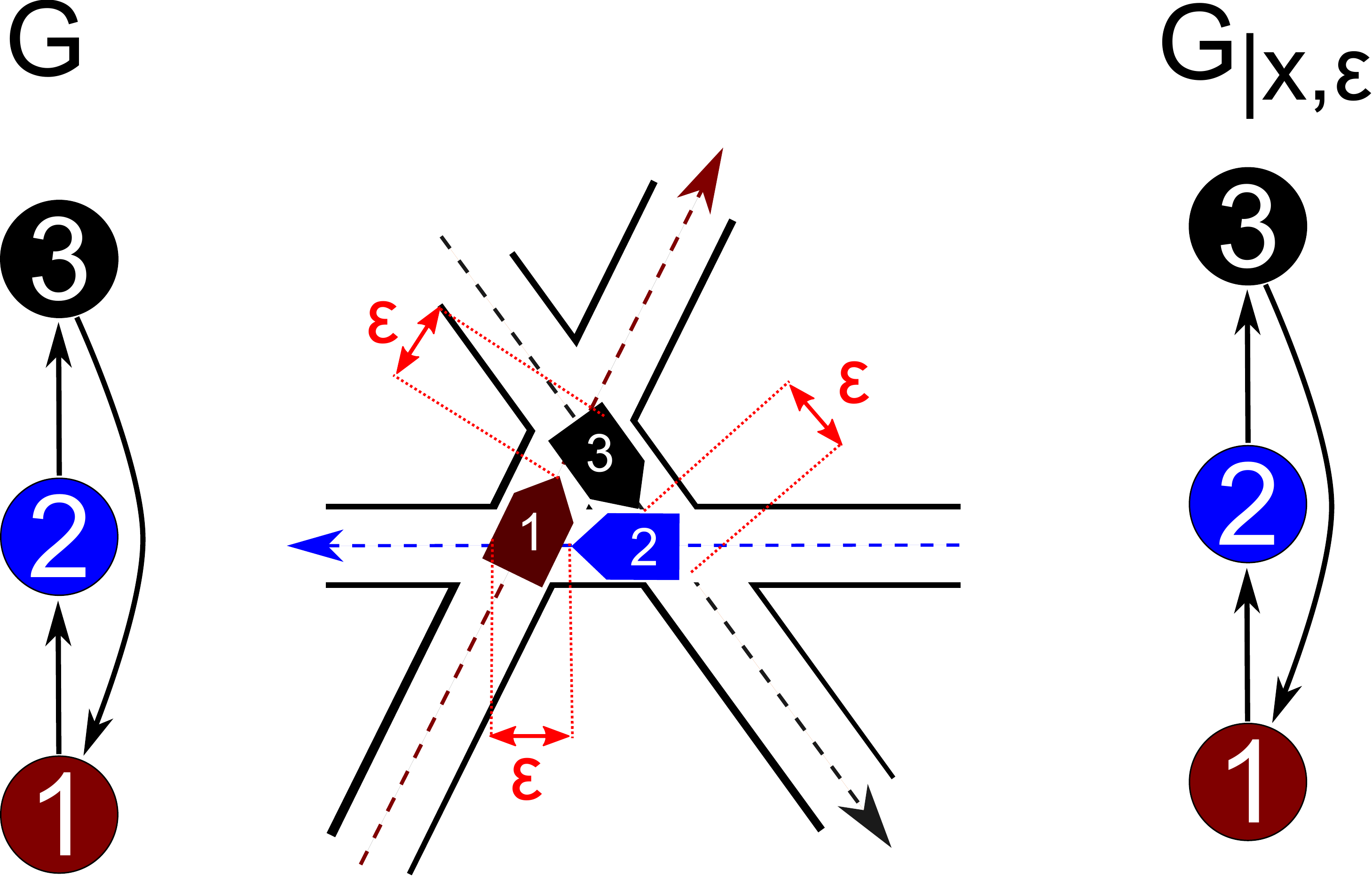}\hfill
\end{center}
\caption{Computation of the local priority graph for a three-robot system in a deadlock configuration. Note that the local priority graph is cyclic at the deadlock configuration, even for very small radius.}
\label{fig:local-priority-graph-deadlock}
\end{figure}
As depicted in Figure~\ref{fig:local-priority-graph}, computing the local priority graph at a given configuration $x$ with a given radius $r\geq 0$ consists in copying $G$ and keeping only edges $(i,j)\in E(G)$ such that $x$ belongs to the dilatation of $\Cobs_{i\succ j}$  by the structuring element $[-r,r]^n$, i.e., we keep only edges $(i,j)$ such that the distance (with the infinity norm) from $x$ to $\chiobs_{i\succ j}$ is strictly lower than $r$.

It is interesting to notice that in the deadlock-free example of Figure~\ref{fig:local-priority-graph}, the depicted local priority graph is acyclic. By contrast, at the deadlock configuration of Figure~\ref{fig:local-priority-graph-deadlock}, even for very small radius, the local priority graph is cyclic. 

\begin{lemma}[Sufficient condition for locally acyclic priority graph] Consider a priority graph $G\in\graphs$ satisfying for all elementary cycles $\C$ in $G$:
\begin{equation}
\bigcap_{(i,j)\in E(\C)}\left(\Cobs_{i\succ j}+[-\epsilon,\epsilon]^n\right)=\emptyset
\label{eq:condition-acyclic-local-priority-graph-corpus}
\end{equation}
for some $\epsilon>0$, then $G_{|x,\epsilon}$ is acyclic at all configurations $x\in\chi$.
\label{lemma:sufficient-condition-locally-acyclic-priority-graph-corpus}
\end{lemma}
\begin{proof}
Take $G\in\graphs$ and assume Equation~\eqref{eq:condition-acyclic-local-priority-graph-corpus} is satisfied for all elementary cycles $\C$ in $G$. By construction, we have:
\begin{equation}
E(G_{|x,\epsilon})=\left\{(i,j)\in E(G): x \in \left(\Cobs_{i \succ j}+[-\epsilon,\epsilon]^n\right)  \right\}
\end{equation}
The existence of a cycle $\C$ in $G_{|x,\epsilon}$ would imply that $x\in\cap_{(i,j)\in E(\C)}\left(\Cobs_{i\succ j}+[-\epsilon,\epsilon]^n\right)$, and would therefore contradict Equation~\eqref{eq:condition-acyclic-local-priority-graph-corpus} for this cycle.
\end{proof}

It is of high interest to know that the local priority graph with radius $\epsilon>0$ is acyclic at all configurations $x\in\chi$. Indeed, when this condition is satisfied, whatever the current configuration $x\in\Cfree_G$ of the system, it is always possible to find a robot $i\in\robots$ which can move forward the distance $\epsilon>0$ without colliding, which enables to construct a feasible path in $\Cfree_G$ by iterations. Based on this idea, we propose now to give a slightly stronger sufficient condition for the existence of feasible paths satisfying a given priority graph $G\in\graphs$. We prove in the sequel that a sufficient condition for $\paths(\chiobs_G)\neq\emptyset$ is that for all elementary cycles $\C$ in $G$:
\begin{equation}
\bigcap_{(i,j)\in E(\C)}\left(\Cobs_{i\succ j}+[-\epsilon,\epsilon]^n\right)=\emptyset
\label{eq:condition-acyclic-local-priority-graph-corpus2}
\end{equation}
for some $\epsilon>0$. It is a slightly stronger condition than in Theorem~\ref{thm:feasible} as $\epsilon>0$ (instead of $\epsilon \equiv 0$).

\begin{proof}[Proof of the sufficient condition under a slightly stronger assumption]
Take $G\in\graphs$ and assume that we have $\epsilon>0$ such that for all elementary cycles $\C$ in $G$, Equation~\eqref{eq:condition-acyclic-local-priority-graph-corpus2} holds. We will provide a constructive proof of the existence of a path $\path\in\paths(\Cfree_G)$. By Lemma~\ref{lemma:sufficient-condition-locally-acyclic-priority-graph-corpus}, the local priority graph $G_{|x,\epsilon}$ is acyclic at all configurations $x\in\chi$. Let $x^\mathrm{goal}\in\chifree$ denote the desired final configuration defined componentwise as: $x^\mathrm{goal}_i=\xobsmax_i+\epsilon$ (satisfying $\left(x^\mathrm{goal}+\RR_+^n\right) \subset \Cfree_G$). We define the finite time flow $\phi(x,t)$ starting at initial condition $x$ component-wise as follows for $t\in[0,1]$ and $j\in\robots$:
\begin{equation}
\phi_j(t,x):=
\begin{cases}
x_j & \text{if } \exists (i,j)\in 
E(G_{|x,\epsilon})
\\
\min(x^\mathrm{goal}_j, x_j+ t \epsilon) & \text{else.}
\end{cases}
\end{equation}
First we prove that the flow starting from an initial configuration in $\Cfree_G$ remains in $\Cfree_G$. Consider $x\in \Cfree_G$ and $(i,j)\in E(G)$. By construction of $\phi$, we have for all $t\in[0,1]$:
\begin{equation}
\phi_i(t,x) \geq x_i
\label{eq:flow-collision-free-eq1-corpus}
\end{equation}
For $j$, consider the two following options:
\begin{itemize}
\item 
$(i,j)\in E(G_{|x,\epsilon})$. Then, we have for all $t\in[0,1]$:
\begin{equation}
\phi_j(t,x) = x_j
\label{eq:flow-collision-free-eq2-corpus}
\end{equation}
By Property~\ref{property:geometric-invariance}, since Equations~\eqref{eq:flow-collision-free-eq1-corpus} and~\eqref{eq:flow-collision-free-eq2-corpus} hold, $x\in\Cfree_{i \succ j}$ implies that $\phi(t,x)\in\Cfree_{i \succ j}$.
\item 
$(i,j)\notin E(G_{|x,\epsilon})$. Then, we have for all $t\in[0,1]$: 
\begin{equation}
\phi_j(t,x) = \min(x^\mathrm{goal}_j, x_j+ t \epsilon) \leq x_j+ \epsilon
\label{eq:flow-collision-free-eq3-corpus}
\end{equation}
Moreover, by construction of the local priority graph, 
$(i,j)\notin E(G_{|x,\epsilon})$
 is equivalent to:
\begin{equation}
x \notin \Cobs_{i\succ j}+[-\epsilon,\epsilon]^n
\end{equation}
which implies that:
\begin{equation}
x+\epsilon\e_j \in \Cfree_{i\succ j}
\label{eq:x-plus-epsilon-in-chi-free-corpus}
\end{equation}
By Property~\ref{property:geometric-invariance}, since Equations~\eqref{eq:flow-collision-free-eq1-corpus} and~\eqref{eq:flow-collision-free-eq3-corpus} hold, Equation~\eqref{eq:x-plus-epsilon-in-chi-free-corpus} implies that $\phi(t,x)\in \Cfree_{i\succ j}$.
\end{itemize}
In conclusion, for all $x\in\Cfree_G$ and $t\in[0,1]$, $\phi(t,x)\in\Cfree_G$. Now, consider the path $\path(t)$ defined iteratively as follows:
\begin{eqnarray}
\path(0)&:=&\xobsmin\\
\forall p\in\NN, \forall t\in[0,1], \path(p+t)&:=&\phi(t,\path(p))
\end{eqnarray}
$\path(0)\in \Cfree_G$ and by induction, $\path$ takes values in $\Cfree_G$. It is non-decreasing as $\phi_j(t,x)\geq x_j$, and we are going to prove that it reaches $x^\mathrm{goal}$ in finite time. The local priority graph at configuration $x$ only contains edges $(i,j)$ such that $x_j < x_j^\mathrm{goal}$ ($x_j=x_j^\mathrm{goal}=\xobsmax_j+\epsilon$ implies that $x\notin(\chiobs_{i\succ j}+[-\epsilon,\epsilon]^n)$). Since the local priority graph is acyclic, for all $x\in\chi$ with $x < x^\mathrm{goal}$, there exists a maximal element $j\in\robots$ satisfying $x_j<x^\mathrm{goal}_j$ and $\forall i\in\robots$, $(i,j)\notin E(G_{|x,\epsilon})$.
 By construction of $\phi$, it results that for all $p\in\NN$, if $\path(p)\neq x^\mathrm{goal}$, then there exists at least one robot $j$ such that $\path_j(p+1)=\min(x_j^\mathrm{goal}, x_j+\epsilon)$, i.e., robot $j$ travels a distance $\epsilon$ or reaches its goal configuration in time interval $[p,p+1]$. The distance to travel considering all robots is finite: $\sum_{i\in\robots}x_i^\mathrm{goal}-x_i^0$. As a result, $x^\mathrm{goal}$ is reached in finite time $T$, $\path(T)=x^\mathrm{goal}$ and $T$ satisfies:
\begin{equation}
T \leq \left\lceil \frac{\sum_{i\in\robots}x_i^\mathrm{goal}-\xobsmin_i}{\epsilon} \right\rceil 
\end{equation}
where $\lceil . \rceil$ denotes the ceiling function. Rescaling time by a factor $1/T$ yields a path $\tilde\path\in\paths(\Cfree_G)$.
\end{proof}

\subsection{Absence of deadlocks}
\label{subsec:absence-deadlocks}

To this point, we have proved a necessary and sufficient condition for $\paths(\Cfree_G)\neq\emptyset$, i.e., for the existence of feasible paths whose priority is graph $G$. In the following, we prove that provided $\paths(\Cfree_G)\neq\emptyset$, there is no deadlock configuration in $\Cfree_G$. This means that for all configurations $x\in\Cfree_G$, there exists a path $\path\in\paths(\Cfree_G)$ going through configuration $x$. It is a very valuable result as a direct consequence is that provided the assigned priorities are feasible, there will be no deadlock, as long as priorities are "respected", i.e., as long the configuration of the system remains in $\Cfree_G$. 

\begin{theorem}[Absence of deadlocks]
\label{thm:no-deadlock}
Given $G\in\graphs$ satisfying $\paths(\Cfree_G)\neq\emptyset$, for all $x\in\Cfree_G$, there exists $\path\in\paths(\Cfree_G)$ going through $x$.
\end{theorem}
\begin{proof}
Take a priority graph $G\in\graphs$, $x\in\Cfree_G$, assume $\paths(\Cfree_G)\neq\emptyset$ and take $\path\in\paths(\Cfree_G)$. First of all, note that concatenating $\path$ with the segment joining $\path(0)$ and $\min(x,\path(0))$ and with the segment joining $\path(1)$ with $\max(x,\path(1))$ gives a path in $\paths(\Cfree_G)$ starting from a configuration lower than or equal to $x$ and ending at a configuration greater than or equal to $x$. Hence, assume without loss of generality that $\path(0)\leq x$ and $\path(1)\geq x$. Define $\tilde{\path}^1:=\max(x,\path)$ and $\tilde{\path}^2:=\min(x,\path)$. These paths take values in $\chifree_G$ by Property~\ref{property:min-max}. The concatenation of $\tilde{\path}^1$ and $\tilde{\path}^2$ gives a path $\tilde{\path}\in\paths(\Cfree_G)$  going through $x$.
\end{proof}

\renewcommand{\Cfree}{C^{\mathrm{free}}}
\renewcommand{\Cobs}{C^{\mathrm{obs}}}

\chapter*{Conclusions}

The present part proposed a novel tool in multi robot coordination: the priority graph. It enables to go one step ahead in the understanding of the structure of the solutions to the coordination problem. Previous work noticed the existence of homotopy classes of feasible paths in the coordination space~\cite{Ghrist2005}. Our results demonstrate that priorities uniquely encode these homotopy classes. The existence of a finite number of homotopy classes of feasible paths then merely appears as the consequence of the finiteness of possible priority graphs. Assigning priorities plans a high-level coordination strategy represented by the priority graph describing the relative order of robots through the intersection. Under assigned priorities, the path of robots in the coordination space is just required to remain in a homotopy class of feasible paths continuously deformable into each other. Respecting assigned priorities is weaker than following a particular feasible path as a large (homotopy) class of feasible paths induce the same priorities. The size of the homotopy class provides some freedom of action. Therefore, priorities appear a relevant resource to guide robots through an intersection area.  A key asset of planning priorities is that well-chosen feasible priorities -- in particular, acyclic priorities -- completely solves the deadlock avoidance problem (see Threorem~\ref{thm:feasible}). Theorem~\ref{thm:no-deadlock} proves that provided feasible priorities are respected, robots will never be stuck in a deadlock configuration. It is key to solve the deadlock avoidance problem at the planning level as deadlocks are difficult to avoid in a reactive manner.

The results of the present part are quite conceptual with little care about the dynamics model and control issues. It does not specify how to use priorities to control robots. In traditional planning, the control part consists of executing the plan by tracking the planned reference trajectory. This is known as the trajectory tracking problem~\cite{Jiang1997,Lee2001,Micaelli1993,Soetanto2003,Yang1999}. The reference trajectory configures the control law which tries to minimize the tracking error (e.g., using a linear-quadratic regulator~\cite{Todorov2006}). In priority-based coordination, there is no reference trajectory to track, the plan is merely the priority graph. The next part of the thesis assumes that priorities are assigned and it aims at building control laws configured by the priority graph ensuring that priorities are respected and that all robots eventually go through the intersection. From the point of view of the present part, control laws proposed in the next part ensure that the resulting path described by robots in the coordination space belongs to the homotopy class encoded by the assigned priorities.

\part{Priority preserving control}
\label{part:priorities-to-guide-robots}

\chapter*{Introduction}

\begin{minipage}{\linewidth}

Previous work noticed the combinatorial complexity of multi robot control (see, e.g.,~\cite{Colombo2012}). In~\cite{Colombo2012}, the expected application is a driver assistance system to avoid crashes between human driven vehicles just in time. It is thus completely right to try to find a particular schedule to avoid the crash. In this thesis, we are in a much different context and we assume that robots are in a safe state when approaching the intersection, and a lot of different schedules -- more precisely, a lot of different priority graphs -- are possible to safely coordinate robots. The present part assumes that feasible priorities are assigned, that the assigned priorities are compatible with the initial state of the robots, and focuses on how to use the assigned priorities to guide robots through the intersection. As priorities are assigned, there is no combinatorial problem, and so-called priority preserving control can be carried out in polynomial time. In traditional planning, the plan is a reference trajectory which configures a control law in charge of tracking the reference trajectory. In priority-based coordination, the plan is the priority graph, so the control law is configured by the priority graph and is in charge of ensuring priority preservation (no collision occurs and priorities are respected). Ensuring priority preservation is much weaker than tracking a reference trajectory, so robots retain some freedom of action. The proposed control law guarantees liveness, i.e., following the control law, all robots eventually go through the intersection. The freedom of action enabled by planning only priorities is highlighted, as under the presented control law, robots may brake at any point of time without violating priorities, in particular without colliding. This robustness property is quite novel among existing coordination systems and is highly valuable as it is very likely to happen that robots need to brake to handle some unexpected event (e.g., a pedestrian crossing the road, a loss of communication abilities, a congestion at the exit of the intersection). Finally, the proposed control scheme in Chapters~\ref{chap:control-acceleration} and~\ref{chap:control-uncertainty} is decentralized as the output of the control law can be computed on each robot independently without an agreement through communication links.

\paragraph{Sketch of the part} The assigned priorities are assumed to be acyclic. Under this assumption, Chapter~\ref{chap:optimal-control-velocity} provides a priority preserving control law for robots controlled in velocity; Chapter~\ref{chap:control-acceleration} examines the case of robots controlled in acceleration; and in Chapter~\ref{chap:control-uncertainty}, robustness of priority preserving control with respect to bounded noise is illustrated. The reader is referred to Appendix~\ref{app:control-extension-all-feasible-priorities} for an extension of the results of this part to feasible cyclic priority graphs under mild assumptions.

\paragraph{Note to the reader} The two first chapters of the present part are independent. However, it is advised to start with the first chapter for a gradual understanding of the proposed method. The last chapter is not necessary to the understanding of the rest of the thesis. The reader without special interest in considering uncertainty concerns can skip Chapter~\ref{chap:control-uncertainty} and go directly to Part~\ref{part:priority-based-coordination}. 

\end{minipage}

{\pagestyle{plain}
\clearpage
\topskip0pt
\vspace*{\fill}
\includegraphics[height=1.0\linewidth,angle=-90,trim=160 60 160 60, clip]{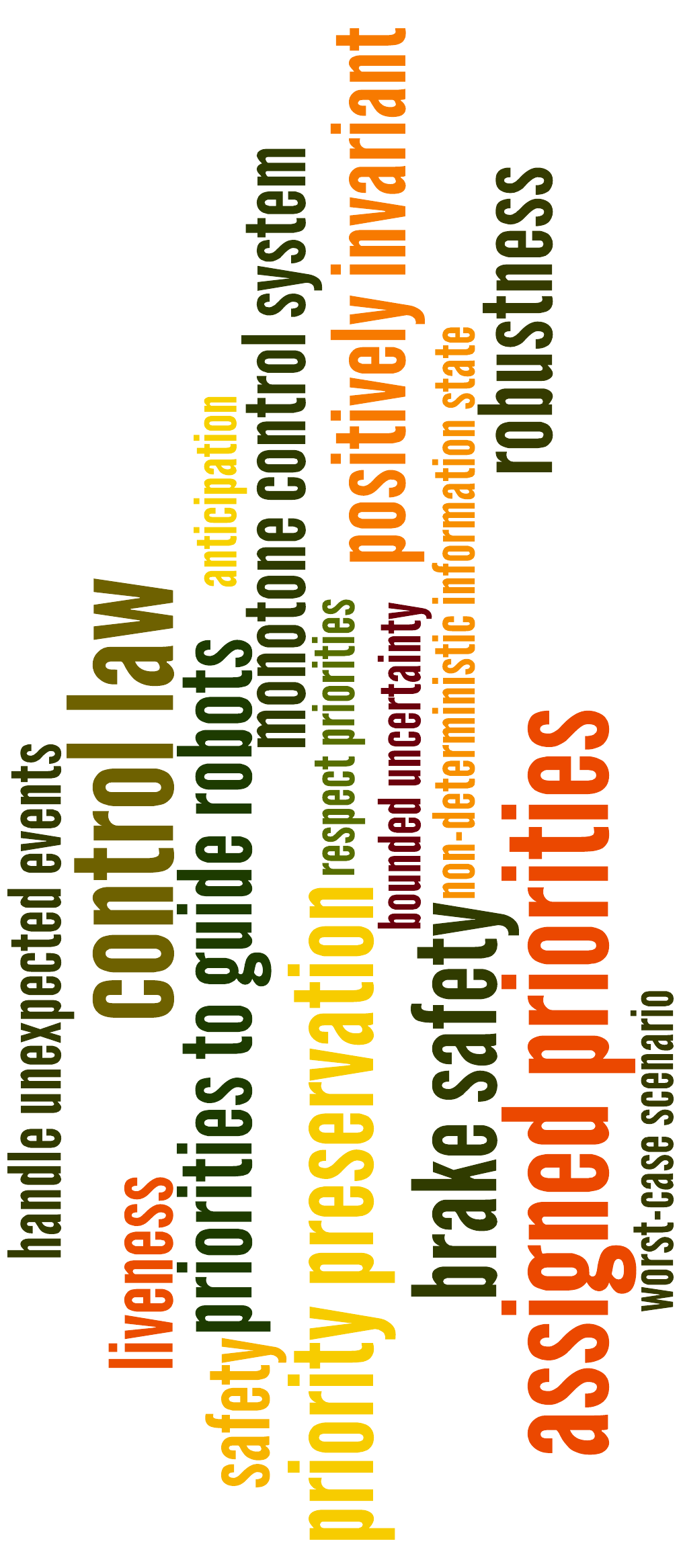}
\vspace*{\fill}
\parttoc}

\chapter[Priority preserving control in the absence of inertia]{Priority preserving control \\in the absence of inertia}
\label{chap:optimal-control-velocity}
\minitoc

In the present chapter, the velocity of the robots is assumed to be controlled, and a control law aimed at coordinating multiple robots with assigned priorities is proposed.

\paragraph{Sketch of the chapter} Section~\ref{sec:control-velocity-model} exposes the dynamics model and shows that the resulting system is a monotone control system~\cite{Angeli2003}. Section~\ref{sec:control-velocity-control-law} constructs a priority preserving control law. Optimality and liveness properties are provided.

\section{A monotone control system}
\label{sec:control-velocity-model}

Each robot $i$ is modeled as a first-order control system with state $x_i\in\RR$, whose evolution is described by the differential equation:
\begin{equation}
\dot{x_i}(t) = \vbf_i(t) \label{eq-diff-control-velocity}
\end{equation}
where $\vbf_i:\RR_+ \to V_i$ is the control of robot $i$. We let $V_i:=\{0,\vmax_i\}$ be the set of feasible control values. The control is assumed to be updated in discrete time every $\dt>0$: 
\begin{equation}
\forall k\in\NN, \forall t\in[k\dt, (k+1)\dt), \vbf_i(t) \equiv \vbf_i(k\dt)
\end{equation}
The time interval $[k\dt, (k+1)\dt)$ will be referred to as (time) slot $k$. For the sake of simplicity we let $\dt := 1$ in the sequel. We let $\vcontrols_i$ denote the set of controls $\vbf_i:\RR_+ \to V_i$ piecewise constant on intervals $[k,k+1)$, $k\in\NN$. We let $t \mapsto \phi_i(t,x_i,\vbf_i)$ denote the flow of the system starting at initial configuration $x_i\in \RR$  with control $\vbf_i \in \vcontrols_i$ as depicted in Figure~\ref{fig:velocity-control}. 

\begin{figure}[!htbp]
\begin{center}
\includegraphics[width=1.0\linewidth]{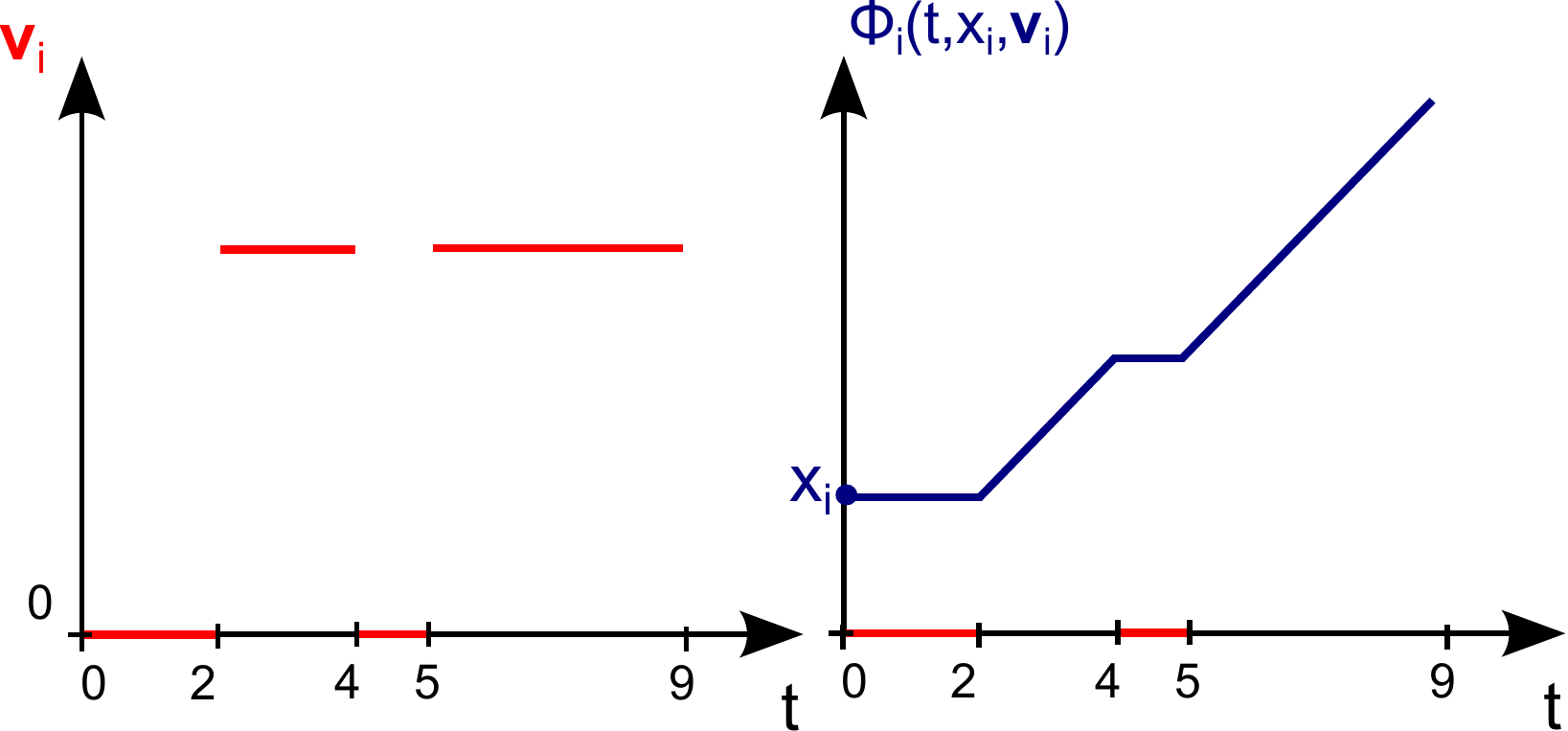}\hfill
\end{center}
\caption{An example of piecewise constant velocity control $\vbf_i$ (left) and the corresponding flow $t \mapsto \phi_i(t,x_i,\vbf_i)$ starting from the initial configuration $x_i$ (right).}
\label{fig:velocity-control}
\end{figure}

We also define the vectorial control $\vbf:=(\vbf_i)_{i\in\robots}\in \vcontrols:=\prod_{i\in\robots}\vcontrols_i$, and the vectorial flow: $\phi(t,x,\vbf):=(\phi_i(t,x_i,\vbf_i))_{i\in\robots}$. We let $\vmax:=(\vmax_i)_{i\in\robots}$ and we define the constant control $\vmaxbf(t):=\vmax$. We introduce partial orders as follows:
\begin{eqnarray}
\forall \vbf_i^1,\vbf_i^2\in \vcontrols_i, \vbf_i^1 \preceq \vbf_i^2 &\text{if}& \forall t\geq 0, \vbf_i^1(t) \leq \vbf_i^2(t)\\
\forall \phi^1,\phi^2:\RR_+\to \chi, \phi^1 \preceq \phi^2 &\text{if}& \forall t\geq 0,\phi^1(t) \preceq \phi^2(t)
\end{eqnarray}
The control system~\eqref{eq-diff-control-velocity} is a monotone control system~\cite{Angeli2003} with regards to the relative orders defined above. More precisely, the following key property holds:

\begin{property}[Order preservation]
The flow $t \mapsto \phi_i(t,x_i,\vbf_i)$ is order-preserving with regards to $x_i$ and $\vbf_i$.
\end{property}

Note that in our open loop model, control $\vbf_i$ only acts on robot $i$, that is, $\vbf$ is a collection of independent controls: it does not achieve any kind of coordination between the robots. The control law introduced in the sequel is precisely aiming at coordinating the robots to avoid collisions and respect priorities.

\section{The proposed control law}
\label{sec:control-velocity-control-law}

Now, we propose to build a control law $f^G:\chi\to V$ such that starting from an initial collision-free configuration, the flow of the system controlled by the control law $f^G$ is ensured to remain in $\chifree_G$ (thus being collision-free and respecting priorities $G$). In other words, using the terminology of~\cite{Kerrigan2000}, $\chifree_G$ shall be positively invariant for the system under control law $f^G$.

The rationale for our control law is as follows. Each robot $i\in\robots$ moves forward, unless moving forward violates the priority with regards to some robot $j\in\robots$ with $(j,i)\in E(G)$. In the coordination space, violating such a priority means that the configuration of the system would collide with $\chiobs_{j\succ i}$. The control law can then be formulated synthetically component-wise:
\begin{equation}
f_i^G(x):= \begin{cases}
0&\text{ if } \exists(j,i)\in E(G),\exists t\in[0,1] \text{ s.t. } \left(x+t\left(\vmax_i\e_i+f_j^G(x)\e_j\right)\right)\in\chiobs_{j\succ i} \\
\vmax_i&\text{ else.}
\end{cases}
\label{eq:definition-control-law-velocity}
\end{equation} 
First of all, note that $f^G$ appears in both the left-hand side and the right-hand side in Equation~\eqref{eq:definition-control-law-velocity}. Hence, it is not obvious that Equation~\eqref{eq:definition-control-law-velocity} effectively defines a control law which is stated by the following theorem. Note that a decentralized version of the proposed control law could be used alternatively by considering the worst case scenario for each robots $(j,i)\in E(G)$, i.e., when robot $j$ stops (see the decentralized control law of Chapter~\ref{chap:control-acceleration}). However, the optimality result that we obtain in the present chapter would not hold anymore.
\begin{theorem}[Control law existence]
Given an acyclic priority graph $G$, Equation~\eqref{eq:definition-control-law-velocity} uniquely defines a control law $f^G:\chi\to V$. 
\label{thm:control-law-existence}
\end{theorem}
\begin{proof}
The priority graph is assumed to be acyclic. Hence, there exists a topological ordering of the graph such that for every edge $(j,i)\in E(G)$, $j$ comes before $i$ in the ordering. Following the topological order induced by $G$, it is possible to compute $f_i^G(x)$ for all $i\in\robots$ iteratively. As a result,  Equation~\eqref{eq:definition-control-law-velocity} uniquely defines a control law $f^G:\chi\to V$.
\end{proof}
Figure~\ref{fig:control-law-velocity-example} and~\ref{fig:trajectory-velocity-control-example} show the evolution of a three-robot system under control law $f^G$ under acyclic priorities. It is clear in Figure~\ref{fig:trajectory-velocity-control-example} that the control law belongs to the  "bug"  family, emanating  from the work  of~\cite{Lumelsky1987}. Indeed, the robots go at maximum speed until they are too close to the boundary of the obstacle region. Then, they follow the boundary with a certain distance as long as necessary.
\begin{figure}[p]
\begin{center}
\includegraphics[width=1.0\linewidth]{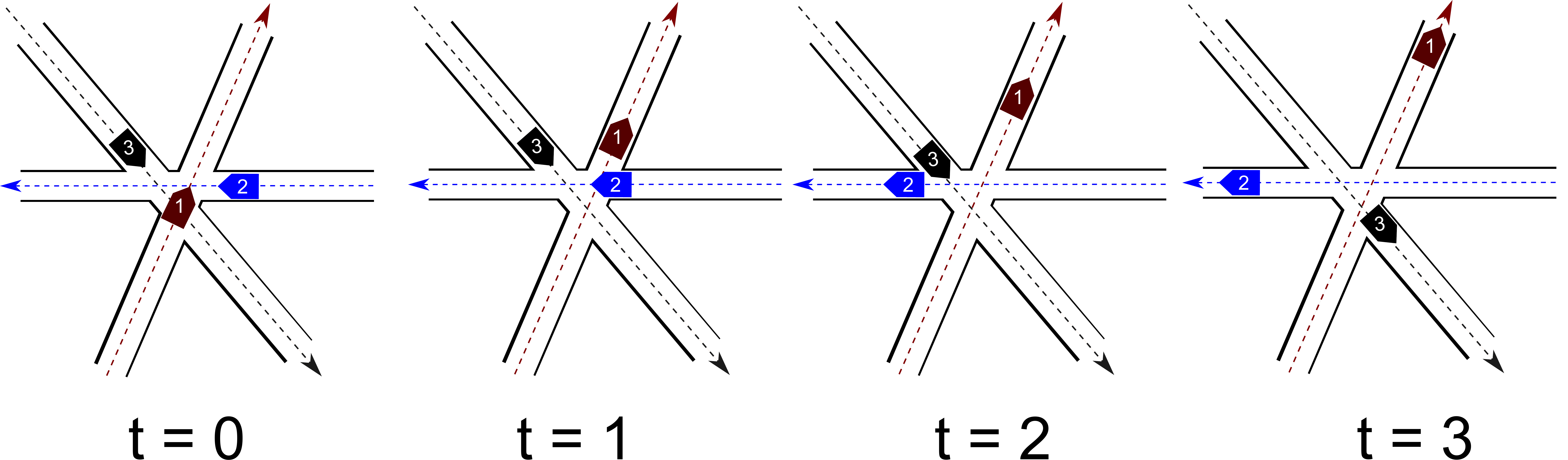}\hfill
\end{center}
\caption{A three-robot system with acyclic assigned priorities: $1\succ 2$, $2\succ 3$, and $1\succ 3$. Robots are controlled under control law $f^G$. The drawings show the evolution of the robots along their paths.}
\label{fig:control-law-velocity-example}
\end{figure}
\begin{figure}[p]
\begin{center}
\includegraphics[width=1.0\linewidth]{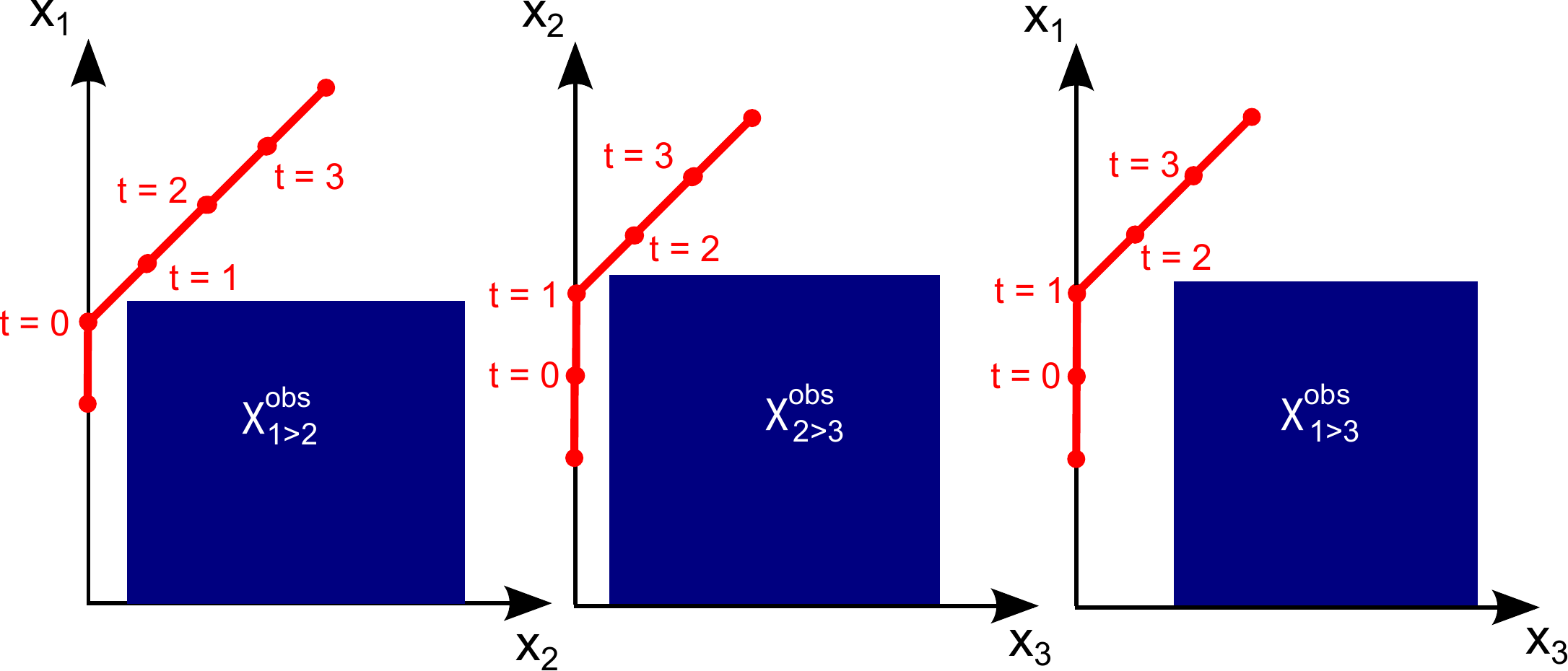}\hfill
\end{center}
\caption{Trajectory in the coordination space for the three-robot system under control law $f^G$ of Figure~\ref{fig:control-law-velocity-example}.}
\label{fig:trajectory-velocity-control-example}
\end{figure}

Now, we need to introduce the following notation. Given a feedback control law $f:\chi\to V$,  with a slight abuse of notation we let $t \mapsto \phi(t,x,f)$ denote the vectorial flow of the system starting at initial condition $x\in \chi$ and controlled by $\vbf\in \vcontrols$ satisfying:
\begin{equation}
\forall k\in\NN, \vbf(k) \equiv f(\phi(k,x,\vbf))
\end{equation}

\subsection{Priority preservation}

First of all, we prove the key property of our control law that is the safety guarantee. More precisely, starting from a configuration in $\chifree_G$, the system under control law $f^G$ is ensured to remain in $\chifree_G$, i.e., priorities $G$ are preserved. Following the terminology of~\cite{Kerrigan2000}, $\chifree_G$ is positively invariant under control law $f^G$ as stated in the following theorem:
\begin{theorem}[Priority preservation]
Given an acylic priority graph $G$, $\chifree_{G}$ is positively invariant for the system under control law $f^G$, i.e., 
\begin{equation}
\forall x\in \chifree_{G}, \forall t\geq 0, \phi(t,x,f^G) \in \chifree_{G}
\end{equation}
\label{thm:safety-robust-control-map-velocity}
\end{theorem}
\begin{proof}
Take an acyclic priority graph $G$ and an initial configuration $x\in\chifree_G$. By induction, it is sufficient to prove that the flow satisfies: 
\begin{equation}
\forall t\in[0,1], \phi(t,x,f^G)\in\chifree_G
\end{equation}
By construction, for all $i\in\robots$ and for all $(j,i)\in E(G)$, we have two options:
\begin{itemize}
\item either $f_i^G(x)=0$. For all $t\in[0,1]$, we have: 
\begin{eqnarray}
x_i+ t f_i^G(x) &=& x_i\\
x_j+t f_j^G(x) &\geq& x_j
\end{eqnarray}
Hence, by Property~\ref{property:geometric-invariance}, $x\in\chifree_{j \succ i}$ implies that for all $t\in[0,1]$, $\left(x+t \left(f_i^G(x)\e_i+f_j^G(x)\e_j\right)\right) \in \chifree_{j \succ i}$.
\item or $f_i^G(x)=\vmax_i$. Then, by construction of the control law, we have for all $t\in[0,1]$, $\left(x+t\left(f_i^G(x)\e_i+f_j(x)^G\e_j\right)\right) \equiv \left(x+t\left(\vmax_i\e_i+f_j^G(x)\e_j\right)\right) \in\chifree_{j \succ i}$.
\end{itemize}
Hence, in both cases, we obtain: 
\begin{equation}
x+t\left(f^G_i(x) \e_i+f^G_j(x)\e_j\right) \in \chifree_{j \succ i}
\label{eq:eq1-pf-safety-control-law-velocity}
\end{equation}
Moreover, we have:
\begin{eqnarray}
\phi_i(t,x,f^G)&=&x_i+t f^G_i(x) \\\
\phi_j(t,x,f^G)&=&x_j+t f^G_j(x)
\end{eqnarray}
As a result, Equation~\eqref{eq:eq1-pf-safety-control-law-velocity} implies that $\phi(t,x,f^G)\in\chifree_{j\succ i}$ for all $(j,i)\in E(G)$ and $t\in[0,1]$, i.e., $\phi(t,x,f^G)\in\chifree_G$ for all $t\in[0,1]$.
\end{proof}

Given a configuration $x\in\chifree_G$ and a priority graph $G$, we say $\vbf$ is a collision-free control for the pair $(x,G)$ if the flow starting from $x\in\chifree_G$ remains in $\chifree_G$. We write $\vbf\in\vcontrols_G^\free(x)$ defined as follows:
\begin{equation}
\vcontrols_G^\free(x):=\left\{\vbf\in\vcontrols: \phi(\RR_+,x,\vbf)\subset\chifree_G\right\}
\end{equation}

\subsection{Optimality}

First of all, we define the notion of optimality under assigned priorities used in the sequel. Given a priority graph $G$ and a control law $f$, we say $f$ is optimal for the priority graph $G$ if for all configurations $x\in\chifree_G$ and for all controls $\vbf\in\vcontrols_G^\free(x)$, we have:
\begin{equation}
\forall t\geq 0, \phi(t,x,f) \geq \phi(t,x,\vbf)
\end{equation}
In other words, the control law is optimal if for each robot, it maximizes the distance travelled through time while respecting priorities $G$. Note that this kind of optimality is even stronger than the family of Pareto optimality. Pareto optimality would state that it is impossible to make any individual robot travel farther without making at least one robot travel less. By contrast, our optimality result states that even if other robots travel less, it's impossible to make one robot travel farther while respecting the assigned priorities, i.e., all individual objectives are optimized. As a consequence, the obtained trajectory is optimal for a whole set of utility functions, more precisely, all utility functions which grow with the distance traveled by robots. For example, it minimizes the average exit time of robots, it also minimizes the maximum exit time of robots (the time at which the last robot exits the intersection). However, it is important to note that the optimality result is conditioned on the assigned priorities. Note that the trajectory resulting from the application of the proposed control law corresponds to the left-greedy optimal trajectory of References~\cite{Ghrist2005,Ghrist2006}, where it is noticed that it is a local optimum, in that it is optimal over trajectories belonging to the same homotopy class. Hence, obtaining a globally optimal trajectory would require exploring all feasible priorities, i.e., exploring all homotopy classes. 
\begin{theorem}[Optimality]
Given an acyclic priority graph $G$, the control law $f^G$ is optimal for the priority graph $G$, in the sense that for all controls $\vbf\in\vcontrols_G^\free(x)$, we have:
\begin{equation}
\forall t\geq 0, \phi(t,x,f^G) \geq \phi(t,x,\vbf)
\end{equation}
\label{thm:control-law-optimality}
\end{theorem}
\begin{proof}
We will prove Theorem~\ref{thm:control-law-optimality} by contraposition. Take an acyclic priority graph $G$, an initial condition $x\in\chifree$ and a control $\vbf\in\vcontrols$, Assume that there exists $i\in\robots$ and $t\geq 0$ such that $\phi_i(t,x,f^G) < \phi_i(t,x,\vbf)$. We have to prove that $\phi(\RR_+,x,\vbf)\cap\chiobs_G\neq\emptyset$. Consider $I:=\{t \geq 0: \exists i\in\robots: \phi_i(t,x,f^G) < \phi_i(t,x,\vbf)\}$. By assumption, $I\neq\emptyset$, then $I$ is a lower-bounded non-empty subset of $\RR$, so that $t^0=\inf I$ exists. Let $k^0$ be the unique $k\in\NN$ such that $t^0\in[k,k+1)$. By definition of $t^0$ and as the velocity control is piece-wise constant, we have:
\begin{equation}
\forall j\in\robots, \phi_j(k^0,x,f^G) \geq \phi_j(k^0,x,\vbf)\label{eq:not-already-overtaking-k}
\end{equation}
and there exists $i\in\robots$ such that:
\begin{equation}
\phi_i(k^0+1,x,f^G) < \phi_i(k^0+1,x,\vbf) \label{eq:overtaking-k-plus-1}
\end{equation}
As $V_i=\{0,\vmax_i\}$ (binary velocity control), Equations~\eqref{eq:overtaking-k-plus-1} and~\eqref{eq:not-already-overtaking-k} imply that:
\begin{equation}
\phi_i(k^0,x,f^G) = \phi_i(k^0,x,\vbf) \label{eq:phi-i-equality-k}
\end{equation}
\begin{equation}
\vbf_i(k)=\vmax_i>f^G_i(x^0)=0\label{eq:velocity-strictly-greater}
\end{equation}
where $x^0:=\phi(k^0,x,f^G)$. As $f^G_i(x^0)=0$, by construction of the control law $f^G$, there is necessarily an edge $(j,i)$ in the graph $G$ satisfying:
\begin{equation}
\left(x^0+t\left(\vmax_{i}\e_{i}+f_j^G(x^0)\e_j\right)\right) \in \chiobs_{j\succ i} \label{eq:collision-proof-optimality}
\end{equation}
for some $t\in[0,1]$. Assume additionally that $i$ is chosen to be a maximal element of the (acyclic) sub-graph of $G$ containing only vertices satisfying Equation~\eqref{eq:overtaking-k-plus-1}. Then, as $(j,i)\in E(G)$, $j$ does not satisfy Equation~\eqref{eq:overtaking-k-plus-1} and we have:
\begin{equation}
\phi_j(k^0+t,x,f^G) \geq\phi_j(k^0+t,x,\vbf)
\label{eq:j-not-overtaking}
\end{equation}
Combining Equations~\eqref{eq:phi-i-equality-k} and~\eqref{eq:j-not-overtaking}, we obtain:
\begin{eqnarray}
x^0_i+t\vmax_i &=\phi_i(k^0,x,\vbf)+t\vmax_i=& \phi_i(k^0+t,x,\vbf) \label{eq:ineq-i-to-prove-collision}\\
x^0_j+t f_j^G(x^0) &=  \phi_j(k^0+t,x,f^G) \geq& \phi_j(k^0+t,x,\vbf)\label{eq:ineq-j-to-prove-collision}
\end{eqnarray}
By Property~\ref{property:geometric-invariance}, as Equations~\eqref{eq:ineq-i-to-prove-collision} and~\eqref{eq:ineq-j-to-prove-collision} are satisfied, Equation~\eqref{eq:collision-proof-optimality} implies that:
\begin{equation}
\phi(k^0+t,x,\vbf)\in\chiobs_{j\succ i} \subset \chiobs_G
\end{equation}
In conclusion, 
\begin{equation}
\phi(\RR_+,x,\vbf) \cap \chiobs_G \neq \emptyset
\end{equation}
\end{proof}

The above theorem is illustrated in Figure~\ref{fig:trajectory-velocity-control-example-optimality}.
\begin{figure}[!htbp]
\begin{center}
\includegraphics[width=1.0\linewidth]{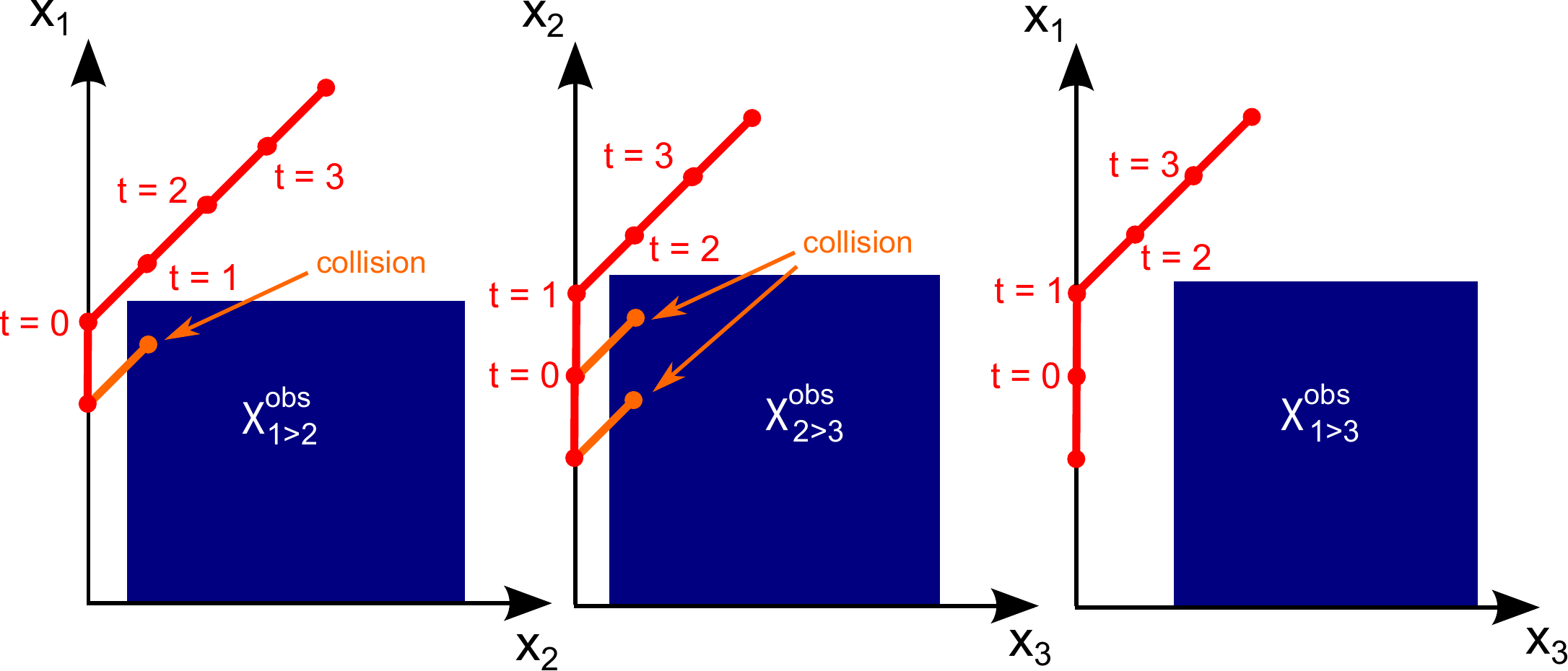}\hfill
\end{center}
\caption{Illustration of the optimality of the control law for a three-robot system.}
\label{fig:trajectory-velocity-control-example-optimality}
\end{figure}
It is clear that if at some point of time one robot tries to go faster than prescribed by the control law, a collision would occur. For example, in the left drawing, one can see that if robot $3$ tries to move forward at time $t=0$ instead of stopping as prescribed by the control law, a collision occurs (see the orange segment). 

\subsection{Liveness}

A key property in motion planning is liveness, i.e., the guarantee that every robot eventually reaches its goal. In the particular case of the problem studied here, every robot is expected to exit the obstacle region. Hence, liveness is guaranteed if every robot $i\in\robots$ eventually reaches the region $\chigoal:=\xobsmax+\RR_+^n$.

\begin{theorem}[Liveness]
Given an acyclic priority graph $G$ and a configuration $x^0\in\chifree_G$, there exists $T>0$ such that:
\begin{equation}
\phi(T,x^0,f^G)\in\chigoal
\end{equation}
\label{thm:liveness-control-law-velocity}
\end{theorem}
The idea of the proof is that under acyclic priorities there is always at least one non exited robot able to move forward at maximum velocity until it exits the intersection.
\begin{proof}
Take an acyclic priority graph $G$. Consider the trajectory of the robots under control law $f^G$. $G$ being acyclic, there exists an extremal vertex $i_1\in\robots$ such that for all $j\in\robots$, $(j,i_1)\notin E(G)$. As a result, under the control law $f^G$, robot $i_1$ will always travel at maximal velocity and it will exit the intersection (it will reach position $\xobsmax_i$) in finite time $T_1$.

Now, assume that at time $T_m$, robots $i_1 \cdots i_m$ have exited the intersection and $m<n$ (there remain some robots). $G$ being acyclic, there exists an extremal element for the remaining robots denoted $i_{m+1}\in\robots\setminus \{i_1 \cdots i_m\}$ such that for all $j\in\robots\setminus \{i_1 \cdots i_m\}$, $(j,i_{m+1})\notin E(G)$. Collisions occurring only with non exited robots, for $t\geq T_m$ $j$ will always be at maximum velocity and it will exit the intersection in finite time at instant $T_{m+1}\geq T_m$. 

Iterating this process yields a sequence $(T_1 \cdots T_n)$ and all robots have exited the intersection at time $T:=T_n$.
\end{proof}

\chapter[Priority preserving control under kinodynamic constraints]{Priority preserving control \\under kinodynamic constraints}
\label{chap:control-acceleration}
\minitoc

In the present chapter, the acceleration of the robots is assumed to be controlled, and a control law is proposed aiming at coordinating multiple robots with assigned priorities under second-order kinodynamic constraints. The method is inspired by References~\cite{DelVecchio2009,Colombo2012,Kowshik2011,Verma2012,Hafner2011} dealing with the coordination of a small number of vehicles without explicit notion of priorities. Preliminaries of the presented results are presented in our conference paper~\cite{Gregoire2013-dynamic-constraints}, and a more accomplished version in our article~\cite{Gregoire2013-priority-based}.

\paragraph{Sketch of the chapter} Similarly to the previous chapter, the first section exposes the second-order dynamics model and shows that the resulting system is a monotone control system. The second section constructs a priority preserving control law. Liveness is provably guaranteed and a robustness property is provided stating that a robot may brake at any moment with neither colliding, nor violating priorities.

\section{A monotone control system}

Each robot $i$ is modeled as a second-order control system with state $s_i=(x_i,v_i)\in S_i:=\RR \times [0,\vmax_i]$, whose evolution is described by the differential equation:
\begin{eqnarray}
\dot{x_i}(t) &= & v_i(t)
\label{eq-diff-state1-deterministic}
\\
\dot{v_i}(t) & =& \ubf_i(t) ~\delta(\ubf_i(t),v_i(t))
\label{eq-diff-state2-deterministic}
\end{eqnarray}
where $\ubf_i:\RR_+ \to U_i$ is the control of robot $i$ and $\vmax_i$ denotes the non-negative speed limit for robot $i$. We let $U_i:=[\uL_i,\uH_i]$ be the set of feasible control values. $\uL_i<0$ represents the maximum brake control value and $\uH_i>0$ represents the maximum throttle control value. $\delta$ is a binary function merely ensuring that $v_i\in [0,\vmax_i]$ at all times, that is,
$\delta(\ubf_i(t),v_i(t))=1$  except for $v_i(t)=0$ and $\ubf_i(t)<0$, and for $v_i(t)=\vmax_i$ and $\ubf_i(t)>0$, where it vanishes.

The control is assumed to be updated in discrete time every $\dt>0$: 
\begin{equation}
\forall k\in\NN, \forall t\in[k\dt, (k+1)\dt), \ubf_i(t) \equiv \ubf_i(k\dt)
\end{equation}
The time interval $[k\dt, (k+1)\dt)$ will be referred to as (time) slot $k$. For the sake of simplicity we let $\dt := 1$ in the sequel. We let $\controls_i$ denote the set of controls $\ubf_i:\RR_+ \to U_i$ piecewise constant on intervals $[k,k+1)$, $k\in\NN$. We let $t \mapsto \phi_i(t,s_i,\ubf_i)$ denote the flow of the system starting at initial condition $s_i\in S_i$  with control $\ubf_i \in \controls_i$. We also define the vectorial state $s:=(s_i)_{i\in\robots}\in S$, the vectorial control $\ubf:=(\ubf_i)_{i\in\robots}\in \controls:=\prod_{i\in\robots}\controls_i$, and the vectorial flow: $\phi(t,s,\ubf):=(\phi_i(t,s_i,\ubf_i))_{i\in\robots}$. We let $\uL:=(\uL_i)_{i\in\robots}$, $\uH:=(\uH_i)_{i\in\robots}$ and we define the constant controls $\uLbf(t):=\uL$ and $\uHbf(t):=\uH$. 

We define projection operators as follows: given a state $s=(x,v)=(s_i)_{i\in\robots}=((x_i,v_i))_{i\in\robots}$, we let $\pi_x(s):=x$, $\pi_{x,i}(s):=\pi_{x,i}(s_i):=x_i$, $\pi_v(s):=v$, and $\pi_{v,i}(s):=\pi_{v,i}(s_i):=v_i$. We also define projected flows as follows: $\phi_x=\pi_x \circ \phi$, $\phi_{x,i}=\pi_{x,i} \circ \phi$, $\phi_v=\pi_v \circ \phi$ and $\phi_{v,i}=\pi_{v,i} \circ \phi$. Figure~\ref{fig:control-in-acceleration} depicts the projected flow $t \mapsto \phi_{x,i}(t,s_i,\ubf_i)$ for a particular control $\ubf_i$.
\begin{figure}[!htbp]
\begin{center}
\includegraphics[width=1.0\linewidth]{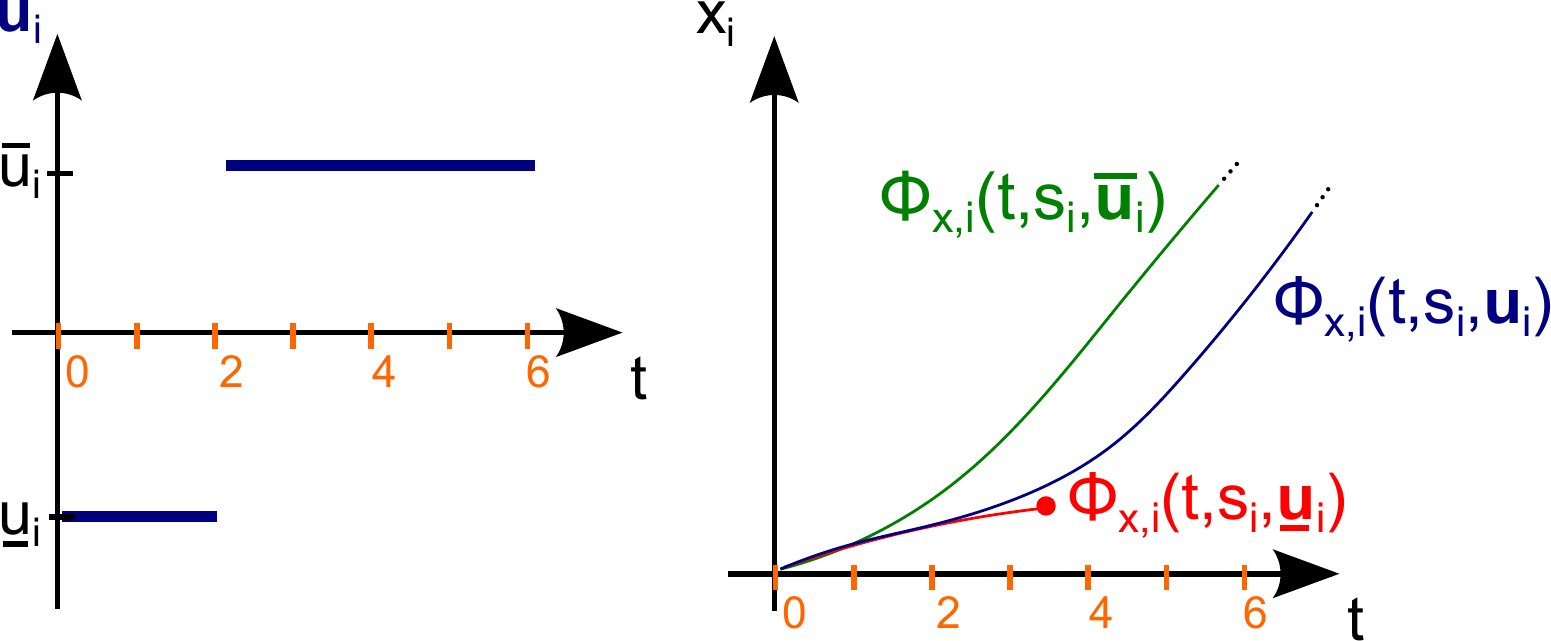}\hfill
\end{center}
\caption{An example of piecewise constant control $\ubf_i$ (left) and the corresponding projected flow $t \mapsto \phi_{x,i}(t,s_i,\ubf_i)$ starting from the initial configuration $s_i$ (right).}
\label{fig:control-in-acceleration}
\end{figure}

We introduce partial orders as follows:
\begin{eqnarray}
\forall \ubf_i^1,\ubf_i^2\in \controls_i, \ubf_i^1 \preceq \ubf_i^2 &\text{if}& \forall t\geq 0, \ubf_i^1(t) \leq \ubf_i^2(t)\\
\forall s_i^1=(x_i^1,v_i^1),s_i^2=(x_i^2,v_i^2)\in S_i, s_i^1 \preceq s_i^2 &\text{if}& x_i^1 \leq x_i^2 \text{ and } v_i^1 \leq v_i^2\label{orderr:eq}\\
\forall \phi^1,\phi^2:\RR_+\to S, \phi^1 \preceq \phi^2 &\text{if}& \forall t\geq 0,\phi^1(t) \preceq \phi^2(t)
\end{eqnarray}
The control system~\eqref{eq-diff-state1-deterministic}-\eqref{eq-diff-state2-deterministic}  is a monotone control system~\cite{Angeli2003} with regards to the relative orders defined above as easily seen in Figure~\ref{fig:control-in-acceleration} (in this example, we have $\uLbf_i \preceq \ubf_i \preceq \uHbf_i$). More precisely, the following key property holds:

\begin{property}[Order preservation]
The flow $t \mapsto \phi_i(t,s_i,\ubf_i)$ is order-preserving with regards to $s_i$ and $\ubf_i$.
\end{property}

Note that in our open loop model, control $\ubf_i$ only acts on robot $i$, that is, $\ubf$ is a collection of independent controls: it does not achieve any kind of coordination between the robots. The control law introduced in the sequel is precisely aiming at coordinating the robots to avoid collisions and respect priorities.

\section{The proposed decentralized control law}
\label{law:sec}

In the absence of inertia as in Chapter~\ref{chap:optimal-control-velocity}, robots can stop instantly to respect priorities. With second-order dynamics, robots cannot stop instantly anymore and need to anticipate, taking into account their brake distance, to effectively respect priorities. The idea proposed here is to constrain the multi robot system to remain in so-called brake safe states where robots can always safely brake without colliding. 

Define the set of brake safe states as follows:
\begin{equation}
B_{G}:=\{ s\in S: \phi_x\left(\RR_+,s,\uLbf\right) \subset \chifree_G\} \subset S
\end{equation} 
According to the above definition, a state $s\in S$ is brake safe if, starting at initial condition $s$ under maximum brake control, the system remains in $\chifree_G$ (see Figure~\ref{fig:brake-safety}). In particular, a state $(x,0)$ with $x\in\chifree_G$ is brake safe, so $B_G$ is not empty provided $\chifree_G$ is not empty. Figure~\ref{fig:brake-safety} illustrates brake safety in the coordination space and Figure~\ref{fig:brake-safety-examples} attempts to represent the concept in the real space.
\begin{figure}[p]
\begin{center}
\includegraphics[width=0.5\linewidth]{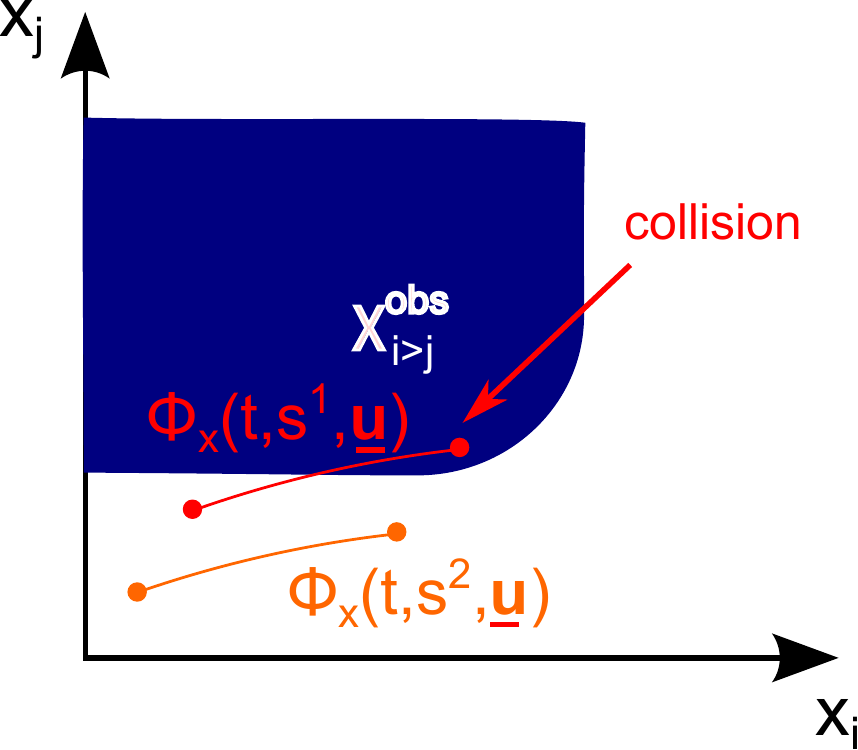}\hfill
\end{center}
\caption{Illustration of brake safety in the coordination space for a two-robot scenario with the assigned priority $(i,j)\in E(G)$. The flow starting from $s^1$ (resp. $s^2$) under control $\uLbf$ is constructed. $s^1$ is not brake safe as this flow collides $\chiobs_{i\succ j}$. $s^2$ is brake safe as the flow is collision-free with regards to $\chiobs_{i\succ j}$.}
\label{fig:brake-safety}
\end{figure}
\begin{figure}[p]
\begin{center}
\includegraphics[width=0.8\linewidth]{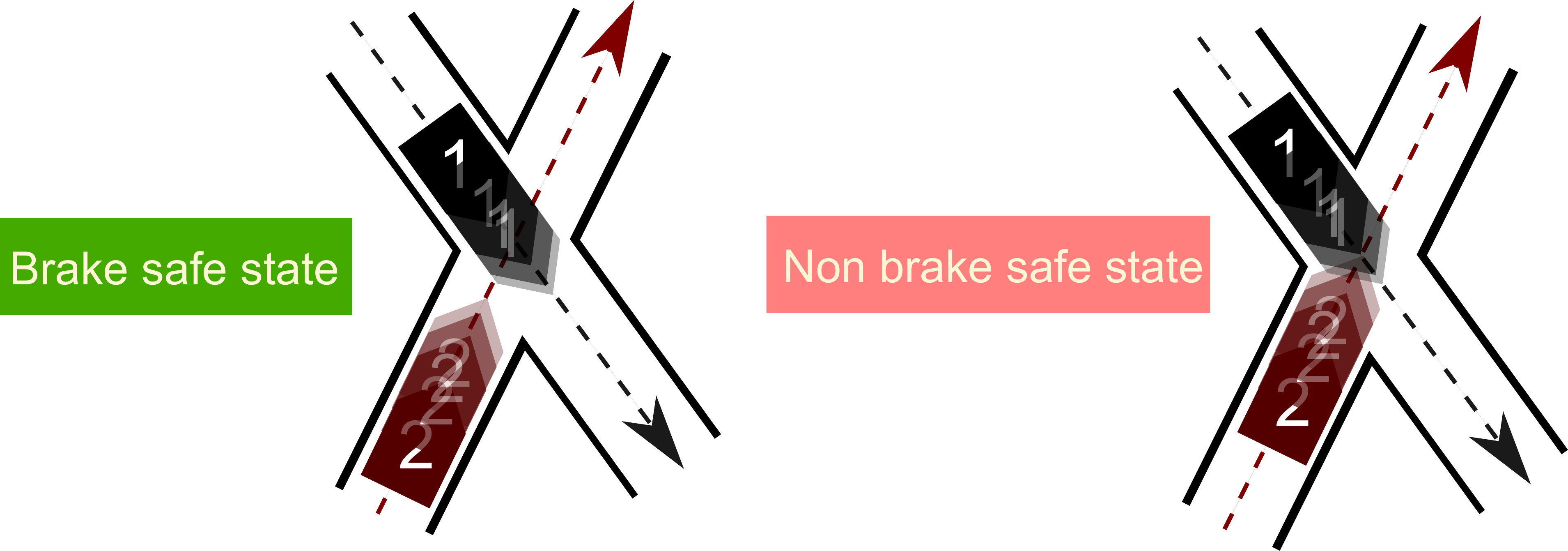}\hfill
\end{center}
\caption{Illustration of brake safety in the real space for a two-robot scenario. Robots with lower opacity are used to represent the flow under maximum brake command. In the left drawing, the two robots stop without colliding when applying maximum brake command: they are in a brake safe state. In the right drawing, a collision occurs when the two robots brake maximally: they are not in a brake safe state.}
\label{fig:brake-safety-examples}
\end{figure}
Brake safety is more conservative than remaining in the escape set proposed in~\cite{DelVecchio2009}, which includes all states from which there exists at least one control (not necessarily $\uLbf$) avoiding future collisions. It is also more conservative than not entering an inevitable collision state as defined in~\cite{Fraichard2004,Bouraine2011} where neither the geometric path in $\RR_2$ nor the control to avoid collisions are fixed. The idea behind this quite conservative approach is twofold:
\begin{enumerate}[(a)]
\item designing a decentralized control law: the output of the control law proposed in the following can be computed independently on each robot. It is much different from approaches where collision-free trajectories are computed, either in a centralized manner (see, e.g.,~\cite{Colombo2012} where decentralization is considered as a possible extension), or with some agreement with message-passing through communication links (see, e.g.,~\cite{Bekris2007}).
\item demonstrating robustness regarding unexpected deceleration of some robots: we believe that this is a highly valuable property as many unpredictable events requiring a robot to brake may happen in real applications.
\end{enumerate}
Importantly, note that checking whether a state $s \in S$ is brake safe consists in computing a finite time single flow $t\mapsto\phi(t,s,\uL)$ and checking for collisions with respect to each completed obstacle region $\chiobs_{i\succ j}$ for all $(i,j)\in E(G)$, yielding a quadratric complexity.

We propose to build a control law $g^G:S\to U$ such that starting from an initial brake safe state in $B_G$, the flow of the system controlled by the control law $g^G$ is ensured to remain in $B_G$ (thus being collision-free and respecting priorities $G$). In other words, using the terminology of~\cite{Kerrigan2000}, $B_G$ shall be positively invariant for the system under control law $g^G$. 

The rationale for our control law is as follows. Consider a robot $i$ and a robot $j$ that has priority over $i$. Given an initial configuration of the two robots, the worst-case scenario is when $j$ brakes whereas $i$ accelerates in the next time slot. If the trajectory of the system in the next time slot under this worst-case scenario is collision-free and if the reached state is brake safe, robot $i$ may accelerate in any case. Otherwise, it is required to brake. This is formalized below.

Let $\ubf_i^\imp\in \controls_i$ denote the impulse control for robot $i$ defined by (see Figure~\ref{fig:control-utilde}): 
\begin{equation}
\ubf_i^\imp(k) := 
\begin{cases}
\uH_i & \text{if } k=0\\
\uL_i & \text{if } k\geq 1
\end{cases}
\label{eq:impulse-control}
\end{equation}
\begin{figure}[!htbp]
\begin{center}
\includegraphics[width=0.7\linewidth]{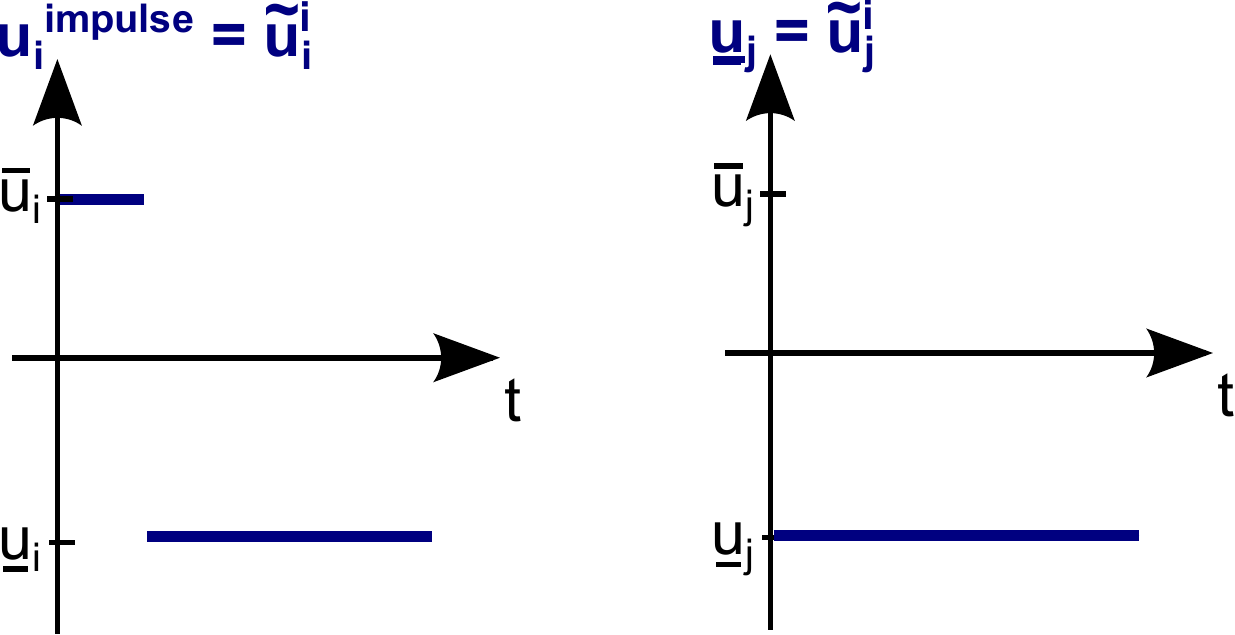}\hfill
\end{center}
\caption{The control $\tilde{\ubf}^i$ used in the formulation of the control law. For $j\neq \i$ (right drawing), $\tilde{\ubf}^i_j$ is simply the maximum brake command $\uLbf_j$. For $j=i$ (left drawing), $\tilde{\ubf}^i_i$ is the impulse control $\ubf_i^\imp$.}
\label{fig:control-utilde}
\end{figure}
Now let $\tilde{\ubf}^i$ denote the worst-case vectorial control with regards to $i$ defined componentwise by (see Figure~\ref{fig:control-utilde}): 
\begin{equation}
\tilde{\ubf}^i_j:=\begin{cases}
 \ubf_i^\imp & \text{if } j=i\\
\uLbf_j & \text{if } j\neq i
\end{cases}
\end{equation}
The control law can then be formulated synthetically:
\begin{equation}
g_i^G(s):=\begin{cases}
\uL_i & \text{if } \exists (j,i)\in E(G), \exists t\geq  0 \text{ s.t. } \phi_x(t,s,\tilde{\ubf}^i) \in \chiobs_{j\succ i} \\
\uH_i & \text{ else.}
\end{cases}
\label{eq-control-map}
\end{equation}
\begin{figure}[!htbp]
\begin{center}
\includegraphics[width=1.0\linewidth]{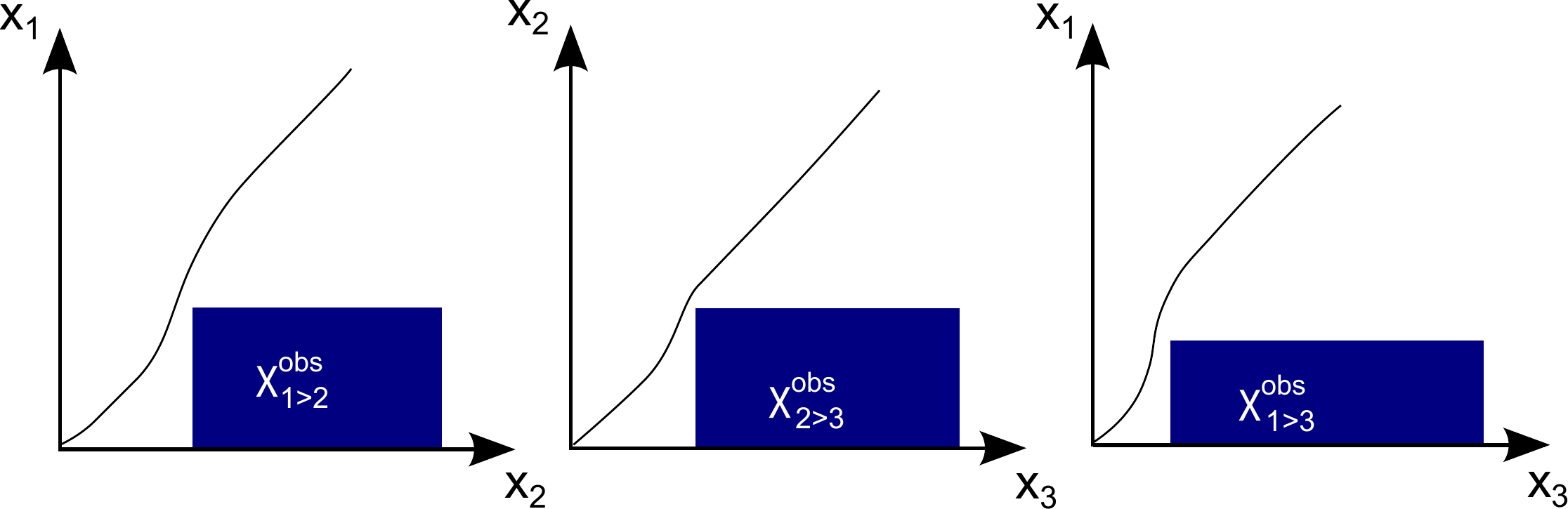}\hfill
\end{center}
\caption{Look of the trajectory for a three-robot system with acyclic assigned priorities $1\succ 2$, $2\succ 3$ and $1\succ 3$ under control law $g^G$. }
\label{fig:trajectory-acceleration-control-example}
\end{figure}
This simply means that robot $i$ always keeps a safe distance so that if a higher-priority robot $j$ suddenly brakes, robot $i$ may apply the maximum brake command until possibly stop without violating the priority. To this purpose, robot $i$ looks at the state that would be reached if it accelerates while the higher-priority robot $j$ brakes. If the simulated reached state is brake safe, $i$ may accelerate; otherwise, it must brake (see the two cases in Figures~\ref{fig:control-law-case1} and~\ref{fig:control-law-case1}). The look of the trajectory under control law $g^G$ is depicted in Figure~\ref{fig:trajectory-acceleration-control-example}. Note that, as for brake safety checking, computing the output of the control law is of quadratic complexity for the same reasons: it requires to compute a finite time single flow and to check for collisions. Note also that each component $g_i^G(s)$ can be computed independently for each $i\in\robots$, which means that the proposed control law is decentralized.

\begin{figure}[p]
\begin{center}
\includegraphics[width=0.7\linewidth]{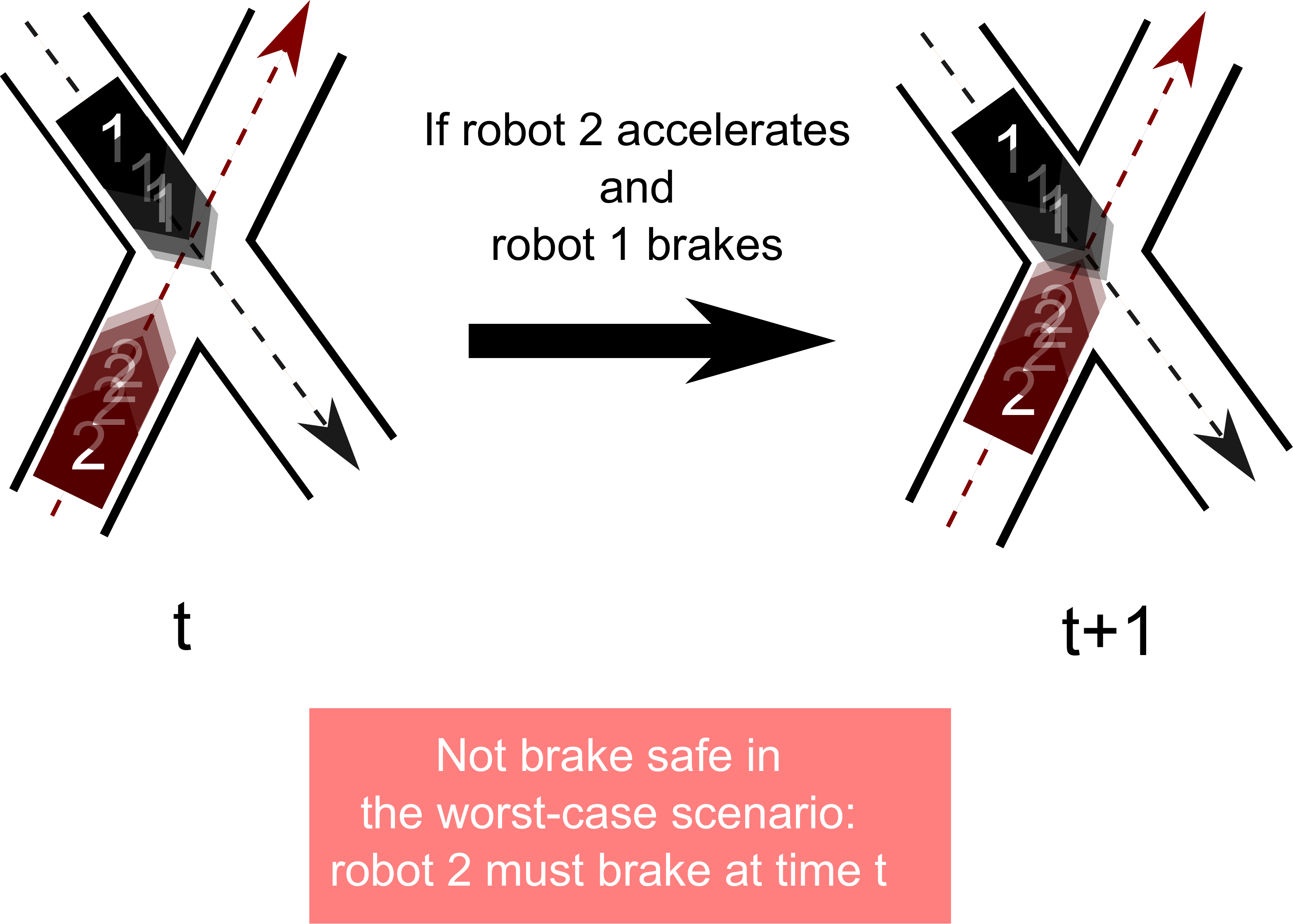}\hfill
\end{center}
\caption{In this setting where robot 1 is assumed to have priority, the worst-case scenario under which robot 2 accelerates and robot 1 brakes leads to a new state which is not brake safe. In this case, the control law requires robot 2 to brake. Robots with lower opacity are used to represent the brake trajectory.}
\label{fig:control-law-case1}
\end{figure}
\begin{figure}[p]
\begin{center}
\includegraphics[width=0.7\linewidth]{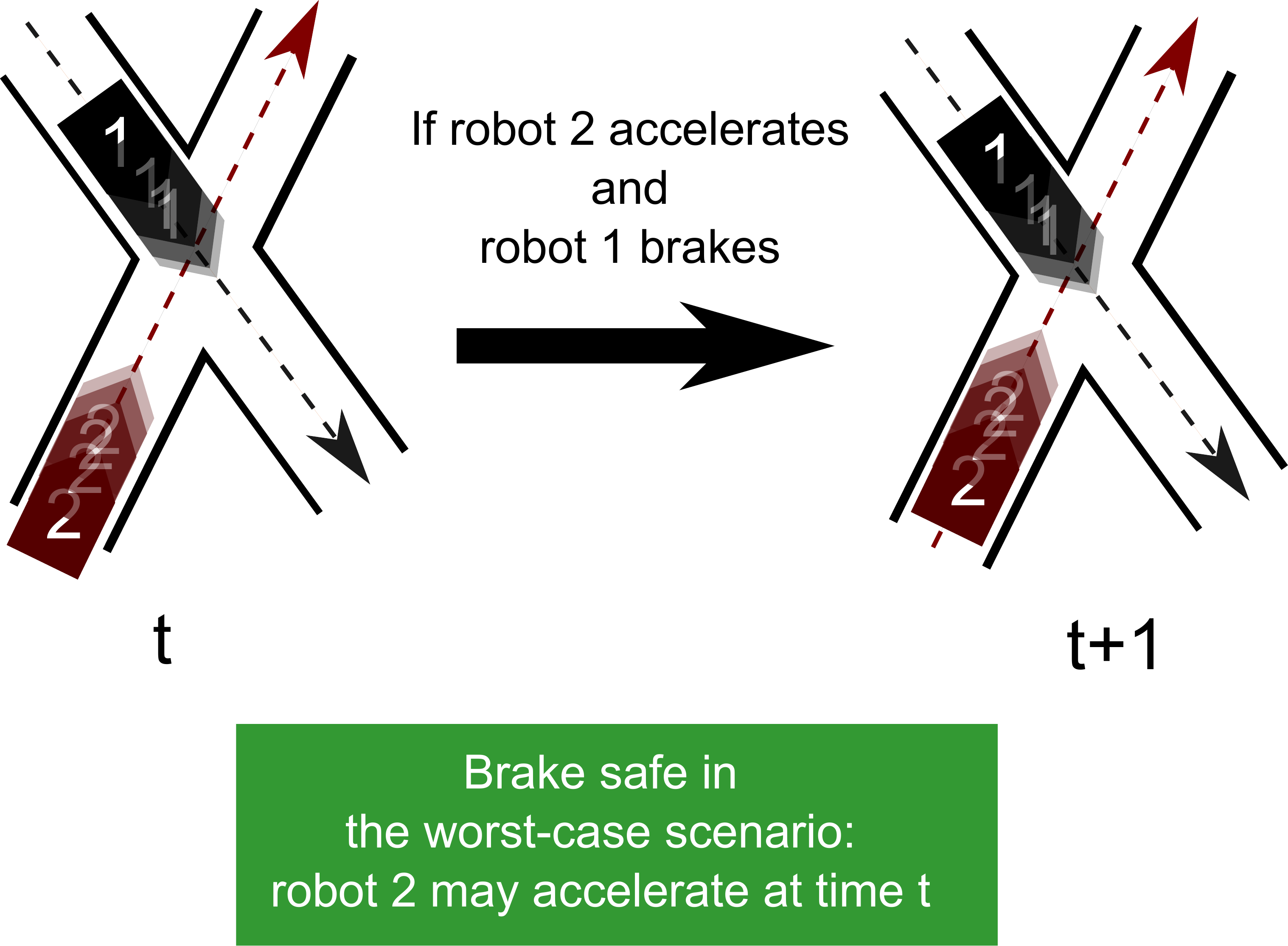}\hfill
\end{center}
\caption{In this setting where robot 1 is assumed to have priority, the worst-case scenario under which robot 2 accelerates and robot 1 brakes leads to a new state which is brake safe. In this case, the control law allows robot 2 to accelerate. Robots with lower opacity are used to represent the brake trajectory.}
\label{fig:control-law-case2}
\end{figure}

\subsection{Priority preservation}

Now, we need to introduce the following notation. Given a feedback control law $g:S\to U$,  with a slight abuse of notation we let $t \mapsto \phi(t,s,g)$ denote the vectorial flow of the system starting at initial condition $s\in S$ and controlled by $\ubf\in \controls$ satisfying:
\begin{equation}
\forall k\in\NN, \ubf(k) \equiv g(\phi(k,s,\ubf))
\end{equation}
First of all, we prove the key property of our control law that is the safety guarantee. More precisely, starting from a brake safe state in $B_G$, the system under control law $g^G$ is ensured to remain in $B_G$, i.e., priorities $G$ are preserved and the system is always in a brake safe state. Following the terminology of~\cite{Kerrigan2000}, $B_G$ is positively invariant under control law $g^G$ as stated in the following theorem:
\begin{theorem}[Priority preservation]
Given a priority graph $G\in\graphs$, the set of brake safe states $B_G$ is positively invariant (in discrete time) for the system under control law $g^G$, i.e.,
\begin{equation}
\forall s\in B_G, \forall k\in\NN, \phi(k,s,g^G) \in B_G
\label{eq:invariance-B-G}
\end{equation}
Moreover, the configuration of the system remains in $\chifree_G$ through time, i.e.,
\begin{equation}
\forall s\in B_G, \forall t\geq 0, \phi_x(t,s,g^G)\in \chifree_G
\label{eq:no-collision-during-time-slots}
\end{equation}
\label{thm-safe-control-map}
\end{theorem}
The above theorem asserts that under control law $g^G$, provided the system starts in a brake safe state, the sequence of future states at the beginning of each time slot is a sequence of brake safe states (see Equation~\eqref{eq:invariance-B-G}). Moreover, the flow of the system remains in $\chifree_G$ in continuous time (see Equation~\eqref{eq:no-collision-during-time-slots}), i.e., no collision occurs and priorities are preserved. It is a direct consequence of Theorem~\ref{thm-robustness-brake-application} and appears as a limiting case. 

\subsection{Robustness}

The control law $g_i^G$ returns the maximum control value that robot $i$ can safely apply, but it is in fact always safe to apply a lower control value, including letting all robots brake as much as possible, i.e., leading to an emergency stop. This property stated in Theorem~\ref{thm-robustness-brake-application} below is very valuable because for applications in intelligent transportation systems, even without considering extreme situations such as emergency stops, it is very usual that a vehicle needs to brake because of an unpredictable event such as a pedestrian crossing the road, or a loss of sensing/communication abilities.

\begin{theorem}[A broad class of priority preserving controls]
Given a priority graph $G\in\graphs$, an initial condition $s\in B_G$, and a control $\ubf \in \controls$ that satisfies: 
\begin{equation}
\forall k\in\NN, \ubf(k) \leq g^G(\phi(k,s,\ubf))
\label{eq-inequality-control-map}
\end{equation}
The set of brake safe states $B_G$ is positively invariant (in discrete time), i.e.,
\begin{equation}
\forall k\in\NN, \phi(k,s,\ubf) \in B_G
\end{equation}
Moreover, the configuration of the system remains in $\chifree_G$ through time, i.e.,
\begin{equation}
\forall t\geq 0, \phi_x(t,s,\ubf)\in\chifree_G
\end{equation}
\label{thm-robustness-brake-application}
\end{theorem}

\begin{proof}
Take a priority graph $G\in\graphs$, an initial condition $s\in B_G$ and a control $\ubf\in\controls$ satisfying Equation~\ref{eq-inequality-control-map}. By induction, it is sufficient to prove that the flow is collision-free for $t\in[0,1]$ and the reached state $\phi(1,s,\ubf)$ is brake safe. Now, we prove that the flow of Theorem~\ref{thm-robustness-brake-application} does not intersect $\chiobs_G$ for $t\in[0,1]$. Take arbitrary $t\in[0,1]$: we have to prove that for all $(j,i)\in E(G)$, $\phi_x(t,s,\ubf)\in\chifree_{j \succ i}$. By construction of $g^G$, for each robot $i$, there are two cases:
\begin{itemize}
\item $g_i^G(s)=\uL_i$: in this case,
\begin{equation}
\phi_i(t,s,\ubf)=\phi_i(t,s,\uLbf)
\label{eq:pf-discrete-case1-eq1}
\end{equation}
and by order-preservation, for all robots $j$ such that $(j,i)\in E(G)$ we have: 
\begin{equation}
\phi_j(t,s,\ubf) \geq \phi_j(t,s,\uLbf)
\label{eq:pf-discrete-case1-eq2}
\end{equation}
Since $s$ is brake safe, $\phi_x(t,s,\uLbf)\in\chifree_{j \succ i}$. Hence, by Property~\ref{property:geometric-invariance}, Equations~\eqref{eq:pf-discrete-case1-eq1} and~\eqref{eq:pf-discrete-case1-eq2} ensure that $\phi_x(t,s,\ubf)\in\chifree_{j \succ i}$ as well.
\item $g_i^G(s)=\uH_i$: by construction of the control law, $\phi_x(t,s,\tilde{\ubf}^i) \in \chifree_{G}$. By order-preservation, using $\tilde{\ubf}^i_i(0)=\uH_i$, we obtain:
\begin{equation}
\phi_i(t,s,\tilde{\ubf}^i)=\phi_i(t,s,\uHbf) \geq \phi_i(t,s,\ubf)
\label{eq:pf-discrete-case2-eq1}
\end{equation}
For all robots $j$ such that $(j,i)\in E(G)$, using $\tilde{\ubf}^i_j(0)=\uL_j$, we have:
\begin{equation}
\phi_j(t,s,\tilde{\ubf}^i)=\phi_j(t,s,\uLbf) \leq   \phi_j(t,s,\ubf)
\label{eq:pf-discrete-case2-eq2}
\end{equation}
Since $\phi_x(t,s,\tilde{\ubf}^i) \in \chifree_{G}$, $\phi_x(t,s,\tilde{\ubf}^i)\in\chifree_{j \succ i}$, and by Property~\ref{property:geometric-invariance}, Equations~\eqref{eq:pf-discrete-case2-eq1} and~\eqref{eq:pf-discrete-case2-eq2} ensure that $\phi_x(t,s,\ubf)\in\chifree_{j \succ i}$ as well.
\end{itemize}

As a final step, we prove that the reached state $s^1:=\phi(1,s,\ubf)$ is brake safe. Take arbitrary $t \geq 0$: we have to prove that for all $(j,i)\in E(G)$, $\phi_x(t,s^1,\uLbf)\in\chifree_{j \succ i}$. As previously, there are two cases: 
\begin{itemize}
\item $g_i^G(s)=\uL_i$: then, $s^1_i=\phi_i(1,s,\uLbf)$ and we have: 
\begin{equation}
\phi_i(t,s^1,\uLbf) = \phi_i(1+t,s,\uLbf)
\label{eq:pf-continuous-case1-eq1}
\end{equation}
Moreover, by order-preservation, for all $j$ such that $(j,i)\in E(G)$: $s_j^1 \geq \phi_j(1,s,\uLbf)$. As a result, by order-preservation:
\begin{equation}
\phi_j(t,s^1,\uLbf) \geq \phi_j(1+t,s,\uLbf)
\label{eq:pf-continuous-case1-eq2}
\end{equation}
Since $s$ is brake safe, $\phi_x(1+t,s,\uLbf)\in\chifree_{j \succ i}$. Hence, by Property~\ref{property:geometric-invariance}, Equations~\eqref{eq:pf-continuous-case1-eq1} and~\eqref{eq:pf-continuous-case1-eq2} ensure that $\phi_x(t,s^1,\uLbf)\in\chifree_{j \succ i}$ as well.

\item $g_i^G(s)=\uH_i$: then, by construction of the control law, $\phi_x(1+t,s,\tilde{\ubf}^i) \in \chifree_G$. Define $\tilde s^1:=\phi(1,s,\tilde{\ubf}^i)$. We have $\tilde{\ubf}^i(1+\tau)=\uL$ for $\tau\geq 0$. As a result, $\phi(1+t,s,\tilde{\ubf}^i)=\phi(t,\tilde{s}^1,\uLbf)$. Since $\phi_x(1+t,s,\tilde{\ubf}^i) \in \chifree_G$,  $\phi_x(t,\tilde s^1,\uLbf)\in \chifree_G$.

By order-preservation, using $\tilde{\ubf}^i_i(0)=\uH_i$, we obtain:
\begin{equation}
\tilde s^1_i = \phi_i(1,s,\tilde{\ubf}^i) = \phi_i(1,s,\uHbf)  \geq   \phi_i(1,s,\ubf) = s^1_i
\label{eq:pf-continuous-case2-eq1-intermediate-equation}
\end{equation}
For all robots $j$ such that $(j,i)\in E(G)$, using $\tilde{\ubf}^i_j(0)=\uL_j$, we have:
\begin{equation}
\tilde s^1_j = \phi_j(1,s,\tilde{\ubf}^i) = \phi_j(1,s,\uLbf)   \leq \phi_j(1,s,\ubf) = s^1_j
\label{eq:pf-continuous-case2-eq2-intermediate-equation}
\end{equation}
Hence, by order-preservation, Equations~\eqref{eq:pf-continuous-case2-eq1-intermediate-equation} and~\eqref{eq:pf-continuous-case2-eq2-intermediate-equation} imply:
\begin{eqnarray}
\phi_i(t,\tilde{s}^1,\uLbf) & \geq & \phi_i(t,s^1,\uLbf)\label{eq:pf-continuous-case2-eq1}\\
\phi_j(t,\tilde{s}^1,\uLbf) & \leq & \phi_j(t,s^1,\uLbf)\label{eq:pf-continuous-case2-eq2}
\end{eqnarray}
Since $\phi_x(t,\tilde s^1,\uLbf)\in \chifree_G$, $\phi_x(t,\tilde s^1,\uLbf)\in \chifree_{j \succ i}$, and by Property~\ref{property:geometric-invariance}, Equations~\eqref{eq:pf-continuous-case2-eq1} and~\eqref{eq:pf-continuous-case2-eq2} ensure that $\phi_x(t,s^1,\uLbf)\in \chifree_{j \succ i}$ as well.
\end{itemize}
\end{proof}

To illustrate the interest of Theorem~\ref{thm-robustness-brake-application}, given priorities $G$ and an initial condition $s\in B_G$ consider the two examples below.

\begin{example}[Individual brake application]
Consider a control $\ubf \in\controls$ satisfying: 
\begin{eqnarray}
\forall k\in\NN,  \ubf_i(k)&=&\begin{cases}
\uL_i & \text{ if } k\in K \\
g_i^G(\phi(k,s,\ubf)) & \text{ else.}
\end{cases} \label{eq-example-control}\\
\forall j\in\robots, j\neq i, \ubf_j(k)&=&g_j^G(\phi(k,s,\ubf))
\end{eqnarray}
$i\in\robots$ is a particular robot and $K \subset \NN$ is a subset of slots. Under the control described above, the system is perfectly controlled by the control law, except during slots $K$ where the particular robot $i$ brakes while other robots are still perfectly controlled by the control law. Such a scenario may arise, for instance, in case of a momentary communication/sensing failure for one robot: if the current state is not available, the control law cannot be applied, and a brake maneuver is performed instead. The condition of Theorem~\ref{thm-robustness-brake-application} is clearly respected since for $j\neq i$, $\ubf_j(k)=g_j^G(\phi(k,s,\ubf))\leq g_j^G(\phi(k,s,\ubf))$, and  $\ubf_i(k)=g_i^G(\phi(k,s,\ubf))\leq g_i^G(\phi(k,s,\ubf))$ or $\ubf_i(k)=\uL_i \leq g_i^G(\phi(k,s,\ubf))$. Hence, the flow $t \mapsto \phi(t,s,\ubf)$ is collision-free and preserves priorities $G$. This illustrates that the control law is robust with regards to an individual brake application of a particular robot for an arbitrary long time, yielding a deviated but still collision-free flow respecting the assigned priorities.
\end{example}

\begin{example}[Simultaneous brake application]
Consider a control $\ubf \in\controls$ satisfying: 
\begin{equation}
\forall k\in\NN,  \ubf(k)=\begin{cases}
\uL & \text{ if } k\in K \\
g^G(\phi(k,s,\ubf)) & \text{ else.}
\end{cases}
\end{equation}
Again, $K \subset \NN$ is a subset of slots. Under the control described above, the system is perfectly controlled by the control law, except during slots $K$ where all robots brake simultaneously. It may arise in case of a global failure requiring an emergency brake to be performed. Again, the condition of Theorem~\ref{thm-robustness-brake-application} is clearly respected since $\ubf(k)=g^G(\phi(k,s,\ubf)) \leq g^G(\phi(k,s,\ubf))$ or $\ubf(k)=\uL \leq g^G(\phi(k,s,\ubf))$. It illustrates that the control law is robust with regards to a simultaneous brake application of all robots for an arbitrary long time, yielding again a deviated but still collision-free flow respecting the assigned priorities.
\end{example}

\subsection{Liveness}

As in the case of velocity control, we aim at guaranteeing liveness, i.e., the guarantee that every robot $i\in\robots$ eventually reaches the region $\chigoal:=\xobsmax+\RR_+^n$. 

\begin{theorem}[Liveness]
Given an acyclic priority graph $G$ and an initial brake safe state $s\in B_G$, there exists $T>0$ such that:
\begin{equation}
\phi_x(T,s,g^G)\in\chigoal
\end{equation}
\label{thm:liveness-control-law}
\end{theorem}
Again, the idea of the proof is that under acyclic priorities, there is always a non exited robot able to travel at maximum throttle command until it exits the intersection.
\begin{proof}
Take an acyclic priority graph $G$. Consider the trajectory of the robots under control law $g^G$. $G$ being acyclic, there exists an extremal vertex $i_1\in\robots$ such that for all $j\in\robots$, $(j,i_1)\notin E(G)$. As a result, under the control law $g^G$, robot $i_1$ will always accelerate as much as possible and it will exit the intersection (it will reach position $\xobsmax_i$) in finite time $T_1$.

Now, assume that at time $T_m$, robots $i_1 \cdots i_m$ have exited the intersection and $m<n$ (there remain some robots). $G$ being acyclic, there exists an extremal element for the remaining robots denoted $i_{m+1}\in\robots\setminus \{i_1 \cdots i_m\}$ such that for all $j\in\robots\setminus \{i_1 \cdots i_m\}$, $(j,i_{m+1})\notin E(G)$. Collisions occurring only with non exited robots, for $t\geq T_m$ $j$ will always accelerate and it will exit the intersection in finite time at instant $T_{m+1}\geq T_m$. 

Iterating this process yields a sequence $(T_1 \cdots T_n)$ and all robots have exited the intersection at time $T:=T_n$.
\end{proof}

\chapter[Robustness with respect to bounded noise]{Robustness with respect\\ to bounded noise}
\label{chap:control-uncertainty}
\minitoc

In the plan-as-program approach, a low-level controller is assumed to be able to follow the planned trajectory. Uncertainty is taken into account at the control phase. This is known as the trajectory tracking problem. Many trajectory tracking systems have been proposed for different robot dynamics models~\cite{Jiang1997,Lee2001,Micaelli1993,Soetanto2003,Yang1999}. In~\cite{VanDenBerg2011}, a linearisation of the robot dynamics model around the tracked trajectory enables to obtain a linear-quadratic regulator~\cite{Todorov2006}, and under Gaussian models of uncertainty, the a-priori distribution of the trajectory around the tracked trajectory can be computed. It makes possible to compare several possible motion planning strategies in terms of collision probability, and to select one of them based on some criteria/cost function. However, even when a-priori knowledge on uncertainty is used as in~~\cite{VanDenBerg2011} to plan the trajectory, the trajectory tracking approach is still quite decoupled as the reference trajectory remains unchanged as new information comes in during the execution of the plan. It can result in undesirable behaviors particularly in case of large deviation from the reference trajectory.

In priority-based coordination, there is no reference trajectory to track. There are assigned priorities to preserve which is ensured by a control law configured by the assigned priorities (see previous chapters of the present part). In this setting, information on uncertainty can be used as an additional resource to take into account when acting, i.e., as an additional input for the control law. In~\cite{LaValle1996-uncertainty,LaValle1998-uncertainty-objective-based}, the information space approach is proposed. The information state at time $t$ contains all the information history up to date $t$. Under probabilistic uncertainty, the current information state can be considered as the distribution of the current state of the system conditionally to current history. Under non deterministic uncertainty, the current information state is the set of all possible "true" current states of the robot: one can see the non-deterministic information state as a "bubble" of possible current "true" states. In this approach, the action of robots is a function of the current information state: the control law takes into account the uncertainty on the estimated current state to decide action. The approach has since become standard (see, e.g., \cite{Petti2005-reactive-planning, Ghaemi2014}). This chapter espouses the information space approach.

Priority preserving control under bounded noise in the coordination space is considered. As in Chapter~\ref{chap:control-acceleration}, priorities guide the action of robots by configuring the control law. However, in this chapter, the control law does not take the current state of the robot as input as it is not known. We take a model of bounded uncertainty which enables to apply the non-deterministic information space approach~\cite{LaValle1996-uncertainty,LaValle1998-uncertainty-objective-based}. The control law takes as input the non-deterministic information state of the robots, i.e., the set of all possible current positions and velocities of robots along their paths. Uncertainty is only considered from the coordination space point of view. Uncertainty on the path following assumption (lateral control) and more generally realistic models of uncertainty based on real sensors/actuators models are beyond the ambition of the present thesis (it is mentioned as a perspective in the concluding part). The present chapter only aims at providing elements demonstrating the robustness abilities of the proposed priority preserving approach in the presence of bounded noise.

\paragraph{Sketch of the chapter} Section~\ref{sec:control-model-uncertainty} exposes the second-order dynamics model with bounded noise. Section~\ref{sec:evolution-information-state} defines the so-called non-deterministic information state for our particular multi robot system and provides the equations describing its evolution through time. The last section of the chapter builds a priority preserving control law taking into account uncertainty information by considering the worst-case scenario. It is guaranteed that for all possible errors/perturbations in sensing/control, no collision occurs, priorities are respected and all robots eventually go through the intersection. Additionally, the brake safety property of Chapter~\ref{chap:control-acceleration} stating that robots may safely brake at any point of time without violating priorities still holds.

\section{Control model with bounded noise}
\label{sec:control-model-uncertainty}

We slightly modify the model of Chapter~\ref{chap:control-acceleration} to account for bounded control noise. Each robot $i$ is modeled as a second-order control system with state $s_i=(x_i,v_i)\in S_i:=\RR \times [0,\vmax_i]$, whose evolution is described by the differential equation:
\begin{eqnarray}
\dot{x_i}(t) &= & v_i(t)+\mathbf{1}_{v_i(t)=\vmax_i} \dbf^v_i(t)
\label{eq-diff-state1-non-deterministic}
\\
\dot{v_i}(t) & =& (\ubf_i(t)+\dbf^u_i(t)) ~\delta(\ubf_i(t)+\dbf^u_i(t),v_i(t))
\label{eq-diff-state2-non-deterministic}
\end{eqnarray}
with the same notations as in Chapter~\ref{chap:control-acceleration} and with $\mathbf{1}_C$ returning $1$ if condition $C$ holds, $0$ else. Basically, $\dbf^v$ models the uncertainty on maintaining maximum velocity and $\dbf^u$ models the uncertainty on the brake command. $\dbf=(\dbf^v,\dbf^u)$ is the overall exogenous control uncertainty signal. We assume that control uncertainty is bounded and we let $D:=\prod_{i\in\robots} D_i$ with $D_i:=[\dmin_i,\dmax_i]$ and $\dmin_i\in\RR_-^2$ and $\dmax_i\in\RR_+^2$. We let $\dsignals_i$ denote the set of uncertainty controls $\dbf_i$ taking values in $D_i$ and $\dsignals:=\prod_{i\in\robots} \dsignals_i$. We let $t \mapsto \phi_i(t,s_i,\ubf_i,\dbf_i)$ denote the flow of the system starting at initial condition $s_i\in S_i$  with control $\ubf_i \in \controls_i$ and uncertainty control $\dbf_i$. As in Chapter~\ref{chap:control-acceleration}, projected flows are defined as follows: $\phi_x:=\pi_x\circ\phi$ and $\phi_v:=\pi_v\circ\phi$. We introduce a partial order for uncertainty signals as follows:
\begin{equation}
\forall \dbf_i^1,\dbf_i^2\in \dsignals_i, \dbf_i^1 \preceq  \dbf_i^2 \text{ if } \forall t\geq 0, \dbf_i^1(t) \leq \dbf_i^2 (t) 
\end{equation}

\begin{property}[Order preservation]
The flow $t \mapsto \phi_i(t,s_i,\ubf_i,\dbf_i)$ is order-preserving with regards to $s_i$,  $\ubf_i$ and $\dbf_i$.
\end{property}

Finally, we make the following assumptions for all $i\in\robots$:
\begin{eqnarray}
\dmin_i^v+\vmax_i&>&0 \label{eq:assumption-uncertainty-velocity}\\
\uH_i+\dmin_i^u&>&0\label{eq:assumption-uncertainty-acceleration}\\
\uL_i+\dmax_i^u&<&0\label{eq:assumption-uncertainty-brake}
\end{eqnarray}
Basically, it means that:
\begin{itemize}
\item even with uncertainty on maintaining maximum velocity, the velocity is always positive;
\item even with uncertainty on control, when a robot applies maximum throttle command, it will effectively accelerate;
\item and when a robot applies maximum brake command, it will effectively brake.
\end{itemize}

\section{Evolution of the non-deterministic information state}
\label{sec:evolution-information-state}

We let $2^A$ denote the power set of any set $A$. We assume that we have observations at the beginning of every time slot. We model observations as a signal $\ybf:\RR_+^* \to 2^S$ satisfying $\ybf(t)=S$ if $t\notin\NN$, and $\ybf(k)$ is a parallelepiped:
\begin{equation}
\forall k\in\NN^*_+, \ybf(k)=\prod_{i\in\robots}\ybf_i(k)=\prod_{i\in\robots} \ybf^x_i(k) \times \ybf^v_i(k)
\end{equation}
$\ybf^x_i(k)$ denotes the observation on the position of robot $i$ at time slot $k$ and $\ybf^v_i(k)$ the observation on the velocity of robot $i$ at time slot $k$. An observation $\ybf(k)$ provides a set of possible (true) states given the sensors information. $\ybf(t)=S$ if $t\notin\NN$ means that there is no observation data at time $t\notin\NN$. We let $\Ybf_i$ denote the set of observation signals $\ybf_i$ satisfying the above assumptions and $\Ybf:=\prod_{i\in\robots} \Ybf_i$.

We remind that the non-deterministic information state at time $t$ provides the set of possible (true) states at time $t$: if the current non-deterministic information state is $\hat s\in 2^S$, the current (true) state $s\in S$ satisfies $s\in\hat s$. The evolution of the non-deterministic information state accounts for both the uncertainty on control and on sensing. Given $k\in\NN$, for $t\in (k,k+1)$, the uncertainty on control (through $\dbf$) increases the size of possible states as time goes by. At time $k+1$, a new observation is available, and the new state of the system necessarily belongs to the range given by the observation. 

The above statements lead to the non-deterministic information state flow $t\mapsto \hat\phi(t,\hat{s},\ubf,\ybf)$ associated to an initial condition $\hat s$, a control $\ubf$ and an observation signal $\ybf\in\Ybf$ defined as follows.:
\begin{equation}
\hat \phi(0,\hat{s},\ubf,\ybf):=\hat{s}
\end{equation}
\begin{equation}
\forall k\in\NN, \forall t\in(0,1),
\hat\phi(k+t,\hat{s},\ubf,\ybf):=\left\{\phi(t,s,\ubf,\dbf):\dbf\in\dsignals, s\in \hat\phi(k,\hat{s},\ubf,\ybf)\right\} 
\end{equation}
\begin{equation}
\forall k\in\NN,
\hat\phi(k+1,\hat{s},\ubf,\ybf):=\left\{\phi(1,s,\ubf,\dbf):\dbf\in\dsignals, s\in \hat\phi(k,\hat{s},\ubf,\ybf)\right\} \cap \ybf(k+1) \label{eq:non-deterministic-flow-intersection-with-observation}
\end{equation}
The evolution of the non-deterministic information state flow is illustrated in Figure~\ref{fig:evolution-non-deterministic-information-state}. By order-preservation, as $D$ and $y(k)$ are parallelepipeds, and as the intersection of two parallelepipeds is a parallelepiped, it is clear that starting from an initially parallelepipedic non-deterministic information state $\hat{s}$, $\hat\phi(t,\hat{s},\ubf,\ybf)$ is a parallelepiped at any point of time.  Note that it makes its computation easier as only two extremal points need to be computed. As for deterministic flows, projected flows are defined as follows: $\hat\phi_x:=\pi_x\circ\hat\phi$ and $\hat\phi_v:=\pi_v\circ\hat\phi$.
\begin{figure}[!htbp]
\begin{center}
\includegraphics[width=1.0\linewidth]{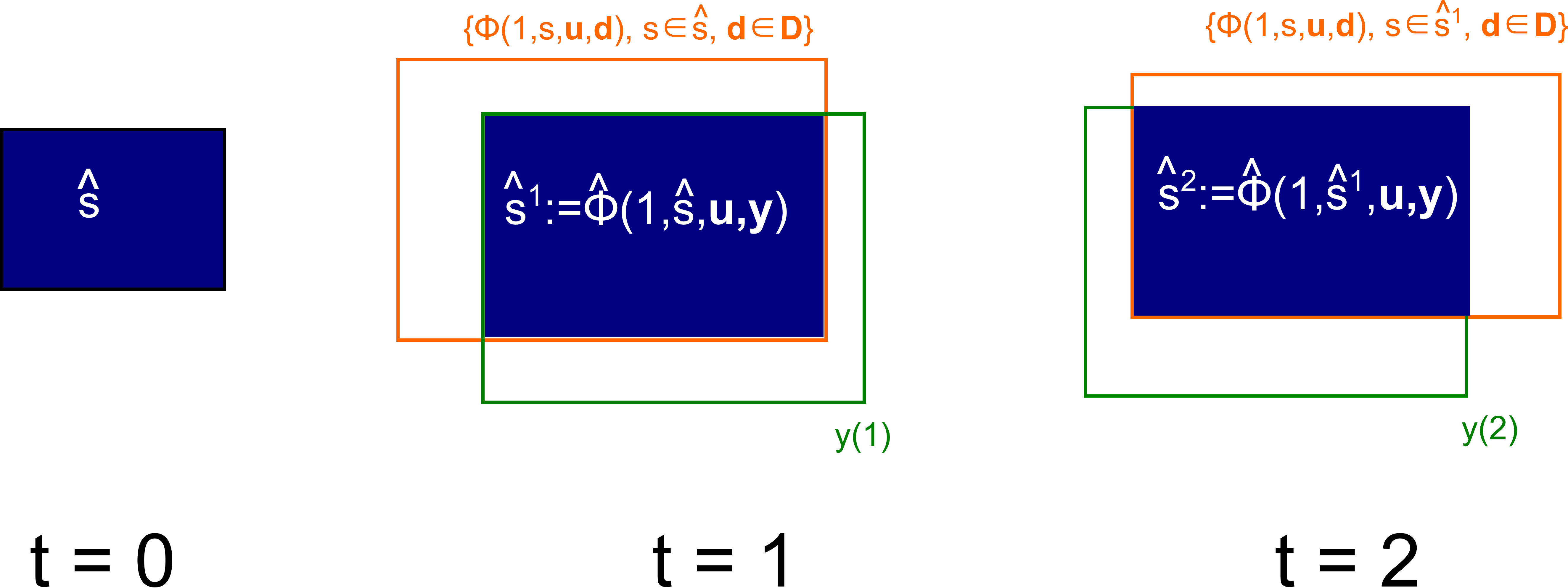}\hfill
\end{center}
\caption{Evolution of the non-deterministic information state. For $t\in(0,1)$, the "size" of the non-deterministic information state grows until a new observation is available at $t=1$. The intersection with $y(1)$ at $t=1$ enables to take into account the new observation and reduces the "size" of the updated non-deterministic information state $\hat s^1$.}
\label{fig:evolution-non-deterministic-information-state}
\end{figure}

We let $\mathbf{S}$ denote the constant observation signal $\ybf(t)\equiv S$, i.e., there is no observation data for all $t\geq 0$. We define the set of brake safe non-deterministic states $\hat B_G$ as follows:
\begin{eqnarray}
\hat B_G &:=& \left\{ \hat s \in 2^\chi: \forall t\geq 0, \hat \phi_x(t, \hat s, \uLbf, \mathbf{S}) \in 2^{\chifree_G} \right\}\\
&=&\left\{ \hat s \in 2^\chi: \forall s\in\hat s, \forall \dbf\in\dsignals, \forall t\geq 0, \phi_x(t, s, \uLbf, \dbf) \in \chifree_G \right\}
\end{eqnarray}

A non-deterministic state $\hat s\in 2^\chi$ is brake safe if starting from any true state $s\in\hat s$, the flow under maximum brake command $\uLbf$ is collision-free for all possible control uncertainty signals $\dbf\in\dsignals$.

\section{The proposed decentralized control law}

We are going to build a control law for the non-deterministic system. The control law maps the current information state of the system to the control to apply, i.e., it is a map $\hat g: 2^S \to U$. We let $t\mapsto \hat\phi(t,\hat{s},\hat g,\ybf)$ denote the flow $t\mapsto \hat\phi(t,\hat{s},\ubf,\ybf)$ where $\ubf$ satisfies:
\begin{equation}
\forall k\in\NN, \ubf(k)\equiv \hat{g}(\hat\phi(k,\hat{s},\ubf,\ybf))
\end{equation}

For all $i\in\robots$, we define the operator $\supi$ and the uncertainty signal $\tilde{\dbf}^i$ as follows: 

\begin{eqnarray}
\forall j\in\robots,  \supi_j(\hat s)&:=&\begin{cases}
\sup \hat s_i& \text{if } j=i\\
\inf \hat s_j & \text{if } j\neq i
\end{cases}\\
\forall j\in\robots,   \tilde{\dbf}^i_j&:=&\begin{cases}
\dbfmax_i & \text{if } j=i\\
\dbfmin_j & \text{if } j\neq i
\end{cases}
\end{eqnarray}
Basically, $\supi(\hat s)$ represents the worst-case possible true state of the system for collisions with regards to $\chiobs_{j\succ i}$. Similarly, $\tilde{\dbf}^i_j$ is the worst-case possible disturbance with regards to $\chiobs_{j\succ i}$. The rationale of the proposed control law is to apply the same control law as without uncertainty but considering the worst-case scenario (both worst-case disturbance and worst-case true state). The control law can be formulated synthetically as follows:
\begin{equation}
\hat g_i^G(\hat s):=\begin{cases}
\uL_i & \text{if } \exists (j,i)\in E(G), \exists t\geq  0 \text{ s.t. } \phi_x(t,\supi(\hat{s}),\tilde{\ubf}^i, \tilde{\dbf}^i) \in \chiobs_{j\succ i} \\
\uH_i & \text{ else.}
\end{cases}
\label{eq-control-map-uncertainty}
\end{equation}

\subsection{Priority preservation}

As previously, we first focus on the most important property: safety, i.e., priority preservation.

\begin{theorem}[Priority preservation]
Given a priority graph $G\in\graphs$, the set of brake safe non-deterministic information states $\hat B_G$ is positively invariant (in discrete time) for the non-deterministic system under control law $\hat g^G$, i.e.,
\begin{equation}
\forall \ybf\in\Ybf, \forall \hat s\in \hat B_G, \forall k\in\NN, \hat\phi(k,\hat{s},\hat g^G,\ybf) \in \hat B_G
\label{eq:invariance-B-G-uncertainty}
\end{equation}
Moreover, the non-deterministic system remains in $2^{\chifree_G}$ through time, i.e.,
\begin{equation}
\forall \ybf\in\Ybf, \forall \hat s\in \hat B_G, \forall t\geq 0, \hat\phi_x(t,\hat{s},\hat g^G,\ybf)\in 2^{\chifree_G}
\label{eq:no-collision-during-time-slots-uncertainty}
\end{equation}
\label{thm-safe-control-map-uncertainty}
\end{theorem}
As in Chapter~\ref{chap:control-acceleration}, the above theorem is a limit case of Theorem~\ref{thm-robustness-brake-application-uncertainty} proved in the sequel.
\begin{figure}[!htbp]
\begin{center}
\includegraphics[width=1.0\linewidth]{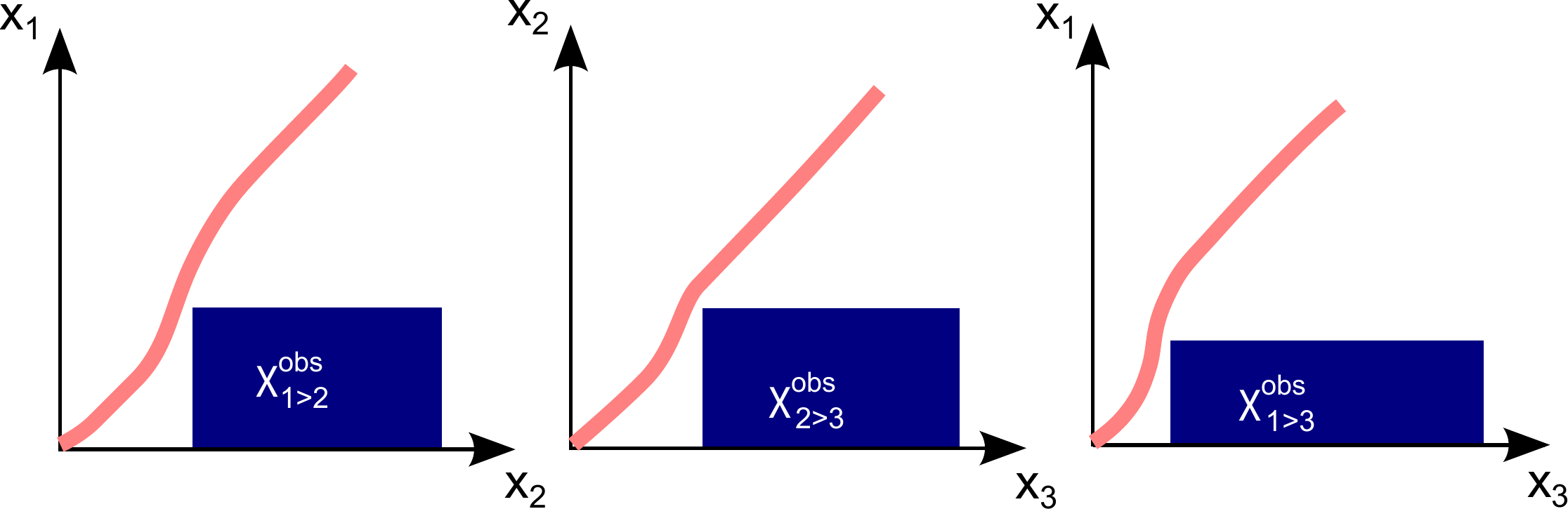}\hfill
\end{center}
\caption{Look of the trajectory for a three-robot system with acyclic assigned priorities $1\succ 2$, $2\succ 3$ and $1\succ 3$ under control law $g^G$. A band is used to represent the set of all possible real configurations through time.}
\label{fig:trajectory-acceleration-control-example-with-uncertainty}
\end{figure}

\subsection{Robustness}

As in Chapter~\ref{chap:control-acceleration}, the control law $\hat g_i^G$ returns the maximum control value that robot $i$ can safely apply, but it is in fact always safe to apply a lower control value. Hence, we obtain the same robustness property that is highly valuable for applications in autonomous cars.

\begin{theorem}[A broad class of priority preserving controls]
Given a priority graph $G\in\graphs$, an initial condition $\hat s\in \hat B_G$, an observation signal $\ybf\in\Ybf$ and a control $\ubf \in \controls$ that satisfies: 
\begin{equation}
\forall k\in\NN, \ubf(k) \leq \hat g^G(\hat \phi(k,\hat s,\ubf,\ybf))
\label{eq-inequality-control-map-uncertainty}
\end{equation}
The set of non-deterministic brake safe states $\hat B_G$ is positively invariant (in discrete time), i.e.,
\begin{equation}
\forall k\in\NN, \hat \phi(k,\hat s,\ubf,\ybf) \in \hat B_G
\end{equation}
Moreover, the configuration of the non-deterministic system remains in $2^{\chifree_G}$ through time, i.e.,
\begin{equation}
\forall t\geq 0, \hat \phi_x(t,\hat s,\ubf,\ybf)\in 2^{\chifree_G}
\end{equation}
\label{thm-robustness-brake-application-uncertainty}
\end{theorem}
\begin{proof}
Take a priority graph $G\in\graphs$, an initial condition $\hat s\in \hat B_G$ and a control $\ubf\in\controls$ satisfying Equation~\eqref{eq-inequality-control-map-uncertainty}. By induction, it is sufficient to prove that the flow remains in $2^{\chifree_G}$ for $t\in[0,1]$ and the reached state $\hat \phi(1,\hat s,\ubf,\ybf) \in \hat B_G$. First, we prove that the flow of Theorem~\ref{thm-robustness-brake-application-uncertainty} does not intersect $2^{\chiobs_G}$ for $t\in[0,1]$. Take arbitrary $t\in[0,1]$ and $(j,i)\in E(G)$: we have to prove that $\hat\phi_x(t,\hat s,\ubf,\ybf)\in 2^{\chifree_{j \succ i}}$. By construction of $\hat \phi$, it is equivalent to prove that for all $\dbf\in\dsignals$ and $s\in\hat s$, $\phi_x(t,s,\ubf,\dbf)\in \chifree_{j \succ i}$. By construction of $\hat g^G$, there are two cases:
\begin{itemize}
\item $\hat g_i^G(\hat s)=\uL_i$: in this case,
\begin{equation}
\phi_i(t,s,\ubf, \dbf)=\phi_i(t, s,\uLbf, \dbf) 
\label{eq:pf-discrete-case1-eq1-uncertainty}
\end{equation}
and by order-preservation: 
\begin{equation}
\phi_j(t, s,\ubf,\dbf) \geq \phi_j(t,s,\uLbf, \dbf) 
\label{eq:pf-discrete-case1-eq2-uncertainty}
\end{equation}
Since $\hat s\in \hat B_G$, $\phi_x(t, s,\uLbf,\dbf)\in \chifree_{j \succ i}$. Hence, by Property~\ref{property:geometric-invariance}, Equations~\eqref{eq:pf-discrete-case1-eq1-uncertainty} and~\eqref{eq:pf-discrete-case1-eq2-uncertainty} ensure that $\phi_x(t, s,\ubf,\dbf)\in \chifree_{j \succ i}$ as well.
\item $g_i^G(s)=\uH_i$: by construction of the control law, $\phi_x(t,\supi(\hat s),\tilde{\ubf}^i, \tilde{\dbf}^i) \in \chifree_{G}$. By order-preservation, using $\tilde{\ubf}^i_i(0)=\uH_i$, $\tilde{\dbf}^i_i=\dbfmax_i$ and $\supi_i(\hat s)=\sup \hat s_i$, we obtain:
\begin{equation}
\phi_i(t,\supi(\hat s),\tilde{\ubf}^i, \tilde{\dbf}^i)=\phi_i(t,\sup \hat s, \uHbf, \dbfmax) \geq \phi_i(t,s, \ubf, \dbf)
\label{eq:pf-discrete-case2-eq1-uncertainty}
\end{equation}
Again, by order-preservation, using $\tilde{\ubf}^i_j(0)=\uL_j$, $\tilde{\dbf}^i_j=\dbfmin_j$ and $\supi_j(\hat s)=\inf \hat s_j$, we have:
\begin{equation}
\phi_j(t,\supi(\hat s),\tilde{\ubf}^i, \tilde{\dbf}^i)=\phi_j(t,\inf \hat s, \uLbf, \dbfmin) \leq \phi_j(t,s, \ubf, \dbf)
\label{eq:pf-discrete-case2-eq2-uncertainty}
\end{equation}
Since $\phi_x(t,\supi(\hat s),\tilde{\ubf}^i, \tilde{\dbf}^i) \in \chifree_{j \succ i}$, by Property~\ref{property:geometric-invariance}, Equations~\eqref{eq:pf-discrete-case2-eq1-uncertainty} and~\eqref{eq:pf-discrete-case2-eq2-uncertainty} ensure that $\phi_x(t,s,\ubf, \dbf)\in\chifree_{j \succ i}$.
\end{itemize}

As a final step, we  prove that the reached state $\hat \phi(1,\hat s,\ubf,\ybf)$ is brake safe. It is sufficient to prove that:
\begin{equation}
\hat s^1:=\left\{\phi(1,s,\ubf,\dbf):\dbf\in\dsignals, s\in \hat s\right\} \in \hat B_G
\end{equation}
$\hat \phi(1,\hat s,\ubf,\ybf)$ is indeed a subset of $\hat s^1$ by construction of $\hat \phi$ (see Equation~\eqref{eq:non-deterministic-flow-intersection-with-observation}). Take arbitrary $t \geq 0$ and $(j,i)\in E(G)$: we have to prove that for all $s^1\in\hat s^1$ and  $\dbf^1\in\dsignals$, we have $\phi_x(t,s^1,\uLbf, \dbf^1)\in \chifree_{j \succ i}$. 

Take arbitrary $s^1\in\hat s^1$ and $\dbf^1\in\dsignals$. We have $s^1=\phi(1,s,\ubf,\dbf)$ with $\dbf\in\dsignals$ and $s\in \hat s$. Consider $\dbf^2\in\dsignals$ and $\ubf^2\in\controls$ satisfying $\dbf^2(0)=\dbf(0)$, $\ubf^2(0)=\ubf(0)$ and for all $k\in\NN$, $\dbf^2(k+1)=\dbf^1(k)$ and $\ubf^2(k+1)=\uL$.  By construction, we have: $\phi(t,s^1,\uLbf,\dbf^1)=\phi(1+t,s,\ubf^2,\dbf^2)$.

As a result, we have to prove that $\phi_x(1+t,s,\ubf^2, \dbf^2)\in \chifree_{j \succ i}$. As previously, there are two cases: 
\begin{itemize}
\item $g_i^G(s)=\uL_i$: then, we have $\ubf^2_i=\uLbf_i$, so that:
\begin{equation}
\phi_i(1+t,s,\ubf^2, \dbf^2) = \phi_i(1+t,s,\uLbf, \dbf^2)
\label{eq:pf-continuous-case1-eq1-uncertainty}
\end{equation}
Moreover, by order-preservation, we have:
\begin{equation}
\phi_j(1+t,s,\ubf^2, \dbf^2) \geq \phi_j(1+t,s,\uLbf, \dbf^2)
\label{eq:pf-continuous-case1-eq2-uncertainty}
\end{equation}
Since $s$ is brake safe, $\phi_x(1+t,s,\uLbf,\dbf^2)\in\chifree_{j \succ i}$. Hence, by Property~\ref{property:geometric-invariance}, Equations~\eqref{eq:pf-continuous-case1-eq1-uncertainty} and~\eqref{eq:pf-continuous-case1-eq2-uncertainty} ensure that $\phi_x(1+t,s,\ubf^2, \dbf^2)\in\chifree_{j \succ i}$ as well.

\item $g_i^G(s)=\uH_i$: then by construction of the control law, $\phi_x(1+t,\supi(\hat{s}),\tilde{\ubf}^i, \tilde{\dbf}^i) \in \chifree_G$. Using $\ubf^2_i \leq \tilde{\ubf}^i_i$, $\supi_i(\hat{s})=\sup\hat s_i$ and  $\tilde{\dbf}^i_i=\dbfmax_i$, by order-preservation, we have:
\begin{equation}
\phi_i(1+t,\supi(\hat{s}),\tilde{\ubf}^i, \tilde{\dbf}^i) =  \phi_i(1+t,\sup \hat s,\tilde{\ubf}^i, \dbfmax) \geq \phi_i(1+t,s,\ubf^2, \dbf^2)
\label{eq:pf-continuous-case2-eq1-uncertainty}
\end{equation}
Moreover, by order preservation, using $\tilde{\ubf}^i_j=\uLbf_j$, $\supi_j(\hat{s})=\inf\hat s_j$ and  $\tilde{\dbf}^i_j=\dbfmin_j$, we have:
\begin{equation}
\phi_j(1+t,\supi(\hat{s}),\tilde{\ubf}^i, \tilde{\dbf}^i) =  \phi_j(1+t,\inf \hat s,\uLbf, \dbfmin) \leq \phi_j(1+t,s,\ubf^2, \dbf^2)
\label{eq:pf-continuous-case2-eq2-uncertainty}
\end{equation}
Since $\phi_x(1+t,\supi(\hat{s}),\tilde{\ubf}^i, \tilde{\dbf}^i) \in \chifree_{j \succ i}$, by Property~\ref{property:geometric-invariance}, Equations~\eqref{eq:pf-continuous-case2-eq1-uncertainty} and~\eqref{eq:pf-continuous-case2-eq2-uncertainty} ensure that $\phi_x(1+t,s,\ubf^2, \dbf^2)\in\chifree_{j \succ i}$ as well.
\end{itemize}
\end{proof}

\subsection{Liveness}

Despite uncertainty, the proposed control still ensures all robots will eventually go through the intersection. As in previous chapters, robots are expected to eventually reach the region $\chigoal:=\xobsmax+\RR_+^n$. 

\begin{theorem}[Liveness]
Given an acyclic priority graph $G$, an initial brake safe non-deterministic state $\hat s\in \hat B_G$ and an observation signal $\ybf\in\Ybf$, there exists $T>0$ such that:
\begin{equation}
\hat \phi_x(T,\hat s,\hat g^G, \ybf)\in 2^{\chigoal}
\end{equation}
\label{thm:liveness-control-law-uncertainty}
\end{theorem}

A proof of the above theorem is provided in Appendix~\ref{app:liveness-control-law-uncertainty} (under weaker assumptions).

\chapter*{Conclusions}

{\pagestyle{empty}

Part~\ref{part:priority-framework} suggested to use priorities as a plan to guide robots. In traditional planning, the plan is a reference trajectory to track and the trajectory tracking problem is a well-established problem with many existing solutions consisting in devising a control law configured by the reference trajectory in charge of tracking. However, in priority-based coordination, there is no reference trajectory. The plan is the priority graph and there is no standard tool to control robots under assigned priorities. Devising such tools has been the topic of the present part. The proposed control laws are configured by the assigned priorities, and guarantee priority preservation and liveness (all robots eventually go through the intersection). First of all, as priorities are assigned, the combinatorial complexity of multi robot control (see~\cite{Colombo2012}) is avoided, and computing the output of the control laws proposed in this part is of polynomial complexity. Moreover, in contrast with a trajectory tracking approach, the priority preservation approach retains some freedom of action at the control phase, as there is a large class of trajectories respecting the assigned priorities (instead of only one reference trajectory). In particular, the proposed control law ensures that robots -- only one, several, or even all -- may brake at any time without violating priorities. In Chapter~\ref{chap:optimal-control-velocity}, we have proposed priority preserving control for robots controlled in velocity. The trajectory resulting from the application of the proposed control law is optimal for the assigned priorities, recovering the existence of a left-greedy optimal trajectory in a given homotopy class noticed in~\cite{Ghrist2005}. Chapter~\ref{chap:control-acceleration} demonstrates that the presence of inertia can be handled using the notion of brake safety that merely consists in some kind of anticipation. The proposed control law is decentralized and demonstrates a remarkable robustness regarding unexpected deceleration of robots. The final chapter of the present part has given some elements to take into account uncertainty in priority-based coordination. Under bounded uncertainty, the idea is to consider the worst-case scenario which is well defined when priorities are assigned. Priority preservation and liveness can still be guaranteed as in the deterministic case. This chapter demonstrates the ability of the priority-based approach to handle uncertainty in a reactive manner. For example, if the current uncertainty on the position of robots is very large, under priority preserving control, all robots will brake and eventually stop safely, and will not restart until a sufficiently small uncertainty enables to go safely through the intersection. By contrast, tracking a reference planned trajectory when the uncertainty on position is very large would likely result in collisions. Hence, with a plan execution approach, if the uncertainty becomes very large, the designer should anticipate by providing an emergency maneuver to execute. Then, a new planning phase should be carried out before restarting. With priority-based coordination, such change in uncertainty -- even a complete lost of sensing capabilities -- can be handled in a reactive manner. 
\cleardoublepage}

\part{Priority-based coordination}
\label{part:priority-based-coordination}

\chapter*{Introduction}

\begin{minipage}{\linewidth}
Part~\ref{part:priority-framework} suggested using priorities to guide robots and Part~\ref{part:priorities-to-guide-robots} provided solutions to control robots under assigned priorities. To this point, many aspects of the design of a coordination system at intersections have been left behind. Most importantly, the multi robot coordination system is an open system as robots arrive and exit the intersection through time. Hence, priorities need to be assigned dynamically. Moreover, priority assignment and control under assigned priorities need to be executed in parallel. This part has a more engineering flavor, it specifies the system architecture, how priority assignment and priority preserving control are integrated and how they interface. The proposed approach is inspired from drivers' behavior at signalized intersections. Before entering the intersection, the driver follows the preceding vehicles without colliding, and as long as the traffic signal does not give him/her the right of way, the driver does not go through the intersection. Once the vehicle is given the right of way (green signal), the driver goes through the intersection. However, the driver still retains some reactive abilities and will hopefully not enter the intersection if other vehicles are blocked and/or a pedestrian crosses the road. \parindent2em

In priority-based coordination, the so-called control area is a region of space that robots should not enter unless they have been accepted and assigned a priority with respect to other accepted robots. We adopt a three-layer architecture~\cite{Gat1998}, particularly adapted to the approach considering plans as a resource to guide action. The reactive quality of the system is ensured by a behavior-based layer. Robots' behaviors include 'follow geometric path', 'move forward', 'do not enter the control area', 'respect priorities', 'avoid pedestrians'. The entry of the control area is managed by a central agent, the intersection controller. The intersection controller assigns priorities, yet it does not assign a precise trajectory for the accepted robots. It constitutes the deliberative layer of the system, processing time-consuming tasks reasoning about the future. Finally, robots have a sequencing layer in charge of activating/deactivating/configuring behaviors. The robustness property of the control law ensuring robots may safely brake at any point of time is shown to be of high interest in the proposed architecture. It is indeed possible for, e.g., behavior 'avoid pedestrians' to require the robot to brake to avoid a detected pedestrian, without conflicting with behavior 'respect priorities', as the control law ensures that it is always priority preserving to brake at any point of time (see Theorem~\ref{thm-robustness-brake-application}). Preliminaries of the presented results can be found in our article~\cite{Gregoire2013-priority-based}.

\paragraph{Sketch of the part} Chapter~\ref{chap:coordination-system} describes the system architecture and how priorities may be assigned. Chapter~\ref{chap:simulations} provides simulation results demonstrating safety and robustness of priority-based coordination.
\end{minipage}

{\pagestyle{plain}
\clearpage
\topskip0pt
\vspace*{\fill}
\includegraphics[height=1.0\linewidth,angle=-90,trim=160 60 160 60, clip]{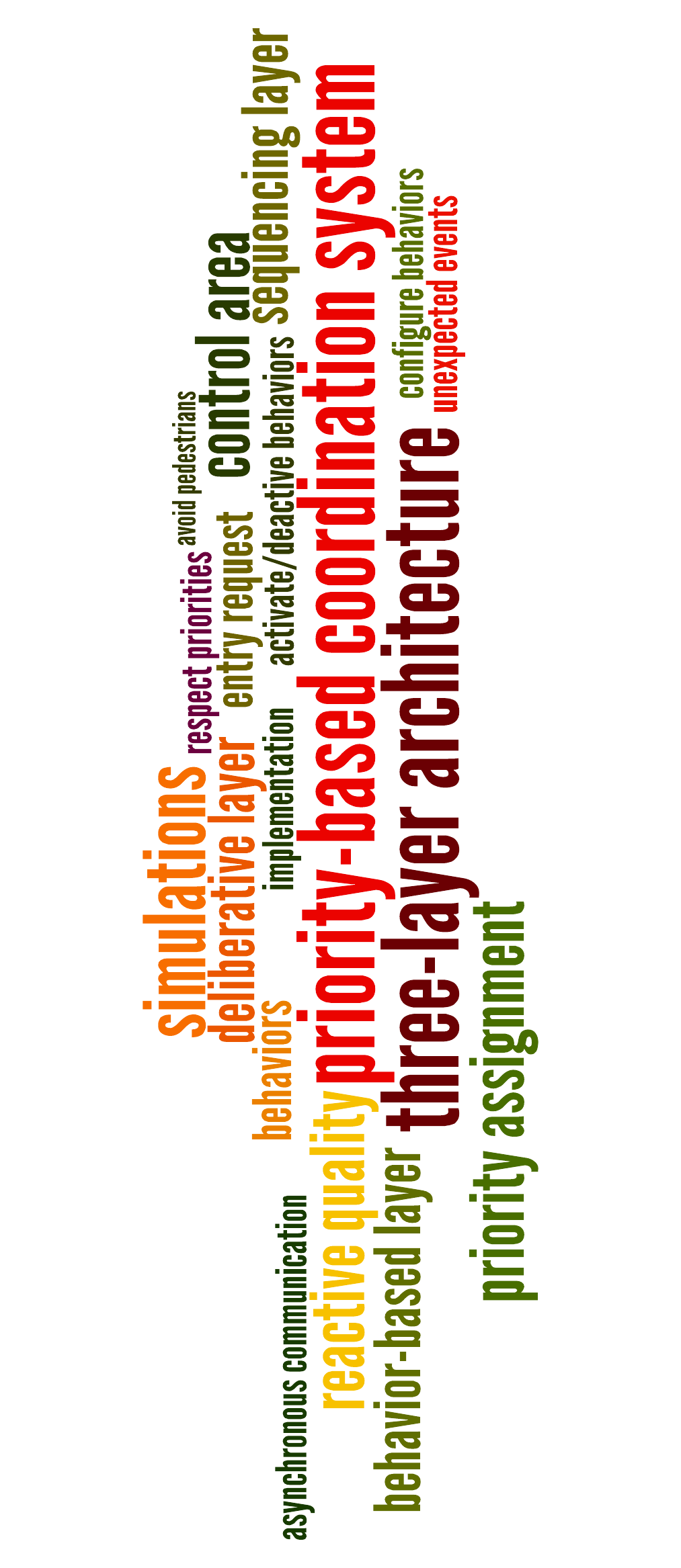}
\vspace*{\fill}
\parttoc}

\chapter[Overall priority-based coordination system]{Overall priority-based\\ coordination system}
\label{chap:coordination-system}
\minitoc

\paragraph{Sketch of the chapter} The first section presents the proposed three-layer architecture and provides details on how layers interact. The second chapter focuses on priority assignment: a simple and easily implementable priority assignment policy is described, and some adaptations in order to guarantee request processing liveness and queues stability are discussed.

\section{Three-layer architecture}

For its ability to design systems with reactive qualities yet retaining planning capabilities, a three-layer architecture is proposed. As noticed in~\cite{Gat1998}, such an architecture organizes control algorithms according to whether their internal state reflects the present, the past, or predictions of the future. Figure~\ref{fig:three-layer-architecture} gives a quick overview of the proposed architecture detailed in the sequel.

\begin{figure}[p]
\begin{center}
\includegraphics[width=1.0\linewidth]{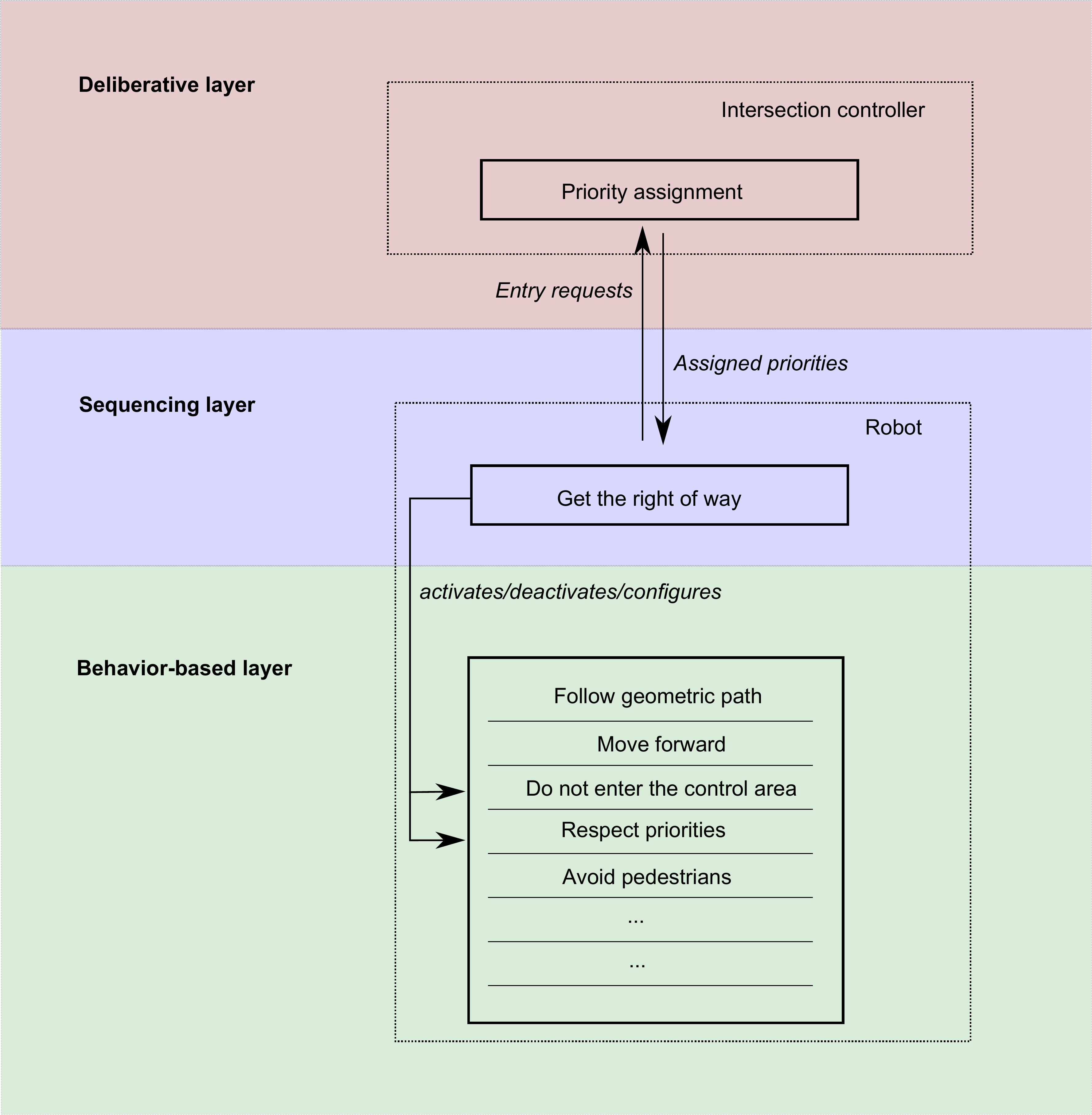}\hfill
\end{center}
\caption{The three-layer architecture of priority-based coordination}
\label{fig:three-layer-architecture}
\end{figure}

\subsection{The intersection controller}

The intersection controller constitutes the deliberative layer of the proposed architecture -- reasoning on the future -- and manages the control area, defined as a subset of the two-dimensional real space in which the collision area wholly resides. The control area must contain, at least, the subset of the two-dimensional space corresponding to all possible collisions between robots, excluding only regions where collision avoidance is reduced to safe car following (see Figure~\ref{fig-controlled-area}).
\begin{figure}[!htbp]
\begin{center}
\includegraphics[width=0.5\linewidth]{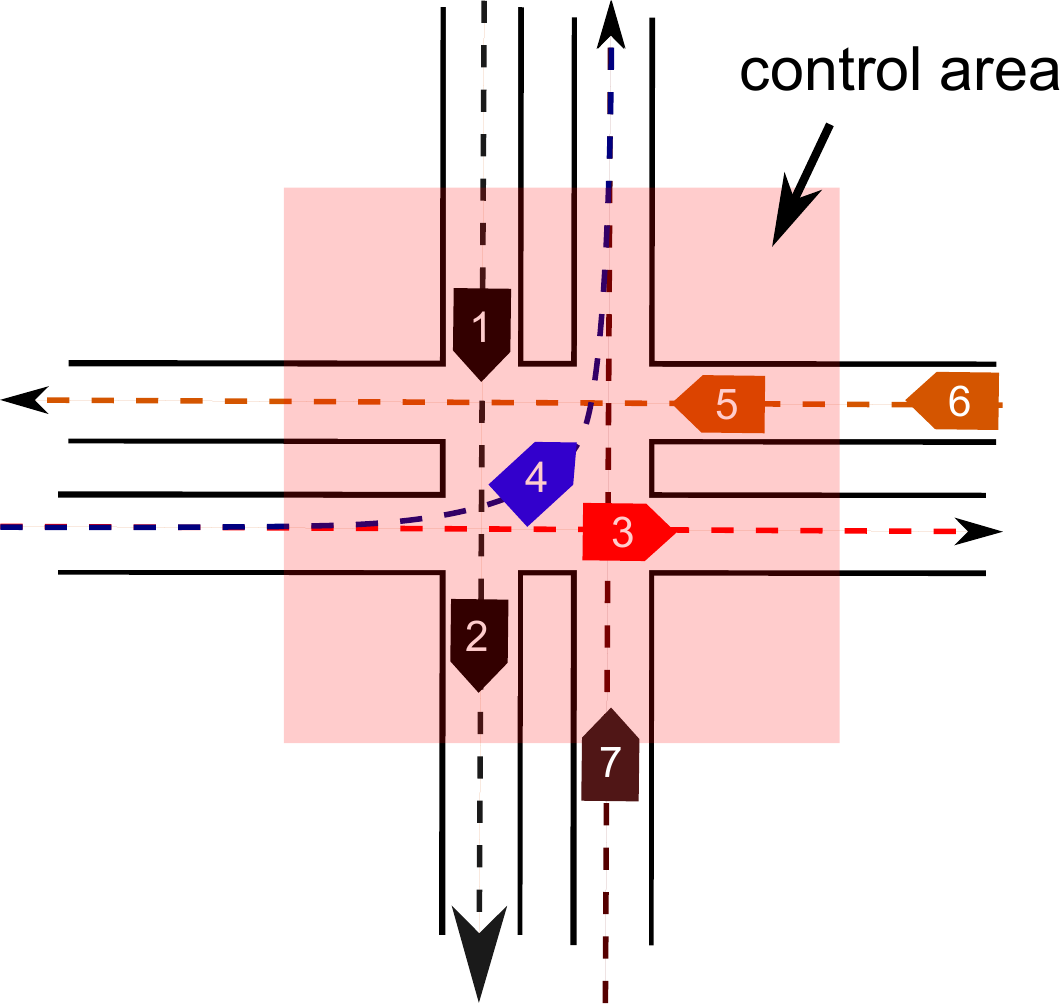}\hfill
\end{center}
\caption{The control area.}
\label{fig-controlled-area}
\end{figure}
For each path, an entry position and an exit position are defined. Let $x_i^\entry$ denote the entry position for robot $i$ and $x_i^\exit$ its exit position. Robot $i$ is in the control area if $x_i\in[x_i^\entry,x_i^\exit]$. To enter the control area, robots send a request to the intersection controller. The job of the intersection controller is to process requests. Either the request is rejected, the robot cannot enter the control area and will have to send a new request; or, the request is accepted and priorities with respect to robots already accepted in the control area are assigned. This task, referred as priority assignment is the topic of Section~\ref{sec:priority-assignment}.

\subsection{The behavior-based layer}

The behavior-based layer ensures the reactive quality of the system and is implemented by designing control laws.

\paragraph{Move forward} All robots implement behavior 'move forward'. Using the second-order dynamics model of Chapter~\ref{chap:control-acceleration}, 'move forward' behavior for robot $i$ consists in applying maximum throttle command $\uH_i$. However, this behavior is subsumed~\cite{Brooks1986} by all other behaviors, e.g., the robot will brake if another behavior like 'respect priorities' requires braking. 

\paragraph{Do not enter the control area} Robots also implement 'do not enter the control area' behavior. 
Each robot $i$ that has not already been accepted in the control area checks at every time slot whether accelerating (or maintaining maximum velocity) during the next time slot will inevitably result in an entry into the control area. If this is the case and if it is not accepted into the control area, robot $i$ must brake. To formulate this mathematically, the final (and maximal) position reached by robot $i$ with initial state $s_i$ under impulse control is computed as follows:
\begin{equation}
x_i^\stop(s_i) := \max \phi_{x,i}(\RR_+,s_i,\ubf_i^\imp)
\label{eq:stop-position}
\end{equation}
The condition for 'do not enter the control area' behavior to make robot brake then simply becomes: $x_i^\stop(s_i)>x_i^\entry$.

\paragraph{Respect priorities} Robots also implement priority preserving control of Part~\ref{part:priorities-to-guide-robots} (Chapter~\ref{chap:control-acceleration}) through 'respect priorities' behavior. This behavior cannot be active before priorities are assigned as it needs priorities as input to configure the control law $g^G$ by specifying the priority graph $G$.

\paragraph{Additional behaviors} In addition to the behaviors presented above, in charge of coordination, robots may implement other behaviors, not directly related to coordination. First of all, the fixed paths assumption (see Section~\ref{sec:coordination-space},~Figure~\ref{fig-paths}) requires robots to implement a 'follow geometric path' behavior using lateral control. More interestingly, robots may implement behaviors to react to unexpected events. For example, for an application in autonomous vehicles, an 'avoid pedestrians' behavior is a must. It is not conceivable to let autonomous vehicles go through an intersection in an urban area, executing an open-loop planned trajectory without implementing a behavior to detect pedestrians and react accordingly. The benefit of the proposed behavior-based architecture is that a behavior like 'avoid pedestrians' can be implemented in a manner that it subsumes all other behaviors. Most of the time, such reactive safety behaviors will require the robot to brake, and priorities will be conserved as the control law of Chapter~\ref{chap:control-acceleration} guarantees that it is always priority preserving to brake at any point of time (see Theorem~\ref{thm-robustness-brake-application}). Hence, priority-based coordination can handle a large class of unexpected events -- all events requiring one or more robots to brake -- without need to replan, i.e., without need to reassign priorities.

\subsection{The sequencer}

In a three-layer architecture, the sequencer's job is to activate/deactivate and/or configure the behaviors~\cite{Gat1998} that we just listed above. When should 'do not enter the control area' behavior be deactivated in favor of 'respect priorities' behavior ? Note that the state of the sequencer reflects the past as it is necessary to store whether the robot is accepted or not into the control area and to store priorities as well, in order to configure and activate/deactivate behaviors accordingly. 

The sequencer communicates with the deliberative layer, i.e., the intersection controller, by sending queries. The goal of these queries is to 'get the right of way'. The condition $x_i^\stop(s_i)>x_i^\entry-\delta$ is used as the condition to request the entry of the control area. The margin $\delta \geq 0$ enables to anticipate the entry of the control area, so that the intersection controller can possibly accept the robot into the control area in the remaining time, before 'do not enter the control area' behavior's brake condition $x_i^\stop(s_i)>x_i^\entry$ holds. The sequencer communicates asynchronously with the intersection controller to ensure a reactive quality. As long as the intersection controller does not accept the robot, the sequencer keeps 'do not enter the control area' behavior active. When the robot is accepted into the control area, 'do not enter the control area' behavior is deactivated in favor of 'respect priorities'. The assigned priorities received by the sequencer in the response of the intersection controller serve as input of 'respect priorities' behavior to configure the control law by specifying the priority graph $G$.

\section{Priority assignment}
\label{sec:priority-assignment}
This section focuses on how priorities are assigned, i.e., how the intersection controller processes entry requests. 

\paragraph{Priorities as a byproduct of traditional trajectory planning algorithms} First of all, it is key to notice that priorities can be obtained as a byproduct of all existing trajectory planning algorithms espousing the plan-as-program paradigm. One can simply assign the priorities induced by the feasible path returned by the planning algorithm. For certain existing algorithms, e.g., in~\cite{Akella2002}, priorities are even directly accessible (in~\cite{Akella2002}, they can be retrieved through the binary variables of the MILP formulation of the problem). Therefore, priority assignment is not the core of the present thesis and we will not provide complex priority assignment policies adapting existing algorithms. In this section, a simple priority assignment policy is proposed resulting in acylic and thus necessarily feasible priorities. Then, perspectives towards "liveness" and "stability" guarantees are presented.

\subsection{A simple priority assignment policy}
\label{subsec:simple-priority-assignment}
The idea of the proposed policy is to let robots spend as little time as possible in the intersection area, inspired from~\cite{Dresner2008-multiagent-approach}. Thus a robot is accepted into the control area only if it can travel with maximum throttle command and with lowest priority. The second point is key: assigning the newly accepted robot the lowest priority with regards to robots already accepted into the control area leads to a necessarily acyclic graph, enforcing liveness (see Theorem~\ref{thm:liveness-control-law}). This can be formulated as follows and implementation aspects are presented in Section~\ref{sec:implementation-aspects}. 

First of all, recall the control law of Chapter~\ref{chap:control-acceleration} when robots are controlled in acceleration:
\begin{equation}
g_i^G(s):=\begin{cases}
\uL_i & \text{if } \exists (j,i)\in E(G), \exists t\geq  0 \text{ s.t. } \phi_x(t,s,\tilde{\ubf}^i) \in \chiobs_{j\succ i} \\
\uH_i & \text{ else.}
\end{cases}
\end{equation}
As $\kappa_{j\succ i}$ is the cross-section of $\chiobs_{j\succ i}$, using the definition of $\tilde{\ubf}^i$, the control law can be formulated as follows:
\begin{equation}
g_i^G(s):=\begin{cases}
\uL_i & \text{if } \exists (j,i)\in E(G), \exists t\geq  0: (\phi_{x,j}(t,s_j,\uLbf_j),\phi_{x,i}(t,s_i,\ubf^\imp_i))\in\kappa_{j\succ i} \\
\uH_i & \text{ else.}
\end{cases}
\end{equation}

Consider a robot $i$ that requests the entry of the control area. To decide to accept it or not, we can simulate a trajectory that consists in applying control $\uH_i$ constantly to robot $i$ while robots $j \neq i$ follow the trajectory that they would have followed in the absence of $i$, i.e., following control law $g^G$. Let $s=(s_j)_{j\in\robots}$ denote the current state of robots $j\in\robots$, let $s_i$ denote the current state of the requesting robot $i$ and let $\varsigma$ denote the simulated trajectory defined as follows:
\begin{eqnarray}
\forall t \geq 0, \varsigma_i(t) &:=& \phi_i(t,s_i,\uHbf_i)\\
\forall j\in \robots, \forall t\geq 0, \varsigma_j(t) &:=& \phi_j(t,s,g^{G}) \label{eq-predicted-trajectory}
\end{eqnarray}
Then, there are two options:
\begin{itemize}
\item if for all $k\in\NN$ and for all $j\in\robots$ satisfying $\kappa_{ij}\neq\emptyset$, we have:
\begin{equation}
\forall t\geq 0, (\phi_{x,j}(t,\varsigma_j(k),\uLbf_j),\phi_{x,i}(t,\varsigma_i(k),\ubf^\imp_i))\notin \kappa_{j\succ i} 
\label{eq:condition-accept-i}
\end{equation}
the request is accepted and we do:
\begin{eqnarray}
\robots &\gets& \robots\cup\{i\}\\
E(G)& \gets& E(G)\cup\{(j,i):j\in\robots, \kappa_{ij}\neq\emptyset\}
\end{eqnarray}
\item else the request is rejected.
\end{itemize}

Note that the described algorithm ensures the priority relation to be a partial order, that is $G$ to be a directed acyclic graph at all times. Each robot is sequentially accepted into the control area by the intersection controller if it can go through the intersection at maximum throttle command and after all already accepted robots. Condition~\eqref{eq:condition-accept-i} ensures that once robot $i$ is accepted and controlled by the control law $g^G$, if all robots follows $g^G$ (no uncertainty, no unexpected event), the control law will always return $\uH_i$. This means that, in the absence of uncertainty, the coordination system will result in robots either waiting at the entry of the control area (possibly stopped at the entry), or accepted into the control area and applying maximum throttle command, thus going through the intersection at maximum speed. This is what is observable in the simulations of Subsection~\ref{subsec:sim-deterministic}. However, remind that a key motivation for our priority-based approach is precisely to handle uncertainty. Hence, if some robot does not apply maximum throttle command at some point for an arbitrary reason, the priority preserving control law will ensure that priorities are nevertheless respected as demonstrated by the simulations of Subsection~\ref{subsec:sim-unexpected-events}.

\subsection{Request processing liveness}

The weakness of the policy presented above is quite similar to the one of the First-Come-First-Serve policy of~\cite{Dresner2008-multiagent-approach}. As highlighted in~\cite{Au2011} for First-Come-First-Serve reservation policy, and it also holds for the priority assignment policy presented above, handling requests separately and not taking into account the history of requests, causes undesired behaviors like a vehicle in an alley waiting indefinitely at the entry of the intersection.

A solution is presented in~\cite{Au2011} to avoid this phenomenon. A batch policy with locking is proposed, consisting of mapping requests to a real value computed using a cost function of the form $f(wait):=a\times wait^b$ where $a,b$ are constants and $wait$ is the estimated amount of the time the robot has been waiting to enter the intersection. The "locking" mechanism is described as follows: when a request $r$ has an associated cost greater than a threshold, then requests from other robots whose path intersects the path of the robot of $r$ will not be granted, until the robot of $r$ is accepted. Interestingly, the proposed policy provably guarantees liveness, i.e., every robot waiting to enter the intersection can eventually enter. This liveness property is different from the one proved in Part~\ref{part:priorities-to-guide-robots} which ensures that once robots are accepted into the intersection, respecting the assigned priorities, they will eventually go through the intersection. The "locking" mechanism can be easily adapted to enhance performance and ensure liveness of the simple priority assignment proposed in Subsection~\ref{subsec:simple-priority-assignment}.

\subsection{Stability guarantees}
\label{subsec:bp-priority-assignment}

In traffic signal control, queue lengths are a standard indicator of a control policy's performance. In particular, recently, back-pressure control~\cite{Tassiulas1992} applied to traffic signals (see, e.g., ~\cite{Varaiya2009, Wongpiromsarn2012,Gregoire2013-capacity,Gregoire2014-unknown-routing}) aims at providing stability guarantees of the control policy. Loosely speaking, stable queues do not grow indefinitely through time. We believe that this work can be used to endow the priority assignment policy with stability guarantees.

We do not aim to formalize the proposed approach in the general case as it is beyond the scope of the present thesis, so the approach is presented for a particular example. Consider the intersection depicted in Figure~\ref{fig:intersection-BP}. 
\begin{figure}[!htbp]
\begin{center}
\includegraphics[width=0.8\linewidth]{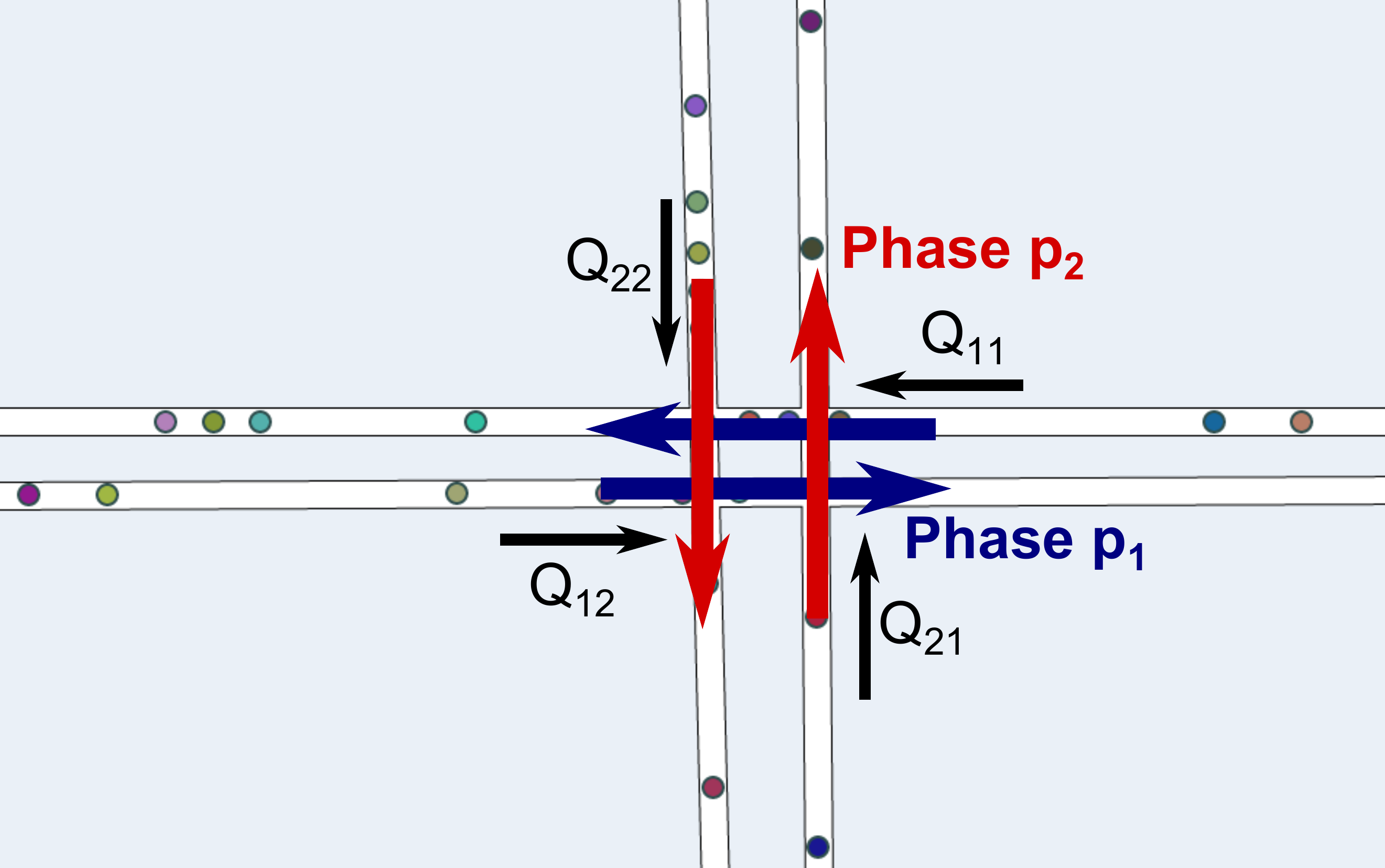}\hfill
\end{center}
\caption{A four-path intersection, with two phases and four queues. Phase $p_1$ empties queues $Q_{11}$ and $Q_{12}$ and phase $p_2$ empties queues $Q_{21}$ and $Q_{22}$.}
\label{fig:intersection-BP}
\end{figure}
We let $Q_{11}$, $Q_{12}$, $Q_{21}$, $Q_{22}$ denote the queues lengths at the entry of each path. If the intersection was controlled by a traffic signal, there would be two phases: $p_1$ and $p_2$. Assuming this intersection is isolated, under standard back-pressure control, phase $p_1$ (resp. $p_2$) is applied if $Q_{11}+Q_{12} \geq Q_{21}+Q_{22}$ (resp. $Q_{11}+Q_{12} < Q_{21}+Q_{22}$). 

Traffic signal control is not efficient at low traffic density because of the phase duration. A typical situation is a single vehicle waiting at the intersection for the right of way. The vehicle needs to wait for the end of the current phase before obtaining the right of way. At high traffic density, there are always queues at the entry of the intersection, so traffic signal control is particularly efficient as it lets vehicles move in platoons. Moreover, under back-pressure control, optimal stability can be proved, i.e., the queuing network is stabilized for all arrival rates that can be stably handled considering all control policies. That is why we propose an adaptive priority assignment policy that consists in applying back-pressure priority assignment if robots accumulate at the entry of the intersection (above a certain threshold), while the basic priority assignment policy presented in Subsection~\ref{subsec:simple-priority-assignment} is applied otherwise. More precisely, a phase duration $T$ and a threshold $\dql \geq 0$ are chosen, and periodically, for $t=0,T,2T\cdots$, the phase update algorithm proceeds as follows:
\begin{itemize}
\item if $(Q_{11}+Q_{12})-(Q_{21}+Q_{22}) \geq \dql$: phase $p_1$ is applied. It means that for all the duration of the phase ($T$), only the entry requests of robots on the corresponding paths will be accepted. However, requests still need to be accepted according to the priority assignment policy presented in Section~\ref{sec:priority-assignment}. Typically, when there is a phase switch, the first requests will probably be rejected as there are still robots of the other phase in the intersection. These first requests which are rejected can be seen as a kind of yellow time. 
\item if $(Q_{21}+Q_{22})-(Q_{11}+Q_{12}) > \dql$: it is the symmetric case, and phase $p_2$ is applied.
\item otherwise, all phases are activated (both $p_1$ and $p_2$), so that the requests of all robots can be potentially accepted. The priority assignment policy is not affected by the phase, and is exactly as presented in Subsection~\ref{subsec:simple-priority-assignment}.
\end{itemize}

\chapter{ Simulations}
\label{chap:simulations}
\minitoc

\paragraph{Sketch of the chapter} The first section provides some details on the implementation of priority-based coordination in simulations. In particular, collision checking, robots random generation and the size of the control area are discussed. Mainly qualitative simulations results are then presented and interpreted.

\section{Implementation aspects}
\label{sec:implementation-aspects}

For the sake of the simplicity, we have implemented our algorithms for circle-shaped robots along straight paths. This choice eases the computation of the obstacle region as every $\kappa_{ij}$ is the interior of an ellipse whose equation can be easily derived from the radius of robots and the angle between the two straight geometric paths. All robots are supposed to be circle-shaped with a common diameter $D$. Note that the collision region between each couple of paths can be precomputed once and for all during the design phase of the intersection controller. The lateral control is not simulated and all robots are assumed to follow their assigned geometric path.

To check whether a trajectory is collision-free, as all we can do is to compute a discrete sequence of points, we have used a conservative collision checking algorithm. Basically, to check whether a given flow $(\phi(t,s,\ubf))_{t\geq 0}$ is collision-free with regards to $\chiobs_{i\succ j}$, we compute the following sequence of points of $\RR^2$:
\begin{eqnarray}
x_i(k)&:=&\phi_{x,i}(k,s,\ubf)\\
x_j(k)&:=&\phi_{x,j}(k+1,s,\ubf)
\end{eqnarray}
Our collision checking algorithm asserts that the flow $(\phi(t,s,\ubf))_{t\geq 0}$ is collision-free with regards to $\chiobs_{i\succ j}$ if and only if the sequence $(x_i(k),x_j(k))_{k\in\NN}$ is collision-free with regards to $\kappa_{i \succ j}$. This method is illustrated in Figure~\ref{fig:collision-checking}. It is direct that it is conservative as some collision-free trajectories are not asserted to be collision-free; yet the difference vanishes for small enough time slot length.
\begin{figure}[!htbp]
\begin{center}
\includegraphics[width=1.0\linewidth]{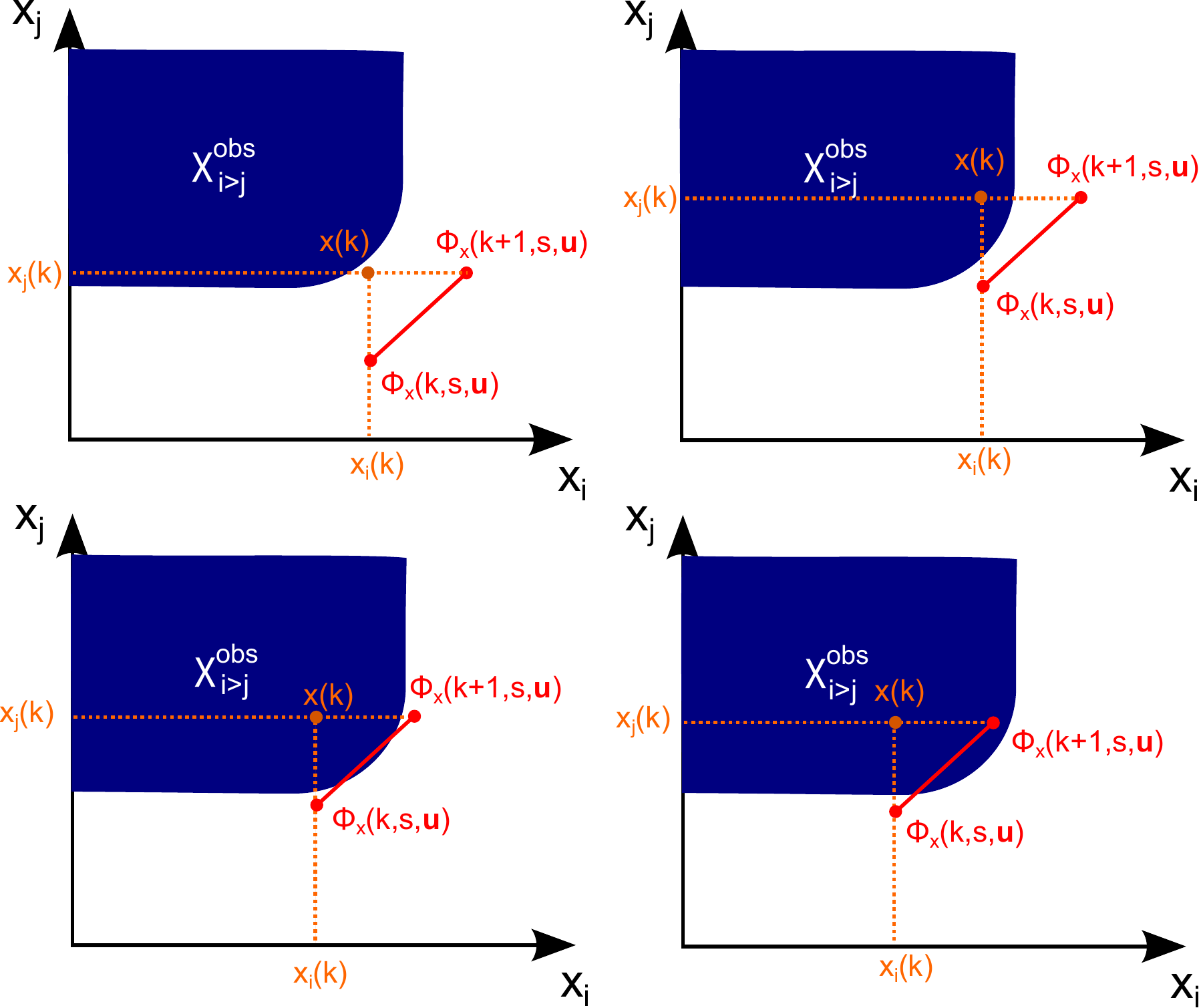}\hfill
\end{center}
\caption{Method used for collision checking based on a discrete sequence of points. In the two top drawings, the trajectory for $t\in[k,k+1]$ is collision-free. However, in the case of the top right drawing the collision checking algorithm will consider the trajectory as non collision-free. In the two bottom drawings, the trajectory for $t\in[k,k+1]$ is not collision-free. In the case of the bottom left drawing, both $\phi_x(k,s,\ubf)$ and $\phi_x(k+1,s,\ubf)$ are collision-free. However, $x(k)$ is not collision-free. This case illustrates that checking whether the endpoints are collision-free is not sufficient, which justifies the use of $x(k)$.}
\label{fig:collision-checking}
\end{figure}

Robots are generated at the origin of each path randomly at a constant rate. Basically, at each time slot, for each path, a random value between $0$ and $1$ according to a uniform distribution is taken, and if this value exceeds a certain threshold, a robot is generated on this path. The value of this threshold is precisely the generation rate at the path. When generated, a robot $i$ is positioned with zero velocity at the coordinate $0$ of the path, or if there is already a robot $j$ at position $x_j \leq D$, $i$ is positioned at the coordinate $x_j-D$. 

As noticed in~\cite{Dresner2008-multiagent-approach}, maximizing the velocity of robots in the intersection minimizes the time spent within the collision region, yielding a better performance. Hence, to ensure that robots have a maximum velocity within the collision region, the entry of the control area is defined to be far enough from the collision region (in the simulation videos we see that the robots that are not already accepted in the control area stop way before potential collision configurations). 

Finally, the priority assignment policy of Subsection~\ref{subsec:simple-priority-assignment} is simplified based on heuristic considerations. To decide whether robot $i$ can be accepted or not, we need to check whether, under maximum throttle command, it can go through the intersection after all robots already accepted in the control area. To do so, note first that it is sufficient to check if it is the case for the lastly accepted robot of each intersecting path. Now, assume that robot $j$ is the lastly accepted robot on path $\gamma_j$. Intuitively, it is clear that, under maximum throttle command, robot $i$ can go through the intersection after robot $j$, if and only if there is a sufficient time offset between their entries. To this purpose, we compute:
\begin{itemize}
\item $\tau_i$~: the number of time slots necessary for robot $i$, under maximum throttle command, to reach "the entry of the collision area between paths $\gamma_i$ and $\gamma_j$", i.e., to reach position $\min\{x_i:x\in\chiobs_{ij}\}$;
\item $\tau_j$~: the number of time slots necessary for robot $j$, under maximum throttle command, to reach "the exit the collision area between paths $\gamma_i$ and $\gamma_j$", i.e., to reach position $\max\{x_j:x\in\chiobs_{ij}\}$.
\end{itemize}
Our heuristic approach considers that, under maximum throttle command, robot $i$ can go through the intersection after robot $j$ if and only if $\tau_i\geq \tau_j$. It looks quite natural as it means that robot $i$ should "enter the collision area between paths $\gamma_i$ and $\gamma_j$" after robot $j$ exits this area. Naturally, due, e.g., to the brake safety constraint, this is not equivalent to the formulation of~Subsection~\ref{subsec:simple-priority-assignment}. However, it is much easier to implement and checking whether the heuristic condition is satisfied is also much less time consuming. Simulation results of Subsection~\ref{subsec:sim-deterministic} confirm the efficiency of our heuristic approach as robots seem to enter the control area at the right time, so as to go through the intersection at maximum speed.

\section{Simulation results}

The purpose of the presented simulations is fourfold; they aim to demonstrate:
\begin{enumerate}[(a)]
\item the ability of priority-based coordination to carry out as complex scheduling as with plan-as-program approaches;
\item the robustness enabled by planning priorities instead of precise trajectories, making possible to handle unexpected events requiring braking without replanning, making also possible to deal with bounded, possibly time-varying, uncertainty;
\item and the ability of priority assignment policies to implement back-pressure algorithms guaranteeing queues stability and opening avenues for the control of a network of autonomous intersections.
\end{enumerate}

\subsection{Simulations under deterministic control}
\label{subsec:sim-deterministic}

The experimental intersection is depicted in Figure~\ref{fig-simulation-intersection}. It is composed of eight straights paths. The maximum velocity of robots is such that a robot at maximum velocity travels $D/2$ (one radius) during one slot. All robots share the same kinodynamic  constraints with $\uL=-\uH$ and 20 slots are necessary to go from stop to full speed (and conversely). Hence, to ensure that robots are at maximum velocity when they reach the first potential collision configuration, the entry position is fixed at a distance $6 D$ from the first potential collision configuration. Symmetrically, the exit position is fixed at a distance $6 D$ after the last potential collision configuration. As communication aspects are not considered in this simulation setting, there is no delay for the intersection controller to respond to requests, so robots do not need to anticipate their entry and we take $\delta\equiv 0$, i.e., robots request the entry of the control area if $x_i^\stop(s_i)>x_i^\entry$, that is just in time.

\begin{figure}[ht]
\begin{center}
\includegraphics[width=0.5\linewidth]{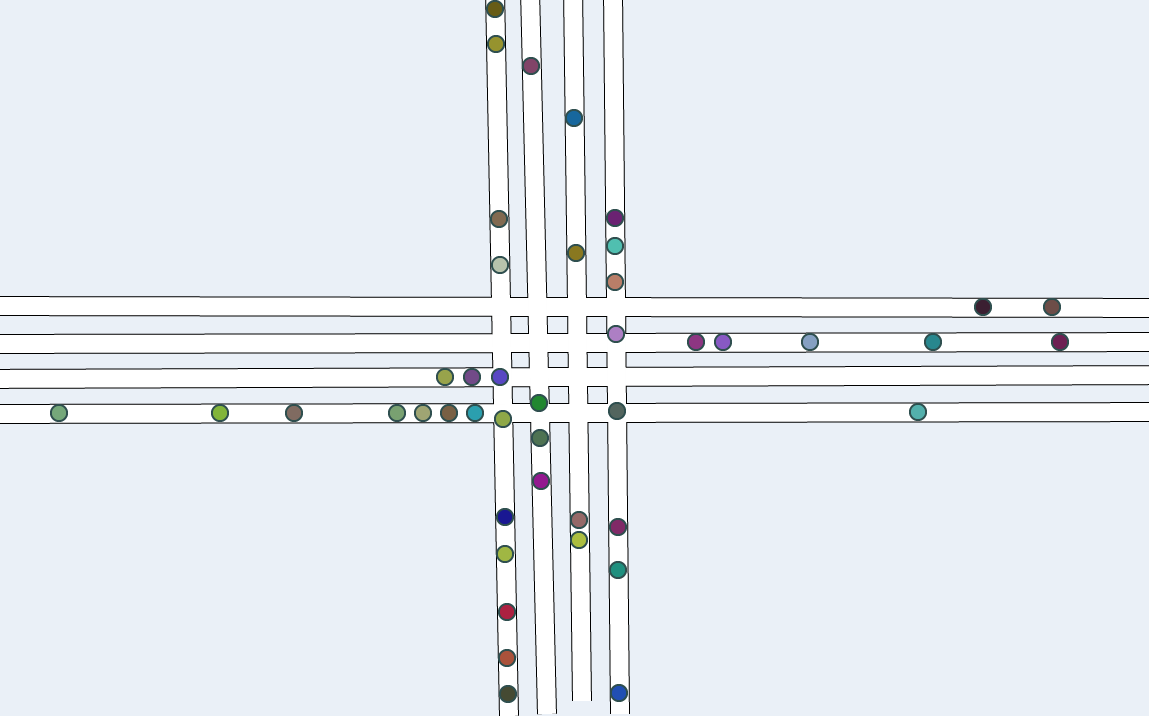}\hfill
\end{center}
\caption{The intersection composed of eight straight paths used for simulations.}
\label{fig-simulation-intersection}
\end{figure}

A video capture of the simulation for an arrival rate of $0.04$ robots per time slot on each path is available \href{http://youtu.be/T5ASnKuJLT4}{here}\footnote{\url{http://youtu.be/T5ASnKuJLT4}}. One can observe that robots not accepted in the control area stop at a distance equivalent to 6 robots before the first potential collision configuration. In this simulation, there is no uncertainty, and the video capture confirms that in the absence of uncertainty, the presented algorithms result in robots always at maximum throttle command inside the control area. Finally, note that the entry management of the control area is not a first come first serve policy. Some robots requesting the entry before another robot may be accepted into the control area after that robot.

The latter phenomenon is more obvious in the video capture of the simulation for an arrival rate of $0.08$ robots per time slot available \href{http://youtu.be/tYC6m7Z-S3Y}{here}\footnote{\url{http://youtu.be/tYC6m7Z-S3Y}}. At such an arrival rate, queues are formed at the entry of the control area, but the size of the queues are not considered for processing the requests. Finally, note that queues are stable at this arrival rate which denotes an ergodic dynamics of the system.

At this point, it just appears that priority-based coordination enables to carry out as complex scheduling as traditional approaches using a plan-as-program approach, e.g.,~\cite{Dresner2008-multiagent-approach}. However, the benefit of the priority-based approach is not visible, because in the absence of uncertainty, the control law under assigned priorities always returns $\uH$, it is very similar to an open-loop plan execution.

\subsection{Robustness regarding unexpected deceleration}
\label{subsec:sim-unexpected-events}

Here, to illustrate the robustness of the proposed coordination system with respect to unexpected events requiring deceleration, we consider a scenario in which robots may decide to brake within the control area unexpectedly. The intersection controller, when assigning priorities, does not know that the robot is going to brake within the control area. At the beginning of every time slot, each robot $i$ may switch from a controlled regime under the control law $g^G$ to an unexpected deceleration under constant control $\uL_i$, and vice versa, with probability transitions displayed in Figure~\ref{fig-transitions}.
\begin{figure}[ht]
\begin{center}
\includegraphics[width=0.6\linewidth]{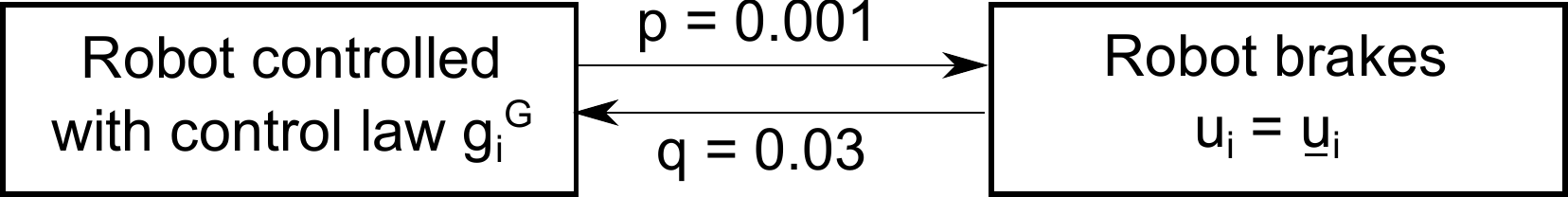}\hfill
\end{center}
\caption{Non-deterministic transitions between control regimes}
\label{fig-transitions}
\end{figure}
The probability values $p,q$ are chosen arbitrarily, as the goal is is not to reproduce a realistic scenario but to test and validate the robustness of the approach. One may consider transitions to brake control regime as modeling some unexpected events subject to occur in applications to transportation systems such as a loss of communication abilities or a pedestrian crossing the road, both requiring the robot to slow down unexpectedly. 

A video capture of the simulation for an arrival rate of $0.04$ robots per time slot on each path is available \href{http://youtu.be/8Xz3S_OhK80}{here}\footnote{\url{http://youtu.be/8Xz3S_OhK80}}. Even if some robots stop within the control area, other robots adapt and brake if necessary thanks to the control law. In contrast with simulations under deterministic control, the control law is useful here and enables to handle robots slowing down unexpectedly. No collision occurs during the simulation, the control law is effectively safe and robust with regards to brake application. We see that the priorities are satisfied, that no collision occurs, and that all robots eventually exit the intersection, although the trajectory may be very far from the trajectory under perfect control law.

\subsection{Robustness regarding bounded uncertainty}

The simulation results that follow aim at demonstrating the robustness of priority-based coordination in the presence of bounded uncertainty in sensing and control. The same inertia/geometrical parameters as for the previous simulations are used. However, uncertainty is additionally considered.

First of all, we assume the presence of control uncertainty, so the dynamics of robots is described by Equations~\eqref{eq-diff-state1-non-deterministic}-\eqref{eq-diff-state2-non-deterministic}. In the presented simulations, the value of control uncertainty bounds are different for each robot. Their average values are (here, $n$ denotes the total number of robots through the simulation run and the sum is over all these robots):
\begin{equation}
\frac{1}{n}\sum_{i} |\dmin_i^v|=\frac{1}{n}\sum_{i}\dmax_i^v=0.1~|\uL_i|
\label{eq:control-uncertainty-values-1}
\end{equation}
\begin{equation}
\frac{1}{n}\sum_{i} |\dmin_i^u|=\frac{1}{n}\sum_{i}\dmax_i^u=0.1~\uH_i = 0.1~|\uL_i|
\label{eq:control-uncertainty-values-2}
\end{equation}
and the actual control uncertainty bounds on each robot vary between $0$ and twice the average values according to a uniform distribution. This enables to illustrate that the proposed approach can deal with different control uncertainty bounds on each robot. In average, the uncertainty in control is $10\%$ of the maximum control value as stated by Equations~\eqref{eq:control-uncertainty-values-1} and~\eqref{eq:control-uncertainty-values-2}.

Uncertainty in sensing is also simulated and again, as for control uncertainty, the value of sensing uncertainty bounds are different for each robot. Let $\delta y_i^x$ and $\delta y_i^v$ denote the respective maximum absolute errors in position and velocity observations on robot $i$, their average values are:
\begin{equation}
\frac{1}{n}\sum_{i} \delta y_i^v=0.1~\vmax_i
\label{eq:sensing-uncertainty-values-1}
\end{equation}
\begin{equation}
\frac{1}{n}\sum_{i} \delta y_i^x = D/2
\label{eq:sensing-uncertainty-values-2}
\end{equation}
and again, the actual observation uncertainty bounds on each robot vary between $0$ and twice the average values according to a uniform distribution. In average, the uncertainty in position is one radius of robot and the uncertainty in velocity is $10\%$ of the maximum velocity as stated by Equations~\eqref{eq:sensing-uncertainty-values-1} and~\eqref{eq:sensing-uncertainty-values-2}.

Note that to decide to accept or not a robot in the control area, the intersection controller can only access to the non-deterministic state of robots. The heuristic approach is adapted to deal with that and the values of $\tau_i$ and $\tau_j$ (see Section~\ref{sec:implementation-aspects}), which are necessary to decide to accept or not a robot in the control area, are computed based on the average state of robots considering all possible current true states.

A video capture of the simulation for an arrival rate of $0.02$ robots per time slot on each path is available \href{http://youtu.be/vpqHbNE6smM}{here}\footnote{\url{http://youtu.be/vpqHbNE6smM}}. The red segments represent the set of positions where a robot believes it is located in. One can see that no collision occurs, neither between robots, nor between the red segments. It confirms that the control law in the non-deterministic information space proposed in Chapter~\ref{chap:control-uncertainty} results in a collision-free trajectory of the non-deterministic information state. 

Finally, to demonstrate the robustness of our approach regarding time-varying uncertainty, we consider a scenario where uncertainty on the observation of position is much higher during a limited time period. In the following simulations, all the parameters are unchanged, but between time slots $t = 500$ and $t = 1000$, the uncertainty on position measures is suddenly multiplied by a factor $10$. A video capture of the simulation for an arrival rate of $0.02$ robots per time slot on each path is available \href{http://youtu.be/k14t-fYpy3g}{here}\footnote{\url{http://youtu.be/k14t-fYpy3g}}. It is remarkable that such a change in the uncertainty of position observation can be handled in a completely reactive manner. Note also that, interestingly, the priority assignment policy during the period of large position uncertainty demonstrates an emerging traffic signal like behavior. 

\subsection{Stability guarantees under back-pressure control}

The following simulation results illustrate the ability of priority-based coordination to ensure both efficiency in term of travel time at low traffic density and stability of the queue lengths at high traffic density, in an adaptive manner. The adaptive priority assignment policy proposed in Subsection~\ref{subsec:bp-priority-assignment} has been implemented with the same inertia/geometrical parameters as the simulations presented previously and simulations results are presented in Figure~\ref{fig:bp-simulations}.
\begin{figure}[!htbp]
\begin{center}
\includegraphics[width=0.5\linewidth]{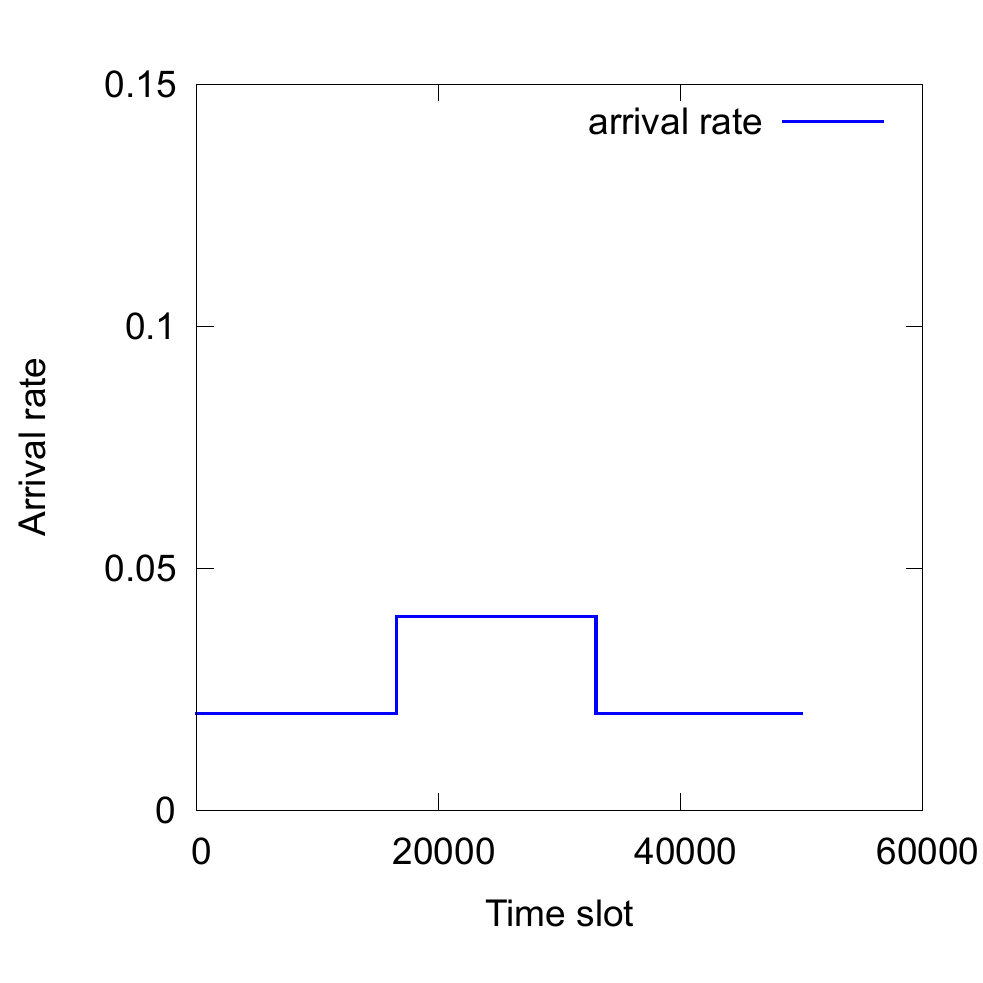}\hfill
\includegraphics[width=0.5\linewidth]{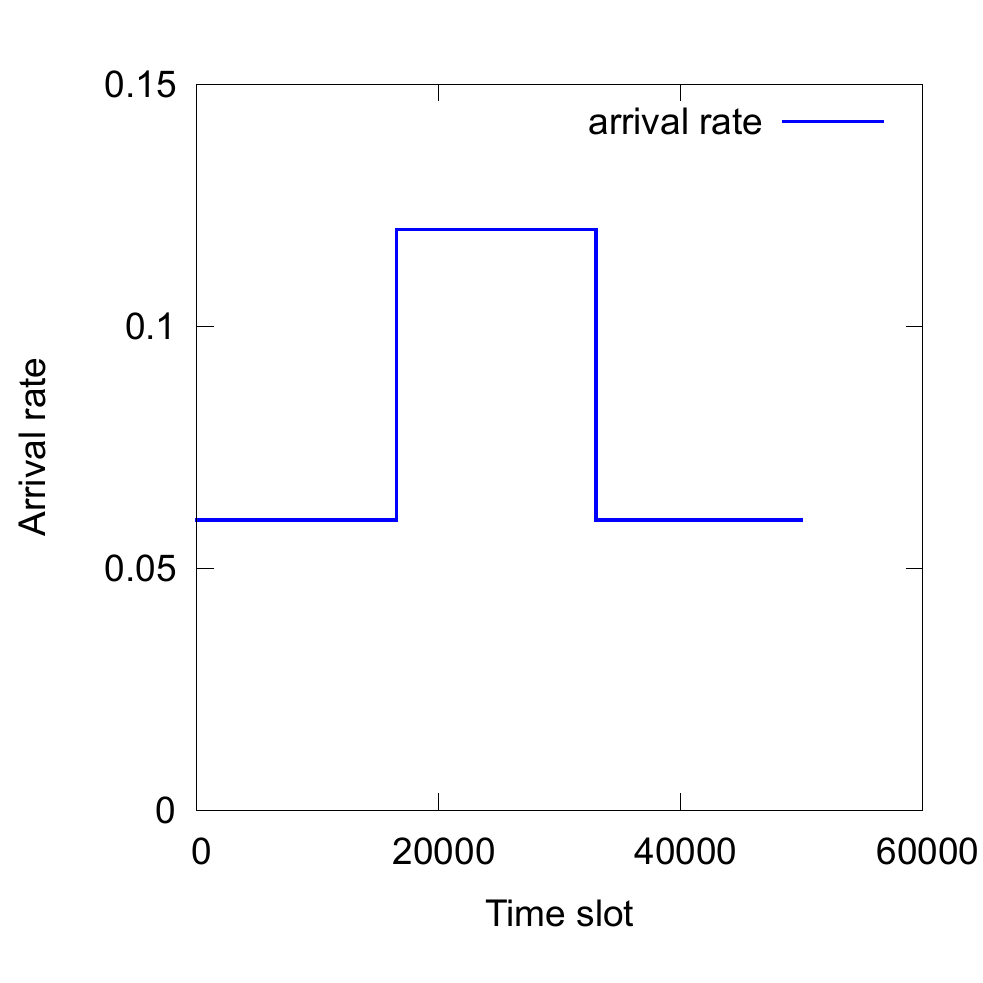}\hfill
\includegraphics[width=0.5\linewidth]{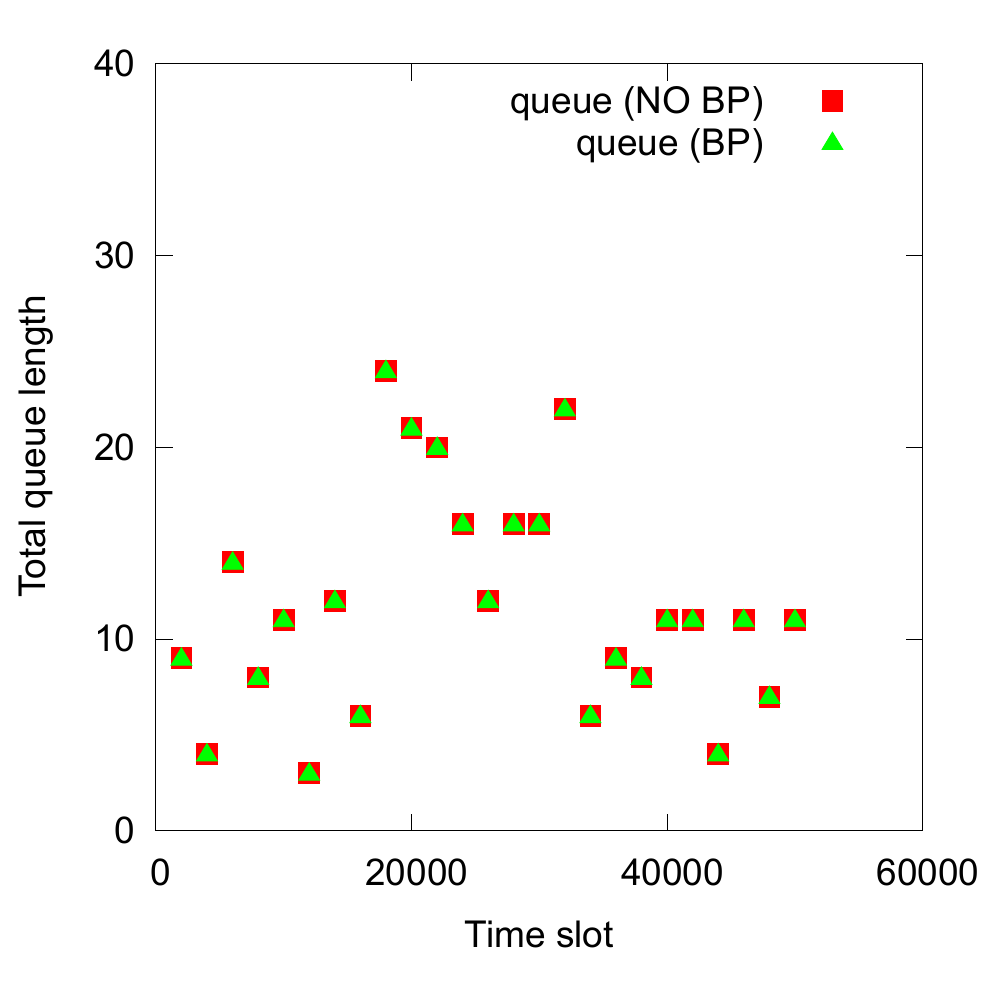}\hfill
\includegraphics[width=0.5\linewidth]{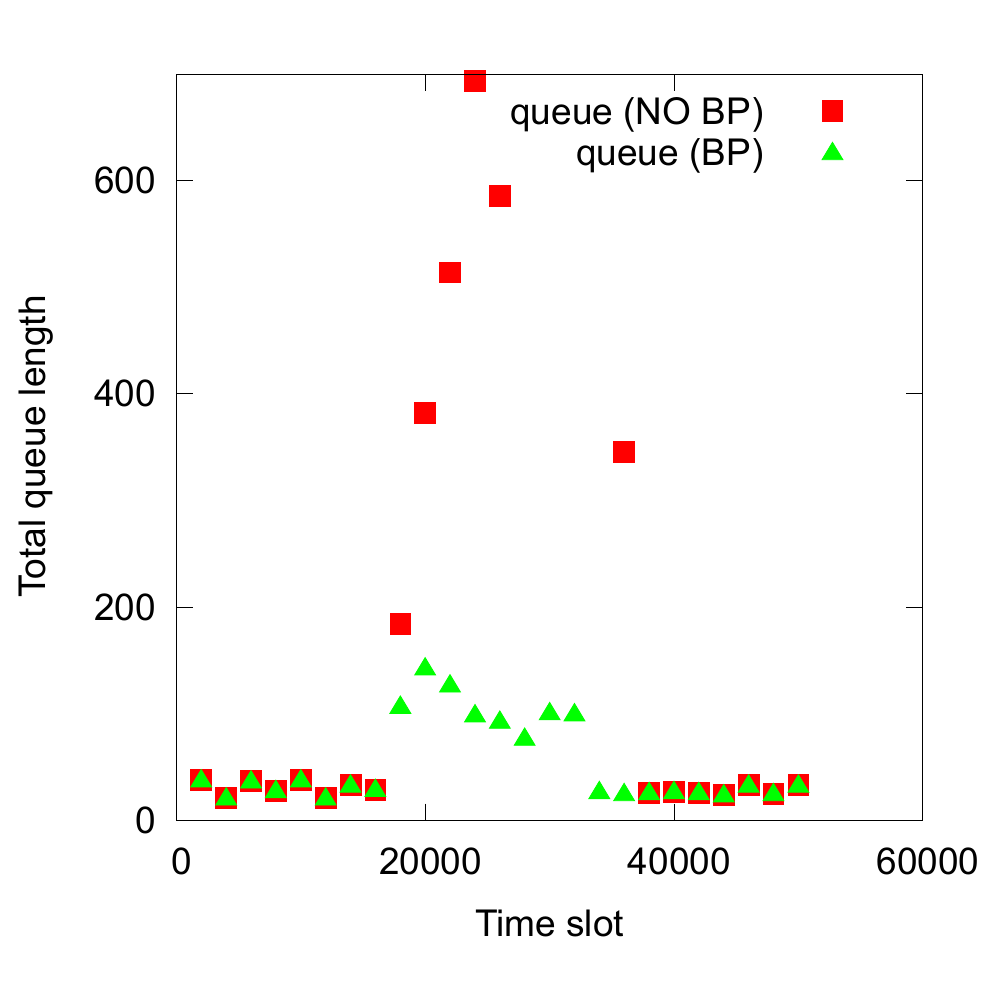}\hfill
\includegraphics[width=0.5\linewidth]{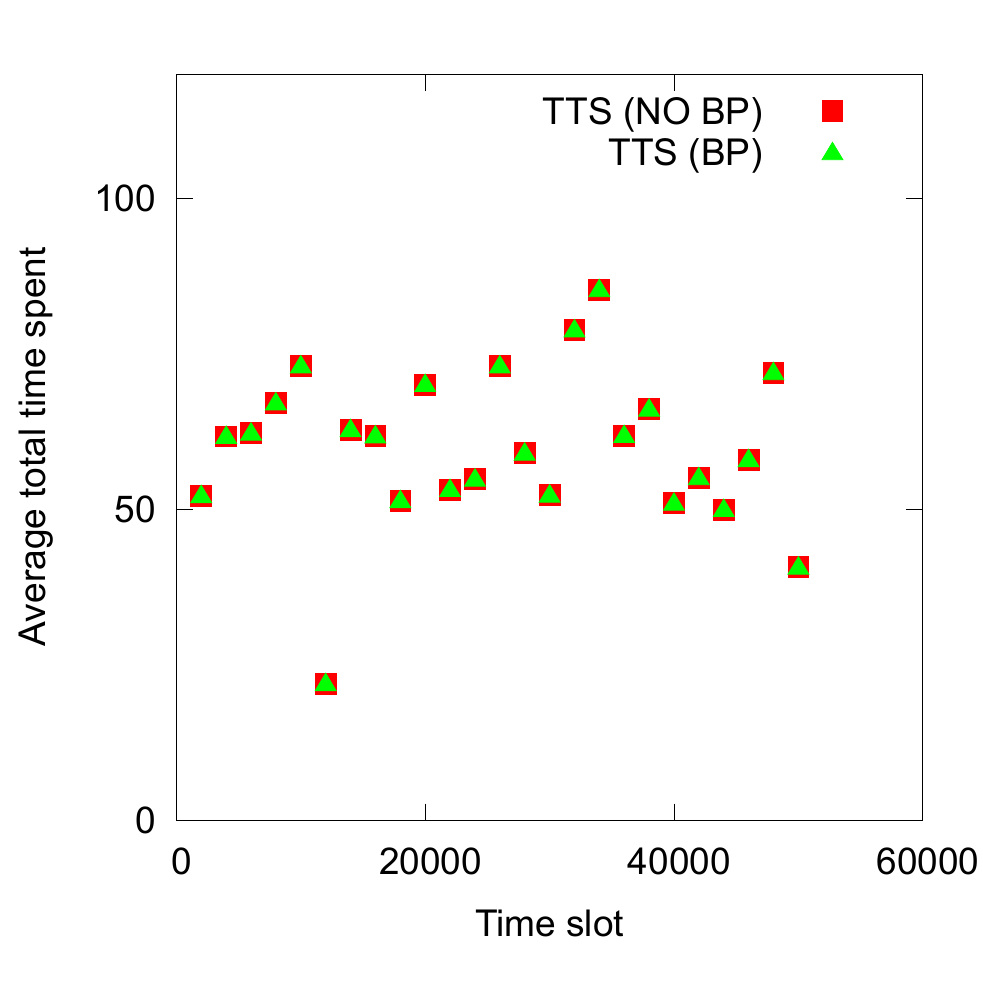}\hfill
\includegraphics[width=0.5\linewidth]{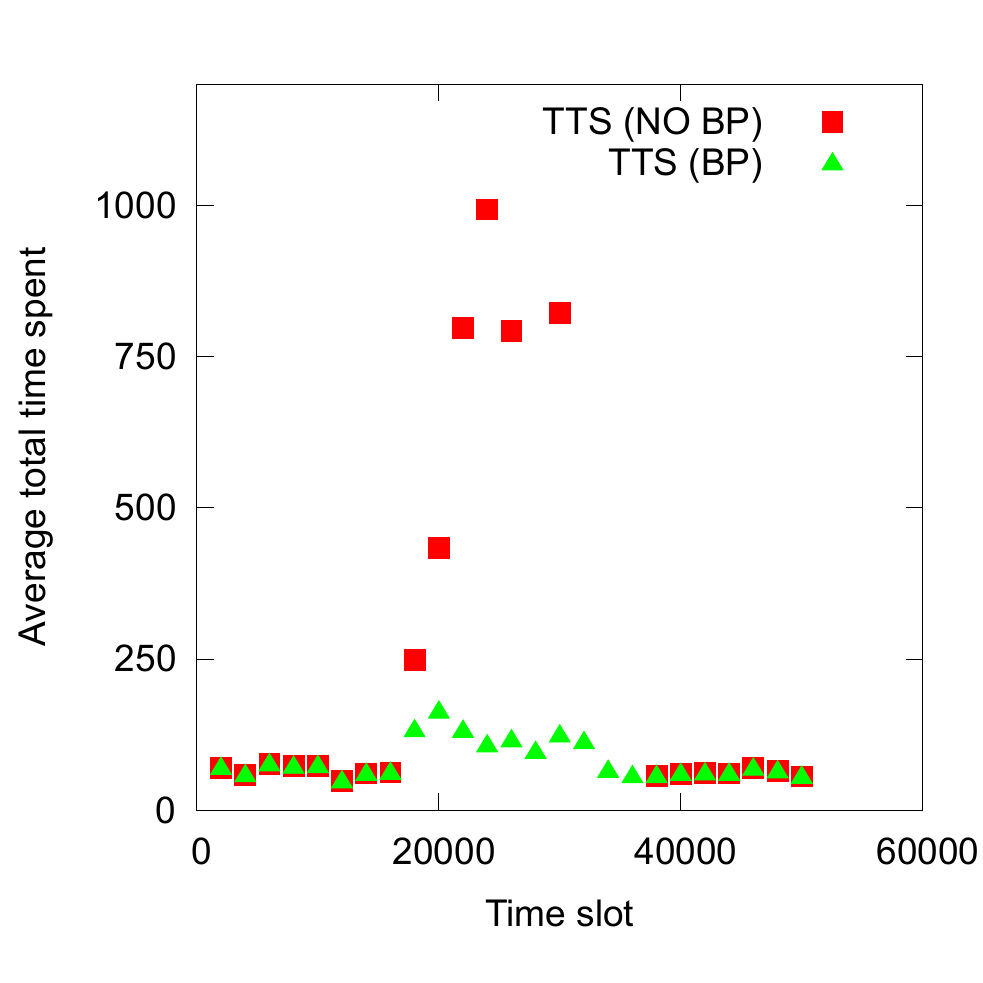}\hfill
\end{center}
\caption{Simulations under adaptive priority assignment policy for two arrival rate scenarios.}
\label{fig:bp-simulations}
\end{figure}
For the phase duration, we take $T=100\text{ time slots}$ and the threshold is $\dql=30$. We compare the behavior of the system under the simple priority assignment policy presented in Section~\ref{sec:priority-assignment} (referred as "NO BP" in Figure~\ref{fig:bp-simulations}) versus the adaptive priority assignment policy proposed here (referred as "BP" in Figure~\ref{fig:bp-simulations}). The top drawings provide the evolution of the arrival rate through time for the two scenarios considered. The drawings in the center depict the evolution of the sum of the four queue lengths ($Q=Q_{11}+Q_{12}+Q_{21}+Q_{22}$) through time. Finally, the bottom drawings represent the average total time spent by robots currently in the intersection. On the left drawings of Figure~\ref{fig:bp-simulations}, it is clear that at low traffic density, the behavior under both policies are identical: it is not surprising as the queue difference is not likely to be greater than the threshold, so that the two priority assignment policies coincide most of the time. On the right drawings of Figure~\ref{fig:bp-simulations}, it appears that the queues get unstable under the priority assignment policy of Section~\ref{sec:priority-assignment} for an arrival rate of $0.12$ robots per time slot. By contrast, under adaptive priority assignment policy, queues are stable. In conclusion, the proposed adaptive priority assignment policy combines the efficiency of the priority assignment policy of Section~\ref{sec:priority-assignment} at low traffic density and the stability guarantees of back-pressure control at high traffic density.

\chapter*{Conclusions}

The present part proposed a priority-based coordination system adopting a three-layer architecture integrating both priority assignment -- in the deliberative layer -- and priority preserving control -- in the behavior-based layer -- which are executed in parallel and interface through the sequencing layer. The proposed coordination system is able to manage continuous arrivals of robots at the intersection and to assign priorities dynamically. The proposed architecture takes full benefit of the brake safety property of priority preserving control of Part~\ref{part:priorities-to-guide-robots} as some behavior may require a robot to brake at any moment to handle some unexpected event with the guarantee that priorities will be respected. Priority-based coordination demonstrates both planning and reactive abilities.

In the proposed setting, priority assignment is centralized and is processed by the intersection controller. It is the planning -- time consuming -- task: assigning efficient priorities requires reasoning about the future. Interestingly, previous work can be used to design efficient priority assignment policies. Priorities are indeed a byproduct of existing trajectory planning algorithms as it suffices to assign the priorities induced by the output planned trajectory. That is why little attention has been paid to the design of efficient priority assignment policies in this part. However, what is quite novel is that, some algorithms commonly used for the control of traffic signals proved useful to assign priorities at high traffic densities as it becomes more efficient to have a traffic-signal-like behavior. Interestingly, priority assignment can adapt to traffic load resulting in a traffic-signal-like behavior ensuring queue lengths stability at high traffic density, and in a much more complex scheduling through combinatorial optimization of priorities to minimize travel delays at low traffic density. This brings elements to the debate about whether autonomous intersection management can really outperform traffic signals. We believe that at high traffic density, a traffic-signal-like behavior is likely to be an optimal strategy as it enables to move robots in platoons, maximizing throughput. However, at low traffic density, optimizing priorities can significantly decrease travel delays and outperform traffic signals. The results of the present part demonstrate that priority-based coordination is able to perform both strategies adaptively. 

As highlighted in the introduction of the thesis, a key motivation to the automation of transportation systems is a reduced need of costly infrastructure. Therefore, a major concern of the architecture proposed here is that it requires a central agent at each junction with processing capabilities. However, the only task which is centralized is priority assignment. We believe that the conciseness of the priority graph that merely maps each couple of robot identifiers to a binary value encoding the priority, is a real asset to design a distributed priority assignment policy. However, if first-come-first-serve distributed priority assignment is likely to be quite easily designed, distributed priority assignment performing some combinatorial optimization will be a much more challenging task.

\part*{Conclusions and perspectives}
\addcontentsline{toc}{part}{Conclusions and perspectives}

\fancyhead[RE]{\bfseries\nouppercase{Conclusions and perspectives}} 
\fancyhead[LO]{\bfseries\nouppercase{Conclusions and perspectives}}  

\section*{Conclusions}

This thesis proposed to study the coordination of mobile robots at intersections espousing an approach considering planning as the computation of resources to guide -- not control -- robots. As a result, the plan representation appeared to differ widely from what it is in traditional motion planning. In traditional approaches, planning consists in computing a planned reference trajectory that robots must execute. The reference trajectory constitutes the plan and in the control phase, a control law configured by the reference trajectory is in charge of tracking the planned trajectory. In the approach of this thesis, the plan is the priority graph encoding the particular homotopy class chosen to solve the coordination problem and to our knowledge, there existed no standard control scheme to ensure the described trajectory belongs to the chosen homotopy class, i.e., to ensure priorities are respected. In priority-based coordination, as there is no reference trajectory, there is no trajectory tracking which is replaced by so-called priority preserving control. Our approach using priorities as a plan to guide -- not control -- robots confirmed suitable to coordinate multiple robots at an intersection area, endowing the system with robustness properties. Part~\ref{part:priority-framework} provides a powerful tool to characterize the structure of solutions to the coordination problem: the priority graph. Previous work already noticed the existence of homotopy classes of feasible paths in the coordination space, yet without providing a meaningful representative of homotopy classes. The main contribution of the first part of the thesis is to provide such a meaningful representative: the priority graph. Priorities uniquely encode the homotopy classes of feasible paths. Choosing a particular priority graph to coordinate robots appears as the discrete part of the coordination problem. It thus provides a geometrical understanding in the coordination space of why planning priorities instead of a precise trajectory results in an increased robustness. It merely appears as the consequence of constraining the path of robots in the coordination space to remain in a homotopy class -- a large set of feasible paths continuously deformable into each other --, instead of assigning a particular precise feasible path to follow. The "size" of the homotopy class provides some freedom of action. Part~\ref{part:priorities-to-guide-robots} demonstrates that robots can easily go through the intersection while respecting priorities in a reactive manner. Under assigned priorities, the combinatorial complexity of multi robot control is avoided as for each pair of robots there is not two strategies to avoid collisions anymore: the robot with lower priority must decelerate in favor of the robot with higher priority. It is thus not surprising that priority preserving control can be carried out in polynomial time. From the coordination space point of view, priority preserving control ensures that the trajectory described by the multi robot system belongs to the homotopy class encoded by the assigned priorities. To this purpose, a control law is configured by the assigned priorities and is in charge of priority preservation. It is very different from the trajectory tracking approach as it allows for example all robots to stop for a while to handle some unexpected event while respecting priorities, i.e., without replanning. By contrast, in traditional motion planning, the plan must be executed, and the only way to recover some freedom of action is replanning. In the absence of inertia, the control law proposed in Chapter~\ref{chap:optimal-control-velocity} ensures that the resulting trajectory is optimal for the assigned priorities, recovering the existence of a local optimum in each homotopy class~\cite{Ghrist2005}. Even though the dynamics model used in Chapter~\ref{chap:control-acceleration} is quite simple, it convinces that the additional complexity in the presence of kinodynamic constraints can be easily tackled by introducing the brake safety constraint. The byproduct of the conservative brake safety constraint is an increased robustness regarding unexpected deceleration of robots. Moreover, the proposed control scheme is decentralized as the output of the control law can be computed on each robot independently, thus not requiring any form of agreement through communication links. This valuable benefit is allowed by the prior agreement on the priority graph which is done at the planning level. This prior agreement which requires some form of communication enables to select a particular strategy for collision avoidance -- the priority graph --, so that no more agreement is required at the control level which can be decentralized. Finally, Chapter~\ref{chap:control-uncertainty} provides some elements to control the system in the presence of uncertainty. It shows that bounded uncertainty can be handled by considering worst-case scenarios. Even though the uncertainty model is quite simple, it demonstrates that the key benefit of priority-based coordination is its ability to handle uncertainty in a reactive manner. For example, robots may stop for a while if uncertainty is very large due to communication concerns and restart without replanning, merely using a control law in the information space. Part~\ref{part:priority-based-coordination} has a more engineering flavor and proposes a three-layer architecture integrating both priority assignment and priority preserving control which are executed in parallel. Priority assignment is carried out by a central deliberative intersection controller. Robots implement multiple behaviors including one ensuring priorities are respected. Robot's sequencer interfaces the reactive and the deliberative layers through asynchronous communication with the intersection controller to negotiate the entry of the control area and by activating/deactivating/configuring primitive behaviors. Compared to traditional plan-as-program approaches, robots retain reactive capabilities through the intersection. A large class of unexpected events -- all events requiring braking -- can be handled in a reactive manner without need to replan endowing the system with significant robustness. This thesis proposes a novel class of coordination systems at intersections -- using priorities to guide robots -- and therefore still suffers from some limitations and opens several perspectives for future work.

\section*{Limitations and perspectives}

\subsection*{From the theoretical point of view}

\paragraph{Homotopy classes under imperfect lateral control} The path-following assumption of Figure~\ref{fig-paths} is key to the definition of priorities and to the existence of homotopy classes of feasible paths uniquely encoded by priorities. In real systems, perfect path following cannot be guaranteed as lateral control is based on imperfect mapping/localization data and imperfect actuators. Hence, future work should investigate which assumptions on lateral control still guarantee all the results of Part~\ref{part:priority-framework} which is the foundation of the priority-based approach. We believe that under bounded uncertainty on lateral control, the results of Part~\ref{part:priority-framework} can be extended by considering the worst-case obstacle region considering all possible geometric paths. However, this could decrease performance and it would raise the problem of handling a lateral error beyond the fixed bound.

\paragraph{Dealing with partial information} In the current setting, all robots are assumed to know the assigned priorities and the current state of other robots. Even though Chapter~\ref{chap:control-uncertainty} provides elements on how to deal with bounded uncertainty, it still assumes that all robots know the complete priority graph and the current non-deterministic information state of all other robots, which is still a strong assumption. Further work should focus on relaxing these assumptions. We believe that priority-based coordination has real strengths to deal with partial information concerns. First of all, some priorities are redundant, as if robots $2$ and $3$ travel along the same path (say $3$ follows $2$) and have both priority over robot $1$, then robot $1$ only needs to know that it has priority over the first of the two robots, i.e., robot $2$. Moreover, to execute the priority preserving control law, a given robot only needs to know the current state of robots in the neighborhood as there is no need to anticipate beyond a certain area.

\paragraph{Distributed priority assignment} As highlighted in the conclusion of Part~\ref{part:priority-based-coordination}, while a key motivation to the automation of transportation systems is a reduced need of costly infrastructure, the proposed three-layer architecture requires a central agent at each junction with processing capabilities. Distributed priority assignment would imply a consensus algorithm as all robots need to agree on a common decision: the assigned priorities. We believe that without efficiency considerations, previous work on consensus algorithms should help to design simple distributed priority assignment policies, e.g., a first-come-first-serve policy. However, optimizing priorities requires to perform time consuming algorithms reasoning about the future and a distributed implementation of such algorithms should prove challenging.

\subsection*{From the application point of view}

\paragraph{Challenges for an implementation in real systems} First of all, localization and mapping aspects have not been addressed in this thesis and are challenges in themselves. These topics are intensive research fields, both in robotics and intelligent transportation systems communities. Priority-based coordination requires building a map specifying predefined geometric paths to go through the intersection and robots need to have an estimate of their position on their path. We believe that for an application in self-driven vehicles, existing maps of the road network and lane markings/panels -- more generally, the physical infrastructure -- should help localization and mapping tasks. Communication is another important aspect that has been left behind. Standardized messages should be designed to support priority-based coordination, taking into account constraints in terms of delay and amount of data. This is one of the tasks currently achieved in the European project Autonet2030 where a complete cooperative system architecture is designed for the cooperation of intelligent vehicles supporting, in particular, coordination at intersections. We believe that the conciseness of the plan representation -- the priority graph is merely mapping couple of robot identifiers to a binary value -- is a valuable feature of priority-based coordination as it limits the amount of data to be exchanged. By contrast, plan-as-program approaches need to exchange precise trajectories which include much more data.

\paragraph{Sharing the road between autonomous and human-driven vehicles} Autonomous vehicles will arrive gradually, and they will have to "share the road" for a while. According to~\cite{LaFortelle2014}, only 50\% of vehicles will be autonomous by 2030. Cooperation between autonomous/semi-autonomous/human-driven vehicles thus appears necessary. As respecting priorities is a capability of both humans and robots, the priority-based approach is particularly adapted for the development of algorithms aiming at coordinate both human-driven and autonomous vehicles. In~\cite{Qian2013}, a priority-based autonomous intersection management system is proposed in this context of "mixed traffic flow". Priorities are assigned by an intersection controller, yet human-driven vehicles are not aware of that, and just respect traffic signals as in a usual signalized intersection. A video capture of simulations is available \href{http://www.youtube.com/watch?v=L3B_FrNn_Pk}{here}\footnote{\url{http://www.youtube.com/watch?v=L3B_FrNn_Pk}}. 

\paragraph{Energy efficient priority preserving control} Priority preserving control proposed in Part~\ref{part:priorities-to-guide-robots} is a kind of bang-bang control switching quite abruptly between maximum brake and maximum throttle. Such control law, while priority preserving, raises energy efficiency concerns as abrupt control switches are energy consuming. To this purpose, we believe that model predictive control with a cost function accounting for both priority preservation, delay reduction and energy efficiency should prove useful by anticipating the need to brake and result in smoother trajectories as demonstrated by recent work~\cite{Makarem2013}.

\paragraph{Handling priority violation} The present work assumed robots respect the assigned priorities. We are convinced that the ability to respect priorities is a must that self-driven vehicles need to possess to be deployed on the roads. People would accept that vehicles cannot follow a precise trajectory through an intersection. By contrast, they would not accept an autonomous vehicle unable to respect assigned priorities. That is why we believe that priority violations should occur mainly under major system failure of one robot. Another reason why it can occur in an application to self-driven cars is when a driver decides to stop the self-driving system and to take back the control of the vehicle for some unexpected reason. Such circumstances would require both priority violation detection and real-time dynamic priority assignment with all the attention paid on safety, i.e., collision avoidance. 

\paragraph{Towards a network of autonomous intersections}

All the work around traffic signal control demonstrates that controlling a network of intersections is of high complexity. Recent work based on back-pressure algorithms tend do demonstrate that queues stability guarantees of the network can be obtained, while each intersection controller uses only local information. Future work should extend the simple back-pressure priority assignment policy at a single intersection proposed in Chapter~\ref{chap:coordination-system} and consider adaptive priority assignment at a network of autonomous intersections using back-pressure algorithms.

\paragraph{Towards coordination of aerial drones}

Finally, even though this work was originally motivated by applications in autonomous vehicles on roadways, applications to other fields should be investigated. For instance, the results of this thesis may be applicable to coordinate aerial drones in a three-dimensional space provided the geometric three-dimensional paths followed by drones are fixed. If so, each drone $i$ still has only one degree of freedom -- its curvilinear coordinate along its three-dimensionnal path $\gamma_i \subset \RR^3$ -- and the priority-based framework proposed in this thesis is still applicable.

%\nocite{*}
\cleardoublepage
\fancyhead[RE]{\bfseries\nouppercase{\leftmark}}      % Chapter in the right on even pages
\fancyhead[LO]{\bfseries\nouppercase{\rightmark}}     % Section in the left on odd pages
\bookmarksetup{startatroot}% this is it
\bibliographystyle{plain}
\bibliography{Thesis}

\begin{thebibliography}{100}

\bibitem{Cybercars2006}
Cybercars.
\newblock \url{http://www.cybercars.org/}, 2006.
\newblock Accessed: 2014-06-11.

\bibitem{CityMobil2011}
Citymobil - towards advanced road transport for the urban environment.
\newblock \url{http://http://www.citymobil-project.eu/}, 2011.
\newblock Accessed: 2014-06-11.

\bibitem{CityMobil2-2012}
Citymobil2 - cities demonstrating automated road passenger transport.
\newblock \url{http://http://http://www.citymobil2.eu//}, 2012.
\newblock Accessed: 2014-06-11.

\bibitem{Autonet2030}
Autonet2030 (2014): Co-operative systems in support of networked automated
  driving by 2030.
\newblock \url{https://www.autonet2030.eu}, 2014.
\newblock Accessed: 2014-06-11.

\bibitem{Aboudolas2009}
K~Aboudolas, M~Papageorgiou, and E~Kosmatopoulos.
\newblock Store-and-forward based methods for the signal control problem in
  large-scale congested urban road networks.
\newblock {\em Transportation Research Part C: Emerging Technologies},
  17(2):163--174, 2009.

\bibitem{Agre1990}
Philip~E Agre and David Chapman.
\newblock What are plans for?
\newblock {\em Robotics and autonomous systems}, 6(1):17--34, 1990.

\bibitem{Akella2002}
S.~Akella and S.~Hutchinson.
\newblock Coordinating the motions of multiple robots with specified
  trajectories.
\newblock In {\em Robotics and Automation, 2002. Proceedings. ICRA '02. IEEE
  International Conference on}, volume~1, pages 624 -- 631 vol.1, 2002.

\bibitem{Alvarez1997}
Luis Alvarez and Roberto Horowitz.
\newblock Safe platooning in automated highway systems.
\newblock Institute of Transportation Studies, Research Reports, Working
  Papers, Proceedings qt1v97t5w1, Institute of Transportation Studies, UC
  Berkeley, January 1997.

\bibitem{Angeli2003}
David Angeli and Eduardo~D Sontag.
\newblock Monotone control systems.
\newblock {\em Automatic Control, IEEE Transactions on}, 48(10):1684--1698,
  2003.

\bibitem{Araki1977}
T.~Araki, Y.~Sugiyama, and T.~Kasami.
\newblock {Complexity of the deadlock avoidance problem}.
\newblock 1977.

\bibitem{Au2011}
Tsz-Chiu Au, Neda Shahidi, and Peter Stone.
\newblock Enforcing liveness in autonomous traffic management.
\newblock In {\em Proceedings of the Twenty-Fifth Conference on Artificial
  Intelligence}, pages 1317--1322, August 2011.

\bibitem{Barraquand1991}
Jerome Barraquand and Jean-Claude Latombe.
\newblock Robot motion planning: A distributed representation approach.
\newblock {\em The International Journal of Robotics Research}, 10(6):628--649,
  1991.

\bibitem{Bekris2007}
Kostas~E Bekris, Konstantinos~I Tsianos, and Lydia~E Kavraki.
\newblock A distributed protocol for safe real-time planning of communicating
  vehicles with second-order dynamics.
\newblock In {\em Proceedings of the 1st international conference on Robot
  communication and coordination}, page~9. IEEE Press, 2007.

\bibitem{Bekris2009}
Kostas~E Bekris, Konstantinos~I Tsianos, and Lydia~E Kavraki.
\newblock Safe and distributed kinodynamic replanning for vehicular networks.
\newblock {\em Mobile Networks and Applications}, 14(3):292--308, 2009.

\bibitem{Benenson2008}
Rodrigo Benenson, St{\'e}phane Petti, Thierry Fraichard, and Michel Parent,
  Null.
\newblock {Towards urban driverless vehicles}.
\newblock {\em The International Journal of Robotics Research}, 1/2(6):4 -- 23,
  2008.

\bibitem{Bennewitz2001}
Maren Bennewitz, Wolfram Burgard, and Sebastian Thrun.
\newblock Optimizing schedules for prioritized path planning of multi-robot
  systems.
\newblock In {\em Robotics and Automation, 2001. Proceedings 2001 ICRA. IEEE
  International Conference on}, volume~1, pages 271--276. IEEE, 2001.

\bibitem{Bennewitz2002}
Maren Bennewitz, Wolfram Burgard, and Sebastian Thrun.
\newblock Finding and optimizing solvable priority schemes for decoupled path
  planning techniques for teams of mobile robots.
\newblock {\em Robotics and Autonomous Systems}, 41(2):89--99, 2002.

\bibitem{Bien1992}
Zeungnam Bien and Jihong Lee.
\newblock A minimum-time trajectory planning method for two robots.
\newblock {\em Robotics and Automation, IEEE Transactions on}, 8(3):414--418,
  1992.

\bibitem{Birman1974}
Joan~S Birman.
\newblock {\em Braids, links, and mapping class groups}.
\newblock Number~82. Princeton University Press, 1974.

\bibitem{Bouraine2011}
Sara Bouraine, Thierry Fraichard, and Hassen Salhi.
\newblock Relaxing the inevitable collision state concept to address provably
  safe mobile robot navigation with limited field-of-views in unknown dynamic
  environments.
\newblock In {\em Intelligent Robots and Systems (IROS), 2011 IEEE/RSJ
  International Conference on}, pages 2985 --2991, sept. 2011.

\bibitem{Brooks1986}
Rodney~A Brooks.
\newblock A robust layered control system for a mobile robot.
\newblock {\em Robotics and Automation, IEEE Journal of}, 2(1):14--23, 1986.

\bibitem{Chang1994}
Cheol Chang, Myung~Jin Chung, and Bum~Hee Lee.
\newblock Collision avoidance of two general robot manipulators by minimum
  delay time.
\newblock {\em Systems, Man and Cybernetics, IEEE Transactions on},
  24(3):517--522, 1994.

\bibitem{Clark2002}
Christopher~M Clark, Tim Bretl, and Stephen Rock.
\newblock Applying kinodynamic randomized motion planning with a dynamic
  priority system to multi-robot space systems.
\newblock In {\em Aerospace Conference Proceedings, 2002. IEEE}, volume~7,
  pages 7--3621. IEEE, 2002.

\bibitem{Coffman1971}
E.~G. Coffman and M.~J. Elphick.
\newblock System deadlocks.
\newblock {\em Computing Surveys}, 3:67--78, 1971.

\bibitem{Colombo2012}
A.~Colombo and D.~Del Vecchio.
\newblock Efficient algorithms for collision avoidance at intersections.
\newblock {\em Hybrid Systems: Computation and Control}, 2012.

\bibitem{AsisGarciaCollado2010}
Francisco de~Asis Garcia~Collado.
\newblock Développement d’un algorithme pour le passage des carrefours.
\newblock Master's thesis, Mines ParisTech, 2010.

\bibitem{LaFortelle2014}
Arnaud de~La~Fortelle, Xiangjun Qian, Sébastien Diemer, Jean Grégoire, Fabien
  Moutarde, Silvère Bonnabel, Ali Marjovi, Ignacio Llatser~Andreas Festag, and
  Sjöberg Katrin.
\newblock Network of automated vehicles: The autonet2030 vision.
\newblock In {\em ITS World Congress}, 2014.

\bibitem{DelVecchio2009}
D.~Del~Vecchio, M.~Malisoff, and R.~Verma.
\newblock A separation principle for a class of hybrid automata on a partial
  order.
\newblock In {\em American Control Conference, 2009. ACC '09.}, pages 3638
  --3643, june 2009.

\bibitem{DiFebbraro2002}
A.~Di~Febbraro, D.~Giglio, and N.~Sacco.
\newblock On applying petri nets to determine optimal offsets for coordinated
  traffic light timings.
\newblock In {\em Intelligent Transportation Systems, 2002. Proceedings. The
  IEEE 5th International Conference on}, pages 773--778, 2002.

\bibitem{Diakaki2002}
Christina Diakaki, Markos Papageorgiou, and Kostas Aboudolas.
\newblock A multivariable regulator approach to traffic-responsive network-wide
  signal control.
\newblock {\em Control Engineering Practice}, 10(2):183--195, 2002.

\bibitem{Dimarogonas2006}
Dimos~V Dimarogonas, Kostas~J Kyriakopoulos, and Dimitris Theodorakatos.
\newblock Totally distributed motion control of sphere world multi-agent
  systems using decentralized navigation functions.
\newblock In {\em Robotics and Automation, 2006. ICRA 2006. Proceedings 2006
  IEEE International Conference on}, pages 2430--2435. IEEE, 2006.

\bibitem{Dimarogonas2004}
Dimos~V Dimarogonas, Savvas~G Loizou, Kostas~J Kyriakopoulos, and Michael~M
  Zavlanos.
\newblock Decentralized feedback stabilization and collision avoidance of
  multiple agents.
\newblock {\em NTUA, http://users. ntua. gr/ddimar/TechRep0401. pdf, Tech.
  Report}, 2004.

\bibitem{Dresner2004}
K.~Dresner and P.~Stone.
\newblock Multiagent traffic management: a reservation-based intersection
  control mechanism.
\newblock In {\em Autonomous Agents and Multiagent Systems, 2004. AAMAS 2004.
  Proceedings of the Third International Joint Conference on}, pages 530 --537,
  july 2004.

\bibitem{Dresner2008-mitigating}
Kurt Dresner and Peter Stone.
\newblock Mitigating catastrophic failure at intersections of autonomous
  vehicles.
\newblock In {\em {AAMAS} Workshop on Agents in Traffic and Transportation},
  pages 78--85, Estoril, Portugal, May 2008.

\bibitem{Dresner2008-multiagent-approach}
Kurt Dresner and Peter Stone.
\newblock A multiagent approach to autonomous intersection management.
\newblock {\em Journal of Artificial Intelligence Research}, 31:591--656, March
  2008.

\bibitem{Elsaesser1994}
Chris Elsaesser and Marc~G Slack.
\newblock Integrating deliberative planning in a robot architecture.
\newblock In {\em NASA CONFERENCE PUBLICATION}, pages 782--782. NASA, 1994.

\bibitem{Erdmann1987}
Michael Erdmann and Tomas Lozano-Perez.
\newblock On multiple moving objects.
\newblock {\em Algorithmica}, 2(1-4):477--521, 1987.

\bibitem{Etemad2013}
Aria Etemad.
\newblock interactive – accident avoidance by active intervention for
  intelligent vehicles.
\newblock In {\em 20th ITS World Congress Tokyo 2013}, 2013.

\bibitem{Farinelli2004}
Alessandro Farinelli, Luca Iocchi, and Daniele Nardi.
\newblock Multirobot systems: a classification focused on coordination.
\newblock {\em Systems, Man, and Cybernetics, Part B: Cybernetics, IEEE
  Transactions on}, 34(5):2015--2028, 2004.

\bibitem{Fraichard1989}
T.~Fraichard and C.~Laugier.
\newblock Planning movements for several coordinated vehicles.
\newblock In {\em Intelligent Robots and Systems '89. The Autonomous Mobile
  Robots and Its Applications. IROS '89. Proceedings., IEEE/RSJ International
  Workshop on}, pages 466 --472, sep 1989.

\bibitem{Fraichard2004}
Thierry Fraichard and Hajime Asama.
\newblock {Inevitable collision states - a step towards safer robots?}
\newblock {\em Advanced Robotics -Utrecht-}, 18(10):1001--1024, 2004.

\bibitem{Gartner1983}
Nathan~H Gartner.
\newblock Opac: A demand-responsive strategy for traffic signal control.
\newblock {\em Transportation Research Record}, (906), 1983.

\bibitem{Gartner1975}
Nathan~H Gartner, John~DC Little, and Henry Gabbay.
\newblock Optimization of traffic signal settings by mixed-integer linear
  programming part i: The network coordination problem.
\newblock {\em Transportation Science}, 9(4):321--343, 1975.

\bibitem{Gat1998}
Erann Gat.
\newblock Three-layer architectures.
\newblock In {\em Artificial intelligence and mobile robots}, pages 195--210.
  MIT Press, 1998.

\bibitem{Ghaemi2014}
R.~Ghaemi and D.~Del~Vecchio.
\newblock Control for safety specifications of systems with imperfect
  information on a partial order.
\newblock {\em Automatic Control, IEEE Transactions on}, 59(4):982--995, April
  2014.

\bibitem{Ghrist2006}
Robert Ghrist and Steven~M. Lavalle.
\newblock Nonpositive curvature and pareto optimal coordination of robots.
\newblock {\em SIAM J. Control Optim.}, 45:1697--1713, November 2006.

\bibitem{Ghrist2005}
Robert Ghrist, Jason~M O'Kane, and Steven~M LaValle.
\newblock Computing pareto optimal coordinations on roadmaps.
\newblock {\em The International Journal of Robotics Research}, 12:997--1012,
  2005.

\bibitem{Gregoire2012-optimal}
Jean Gregoire, Silv{\`e}re Bonnabel, and Arnaud De~La~Fortelle.
\newblock Optimal cooperative motion planning for vehicles at intersections.
\newblock In {\em Navigation, Perception, Accurate Positioning and Mapping for
  Intelligent Vehicles, Workshop, 2012 IEEE Intelligent Vehicles Symposium},
  2012.

\bibitem{Gregoire2013-priority-based}
Jean Gregoire, Silv{\`e}re Bonnabel, and Arnaud de~La~Fortelle.
\newblock Priority-based coordination of robots.
\newblock {\em arXiv preprint arXiv:1306.0785}, 2013.

\bibitem{Gregoire2013-dynamic-constraints}
Jean Gregoire, Silvere Bonnabel, and Arnaud de~La~Fortelle.
\newblock Priority-based intersection management with kinodynamic constraints.
\newblock In {\em Proceedings of the European Control Conference, ECC'14},
  2014.

\bibitem{Gregoire2014-unknown-routing}
Jean Gregoire, Emilio Frazzoli, Arnaud de~La~Fortelle, and Tichakorn
  Wongpiromsarn.
\newblock Back-pressure traffic signal control with unknown routing rates.
\newblock In {\em IFAC World Congress, Cape Town, South Africa}, 2014.

\bibitem{Gregoire2013-capacity}
Jean Gregoire, Emilio Frazzoli, Arnaud de~La~Fortelle, and Tichakorn
  Wongpiromsarn.
\newblock Capacity-aware back-pressure traffic signal control.
\newblock {\em IEEE Transactions on Control of Network Systems}, 2014.
\newblock Accepted.

\bibitem{Gregoire1978}
NM~Gregoire, A~Gjedde, F~Plum, and TE~Duffy.
\newblock Cerebral blood flow and cerebral metabolic rates for oxygen, glucose,
  and ketone bodies in newborn dogs.
\newblock {\em Journal of neurochemistry}, 30(1):63--69, 1978.

\bibitem{Guo2002}
Yi~Guo and L.E. Parker.
\newblock A distributed and optimal motion planning approach for multiple
  mobile robots.
\newblock In {\em Robotics and Automation, 2002. Proceedings. ICRA '02. IEEE
  International Conference on}, volume~3, pages 2612 --2619, 2002.

\bibitem{Hafner2011}
MR~Hafner, D~Cunningham, L~Caminiti, and D~Del~Vecchio.
\newblock Automated vehicle-to-vehicle collision avoidance at intersections.
\newblock In {\em Proceedings of World Congress on Intelligent Transport
  Systems}, 2011.

\bibitem{Hausknecht2011}
Matthew Hausknecht, Tsz-Chiu Au, and Peter Stone.
\newblock Autonomous intersection management: Multi-intersection optimization.
\newblock In {\em Proceedings of IEEE/RSJ International Conference on
  Intelligent Robots and Systems (IROS)}, September 2011.

\bibitem{Henry1984}
Jean-Jacques Henry, Jean-Loup Farges, and J~Tuffal.
\newblock The prodyn real time traffic algorithm.
\newblock In {\em Proceedings of the 4th IFAC/IFORS Conference on Control in
  Transportation Systems,}, 1984.

\bibitem{Hopcroft1984}
J.E. Hopcroft, J.T. Schwartz, and M.~Sharir.
\newblock On the complexity of motion planning for multiple independent
  objects: $pspace$-hardness of the `warehouseman's problem'.
\newblock {\em The International Journal of Robotics Research}, 3(4):76--88,
  1984.

\bibitem{Horowitz2000}
R.~Horowitz and P.~Varaiya.
\newblock Control design of an automated highway system.
\newblock {\em Proceedings of the IEEE}, 88(7):913 --925, jul 2000.

\bibitem{Hu2003}
Jianghai Hu, Maria Prandini, and Shankar Sastry.
\newblock Optimal coordinated motions of multiple agents moving on a plane.
\newblock {\em SIAM Journal on Control and Optimization}, 42(2):637--668, 2003.

\bibitem{Hu1959}
Sze-tsen Hu and Valerie Hu.
\newblock {\em Homotopy theory}, volume~8.
\newblock Academic press, 1959.

\bibitem{Hunt1982}
PB~Hunt, DI~Robertson, RD~Bretherton, and MC~Royle.
\newblock The scoot on-line traffic signal optimisation technique.
\newblock {\em Traffic Engineering \& Control}, 23(4), 1982.

\bibitem{Jiang1997}
Zhong-Ping Jiang and Henk Nijmeijer.
\newblock Tracking control of mobile robots: a case study in backstepping.
\newblock {\em Automatica}, 33(7):1393--1399, 1997.

\bibitem{JingmanFan2011}
Arnaud de La~Fortelle Jingman~Fan and Silvère Bonnabelle.
\newblock Modélisation et résolution du problème de passages de véhicules
  dans un carrefour.
\newblock Technical report, Mines ParisTech -- Centre de robotique, 2011.

\bibitem{Kant1986}
K.~Kant and S.~W. Zucker.
\newblock Toward efficient trajectory planning: The path-velocity
  decomposition.
\newblock {\em International Journal of Robotics Research}, 5(3):72--89, 1986.

\bibitem{Karaman2010}
S.~Karaman and E.~Frazzoli.
\newblock Incremental sampling-based algorithms for optimal motion planning.
\newblock In {\em Proceedings of Robotics: Science and Systems}, Zaragoza,
  Spain, June 2010.

\bibitem{Kavraki1996}
L.E. Kavraki, P.~Svestka, J.-C. Latombe, and M.H. Overmars.
\newblock Probabilistic roadmaps for path planning in high-dimensional
  configuration spaces.
\newblock {\em Robotics and Automation, IEEE Transactions on}, 12(4):566--580,
  1996.

\bibitem{Kerrigan2000}
Eric~C Kerrigan.
\newblock {\em Robust constraint satisfaction: Invariant sets and predictive
  control}.
\newblock PhD thesis, PhD thesis, Cambridge, 2000.

\bibitem{Kim2003}
Jinsuck Kim, Roger~A Pearce, and Nancy~M Amato.
\newblock Extracting optimal paths from roadmaps for motion planning.
\newblock In {\em Robotics and Automation, 2003. Proceedings. ICRA'03. IEEE
  International Conference on}, volume~2, pages 2424--2429. IEEE, 2003.

\bibitem{Kolodko2003}
Julian Kolodko and Ljubo Vlacic.
\newblock Cooperative autonomous driving at the intelligent control systems
  laboratory.
\newblock {\em Intelligent Systems, IEEE}, 18(4):8--11, 2003.

\bibitem{Kowshik2011}
H.~Kowshik, D.~Caveney, and P.R. Kumar.
\newblock Provable systemwide safety in intelligent intersections.
\newblock {\em Vehicular Technology, IEEE Transactions on}, 60(3):804 --818,
  march 2011.

\bibitem{Latombe1991}
Jean-Claude Latombe.
\newblock {\em Robot Motion Planning}.
\newblock Kluwer Academic Publishers, Norwell, MA, USA, 1991.

\bibitem{LaValle2006}
S.~M. LaValle.
\newblock {\em Planning Algorithms}.
\newblock Cambridge University Press, Cambridge, U.K., 2006.
\newblock Available at http://planning.cs.uiuc.edu/.

\bibitem{LaValle1998-RRT}
Steven~M LaValle.
\newblock Rapidly-exploring random trees: A new tool for path planning.
\newblock Technical report, 1998.

\bibitem{LaValle1996-uncertainty}
Steven~M LaValle and Seth~A Hutchinson.
\newblock Evaluating motion strategies under nondeterministic or probabilistic
  uncertainties in sensing and control.
\newblock In {\em Robotics and Automation, 1996. Proceedings., 1996 IEEE
  International Conference on}, volume~4, pages 3034--3039. IEEE, 1996.

\bibitem{LaValle1996-optimal}
Steven~M LaValle and Seth~A Hutchinson.
\newblock Optimal motion planning for multiple robots having independent goals.
\newblock In {\em Robotics and Automation, 1996. Proceedings., 1996 IEEE
  International Conference on}, volume~3, pages 2847 --2852 vol.3, apr 1996.

\bibitem{LaValle1998-uncertainty-objective-based}
Steven~M LaValle and Seth~A Hutchinson.
\newblock An objective-based framework for motion planning under sensing and
  control uncertainties.
\newblock {\em The International Journal of Robotics Research}, 17(1):19--42,
  1998.

\bibitem{Lawley2001}
Mark Lawley and Spyros Reveliotis.
\newblock Deadlock avoidance for sequential resource allocation systems: Hard
  and easy cases.
\newblock {\em IEEE Transactions on Automatic Control}, 46:1572--1583, 2001.

\bibitem{Le2013}
Tung Le, Peter Kovacs, Neil Walton, Hai~L Vu, Lachlan~L Andrew, and Serge~S
  Hoogendoorn.
\newblock Decentralized signal control for urban road networks.
\newblock {\em arXiv preprint arXiv:1310.0491}, 2013.

\bibitem{Lee2001}
Ti-Chung Lee, Kai-Tai Song, Ching-Hung Lee, and Ching-Cheng Teng.
\newblock Tracking control of unicycle-modeled mobile robots using a saturation
  feedback controller.
\newblock {\em Control Systems Technology, IEEE Transactions on},
  9(2):305--318, 2001.

\bibitem{Leroy1999}
S.~Leroy, J.~P. Laumond, and T.~Simeon.
\newblock Multiple path coordination for mobile robots: A geometric algorithm.
\newblock In {\em In Proc. of the International Joint Conference on Artificial
  Intelligence (IJCAI}, pages 1118--1123, 1999.

\bibitem{Lo2001}
Hong~K Lo, Elbert Chang, and Yiu~Cho Chan.
\newblock Dynamic network traffic control.
\newblock {\em Transportation Research Part A: Policy and Practice},
  35(8):721--744, 2001.

\bibitem{Loizou2002}
Savvas~G Loizou and Kostas~J Kyriakopoulos.
\newblock Closed loop navigation for multiple holonomic vehicles.
\newblock In {\em Intelligent Robots and Systems, 2002. IEEE/RSJ International
  Conference on}, volume~3, pages 2861--2866. IEEE, 2002.

\bibitem{Lowrie1990}
PR~Lowrie.
\newblock Scats, sydney co-ordinated adaptive traffic system: A traffic
  responsive method of controlling urban traffic.
\newblock Technical report, 1990.

\bibitem{Lozano-Perez1980}
Tomas Lozano-Perez.
\newblock Spatial planning: A configuration space approach.
\newblock Technical report, 1980.

\bibitem{Lumelsky1997}
Vladimir~J. Lumelsky and KR~Harinarayan.
\newblock Decentralized motion planning for multiple mobile robots: The
  cocktail party model.
\newblock In {\em Robot colonies}, pages 121--135. Springer, 1997.

\bibitem{Lumelsky1987}
Vladimir~J Lumelsky and Alexander~A Stepanov.
\newblock Path-planning strategies for a point mobile automaton moving amidst
  unknown obstacles of arbitrary shape.
\newblock {\em Algorithmica}, 2(1-4):403--430, 1987.

\bibitem{Makarem2013}
Laleh Makarem and Denis Gillet.
\newblock Model predictive coordination of autonomous vehicles crossing
  intersections.
\newblock In {\em Intelligent Transportation Systems-(ITSC), 2013 16th
  International IEEE Conference on}, pages 1799--1804. IEEE, 2013.

\bibitem{Mataric1999}
Maja~J Mataric.
\newblock Behavior-based robotics.
\newblock {\em MIT Encyclopedia of Cognitive Sciences}, pages 74--77, 1999.

\bibitem{Medeiros1998}
Adelardo~AD Medeiros.
\newblock A survey of control architectures for autonomous mobile robots.
\newblock {\em Journal of the Brazilian Computer Society}, 4(3), 1998.

\bibitem{Mehani2007}
Olivier Mehani and Arnaud de~La~Fortelle.
\newblock Trajectory planning in a crossroads for a fleet of driverless
  vehicles.
\newblock In {\em Proceedings of the 11th international conference on Computer
  aided systems theory}, EUROCAST'07, pages 1159--1166, Berlin, Heidelberg,
  2007. Springer-Verlag.

\bibitem{Micaelli1993}
Alain Micaelli, Claude Samson, et~al.
\newblock Trajectory tracking for unicycle-type and two-steering-wheels mobile
  robots.
\newblock 1993.

\bibitem{Miller1963}
Alan~J Miller.
\newblock Settings for fixed-cycle traffic signals.
\newblock {\em Operations Research}, pages 373--386, 1963.

\bibitem{Mirchandani2001}
Pitu Mirchandani and Larry Head.
\newblock A real-time traffic signal control system: architecture, algorithms,
  and analysis.
\newblock {\em Transportation Research Part C: Emerging Technologies},
  9(6):415--432, 2001.

\bibitem{NCSA2004}
NCSA.
\newblock National center for statistics and analysis, traffic safety facts
  2003.
\newblock Technical report, U.S. DOT, Washington, DC, 2004.

\bibitem{Neely2005}
Michael~J Neely, Eytan Modiano, and Charles~E Rohrs.
\newblock Dynamic power allocation and routing for time-varying wireless
  networks.
\newblock {\em Selected Areas in Communications, IEEE Journal on},
  23(1):89--103, 2005.

\bibitem{ODonnell1989}
P.A. O'Donnell and T.~Lozano-Periz.
\newblock Deadlock-free and collision-free coordination of two robot
  manipulators.
\newblock In {\em Robotics and Automation, 1989. Proceedings., 1989 IEEE
  International Conference on}, pages 484 --489 vol.1, may 1989.

\bibitem{Osorio2009-analytic-finite-capacity}
Carolina Osorio and Michel Bierlaire.
\newblock An analytic finite capacity queueing network model capturing the
  propagation of congestion and blocking.
\newblock {\em European Journal of Operational Research}, 196(3):996--1007,
  2009.

\bibitem{Osorio2009-surrogate-model}
Carolina Osorio and Michel Bierlaire.
\newblock A surrogate model for traffic optimization of congested networks: an
  analytic queueing network approach.
\newblock {\em Report TRANSP-OR}, 90825:1--23, 2009.

\bibitem{Pallottino2007}
Lucia Pallottino, Vincenzo~Giovanni Scordio, Antonio Bicchi, and Emilio
  Frazzoli.
\newblock Decentralized cooperative policy for conflict resolution in
  multivehicle systems.
\newblock {\em Robotics, IEEE Transactions on}, 23(6):1170--1183, 2007.

\bibitem{Papageorgiou2003}
Markos Papageorgiou, Christina Diakaki, Vaya Dinopoulou, Apostolos Kotsialos,
  and Yibing Wang.
\newblock Review of road traffic control strategies.
\newblock {\em Proceedings of the IEEE}, 91(12):2043--2067, 2003.

\bibitem{Payton1990}
David~W Payton.
\newblock Internalized plans: A representation for action resources.
\newblock {\em Robotics and Autonomous Systems}, 6(1):89--103, 1990.

\bibitem{Peng2003}
Jufeng Peng and Srinivas Akella.
\newblock Coordinating the motions of multiple robots with kinodynamic
  constraints.
\newblock In {\em Robotics and Automation, 2003. Proceedings. ICRA'03. IEEE
  International Conference on}, volume~3, pages 4066--4073. IEEE, 2003.

\bibitem{Peng2005}
Jufeng Peng and Srinivas Akella.
\newblock Coordinating multiple robots with kinodynamic constraints along
  specified paths.
\newblock {\em The International Journal of Robotics Research}, 24(4):295--310,
  2005.

\bibitem{Pereira2003}
Guilherme~AS Pereira, Aveek~K Das, R~Vijay Kumar, and Mario~FM Campos.
\newblock Decentralized motion planning for multiple robots subject to sensing
  and communication constraints.
\newblock 2003.

\bibitem{Petti2005-reactive-planning}
Stephane Petti and Thierry Fraichard.
\newblock {Reactive Planning Under Uncertainty Among Moving Obstacles}.
\newblock In {\em {Proc. of the Int. Symp. on Robotics}}, Tokyo (JP), France,
  November 2005.
\newblock voir basilic : http://emotion.inrialpes.fr/bibemotion/2005/PF05c/
  address: Tokyo (JP).

\bibitem{Petti2005-partial-motion-planning}
Stéphane Petti and Thierry Fraichard.
\newblock Partial motion planning framework for reactive planning within
  dynamic environments.
\newblock In {\em in AAAI Intl. Conf. ICAR}, 2005.

\bibitem{Qian2013}
Xiangjun Qian, Jean Gregoire, Fabien Moutarde, and Arnaud de~La~Fortelle.
\newblock Autonomous intersection management for mixed traffic flow.
\newblock 2013.

\bibitem{Regele2006}
Ralf Regele and Paul Levi.
\newblock Cooperative multi-robot path planning by heuristic priority
  adjustment.
\newblock In {\em Intelligent Robots and Systems, 2006 IEEE/RSJ International
  Conference on}, pages 5954--5959. IEEE, 2006.

\bibitem{Reveliotis2010}
S.A. Reveliotis and E.~Roszkowska.
\newblock On the complexity of maximally permissive deadlock avoidance in
  multi-vehicle traffic systems.
\newblock {\em Automatic Control, IEEE Transactions on}, 55(7):1646--1651,
  2010.

\bibitem{Reveliotis2011}
Spyros~A Reveliotis and Elzbieta Roszkowska.
\newblock Conflict resolution in free-ranging multivehicle systems: A resource
  allocation paradigm.
\newblock {\em Robotics, IEEE Transactions on}, 27(2):283--296, 2011.

\bibitem{Rimon1992}
Elon Rimon and Daniel~E Koditschek.
\newblock Exact robot navigation using artificial potential functions.
\newblock {\em Robotics and Automation, IEEE Transactions on}, 8(5):501--518,
  1992.

\bibitem{Roszkowska2008}
Elzbieta Roszkowska and Spyros~A Reveliotis.
\newblock On the liveness of guidepath-based, zone-controlled dynamically
  routed, closed traffic systems.
\newblock {\em Automatic Control, IEEE Transactions on}, 53(7):1689--1695,
  2008.

\bibitem{Saha2006}
Mitul Saha and Pekka Isto.
\newblock Multi-robot motion planning by incremental coordination.
\newblock In {\em Intelligent Robots and Systems, 2006 IEEE/RSJ International
  Conference on}, pages 5960--5963. IEEE, 2006.

\bibitem{Schwartz1983}
Jacob~T Schwartz and Micha Sharir.
\newblock On the piano movers' problem: Iii. coordinating the motion of several
  independent bodies: the special case of circular bodies moving amidst
  polygonal barriers.
\newblock {\em The International Journal of Robotics Research}, 2(3):46--75,
  1983.

\bibitem{Serra1983}
Jean Serra.
\newblock {\em Image analysis and mathematical morphology}.
\newblock Academic Press, Inc., 1983.

\bibitem{Sharma2005}
Vikrant Sharma, Michael~A Savchenko, Emilio Frazzoli, and Petros~G Voulgaris.
\newblock Time complexity of sensor-based vehicle routing.
\newblock In {\em Robotics: Science and Systems}, pages 297--304. Citeseer,
  2005.

\bibitem{Sheng2006}
Weihua Sheng, Qingyan Yang, and Yi~Guo.
\newblock Cooperative driving based on inter-vehicle communications:
  Experimental platform and algorithm.
\newblock In {\em Intelligent Robots and Systems, 2006 IEEE/RSJ International
  Conference on}, pages 5073--5078, 2006.

\bibitem{Shepherd1992}
SP~Shepherd.
\newblock A review of traffic signal control.
\newblock Technical report, 1992.

\bibitem{Shin1992}
Kang~G Shin and Zin Zheng.
\newblock Minimum-time collision-free trajectory planning for dual-robot
  systems.
\newblock {\em Robotics and Automation, IEEE Transactions on}, 8(5):641--644,
  1992.

\bibitem{Shladover2008}
S.E. Shladover.
\newblock Ahs research at the california path program and future ahs research
  needs.
\newblock In {\em Vehicular Electronics and Safety, 2008. ICVES 2008. IEEE
  International Conference on}, pages 4--5, Sept 2008.

\bibitem{Soetanto2003}
D~Soetanto, L~Lapierre, and A~Pascoal.
\newblock Adaptive, non-singular path-following control of dynamic wheeled
  robots.
\newblock In {\em Decision and Control, 2003. Proceedings. 42nd IEEE Conference
  on}, volume~2, pages 1765--1770. IEEE, 2003.

\bibitem{Suchman1986}
Lucy Suchman.
\newblock Plans and situated actions.
\newblock {\em New York, Cambridge University}, 1986.

\bibitem{Suh1988}
S-H Suh and Kang~G Shin.
\newblock A variational dynamic programming approach to robot-path planning
  with a distance-safety criterion.
\newblock {\em Robotics and Automation, IEEE Journal of}, 4(3):334--349, 1988.

\bibitem{Sutton1991}
Richard~S Sutton.
\newblock Planning by incremental dynamic programming.
\newblock In {\em ML}, pages 353--357. Citeseer, 1991.

\bibitem{Svestka1995}
P.~Svestka and M.H. Overmars.
\newblock Coordinated motion planning for multiple car-like robots using
  probabilistic roadmaps.
\newblock In {\em Robotics and Automation, 1995. Proceedings., 1995 IEEE
  International Conference on}, volume~2, pages 1631--1636 vol.2, 1995.

\bibitem{Tassiulas1992}
Leandros Tassiulas and Anthony Ephremides.
\newblock Stability properties of constrained queueing systems and scheduling
  policies for maximum throughput in multihop radio networks.
\newblock {\em IEEE Transactions on Automatic Control}, 37(12):1936--1948,
  1992.

\bibitem{Thrun2006}
Sebastian Thrun, Mike Montemerlo, Hendrik Dahlkamp, David Stavens, Andrei Aron,
  James Diebel, Philip Fong, John Gale, Morgan Halpenny, Gabriel Hoffmann,
  et~al.
\newblock Stanley: The robot that won the darpa grand challenge.
\newblock {\em Journal of field Robotics}, 23(9):661--692, 2006.

\bibitem{Todorov2006}
Emanuel Todorov.
\newblock {\em Optimal control theory}.
\newblock MIT Press, Cambridge, 2006.

\bibitem{Treat1977}
J.R. Treat, N.J. Castellan, R.L. Stansifer, R.E. Mayer, R.D. Hume, D.~Shinar,
  S.T. McDonald, and N.S. Tumbas.
\newblock {\em Tri-level Study of the Causes of Traffic Accidents: Final
  Report. Volume I: Causal Factor Tabulations and Assessments}.
\newblock 1977.

\bibitem{VanDenBerg2011}
Jur Van Den~Berg, Pieter Abbeel, and Ken Goldberg.
\newblock Lqg-mp: Optimized path planning for robots with motion uncertainty
  and imperfect state information.
\newblock {\em The International Journal of Robotics Research}, 30(7):895--913,
  2011.

\bibitem{VanDenBerg2005}
Jur~P Van Den~Berg and Mark~H Overmars.
\newblock Prioritized motion planning for multiple robots.
\newblock In {\em Intelligent Robots and Systems, 2005.(IROS 2005). 2005
  IEEE/RSJ International Conference on}, pages 430--435. IEEE, 2005.

\bibitem{Dijke2011}
Jan~P van Dijke.
\newblock Citymobil, advanced road transport for the urban environment. final
  results.
\newblock In {\em 18th ITS World Congress}, 2011.

\bibitem{Varaiya2009}
Pravin Varaiya.
\newblock A universal feedback control policy for arbitrary networks of
  signalized intersections.
\newblock Technical report, 2009.

\bibitem{Varaiya2013}
Pravin Varaiya.
\newblock The max-pressure controller for arbitrary networks of signalized
  intersections.
\newblock In {\em Advances in Dynamic Network Modeling in Complex
  Transportation Systems}, pages 27--66. Springer, 2013.

\bibitem{Verma2012}
Rajeev Verma and Domitilla Del~Vecchio.
\newblock Safety control of hidden mode hybrid systems.
\newblock {\em Automatic Control, IEEE Transactions on}, 57(1):62--77, 2012.

\bibitem{Wallace2012}
Richard Wallace and G~Silberg.
\newblock Self-driving cars: the next revolution.
\newblock {\em KPMG and CAR (Center for Automotive research), source:
  http://www. kpmg.
  com/Ca/en/IssuesAndInsights/ArticlesPublications/Documents/selfdriving-cars-next-revolution.
  pdf,(consulted on April 7th 2014)}, 2012.

\bibitem{Werger1999}
Barry~Brian Werger.
\newblock Cooperation without deliberation: A minimal behavior-based approach
  to multi-robot teams.
\newblock {\em Artificial Intelligence}, 110(2):293--320, 1999.

\bibitem{wiki:traffic}
Wikipedia.
\newblock Traffic --- wikipedia{,} the free encyclopedia, 2014.
\newblock [Online; accessed 17-July-2014].

\bibitem{Wongpiromsarn2012}
Tichakorn Wongpiromsarn, Tawit Uthaicharoenpong, Yu~Wang, Emilio Frazzoli, and
  Danwei Wang.
\newblock Distributed traffic signal control for maximum network throughput.
\newblock In {\em Proceedings of the 15th international IEEE conference on
  intelligent transportation systems}, pages 588--595, 2012.

\bibitem{Yang1999}
Jung-Min Yang and Jong-Hwan Kim.
\newblock Sliding mode control for trajectory tracking of nonholonomic wheeled
  mobile robots.
\newblock {\em Robotics and Automation, IEEE Transactions on}, 15(3):578--587,
  1999.

\bibitem{I.Zohdy2012}
I.~Zohdy and H.~Rakha.
\newblock Optimizing driverless vehicles at intersections.
\newblock In {\em 10th ITS World Congress Vienna, Austria}, October 2012.

\end{thebibliography}

\appendix

\addtocontents{toc}{\protect\newpage}
\part*{Appendix}
\addcontentsline{toc}{part}{Appendix}
{\pagestyle{plain}
\parttoc
\cleardoublepage}
\chapter{ Appendix of Chapter~\ref{chap:priority-graph}}
\section{Topology results}

In the sequel, given a topological space $X$ and a subset $C\subset X$, $C^c$ denotes the complementary of a $C$ in $X$, $\cl{C}$ its closure and $\partial C$ its boundary. 

\label{app:topology-properties}
\begin{lemma}
\label{lemma:two-subsets-frontier}
Let $A$ and $B$ be two disjoint subsets of a topological space $X$ ($A \cap B=\emptyset$). Let $f:[0,1]\to A\cup B$ denote a continuous application taking values in $A\cup B$ and satisfying $f(0)\in A$ and $f(1)\in B$. Then there exists $t^0\in[0,1]$ such that $f(t^0)\in \partial A \cap \partial B$.
\end{lemma}

\begin{figure}[!htbp]
\begin{center}
\includegraphics[width=0.5\linewidth]{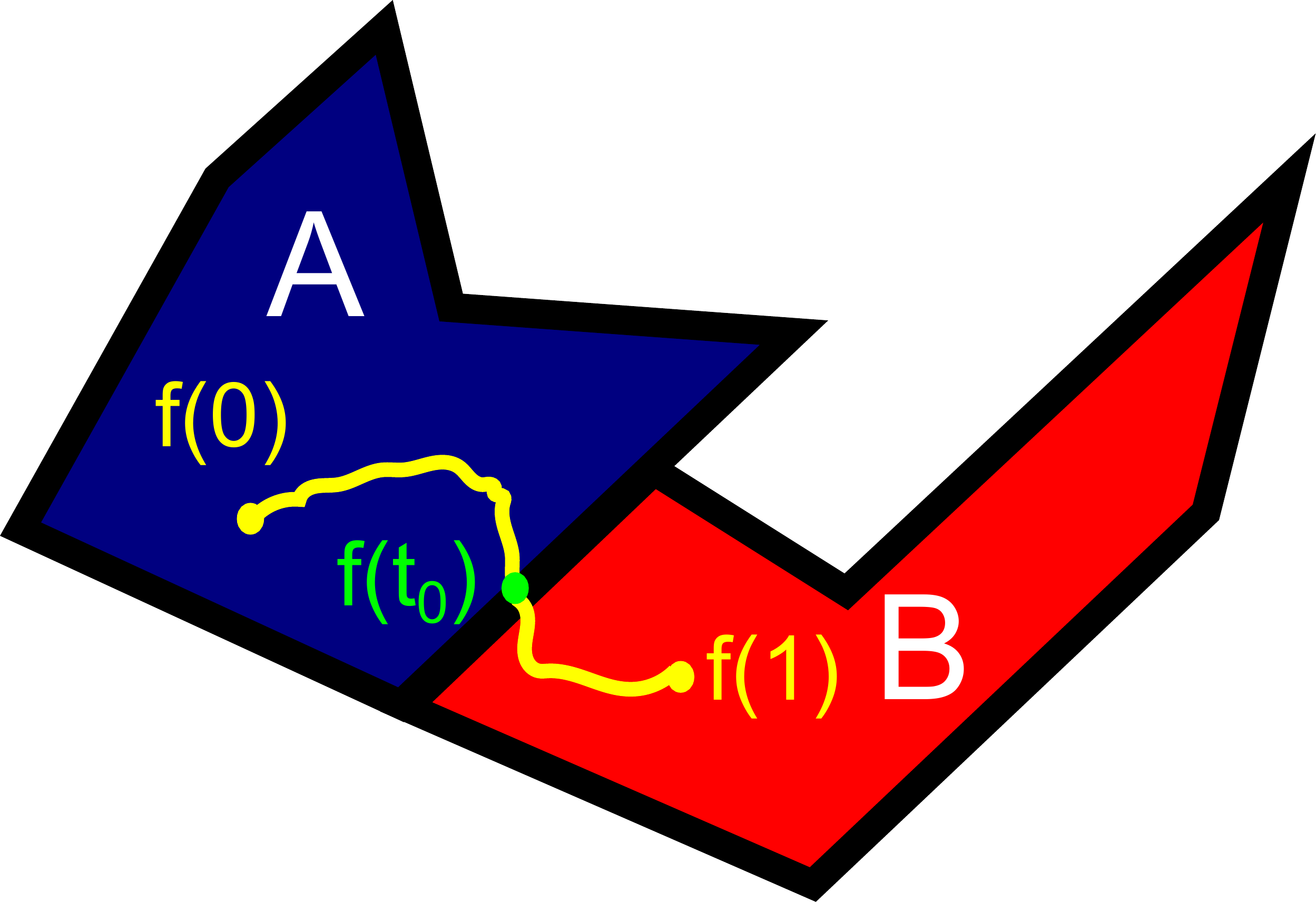}\hfill
\end{center}
\caption{Illustration of Lemma~\ref{lemma:two-subsets-frontier}}
\label{fig:topology-continuous-map-A-B}
\end{figure}

\begin{proof}
Consider $t^0$ defined below:
\begin{eqnarray}
t^0 &:=& \sup\{ t\in[0,1]:f(t)\in A\}\label{eq:definition-t0}\\
&=&\sup\{ t\in[0,1]:f(t)\notin B\} \nonumber
\end{eqnarray}
It exists as the supremum of a non-empty ($f(0)\in A$) upper-bounded subset of $\RR$. 

Take some $r>0$ and let $\ball(f(t^0),r)$ denote the open ball of radius $r$ centered in $f(t^0)$. Consider two options: 
\begin{itemize}
\item either $t^0=1$, and $\ball(f(t^0),r)\cap B \neq \emptyset$ as $f(1)\in B$;
\item or $t^0<1$, then, by continuity of $f$, for small enough $\eta>0$, we have $f(t^0+\eta)\in \ball(f(t^0),r)$ and $f(t^0+\eta)\in B$ as $t^0+\eta>t^0$  by definition of $t^0$.
\end{itemize}
In both cases, we have $\ball(f(t^0),r)\cap B \neq \emptyset$. Again, consider two options:
\begin{itemize}
\item either $t^0=0$, and $\ball(f(t^0),r)\cap A \neq \emptyset$ as $f(0)\in A$;
\item or $t^0>0$, then, by continuity of $f$, for small enough $\eta>0$, we have for all $t\in(t^0-\eta,t^0]$, $f(t)\in \ball(f(t^0),r)$ and there exists some $t^1\in(t^0-\eta,t^0]$ such that $f(t^1)\in A$ by definition of $t^0$ (otherwise, $t^0$ would not be the supremum defined in Equation~\eqref{eq:definition-t0}).
\end{itemize}
In both cases, we have $\ball(f(t^0),r)\cap A \neq \emptyset$.

In conclusion, for all $r>0$, we have:
\begin{eqnarray}
\ball(f(t^0),r)\cap A&\neq&\emptyset\label{eq:boundary-A-1}\\
\ball(f(t^0),r)\cap B&\neq&\emptyset\label{eq:boundary-B-1}
\end{eqnarray}
Let $C^c$ denote the complementary of a subset $C$ of $X$. Since $A\cap B=\emptyset$, we have $B\subset A^c$ and $A\subset B^c$, so that we also have:
\begin{eqnarray}
\ball(f(t^0),r)\cap B^c&\neq&\emptyset\label{eq:boundary-B-2}\\
\ball(f(t^0),r)\cap A^c&\neq&\emptyset\label{eq:boundary-A-2}
\end{eqnarray}
By definition of the boundary, as Equations~\eqref{eq:boundary-A-1} and~\eqref{eq:boundary-A-2} are satisfied for all $r>0$, we have $f(t^0)\in\partial A$ and as Equations~\eqref{eq:boundary-B-1} and~\eqref{eq:boundary-B-2} are satisfied for all $r>0$, we have $f(t^0)\in\partial B$. In conclusion, we have:
\begin{equation}
f(t^0)\in\partial A\cap \partial B
\end{equation} 
\end{proof}

\begin{lemma}
Let $A$ and $B$ be two subsets of a topological space $X$ and assume that $B$ is open, then we have:
\begin{equation}
\partial\left(A\setminus B\right) \cap B = \emptyset
\end{equation}
\label{lemma:frontier-A-minus-B-cap-B}
\end{lemma}
\begin{proof}
By simple manipulations, we obtain:
\begin{equation}
\partial\left(A\setminus B\right)\subset\cl{A\setminus B} = \cl{A\cap B^c} \subset \cl{B^c}
\end{equation}
As $B$ is open, $B^c$ is closed, so that:
\begin{equation}
\partial\left(A\setminus B\right) \subset B^c 
\end{equation}
As a consequence, 
\begin{equation}
\left(\partial\left(A\setminus B\right) \cap B\right)  \subset \left(B\cap B^c\right) = \emptyset
\end{equation}
\end{proof}

\section{Proof of the South-West North-East completion}
\label{app:south-west-completion}

First of all, note that the following property is a direct consequence of the non-decreasing constraint. 
\begin{property}
For all $\path\in\phifree,~ t^0\in[0,1]$, we have:
\begin{equation}
\forall t\in[0,1], \path(t)\in \left(\path(t^0)-\RR_+^{*n}\right)\cup\left(\path(t^0)+\RR_+^n\right)
\end{equation}
\label{property:non-decreasing-path}
\end{property}

\begin{proof}[Proof of Lemma~\ref{lemma:south-west-north-east-completion}]
We will prove Lemma~\ref{lemma:south-west-north-east-completion} by contradiction. Take $i,j\in\robots$, $x^0\in\left(\chiobs_{i\succ j}\cap\chiobs_{j\succ i}\right)$ and assume there exists a feasible path $\path\in\phifree$ such that for some $t^0\in[0,1]$, $\path(t^0)=x^0$. Define $A$ and $B$ as follows:
\begin{eqnarray}
A & := & x^0-\RR_+\e_j+\RR_+\e_i\\
B & := & x^0-\RR_+\e_i+\RR_+\e_j
\end{eqnarray}
As $x^0\in\left(\chiobs_{i\succ j}\cap\chiobs_{j\succ i}\right)$, we have:
\begin{eqnarray}
\chiobs_{ij}\cap A &\neq& \emptyset\\
\chiobs_{ij}\cap B &\neq& \emptyset
\end{eqnarray}
Take $x^1\in\chiobs_{ij}\cap A$ and $x^2\in\chiobs_{ij}\cap B$. By construction of $A$ and $B$, the following inequalities hold:
\begin{eqnarray}
x_i^2\leq x_i^0\leq x_i^1\\
x_j^1\leq x_j^0\leq x_j^2
\end{eqnarray}
By convexity of $\chiobs_{ij}$ and as $\chiobs_{ij}$ is a cylinder,
\begin{equation}
\Sigma:=\left\{x\in\chi: \begin{matrix}
x_i = \alpha x_i^1 + (1-\alpha) x_i^2\\
x_j = \alpha x_j^1 + (1-\alpha) x_j^2
\end{matrix}~,\alpha\in[0,1] \right\}
\label{eq:definition-sigma}
\end{equation}
is a subset of $\chiobs_{ij}$. Define $K^+$ and $K^-$ as follows:
\begin{eqnarray}
K^+&:=&\Sigma+\RR_+\e_i+\RR_+\e_j\\
K^-&:=&\Sigma-\RR_+^*\e_i-\RR_+^*\e_j
\end{eqnarray}
Note that $K^+\cap K^-=\emptyset$ and $\partial K^+ \cap \partial K^- = \Sigma$. By construction, we have: 
\begin{eqnarray}
\left(\path(t^0)+\RR_+^n\right)  &\subset& (K^+\cup K^-)\\
\left(\path(t^0)-\RR_+^{*n}\right)&\subset& (K^+\cup K^-)
\end{eqnarray}
Hence, by Property~\ref{property:non-decreasing-path}, we have for all $t\in[0,1]$, 
\begin{equation}
\path(t)\in\left( K^+ \cup K^-   \right)
\end{equation}

\begin{minipage}{\linewidth}
\begin{center}
\includegraphics[width=0.5\linewidth]{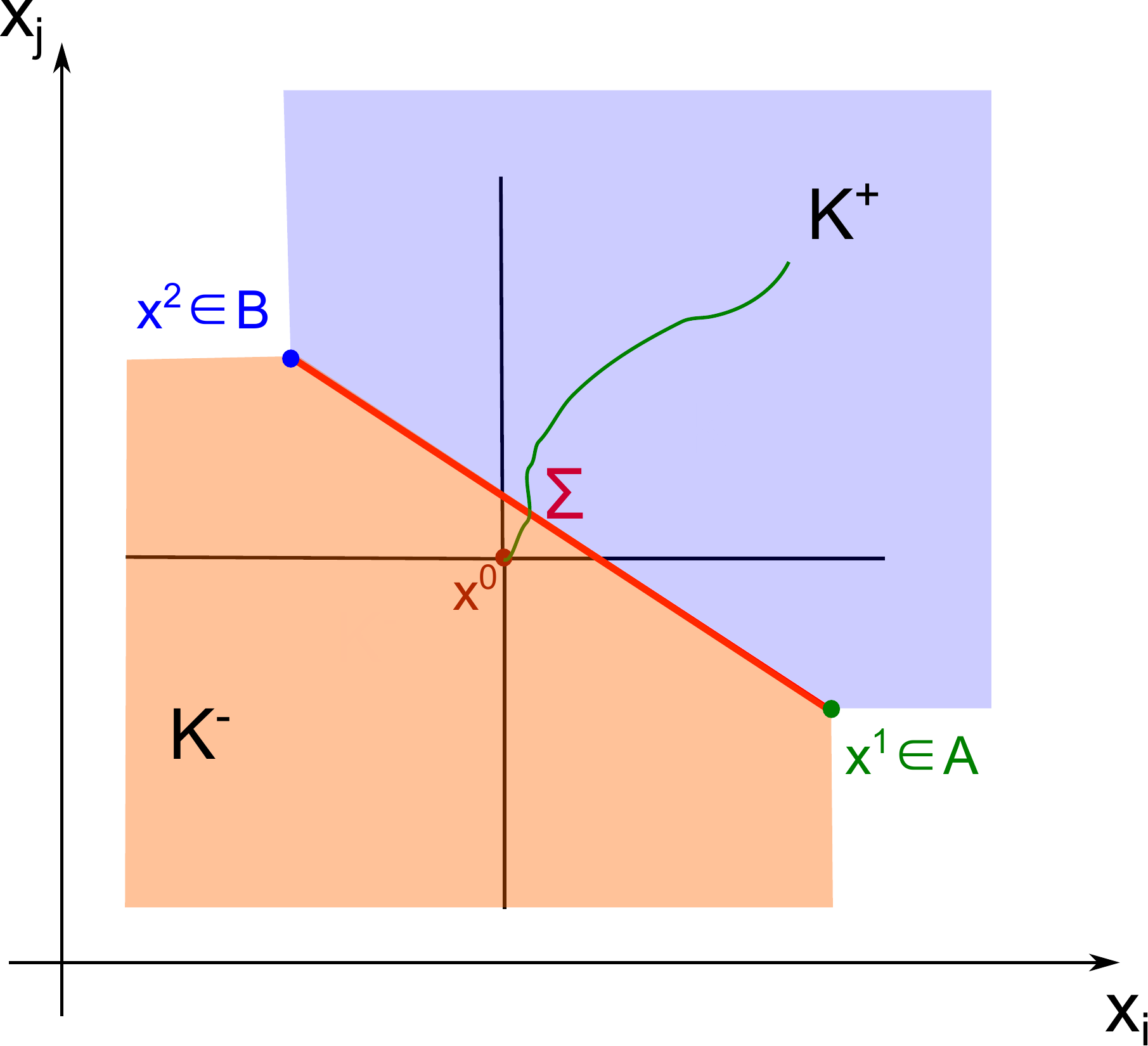}\hfill
\includegraphics[width=0.5\linewidth]{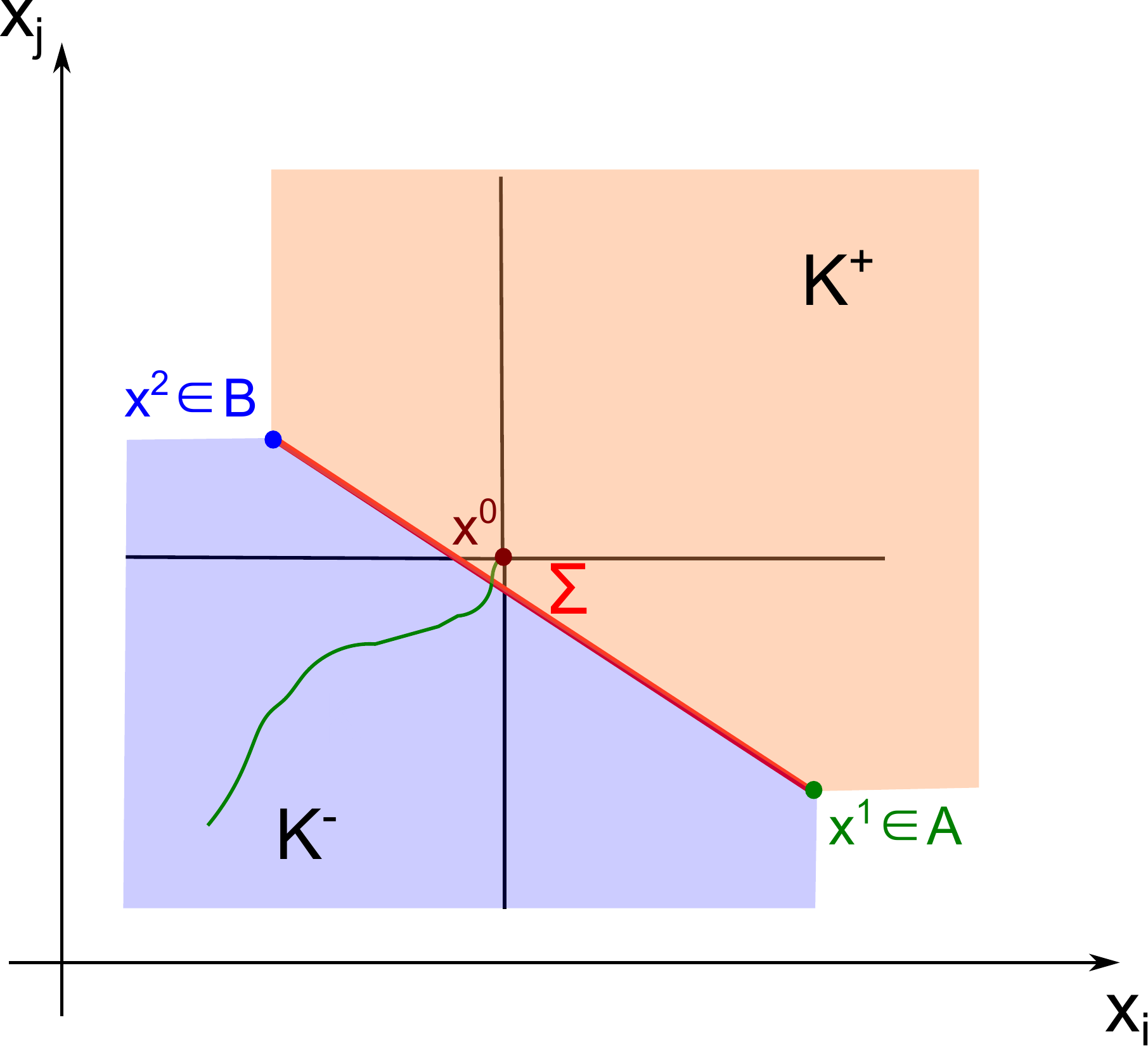}\hfill
\end{center}
\captionof{figure}{The two cases that appear in the proof of Lemma~\ref{lemma:south-west-north-east-completion}}
\label{fig:proof-lemma}
\end{minipage}
\vspace{0.2cm}

By construction, we have two options as depicted in Figure~\ref{fig:proof-lemma}:
\begin{itemize}
\item either $x^0\in K^+$ (right drawing of Figure~\ref{fig:proof-lemma}). This case can be interpreted as the case when $x^0$ is not "reachable". As $(\path(0)-\RR_+^n)\subset\chifree$ and $\path(0)\in(\path(t^0)-\RR_+^n)$, we have:
\begin{eqnarray}
\path(0)&\in&K^-\\
\path(t^0)&\in&K^+
\end{eqnarray}
Hence, as $\path$ is continuous, by Lemma~\ref{lemma:two-subsets-frontier} (see Appendix~\ref{app:topology-properties}), there exists some $t^\Sigma$ such that $\path(t^\Sigma)\in\left(\partial K^-\cap \partial K^+\right) =\Sigma\subset\chiobs_{ij}$ and $\path$ is not collision-free.
\item or $x^0\in K^-$ (left drawing of Figure~\ref{fig:proof-lemma}). This case can be interpreted as the case when $x^0$ will inevitably lead to a deadlock between robots $i$ and $j$. As $(\path(1)+\RR_+^n)\subset\chifree$ and $\path(1)\in(\path(t^0)+\RR_+^n)$, we have:
\begin{eqnarray}
\path(1)&\in& K^+\\
\path(t^0)&\in&K^-
\end{eqnarray}
Hence, as $\path$ is continuous, by Lemma~\ref{lemma:two-subsets-frontier} (see Appendix~\ref{app:topology-properties}), there exists some $t^\Sigma$ such that $\path(t^\Sigma)\in \left(\partial K^+ \cap \partial K^-\right) = \Sigma\subset\chiobs_{ij}$ and $\path$ is not collision-free.
\end{itemize}
In conclusion, there is no feasible path going through any configuration $x^0 \in \chiobs_{i\succ j}\cap\chiobs_{j\succ i}$.
\end{proof}

\section{Proof of the existence of paths with maximal margin}
\label{app:existence-paths-with-maximal-margin}

\begin{lemma}[Existence of paths with maximal margin]
\label{lemma:existence-paths-with-maximal-margin}
Given a priority graph $G\in\graphs$, consider $\rho_G$ defined as follows:
\begin{equation}
\rho_G:=\begin{cases}
~~\sup\left\{r\geq 0:\paths(\chifree_G\ominus[-r,r]^n)\neq\emptyset\right\}&\text{ if } \paths(\chifree_G)\neq\emptyset\\
-\inf\left\{r> 0:\paths(\chifree_G+[-r,r]^n)\neq\emptyset\right\} & \text{ else.}
\end{cases}
\end{equation}
$\rho_G\in\RR\cup\{+\infty\}$ and if $\rho_G\in\RR$, it is attained, so that we can use the maximal element notation:
\begin{equation}
\rho_G:=\begin{cases}
~~\max\left\{r\geq 0:\paths(\chifree_G\ominus[-r,r]^n)\neq\emptyset\right\}&\text{ if } \paths(\chifree_G)\neq\emptyset\\
-\min\left\{r> 0:\paths(\chifree_G+[-r,r]^n)\neq\emptyset\right\} & \text{ else.}
\end{cases}
\end{equation}
\end{lemma}
\begin{proof}[Proof of $\rho_G\in\RR\cup\{+\infty\}$]
Consider a path $\path\in\Phi(\chi)$ whose image is the segment joining $\xobsmin$ to $\xobsmax$. For big enough enough $r\in\RR_+$, this segment is collision-free with regards to $\chiobs_G\ominus[-r,r]^n$. Moreover, $(\xobsmin-\RR_+^n)\subset\chifree_G\subset\chifree_G+[-r,r]^n$ and $(\xobsmax+\RR_+^n)\subset\chifree_G\subset\chifree_G+[-r,r]^n$, so that $\path\in\Phi(\chifree_G+[-r,r]^n)$. Hence, $\rho_G \neq -\infty$ and we obtain $\rho_G\in\RR\cup\{+\infty\}$.
\end{proof}
\begin{proof}[Proof that it is attained] 
Take a priority graph and assume that $\rho_G\in\RR$. Define $\Cfree_G:=\chifree_G\ominus(-\rho_G,\rho_G)^n$ if $\rho_G\geq 0$ (resp. $\Cfree_G:=\chifree_G + (-\vert\rho_G\vert,\vert\rho_G\vert)^n$ if $\rho_G\leq 0$) and $\Cobs_G:=\chi\setminus\Cfree_G$. We have to prove that: 
\begin{equation}
\paths(\Cfree_G)\neq\emptyset
\end{equation}

We have the following identity:
\begin{equation}
\Cfree_G=\bigcap_{r>0}\left( \Cfree_G+[-r,r]^n \right)
\end{equation}
Hence, $\paths(\Cfree_G)$ can be expressed as the limit of a nested sequence of sets of paths as follows:
\begin{equation}
\paths(\Cfree_G)=\bigcap_{r>0}\paths\left( \Cfree_G+[-r,r]^n \right)
\label{eq:nested-sets-paths-sequence}
\end{equation}
By definition of $\rho_G$, we have $\paths\left( \Cfree_G+[-r,r]^n \right)\neq\emptyset$ for all $r>0$, so that $(\paths\left( \Cfree_G+[-r,r]^n \right))_{r>0}$ is a nested sequence of non-empty sets of paths whose limit is $\paths(\Cfree_G)$. We have to prove that this limit is not the empty set. 

Consider the topology of pointwise convergence on $\paths(\chi)$. Define $\phibounded(\chi) \subset \paths(\chi)$ as follows:
\begin{equation}
\phibounded(\chi) := \left\{ \path\in\Phi(\chi): \forall t\in[0,1], \xobsmin \leq \path(t) \leq \xobsmax \right\}
\end{equation}
As for all $\path\in\phibounded(\chi)$ and $i\in\robots$, $\im{\path_i}\subset [\xobsmin_i,\xobsmax_i]$, by Tychonoff's theorem, $\phibounded(\chi)$ with the topology of pointwise convergence is compact in $\paths(\chi)$.  Intersecting with $\phibounded(\chi)$ in~\eqref{eq:nested-sets-paths-sequence}, we obtain:
\begin{equation}
\paths(\Cfree_G)\cap\phibounded(\chi) =\bigcap_{r>0} \left(\paths\left( \Cfree_G+[-r,r]^n \right) \cap \phibounded(\chi) \right)
\end{equation}

Now, we are going to use Cantor's intersection theorem to prove that $\paths(\Cfree_G)\cap\phibounded(\chi) \neq \emptyset$. To this purpose, we are going to prove that for all $r>0$, $\paths\left( \Cfree_G+[-r,r]^n \right) \cap \phibounded(\chi)$ is a non-empty compact in $\paths(\chi)$. As $\phibounded(\chi)$ is compact, it is sufficient to prove that $\paths\left( \Cfree_G+[-r,r]^n \right) \cap \phibounded(\chi)$ is (\ref{item:non-empty}) non-empty and (\ref{item:closed}) closed in $\paths(\chi)$. 
\begin{enumerate}[(a)]
\item $\paths\left( \Cfree_G+[-r,r]^n \right)$ is non-empty. Moreover, taking a path $\path\in\Phi\left( \Cfree_G+[-r,r]^n \right)$, and building the path $\tilde{\path}:=\min(\max(\xobsmin,\path),\xobsmax)$ yields a path in $\phibounded(\chi)\cap\Phi\left( \Cfree_G+[-r,r]^n \right)$. As a result, $\paths\left( \Cfree_G+[-r,r]^n \right) \cap \phibounded(\chi)$ is not empty.\label{item:non-empty}
\item First, we prove that the complementary of $\paths\left( \Cfree_G+[-r,r]^n \right)$ in $\paths(\chi)$ is open. Take a path $\path\in\paths(\chi)$ and assume $\path\notin \paths\left( \Cfree_G+[-r,r]^n \right)$. As $\path\in\paths(\chi)$, $\path$ is necessarily non-decreasing, so that $\path\notin \paths\left( \Cfree_G+[-r,r]^n \right)$ means that $\path$ (or $\path(0)-\RR_+^n$, or $\path(1)+\RR_+^n$) is not collision-free with respect to $\Cobs_G\ominus[-r,r]^n$, which is an open set. Hence, any path $\psi\in \paths(\chi)$ close enough to $\path$ (in the topology of pointwise convergence) also interects $\Cobs_G\ominus[-r,r]^n$, so that $\psi\notin \paths\left( \Cfree_G+[-r,r]^n \right)$. It results that $\paths\left( \Cfree_G+[-r,r]^n \right)$ is closed in $\paths(\chi)$ (as its complementary is open). Moreover, $\phibounded(\chi)$ is also closed as it is compact. In conclusion, $\paths\left( \Cfree_G+[-r,r]^n \right) \cap \phibounded(\chi)$ is closed in $\paths(\chi)$ as the intersection of closed sets in $\paths(\chi)$. \label{item:closed}
\end{enumerate}
Applying Cantor's intersection theorem, we obtain that $\paths(\Cfree_G)\cap\phibounded(\chi)$ is a non-empty compact of $\paths(\chi)$, which implies that $\paths(\Cfree_G)\neq\emptyset$.
\end{proof}

\section{Proof of the characterization of feasible priority graphs}
\label{app:thm-feasible}

Theorem~\ref{thm:feasible} is a direct consequence of the following lemmas. We let $\Cobs_G:=\chiobs_G+[-r,r]^n$ (or $\Cobs_G:=\chiobs_G\ominus[-r,r]^n$) with $\Cobs_{i\succ j}:=\chiobs_{i\succ j}+[-r,r]^n$ (resp. $\Cobs_{i\succ j}:=\chiobs_{i\succ j}\ominus[-r,r]^n$) and we let $\Cfree_G:=\chi\setminus\chiobs_G$. Note that $\Cobs_{G}$ satisfy the same invariance properties as $\chiobs_G$ (Properties~\ref{property:geometric-invariance},~\ref{property:min-max},~\ref{property-union-fixed-priority-cylinder-1}~and~\ref{property-union-fixed-priority-cylinder-1}). Lemma~\ref{lemma:necessary-condition-phi-r-feasible}  provides a necessary condition for the existence of a feasible path taking values in $\Cfree_G$ and Lemma~\ref{lemma:sufficient-condition-phi-r-feasible} provides a sufficient condition for the existence of such path. Note that the lemmas apply in particular for $\Cobs_G\equiv \chiobs_G$ ($r=0$). 

\begin{lemma}
\label{lemma:necessary-condition-phi-r-feasible}
Given a priority graph $G\in\graphs$, a necessary condition for 
 $\paths(\Cfree_G)\neq\emptyset$ is:
\begin{eqnarray}
\forall \C\in\cycles(G), \bigcap_{(i,j)\in E(\C)}\Cobs_{i\succ j}=\emptyset
\end{eqnarray}
\end{lemma}
The proof of the above lemma is omitted as it is exactly the same as the proof of the necessary condition of Theorem~\ref{thm:feasible} where $\Cobs_G$ replaces $\chiobs_G$.

In order to provide a constructive proof of the existence of feasible paths taking values in $\Cfree_G$ under certain conditions, we introduce the concept of local priority graph (an extension of the concept introduced in the corpus of the manuscript with $\Cobs_G$ instead of $\chiobs_G$). Given a radius $r>0$ and a configuration $x\in\chi$, the local priority graph with regards to $\Cobs_G$ at configuration $x$ with radius $r \geq 0$ is the sub-graph $G_{|\Cobs_G,x,r}$ of $G$ with the same vertices and whose edge set is defined below:
\begin{equation}
E(G_{|\Cobs_G,x,r}):=\left\{(i,j)\in E(G): x \in \left(\Cobs_{i \succ j} + [-r,r]^n \right) \right\}
\end{equation}
Note that $G_{|x,r}\equiv G_{|\chiobs_G,x,r}$ which is simply referred as the priority graph at configuration $x$ with radius $r\geq 0$ (without mentioning $\chiobs_G$). 

\begin{lemma}[Sufficient condition for locally acyclic priority graph] Consider a priority graph $G\in\graphs$ satisfying for all elementary cycles $\C$ in $G$:
\begin{equation}
\bigcap_{(i,j)\in E(\C)}\left(\Cobs_{i\succ j}+(-\epsilon,\epsilon)^n\right)=\emptyset
\label{eq:condition-acyclic-local-priority-graph}
\end{equation}
for some $\epsilon>0$, then $G_{|\Cobs_G,x,\epsilon}$ is acyclic at all configurations $x\in\chi$.
\label{lemma:sufficient-condition-locally-acyclic-priority-graph}
\end{lemma}
The proof of the above lemma is omitted as it is exactly the same proof as Lemma~\ref{lemma:sufficient-condition-locally-acyclic-priority-graph-corpus} where $\Cobs_G$ replaces $\chiobs_G$.

It is of high interest to know that the local priority graph with radius $\epsilon>0$ with regards to $\Cobs_G$ is acyclic at all configurations $x\in\chi$. Indeed, when this condition is satisfied, whatever the current configuration $x\in\Cfree_G$ of the system, it is always possible to find a robot $i\in\robots$ which can move forward the distance $\epsilon>0$ without colliding, which enables to construct a feasible path in $\Cfree_G$ by iterations.

\begin{lemma}
\label{lemma:sufficient-condition-phi-r-feasible}
Given a priority graph $G\in\graphs$, a sufficient condition for 
 $\paths(\Cfree_G)\neq\emptyset$  is:
\begin{eqnarray}
\exists\epsilon>0: \forall \C\in\cycles(G), \bigcap_{(i,j)\in E(\C)}(\Cobs_{i\succ j}+(-\epsilon,\epsilon)^n)=\emptyset
\label{eq:condition-acyclic-local-priority-graph-with-epsilon}
\end{eqnarray}
\end{lemma}
The proof of the above lemma is omitted as it is exactly the same proof as Lemma~\ref{lemma:sufficient-condition-locally-acyclic-priority-graph-corpus} where $\Cobs_G$ replaces $\chiobs_G$.

\begin{proof}[Proof of Theorem~\ref{thm:feasible}]
We will prove Theorem~\ref{thm:feasible} by combining the above lemmas. Take a priority graph $G\in\graphs$ and consider two options.
\begin{itemize}
\item If $G$ satisfies:
\begin{equation}
\exists r>0: \forall \C\in\cycles(G), \bigcap_{(i,j)\in E(\C)}(\chiobs_{i\succ j}+[-r,r]^n)=\emptyset
\label{eq:r-exists}
\end{equation}
Then, by Lemma~\ref{lemma:sufficient-condition-phi-r-feasible}, $\paths(\chifree_G)\neq\emptyset$, so $G$ is feasible. Moreover, take $\rmax$ defined as follows:
\begin{multline}
\rmax:=\sup\Bigl\{r> 0: \forall \C\in\cycles(G), \bigcap_{(i,j)\in E(\C)}(\chiobs_{i\succ j}+[-r,r]^n)=\emptyset\Bigr\}
\end{multline}
$\rmax$ exists in $\RR_+\cup\{+\infty\}$ as the supremum of a non-empty subset of $\RR$ (it is non-empty as~\eqref{eq:r-exists} is satisfied). By Lemma~\ref{lemma:necessary-condition-phi-r-feasible}, for all $r>\rmax$, $\paths(\chifree_G\ominus[-r,r]^n)=\emptyset$. And by Lemma~\ref{lemma:sufficient-condition-phi-r-feasible}, for all $r<\rmax$, $\paths(\chifree_G\ominus[-r,r]^n)\neq\emptyset$ (take $0<\epsilon<\rmax-r$). By definition of $\rho_G$, we have $\rho_G=\rmax$, i.e.,
\begin{equation}
\rho_G=\max\left\{r > 0: \forall \C\in\cycles(G), \bigcap_{(i,j)\in E(\C)}(\chiobs_{i\succ j}+[-r,r]^n)=\emptyset\right\}
\label{eq:feasibility-margin-expression}
\end{equation}
where $\max$ replaces $\sup$ since $\paths(\chifree_G\ominus(-\rho_G,\rho_G)^n)\neq\emptyset$ (see Lemma~\ref{lemma:existence-paths-with-maximal-margin}) implies $\bigcap_{(i,j)\in E(\C)}(\chiobs_{i\succ j}+(-\rho_G,\rho_G))^n)=\emptyset$ by Lemma~\ref{lemma:sufficient-condition-phi-r-feasible}.
\item If $G$ satisfies:
\begin{equation}
\forall r>0: \exists \C\in\cycles(G), \bigcap_{(i,j)\in E(\C)}(\chiobs_{i\succ j}+[-r,r]^n)\neq\emptyset
\end{equation}
Then, by Lemma~\ref{lemma:necessary-condition-phi-r-feasible}, $\rho_G\leq 0$ ($G$ is not feasible or it is feasible with a safety margin of $0$). Take $\rmax$ defined as follows:
\begin{multline}
\rmax:=\inf\Bigl\{r\geq 0: \forall \C\in\cycles(G), \bigcap_{(i,j)\in E(\C)}(\chiobs_{i\succ j}\ominus[-r,r]^n)=\emptyset\Bigr\}
\end{multline}
By Lemma~\ref{lemma:sufficient-condition-phi-r-feasible}, for all $r>\rmax$, $\paths(\chifree_G+[-r,r]^n)\neq\emptyset$ (take $0<\epsilon<\rmax-r$) and by Lemma~\ref{lemma:necessary-condition-phi-r-feasible}, for all $r<\rmax$, $\paths(\chifree_G+[-r,r]^n)=\emptyset$. By definition of $\rho_G$, we have $\rho_G=-\rmax$, i.e.,
\begin{equation}
\rho_G=-\min\left\{r \geq 0: \exists \C\in\cycles(G), \bigcap_{(i,j)\in E(\C)}(\chiobs_{i\succ j}\ominus[-r,r]^n)=\emptyset\right\}
\label{eq:feasibility-margin-expression-case-not-feasible}
\end{equation}
where $\min$ replaces $\inf$ and $\geq$ replaces $>$ as $\paths(\chifree_G\ominus(-\rho_G,\rho_G)^n)\neq\emptyset$ (see Lemma~\ref{lemma:existence-paths-with-maximal-margin}) implies $\bigcap_{(i,j)\in E(\C)}(\chiobs_{i\succ j}\ominus(-\rho_G,\rho_G)^n)=\emptyset$ by Lemma~\ref{lemma:sufficient-condition-phi-r-feasible}. Hence, in that case, $G$ is feasible if and only if $\bigcap_{(i,j)\in E(\C)}\chiobs_{i\succ j}=\emptyset$.
\end{itemize}
In conclusion, a necessary and sufficient condition for $G$ being feasible is:
\begin{equation}
\forall\C\in\cycles(G), \bigcap_{(i,j)\in E(\C)}\chiobs_{i\succ j}=\emptyset
\end{equation}
and if this condition is satisfied, the feasibility margin is given by:
\begin{equation}
\rho_G=\max\left\{r \geq 0: \forall \C\in\cycles(G), \bigcap_{(i,j)\in E(\C)}(\chiobs_{i\succ j}+[-r,r]^n)=\emptyset\right\}
\end{equation}
where $\geq$ replaces $>$ of Equation~\eqref{eq:feasibility-margin-expression} to include the case $\rho_G=0$.
\end{proof}

\chapter[Priority preserving control: extension to feasible cyclic priorities]{Priority preserving control: \\extension to feasible cyclic priorities}
\chaptermark{Extension to feasible cyclic priorities}
\label{app:control-extension-all-feasible-priorities}

The results presented in Part~\ref{part:priorities-to-guide-robots} have been proved under the assumption that the assigned priorities are acyclic. However, they can be extended to all feasible priority graphs -- including those containing cycles -- under certain mild conditions. This is the topic of this appendix chapter.

\paragraph{Sketch of the chapter} Sections~\ref{sec:extension-velocity},~\ref{sec:extension-acceleration} and~\ref{sec:extension-uncertainty} respectively extend the results of Chapter~\ref{chap:optimal-control-velocity} (control in velocity),~\ref{chap:control-acceleration} (control in acceleration) and~\ref{chap:control-uncertainty} (bounded uncertainty).

\section{In the absence of inertia}
\label{app:extension-control-velocity}
\label{sec:extension-velocity}

First of all, we examine the case of robots controlled in velocity along their paths. In Chapter~\ref{chap:optimal-control-velocity}, the acyclicity of $G$ is key as it is even necessary to the existence of the control law (see Theorem~\ref{thm:control-law-existence}). A relaxed sufficient condition on the priority graph for all the results of Chapter~\ref{chap:optimal-control-velocity} to hold is that the maximum distance traveled by a robot in one time slot satisfies:
\begin{equation}
\max_{i\in\robots} \vmax_i \leq \rho_G \label{eq:ineq-epsilon}
\end{equation}
In particular, it is necessary for the safety margin to be strictly positive. In practice, to satisfy this assumption, one can choose a sufficiently small time slot length (the maximum distance traveled by a robot in one time slot is proportional to the time slot length). However, due to actuators limitations, there is a lower bound in time slot length. Assume that the minimal time slot length is chosen. A robot $i$ either is stopped or travels the distance $\vmax_i$ which can be seen as the control resolution. Basically, Equation~\eqref{eq:ineq-epsilon} states that the control resolution needs to be lower than the safety margin $\rho_G$.

All the results of Section~\ref{app:extension-control-velocity} (existence, priority preservation, optimality, liveness) are extended under Condition~\eqref{eq:ineq-epsilon} and proofs are provided. The key idea is to use the fact that if the priority graph is feasible with a strictly positive margin $\rho_G$, then the local priority graph is acyclic at all configurations. Hence, it is always possible to move forward some robot provided the control resolution is small enough. The evolution of the multi robot system under control law $f^G$ with cyclic priorities is depicted in Figure~\ref{fig:control-law-velocity-example-cyclic} (the evolution of the local priority graph is also represented). Figure~\ref{fig:trajectory-velocity-control-example-cyclic} depicts the trajectory in the coordination space.
\begin{figure}[p]
\begin{center}
\includegraphics[width=1.0\linewidth]{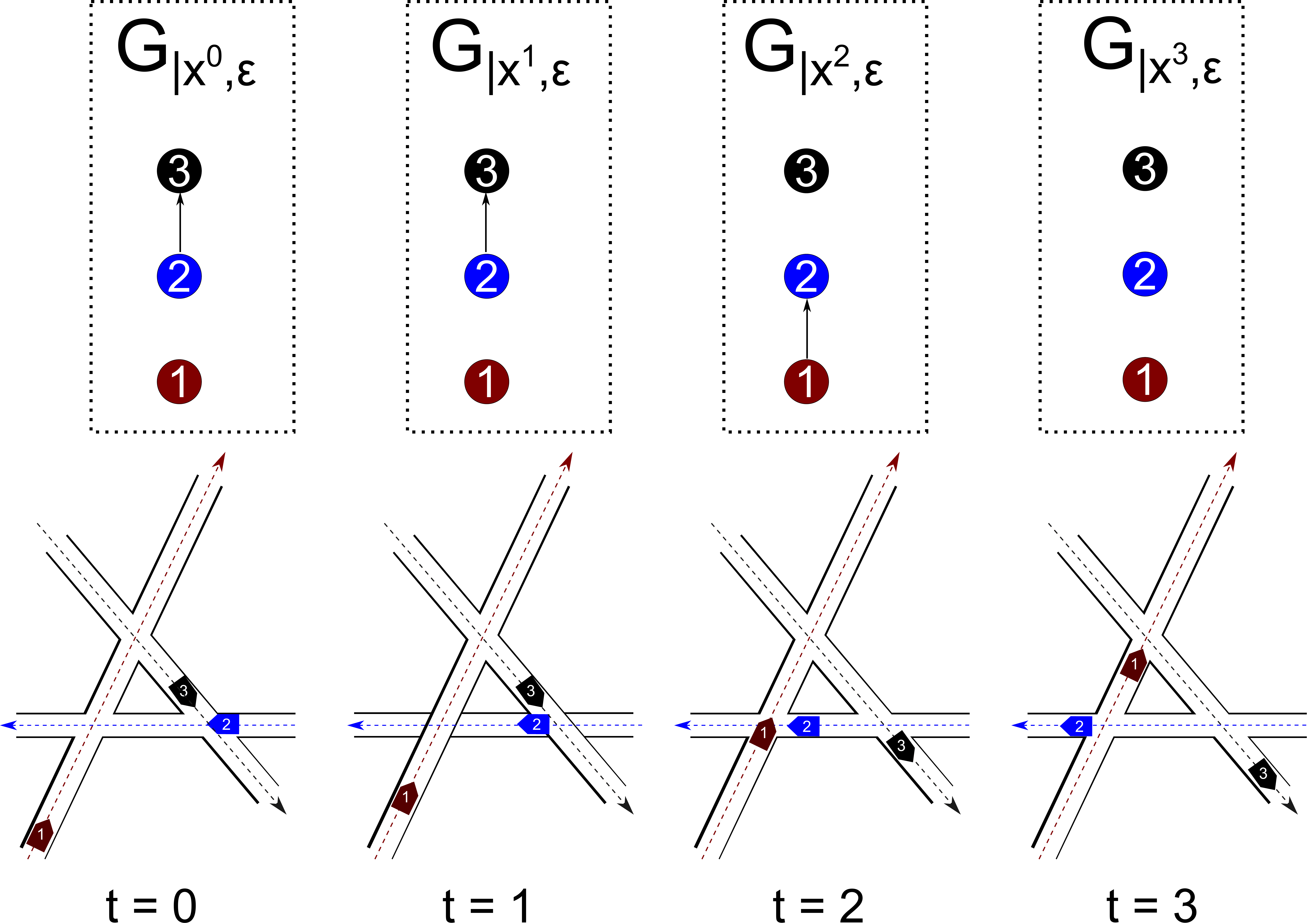}\hfill
\end{center}
\caption{A three-robot system with cyclic assigned priorities: $1\succ 2$, $2\succ 3$, and $3\succ 1$. Robots are controlled under control law $f^G$. The top drawings show the evolution of the local priority graph and the bottom drawings depict the evolution of the robots along their paths.}
\label{fig:control-law-velocity-example-cyclic}
\end{figure}
\begin{figure}[p]
\begin{center}
\includegraphics[width=1.0\linewidth]{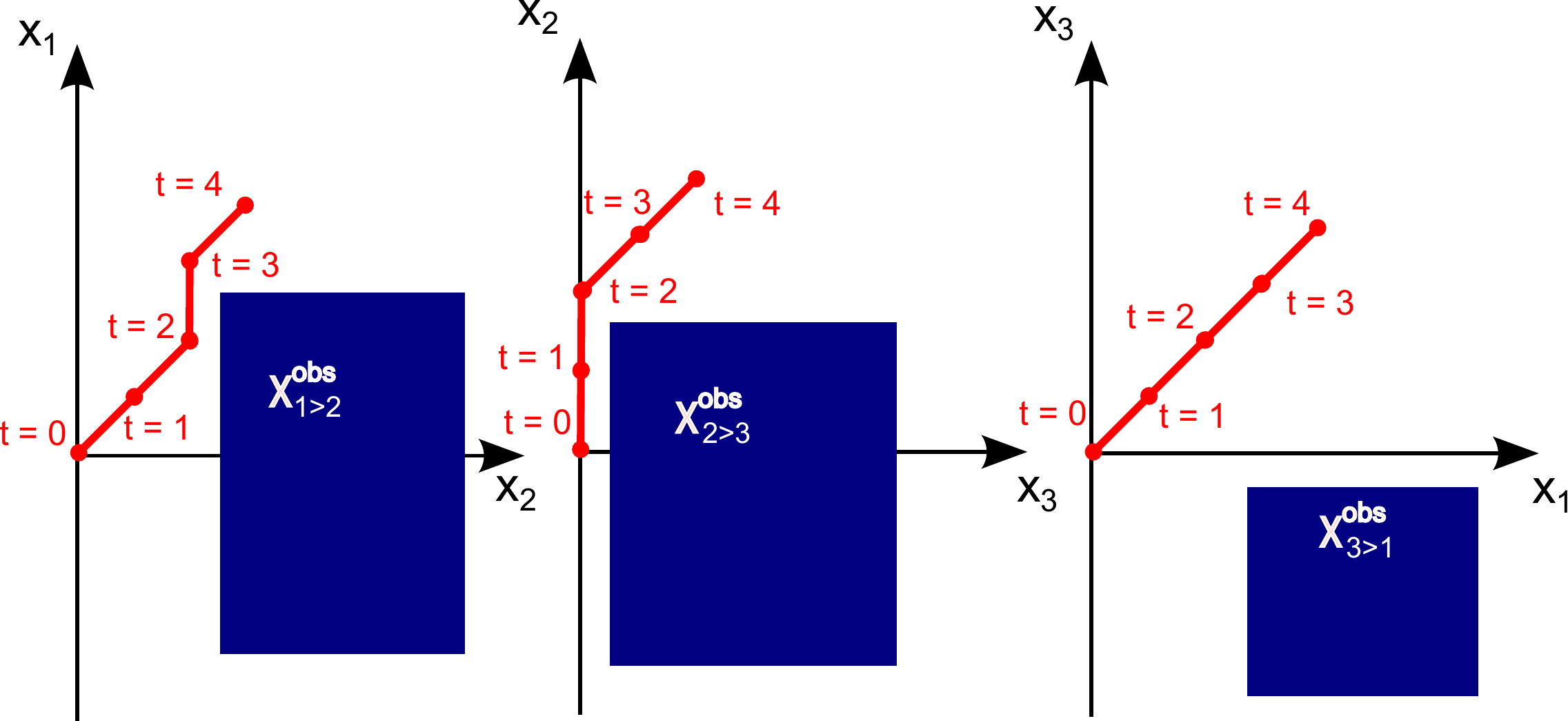}\hfill
\end{center}
\caption{Trajectory in the coordination space for the three-robot system under control law $f^G$ of Figure~\ref{fig:control-law-velocity-example}.}
\label{fig:trajectory-velocity-control-example-cyclic}
\end{figure}

\subsection{Existence}

\begin{theorem}[Control law existence]
Given a	feasible priority graph $G$ with a strictly positive feasibility margin satisfying Condition~\eqref{eq:ineq-epsilon}, Equation~\eqref{eq:definition-control-law-velocity} uniquely defines a control law $f^G:\chi\to V$. 
\label{thm:control-law-existence-app}
\end{theorem}
\begin{proof}
Take a feasible priority graph satisfying Condition~\eqref{eq:ineq-epsilon}. By Lemma~\ref{lemma:sufficient-condition-locally-acyclic-priority-graph-corpus}, $G_{|x,\rho_G}$ is acyclic at all configurations $x\in\chi$. Moreover, by definition of the local priority graph and as Condition~\eqref{eq:ineq-epsilon} is satisfied, Equation~\eqref{eq:definition-control-law-velocity} is equivalent to:
\begin{equation}
f_i^G(x):= \begin{cases}
0&\text{ if } \exists(j,i)\in E(G_{|x,\rho_G}),t\in[0,1] \text{ s.t. } \left(x+t\left(\vmax_i\e_i+f_j(x)\e_j\right)\right)\in\chiobs_{j\succ i} \\
\vmax_i&\text{ else.}
\end{cases}
\label{eq:definition-control-law-velocity-with-epsilon-app}
\end{equation}
Since $G_{|x,\rho_G}$ is a directed acyclic graph, there exists a topological ordering of the graph such that for every edge $(j,i)\in E(G_{|x,\rho_G})$, $j$ comes before $i$ in the ordering. Hence, following the topological order induced by $G_{|x,\rho_G}$, it is possible to compute $f_i^G(x)$ for all $i\in\robots$ iteratively. As a result,  Equation~\eqref{eq:definition-control-law-velocity-with-epsilon-app} uniquely defines a control law $f^G:\chi\to V$.
\end{proof}

\subsection{Safety}

\begin{theorem}[Priority preservation]
Given a	feasible priority graph $G$ with a strictly positive feasibility margin satisfying Condition~\eqref{eq:ineq-epsilon}, $\chifree_{G}$ is positively invariant for the system under control law $f^G$:
\begin{equation}
\forall x\in \chifree_{G}, \forall t\geq 0, \phi(t,x,f^G) \in \chifree_{G}
\end{equation}
\label{thm:safety-robust-control-map-velocity-app}
\end{theorem}

The proof is omitted as it is exactly the same as the proof of Theorem~\ref{thm:safety-robust-control-map-velocity}. Indeed, in the proof of Theorem~\ref{thm:safety-robust-control-map-velocity}, the acyclicity of the priority graph is useful only to ensure that the control law $f_G$ is well defined. Condition~\eqref{eq:ineq-epsilon} precisely ensures that the control law $f_G$ is well defined even if priorities are cyclic (see Theorem~\ref{thm:control-law-existence-app}).

\subsection{Optimality}

\begin{theorem}[Optimality]
Given a	feasible priority graph $G$ with a strictly positive feasibility margin satisfying Condition~\eqref{eq:ineq-epsilon}, the control law $f^G$ is optimal for the priority graph $G$.
\label{thm:control-law-optimality-app}
\end{theorem}
The proof is omitted as it is exactly the same as the proof of Theorem~\ref{thm:control-law-optimality}. Again, in the proof of Theorem~\ref{thm:control-law-optimality} is useful only to ensure that the control law $f_G$ is well defined which is already guaranteed by Condition~\eqref{eq:ineq-epsilon}.

\subsection{Liveness}

\begin{theorem}[Liveness]
Given a	feasible priority graph $G$ with a strictly positive feasibility margin satisfying Condition~\eqref{eq:ineq-epsilon} and a configuration $x^0\in\chifree_G$, there exists $T>0$ such that:
\begin{equation}
\phi(T,x^0,f^G)\in\chigoal
\end{equation}
\label{thm:liveness-control-law-velocity-app}
\end{theorem}
\begin{proof}
Take a feasible priority graph $G$ satisfying Condition~\eqref{eq:ineq-epsilon}. First of all, note that by Lemma~\ref{lemma:sufficient-condition-locally-acyclic-priority-graph-corpus}, the local priority graph of radius $\rho_G>0$ is acyclic at all configurations. Take a configuration $x^0\in\chifree_G$. Define the set $I(t)\subset \robots$ defined as follows for all $t\geq 0$:
\begin{equation}
I(t):=\left\{i\in\robots: \phi_i(t,x^0,f^G)<\xobsmax_i \right\}
\end{equation}
Theorem~\ref{thm:liveness-control-law-velocity-app} is equivalent to the existence of $T>0$ such that:
\begin{equation}
I(T)=\emptyset
\end{equation}

First of all, note that the set $I(t)$ decreases through time, as the trajectory of robots is non-decreasing. At every time slot $k$, if $I(k)$ is not empty, consider $i^*$, a maximal element of the (acyclic) sub-graph of $G_{|\phi(k,x^0,f^G), \rho_G}$ retaining only vertices in $I(k)$. Remember that, by definition of the local priority graph and as Condition~\eqref{eq:ineq-epsilon} is satisfied, the definition of the control law is equivalent to:
\begin{equation}
f_i^G(x)= \begin{cases}
0&\text{ if } \exists(j,i)\in E(G_{|x,\rho_G}),t\in[0,1] \text{ s.t. } \left(x+t\left(\vmax_i\e_i+f_j(x)\e_j\right)\right)\in\chiobs_{j\succ i} \\
\vmax_i&\text{ else.}
\end{cases}
\end{equation} 
Hence, by construction of $i^*$, we have $f_{i^*}^G(\phi(k,x^0,f^G))=\vmax_{i^*}$, so that:
\begin{equation} 
\phi_{i^*}(k+1,x^0,f^G)-\phi_{i^*}(k,x^0,f^G) = \vmax_{i^*}
\end{equation}

Hence, at every time slot $k$, there is at least one robot $i^* \in I(k)$ traveling distance $\vmax_{i^*}$. As there is initially a finite number of robots in $I(0)$ and as each robot $i$ has a finite distance to travel to reach position $\xobsmax_i$, it is guaranteed that $I(t)$ gets empty in finite time, i.e., the multi robot system reaches $\chigoal$ in finite time.
\end{proof}

\section{Under kinodynamic constraints}
\label{sec:extension-acceleration}

When robots are controlled in acceleration, the control law $g^G$ proposed in Chapter~\ref{chap:control-acceleration} is well defined, safe and robust even if the priority graph $G$ is not feasible. However, the liveness guarantee of Theorem~\ref{thm:liveness-control-law} has been proved under the assumption that priorities are acyclic. Actually, liveness is still guaranteed under a less restrictive assumption on the priority graph. Again, the assumption is on the safety margin of the priority graph:
\begin{equation}
\max_{i\in\robots} \sup_{t\in\RR_+} \phi_{x,i}(t,0,\ubf^\imp_i) \leq \rho_G
\label{eq:condition-liveness-control-law}
\end{equation}
Basically, under control in acceleration, $\sup_{t\in\RR_+} \phi_{x,i}(t,0,\ubf_i^\imp)$ represents the control resolution as it is the maximum distance traveled by a robot starting from stop and applying an impulse control considering all robots. Again, Equation~\eqref{eq:condition-liveness-control-law} states that the control resolution needs to be lower than the safety margin. Liveness under Condition~\eqref{eq:condition-liveness-control-law} is proved in the sequel.

\begin{theorem}[Liveness]
Given a feasible priority graph $G$ with a strictly positive feasibility margin satisfying Condition~\eqref{eq:condition-liveness-control-law} and an initial brake safe state $s\in B_G$, there exists $T>0$ such that:
\begin{equation}
\phi_x(T,s,g^G)\in\chigoal
\end{equation}
\label{thm:liveness-control-law-app}
\end{theorem}

We start with some lemmas.

\begin{lemma}
Given $i\in\robots$, $s_i\in S_i$ and $\ubf_i\in\controls_i$ taking values in $\{\uL_i,\uH_i\}$, the following implication holds:
\begin{equation}
\left[\forall k\in\NN, \phi_{x,i}(k,s_i,\ubf_i) \leq \xobsmax_i\right] \Rightarrow \left[\exists k^0\in\NN: \forall k\geq k^0, \ubf_i(k)=\uL_i\right]
\end{equation}
\label{lemma:robot-brakes-indefinitely-if-no-exit}
\end{lemma}
\begin{proof}
We reason by contraposition. Consider $i\in\robots$, $s_i\in S_i$ and $\ubf_i\in\controls_i$ taking values in $\{\uL_i,\uH_i\}$. Assume that $\forall k^0\in\NN, \exists k\in\NN, k\geq k^0: \ubf_i(k)=\uH_i$. Then, we can build an infinite sequence of slots $\{k^0_1, k^0_2, ...\}$ where $\ubf_i$ takes value $\uH_i$. At each of them, the increase in $\phi_{x,i}(k,s_i,\ubf_i)$ within the slot is at least of $\phi_{x,i}(1,0,\uHbf_i)>0$. Hence, we have $\lim_{k\to+\infty}\phi_{x,i}(k,s_i,\ubf_i)=+\infty$ and in particular $\phi_{x,i}(k,s_i,\ubf_i) > \xobsmax_i$ for some $k\in\NN$.
\end{proof}

\begin{lemma}
Given $i\in\robots$, for all $s_i\in S_i$, there exists $k^1\in\NN$ such that:
\begin{equation}
\forall k\geq k^1, \phi_{v,i}(k,s_i,\uLbf_i)=0
\end{equation}
\label{lemma:robot-brakes-velocity-vanishes}
\end{lemma}
\begin{proof}
Take $i\in\robots$, $s_i\in S_i$ and consider the trajectory of robot $i$ under control $\uLbf_i$. At each time slot, there are two options:
\begin{itemize}
\item the velocity of robot $i$ vanishes within the slot;
\item the variation of the velocity of robot $i$ within the slot is $\uL_i<0$.
\end{itemize}
In conclusion, the velocity of robot $i$ vanishes in finite time.
\end{proof}

\begin{lemma}
Given a priority graph $G\in\graphs$ satisfying Condition~\eqref{eq:condition-liveness-control-law} and a state $s\in S$, for all $i\in\robots$ satisfying $\pi_{v,i}(s)=0$, we have:
\begin{equation}
g_i^G(s)=\begin{cases}
\uL_i & \text{if } \exists (j,i)\in E(G_{|\pi_x(s),\rho_G}), \exists t\geq 0 \text{ s.t. } \phi_x(t,s,\tilde{\ubf}^i) \in \chiobs_{j\succ i} \\
\uH_i & \text{ else.}
\end{cases}
\label{eq-control-map-with-local-priority-graph}
\end{equation}
\label{lemma:liveness-control-law}
\end{lemma}
\begin{proof}
Take a priority graph $G\in\graphs$ satisfying Condition~\eqref{eq:condition-liveness-control-law}, a state $s=(x,v)\in S$ and $i\in\robots$ satisfying $\pi_{v,i}(s)=0$. By construction of $g^G$, we have:
\begin{equation}
g_i^G(s):=\begin{cases}
\uL_i & \text{if } \exists (j,i)\in E(G), \exists t\geq  0 \text{ s.t. } \phi_x(t,s,\tilde{\ubf}^i) \in \chiobs_{j\succ i} \\
\uH_i & \text{ else.}
\end{cases}
\end{equation}
Hence, we have to prove that for all $(j,i)\in E(G)$, if $\phi_x(t,s,\tilde{\ubf}^i) \in \chiobs_{j\succ i}$ for some $t\geq 0$, then we have $(j,i)\in E(G_{|x,\rho_G})$. We will prove this by contraposition.

Take $(j,i)\in E(G)$ and assume that $(j,i)\notin E(G_{|x,\rho_G})$, then we have:
\begin{equation}
x+\rho_G \e_i \in\chifree_{j \succ i}
\label{eq:x-plus-epsilon-in-chifree}
\end{equation}
By assumption, we have $\pi_{v,i}(s)=0$. Hence, we obtain for all $t\geq 0$:
\begin{equation}
\phi_i(t,s,\tilde{u}^i)=s_i+\phi_i(t,0,\tilde{u}^i)
\end{equation}
Using Condition~\eqref{eq:condition-liveness-control-law} and as $t \mapsto \phi_{x,j}(t,s,\tilde{u}^i)$ is non-decreasing, we obtain:
\begin{eqnarray}
\phi_{x,i}(t,s,\tilde{u}^i) &\leq& x_i+ \rho_G\\
\phi_{x,j}(t,s,\tilde{u}^i) &\geq& x_j
\end{eqnarray}
By Property~\ref{property:geometric-invariance}, Equation~\eqref{eq:x-plus-epsilon-in-chifree} implies that $\phi_x(t,s,\tilde{u}^i) \in\chifree_{j\succ i}$ for all $t\geq 0$. 
\end{proof}

\begin{lemma}
Given $s\in S$, for all $i,j\in\robots$, for all $\ubf\in\controls$, the following implication holds:
\begin{equation}
\left[\exists t\geq 0: \phi_x(t,s,\ubf) \in \chiobs_{j\succ i}\right]  \Rightarrow \left[\pi_{x,j}(s) < \xobsmax_j\right]
\end{equation}
\label{lemma:flow-intersection-non-empty-implication}
\end{lemma}
\begin{proof}
Take $s\in S$, $i,j\in\robots$ and $\ubf\in\controls$. The flow $t\mapsto \phi_x(t,s,\ubf)$ is non-decreasing. Hence, we have:
\begin{equation}
\phi_{x,j}(t,s,\ubf) \geq \phi_{x,j}(0,s,\ubf) = \pi_{x,j}(s)
\end{equation}
Assume that $\phi_x(t,s,\ubf) \in \chiobs_{j\succ i} $ for some $t\geq 0$. It implies that $\phi_{x,j}(t,s,\ubf) < \sup\{x_j:x\in\chiobs_{j\succ i}\} \leq \xobsmax_j$. By transitivity, we obtain $\pi_{x,j}(s) < \xobsmax_j$.
\end{proof}

\begin{proof}[Proof of Theorem~\ref{thm:liveness-control-law-app}]
Take a priority graph $G\in\graphs$ satisfying Condition~\eqref{eq:condition-liveness-control-law}. Remember that the local priority graph with radius $\rho_G$ is acyclic at all configuration by Lemma~\ref{lemma:sufficient-condition-locally-acyclic-priority-graph-corpus}. Consider the set $I(t)$ defined as follows:
\begin{equation}
\forall t\in\RR_+, I(t):=\left\{i\in\robots: \phi_{x,i}(t,s,g^G) < \xobsmax_i \right\}
\end{equation}
We prove Theorem~\ref{thm:liveness-control-law-app} by contradiction. Theorem~\ref{thm:liveness-control-law-app} is equivalent to the existence of $T>0$ such that:
\begin{equation}
I(T)=\emptyset
\end{equation}
Assume that for all $t \geq 0$, $I(t)\neq\emptyset$. Then, consider $\Imin$ defined as follows:
\begin{equation}
\Imin := \cap_{t\geq 0} I(t) \neq \emptyset
\end{equation}
We have for all $i\in\Imin$, 
\begin{equation}
\forall k\in\NN, \phi_{x,i}(k,s,g^G) < \xobsmax_i
\end{equation}
By Lemma~\ref{lemma:robot-brakes-indefinitely-if-no-exit}, a necessary condition for the above equation to be satisfied is that for all $i\in\Imin$, there exists $k^0_i\in\NN$ such that:
\begin{equation}
\forall k\geq k_i^0, g_i^G(\phi(k,s,g ^G))=\uL_i
\end{equation}
Hence, for $k^0:=\max_{i\in\Imin} k^0_i$, we have:
\begin{equation}
\forall i\in\Imin, \forall k\geq k^0, g_i^G(\phi(k,s,g ^G))=\uL_i
\end{equation}
By Lemma~\ref{lemma:robot-brakes-velocity-vanishes}, the velocity of all robots $i\in\Imin$ will vanish in finite time. There exists $k^1\geq k^0$ such that:
\begin{eqnarray}
\forall k\geq k^1, g_i^G(\phi(k,s,g ^G))&=&\uL_i\\
\phi_{v,i}(k,s,g^G)&=&0
\end{eqnarray}
Take $k \geq k^1$ big enough, so that $I(k)=\Imin$. By Lemma~\ref{lemma:liveness-control-law}, for all $i\in\Imin$, there exists $(j,i)\in E(G_{|\phi_x(k,s,g^G),\rho_G})$ such that $\phi_x(\RR_+,\phi(k,s,g^G),\tilde{\ubf}^i)\cap \chiobs_{j\succ i}\neq \emptyset$ and we have necessarily $\phi_{x,j}(k,s,g^G) <\xobsmax_j$ by Lemma~\ref{lemma:flow-intersection-non-empty-implication}, i.e., $j\in I(k)=\Imin$. As a result, $G_{|\phi_x(k,s,g^G),\rho_G}$ should be cyclic (there would be a cycle involving vertices in $\Imin$). It is absurd.
\end{proof}

\section{Under uncertainty in control and sensing}
\label{sec:extension-uncertainty}
\label{app:liveness-control-law-uncertainty}

Again, the control law $\hat g^G$ under the scenario of bounded uncertainty in control and sensing, as proposed in Chapter~\ref{chap:control-uncertainty} is well defined, safe and robust even if the priority graph $G$ is not feasible. However, if the priority graph is feasible but contains cycles, we need to make an additional assumption in order to prove liveness. First of all, we assume that the observation signal $\ybf$ is such that for all $i\in\robots$ and $k\in\NN$, $\ybf_i^x(k)$ is an interval of length $\sigma_i^x$. We say $\sigma^x$ is the observation precision on the position. The condition to ensure liveness is:
\begin{equation}
\max_{i\in\robots} \left(\sigma_i^x+ \sup_{t\in\RR_+} \phi_{x,i}(t,0,\ubf_i^\imp, \dbfmax_i)\right) \leq \rho_G
\label{eq:condition-liveness-control-law-uncertainty}
\end{equation}
Basically, under control in acceleration with uncertainty, $\sup_{t\in\RR_+} \phi_{x,i}(t,0,\ubf_i^\imp, \dbfmax_i)$ represents the control resolution as it is the maximum distance traveled by a robot starting from stop and applying an impulse control under maximal disturbance considering all robots. However, compared to Equation~\eqref{eq:condition-liveness-control-law}, a new term appears: $\sigma_i^x$, the precision on the position. Hence, Equation~\eqref{eq:condition-liveness-control-law-uncertainty} states that the sum of the control resolution and the observation precision on the position needs to be lower than the safety margin. It is not surprising that the precision on the position appears in the liveness condition. Consider the scenario of Figure~\ref{fig:non-feasible-because-of-uncertainty} where priorities are cyclic. As the uncertainty in position is large, it is very similar to as if robots were much larger than they are. Hence, one can easily see that a deadlock can possibly occur.
\begin{figure}[!htbp]
\begin{center}
\includegraphics[width=0.54\linewidth]{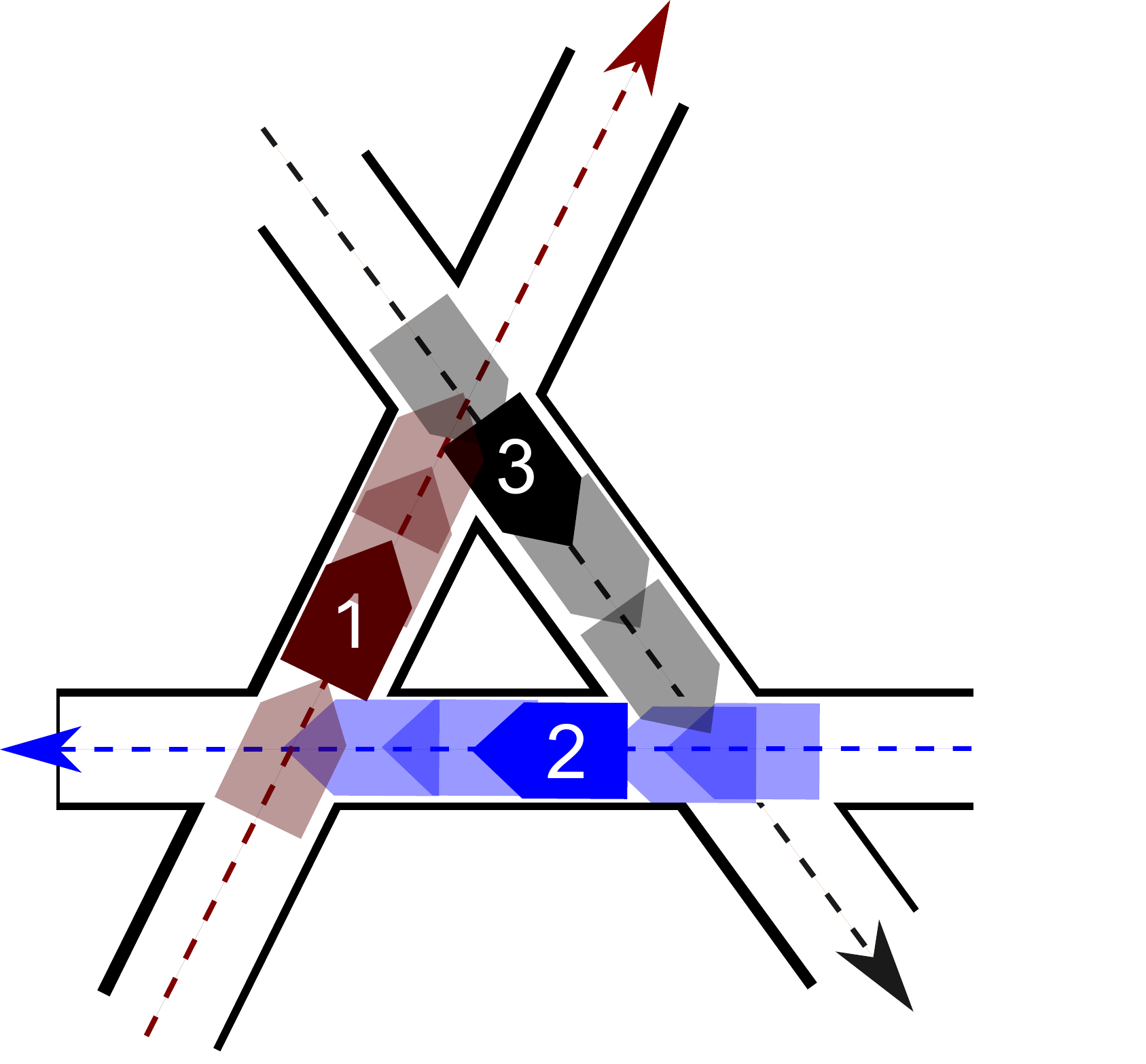}\hfill
\end{center}
\caption{A three-robot scenario where Condition~\eqref{eq:condition-liveness-control-law-uncertainty} is not respected. Even if the priority graph is feasible, due to a large uncertainty on position, liveness cannot be guaranteed and robots following the control law will be possibly stuck in a deadlock configuration.}
\label{fig:non-feasible-because-of-uncertainty}
\end{figure}

Liveness under Condition~\eqref{eq:condition-liveness-control-law-uncertainty} is proved in the sequel.

\begin{theorem}[Liveness]
Given a feasible priority graph $G$ with a strictly positive feasibility margin $\rho_G$ satisfying Condition~\eqref{eq:condition-liveness-control-law-uncertainty}, an initial brake safe non-deterministic state $\hat s\in \hat B_G$ and an observation signal $\ybf\in\Ybf$, there exists $T>0$ such that:
\begin{equation}
\hat \phi_x(T,\hat s,\hat g^G, \ybf)\in 2^{\chigoal}
\end{equation}
\label{thm:liveness-control-law-uncertainty-app}
\end{theorem}

We start with some lemmas.

\begin{lemma}
Given $i\in\robots$, $\hat s_i\in 2^{S_i}$, $\ybf_i\in\Ybf_i$ and $\ubf_i\in\controls_i$ taking values in $\{\uL_i,\uH_i\}$, the following implication holds:
\begin{equation}
\left[\forall k\in\NN, \inf \hat \phi_{x,i}(k,\hat s_i,\ubf_i,\ybf_i) \leq \xobsmax_i\right] \Rightarrow \left[\exists k^0\in\NN: \forall k\geq k^0, \ubf_i(k)=\uL_i\right]
\end{equation}
\label{lemma:robot-brakes-indefinitely-if-no-exit-uncertainty}
\end{lemma}
\begin{proof}
We reason by contraposition. Consider $i\in\robots$, $s_i\in S_i$, $\ybf_i\in\Ybf_i$ and $\ubf_i\in\controls_i$ taking values in $\{\uL_i,\uH_i\}$. Assume that $\forall k^0\in\NN, \exists k\in\NN, k\geq k^0: \ubf_i(k)=\uH_i$. Then, we can build an infinite sequence of slots $\{k^0_1, k^0_2, ...\}$ where $\ubf_i$ takes value $\uH_i$. At each of them, the increase in $\inf \hat \phi_{x,i}(k,\hat s_i,\ubf_i, \ybf_i)$ within the slot is at least of $\phi_{x,i}(1,0,\uHbf_i,\dbfmin_i)>0$ (see Assumptions~\eqref{eq:assumption-uncertainty-velocity} and~\eqref{eq:assumption-uncertainty-acceleration}). Hence, we have $\lim_{k\to+\infty}\inf \hat \phi_{x,i}(k,\hat s_i,\ubf_i, \ybf_i)=+\infty$ and in particular $\inf \hat \phi_{x,i}(k,\hat s_i,\ubf_i, \ybf_i) > \xobsmax_i$ for some $k\in\NN$.
\end{proof}

\begin{lemma}
Given $i\in\robots$, for all $\hat s_i\in 2^{S_i}$ and $\ybf_i\in\Ybf_i$, there exists $k^1\in\NN$ such that:
\begin{equation}
\forall k\geq k^1, \sup \hat \phi_{v,i}(k,\hat s_i,\uLbf_i, \ybf_i)=0
\end{equation}
\label{lemma:robot-brakes-velocity-vanishes-uncertainty}
\end{lemma}
\begin{proof}
Take a priority graph $G\in\graphs$ satisfying Condition~\eqref{eq:condition-liveness-control-law-uncertainty}, a robot $i\in\robots$, a non-deterministic state $\hat s_i\in 2^{S_i}$ and a control $\ybf_i\in\Ybf_i$. Consider the trajectory $t\mapsto \phi(t,\hat s_i,\uLbf_i, \ybf_i)$ of robot $i$ under control $\uLbf_i$. At each time slot, there are two options:
\begin{itemize}
\item $\sup \hat \phi_{v,i}(t,\hat s_i,\uLbf_i, \ybf_i)$ vanishes within the slot;
\item the variation of $\sup \hat \phi_{v,i}(t,\hat s_i,\uLbf_i, \ybf_i)$ within the slot is lower than $\uL_i+\dmax_i<0$ (see Assumption~\eqref{eq:assumption-uncertainty-brake}).
\end{itemize}
In conclusion, $\sup \hat \phi_{v,i}(t,\hat s_i,\uLbf_i, \ybf_i)$ vanishes in finite time.
\end{proof}

\begin{lemma}
Consider a non-deterministic state $\hat s\in 2^S$ included in an open parallelepiped whose side along axis $i$ has length $\sigma_i$. Then, for all $i\in\robots$ satisfying $\pi_{v,i}(\sup \hat s)=0$, we have:
\begin{equation}
\hat g_i^G(\hat s):=\begin{cases}
\uL_i & \text{if } \exists (j,i)\in E(G_{|\pi_x(\inf \hat s),\rho_G}), \exists t\geq  0 \text{ s.t. } \phi_x(t,\supi(\hat{s}),\tilde{\ubf}^i, \tilde{\dbf}^i) \in \chiobs_{j\succ i} \\
\uH_i & \text{ else.}
\end{cases}
\label{eq-control-map-with-local-priority-graph-uncertainty}
\end{equation}
\label{lemma:liveness-control-law-uncertainty}
\end{lemma}
\begin{proof}[Proof of Lemma~\ref{lemma:liveness-control-law-uncertainty}]
Take $\hat s\in 2^S$ and $i\in\robots$ satisfying $\pi_{v,i}(\sup \hat s)=0$, take $i\in\robots$ and $(j,i)\in E(G)$. By construction of $\hat g^G$, we have:
\begin{equation}
\hat g_i^G(\hat s):=\begin{cases}
\uL_i & \text{if } \exists (j,i)\in E(G), \exists t\geq  0 \text{ s.t. } \phi_x(t,\supi(\hat{s}),\tilde{\ubf}^i, \tilde{\dbf}^i) \in \chiobs_{j\succ i} \\
\uH_i & \text{ else.}
\end{cases}
\end{equation}
Hence, we have to prove that for all $(j,i)\in E(G)$, if $\phi_x(t,\supi(\hat{s}),\tilde{\ubf}^i, \tilde{\dbf}^i) \in \chiobs_{j\succ i}$ for some $t\geq 0$, then we have $(j,i)\in E(G_{|\pi_x (\inf \hat s),\rho_G})$. We will prove this by contraposition.

Assume that $(j,i)\notin E(G_{|\pi_x (\inf \hat s),\rho_G})$, then we have:
\begin{equation}
\pi_x (\inf \hat s)+\rho_G \e_i \in\chifree_{j \succ i}
\label{eq:x-plus-epsilon-in-chifree-uncertainty}
\end{equation}
By assumption, we have $\pi_{v,i}(\sup \hat s)=0$, hence $\pi_{v,i}(\supi_i(\hat s)) = 0$. As the length of $\hat s$ along axis $x_i$ is $\sigma_i^x$, we have $\pi_{x,i}(\supi_i(\hat s)) \leq \pi_x(\inf\hat s_i)+ \sigma_i^x$. Hence, by order-preservation, we obtain for all $t\geq 0$:
\begin{equation}
\phi_{x,i}(t,\supi(\hat{s}),\tilde{\ubf}^i, \tilde{\dbf}^i)  \leq \pi_x(\inf\hat s_i)+ \sigma_i^x+\phi_{x,i}(t,0,\tilde{\ubf}^i_i, \tilde{\dbf}^i_i)
\end{equation}
By Equation~\eqref{eq:condition-liveness-control-law-uncertainty}, we obtain:
\begin{equation}
\phi_{x,i}(t,\supi(\hat{s}),\tilde{\ubf}^i, \tilde{\dbf}^i) \leq
\pi_x(\inf\hat s_i)+\rho_G
\end{equation}

As $t\mapsto \phi_j(t,\supi(\hat{s}),\tilde{\ubf}^i, \tilde{\dbf}^i)$ is non-decreasing, we obtain:
\begin{equation}
\phi_{x,j}(t,\supi(\hat{s}),\tilde{\ubf}^i, \tilde{\dbf}^i) \geq \pi_x (\inf\hat s_j) 
\end{equation}
By Property~\ref{property:geometric-invariance}, Equation~\eqref{eq:x-plus-epsilon-in-chifree-uncertainty} implies that $\phi_x(t,\supi(\hat{s}),\tilde{\ubf}^i, \tilde{\dbf}^i) \in\chifree_{j\succ i}$ for all $t\geq 0$. 
\end{proof}

\begin{lemma}
Given $\hat s\in 2^S$, for all $i,j\in\robots$, for all $\ubf\in\controls$ and $\dbf\in\dsignals$, the following implication holds:
\begin{equation}
\left[\exists t\geq 0: \phi_x(t,\supi(\hat s),\ubf,\dbf) \in \chiobs_{j\succ i}\right]  \Rightarrow \left[\pi_{x}(\inf \hat s_j) < \xobsmax_j\right]
\end{equation}
\label{lemma:flow-intersection-non-empty-implication-uncertainty}
\end{lemma}
\begin{proof}
Take $\hat s\in 2^S$, $i,j\in\robots$, $\dbf\in\dsignals$ and $\ubf\in\controls$. The flow $t\mapsto \phi_x(t,\supi(\hat s),\ubf,\dbf))$ is non-decreasing. Hence, we have:
\begin{equation}
\phi_{x,j}(t,\supi(\hat s),\ubf,\dbf) \geq \phi_{x,j}(0,\supi(\hat s),\ubf,\dbf) = \pi_{x}(\supi_j(\hat s))=\pi_x(\inf\hat s_j)
\end{equation}
Assume that $\phi_x(t,\supi(\hat s),\ubf,\dbf) \in \chiobs_{j\succ i} $ for some $t\geq 0$. It implies that $\phi_{x,j}(t,\supi(\hat s),\ubf,\dbf) < \sup\{x_j:x\in\chiobs_{j\succ i}\} \leq \xobsmax_j$. By transitivity, we obtain $\pi_x(\inf\hat s_j) < \xobsmax_j$.
\end{proof}

\begin{proof}[Proof of Theorem~\ref{thm:liveness-control-law-uncertainty-app}]
Take a priority graph $G\in\graphs$ satisfying Condition~\eqref{eq:condition-liveness-control-law-uncertainty}. Consider the set $I(t)$ defined as follows:
\begin{equation}
\forall t\in\RR_+, I(t):=\left\{i\in\robots: \inf \hat \phi_{x,i}(t,\hat s,\hat g^G, \ybf) < \xobsmax_i \right\}
\end{equation}
We prove Theorem~\ref{thm:liveness-control-law-uncertainty-app} by contradiction. Theorem~\ref{thm:liveness-control-law-uncertainty-app} is equivalent to the existence of $T>0$ such that:
\begin{equation}
I(T)=\emptyset
\end{equation}
Assume that for all $t \geq 0$, $I(t)\neq\emptyset$. Then, consider $\Imin$ defined as follows:
\begin{equation}
\Imin := \cap_{t\geq 0} I(t) \neq \emptyset
\end{equation}
We have for all $i\in\Imin$, 
\begin{equation}
\forall k\in\NN, \inf \hat \phi_{x,i}(k,\hat s,\hat g^G, \ybf) < \xobsmax_i
\end{equation}
By Lemma~\ref{lemma:robot-brakes-indefinitely-if-no-exit-uncertainty}, a necessary condition for the above equation to be satisfied is that for all $i\in\Imin$, there exists $k^0_i\in\NN$ such that:
\begin{equation}
\forall k\geq k_i^0, \hat g_i^G(\hat \phi(k,\hat s,\hat g ^G, \ybf))=\uL_i
\end{equation}
Hence, for $k^0:=\max_{i\in\Imin} k^0_i$, we have:
\begin{equation}
\forall i\in\Imin, \forall k\geq k^0, \hat g_i^G(\hat \phi(k,\hat s,\hat g ^G, \ybf))=\uL_i
\end{equation}
By Lemma~\ref{lemma:robot-brakes-velocity-vanishes-uncertainty}, the velocity of all robots $i\in\Imin$ will vanish in finite time. There exists $k^1\geq k^0$ such that:
\begin{eqnarray}
\forall k \geq k^1, \hat g_i^G(\hat \phi(k,\hat s,\hat g ^G, \ybf))&=&\uL_i\\
\sup \hat \phi_{v,i}(k,\hat s,\hat g^G, \ybf)&=&0
\end{eqnarray}
Take $k \geq k^1$ big enough, so that $I(k)=\Imin$. By Lemma~\ref{lemma:liveness-control-law-uncertainty}, for all $i\in\Imin$, there exists $(j,i)\in E(G_{|\inf \hat \phi_x(k,\hat s,\hat g ^G, \ybf),\rho_G})$ such that $\phi_x(\RR_+,\supi(\hat\phi(k,\hat s,\hat g ^G, \ybf)),\tilde{\ubf}^i, \tilde{\dbf}^i)\cap \chiobs_{j\succ i}\neq \emptyset$ and we have necessarily $\inf \hat \phi_{x,j}(k,\hat s,\hat g^G, \ybf) <\xobsmax_j$ by Lemma~\ref{lemma:flow-intersection-non-empty-implication-uncertainty}, i.e., $j\in I(k)=\Imin$. As a result, $G_{|\inf \hat \phi_x(k,\hat s,\hat g ^G, \ybf),\rho_G}$ should be cyclic (there would be a cycle involving vertices in $\Imin$). It is absurd.
\end{proof}

\part*{Résumés en français}
\addcontentsline{toc}{part}{Résumés en Français}
{\pagestyle{plain}
\parttoc
\cleardoublepage}



\section*{Supplementary material}

# Coordination de robots mobiles
## par affectation de priorités

**RESUME :** Depuis la fin des années 1980, le développement de véhicules autonomes est un champ de recherche très actif dans la plupart des grands pays industrialisés. La diminution du nombre d'accidents, des temps de trajet plus courts, une meilleure efficacité énergétique et des besoins en infrastructure plus limités, sont autant d'effets socio-économiques espérés de leur déploiement. Des formes de coopération inter-véhicules et entre les véhicules et l'infrastructure routière sont nécessaires au fonctionnement sûr et efficace du système de transport dans sa globalité. Cette thèse s'intéresse à une forme de coopération particulière en étudiant la coordination de robots mobiles aux intersections. La majorité des systèmes de coordination existants planifie une trajectoire que les robots doivent exécuter afin d'assurer l'absence de collision. C'est une approche classique de la planification, qui est alors considérée comme un mécanisme de génération de l'action. Dans cette thèse, seules les priorités entre les robots sont planifiées – l'ordre relatif de passage dans l'intersection – ce qui est bien plus faible car un grand nombre de trajectoires respectent les même priorités. Plus précisément, les priorités encodent les classes d'homotopie de solutions au problème de coordination. Affecter les priorités revient à choisir une certaine classe d'homotopie pour résoudre le problème de coordination sans planifier une trajectoire particulière. Une fois les priorités affectées, les robots suivent une loi de contrôle qui s'assure de leur respect, c'est-à-dire que la trajectoire effectuée appartient bien à la classe d'homotopie choisie. Il en découle un système de coordination robuste, capable de gérer toute une classe d'événements imprévisibles de façon réactive, ce qui est particulièrement adapté à la coordination de véhicules autonomes aux intersections où voitures, transports en commun et piétons partagent la route.

**PARTIE I :** La première partie de cette thèse s'intéresse à définir formellement une notion de priorité pour le problème de coordination de robots mobiles le long de chemins fixés. Dans le langage commun des automobilistes, la priorité entre deux véhicules définit leur ordre de passage. Pour définir formellement le concept dans le problème de coordination qui nous intéresse, nous utilisons l'approche standard par l'espace de coordination qui transforme le problème de coordination dans l'espace réel en un problème de recherche de chemin dans un espace abstrait appelé l'espace de coordination, de dimension n où n désigne le nombre de véhicule coordonnés. L'évitement de collision dans l'espace réel est équivalent à l'évitement d'une zone de collision dans l'espace de coordination. Cette dernière a une structure bien particulière, puisqu'elle est constituée d'une union de cylindres de collision, chacun correspondant aux collisions potentielles entre chaque couple de robots. La contribution de cette première partie consiste à utiliser cette structure géométrique particulière afin d'équiper l'espace de coordination d'une notion de priorité. Nous proposons un nouveau cadre théorique – basé sur la notion de priorité – pour étudier le problème de coordination. Comme un chemin dans l'espace de coordination doit éviter chaque cylindre de collision, il doit nécessairement passer d'un côté ou de l'autre de ce dernier. Ainsi, nous montrons qu'un chemin sans collision dans l'espace de coordination induit une relation binaire entre chaque couple de robots. Il apparait alors que choisir de quel coté passer par rapport à un cylindre de collision est équivalent à choisir l'ordre de passage entre deux robots, justifiant ainsi d'appeler priorité cette relation binaire. Nous montrons alors deux résultats fondamentaux de la relation de priorité. Le premier stipule que l'ensemble des chemins sans collision dans l'espace de configuration respectant des priorités

données forme une classe de chemins homotopiques, c'est-à-dire continument déformables les uns vers les autres tout en respectant les priorités données. Ainsi, fixer les priorités revient à choisir une classe d'homotopie particulière pour résoudre le problème de coordination : les priorités structurent topologiquement l'espace de coordination. Le second résultat porte sur la question suivante: existe-t-il nécessairement une solution au problème de coordination respectant des priorités données ? Nous mettons en évidence que des cycles dans les priorités peuvent générer des chaînes de blocage où chaque robot de la chaîne est arrêté et attend le suivant pour pouvoir continuer la traversée de l'intersection. Cependant, l'absence de cycles n'est pas une condition nécessaire à l'existence d'une solution respectant des priorités données, et nous présentons une condition nécessaire et suffisante – une caractérisation – des priorités dites faisables, c'est-à-dire pour lesquelles il existe bien une solution au problème de coordination satisfaisant les priorités données. Cette caractérisation raffine le rôle des cycles dans la formation de blocage. Une façon géométrique de comprendre pourquoi tous les graphes de priorités ne sont pas faisables est de noter que selon la position relative des cylindres de collision dans l'espace de configuration, le choix de passer d'un coté ou de l'autre de chaque cylindre, ne peut pas toujours se faire indépendamment pour chacun d'eux. Enfin, nous montrons que choisir des priorités faisables résout entièrement le problème de l'évitement de blocage : plus précisément, soit les priorités choisies ne sont pas faisables et le système va inévitablement atteindre une configuration de blocage, soit elles sont faisables auquel cas il est garanti que le système ne peut jamais se trouver bloqué. C'est une propriété très intéressante dans la mesure où l'évitement de blocage est un problème complexe, qu'il est difficile de résoudre de façon réactive. En définitive, cette partie de la thèse justifie l'utilisation des priorités comme plan de coordination. Contrairement aux approches traditionnelles de planification de trajectoires, le plan est ici le graphe de priorités résumant de façon synthétique la stratégie de coordination en précisant l'ordre relatif de passage entre les robots, et non pas leurs trajectoires spatio-temporelles. Ce plan ne peut être simplement exécuté par un contrôleur en charge du suivi de trajectoires puisque les trajectoires proprement dites ne sont pas planifiées. Le plan de coordination doit être interprété, utilisé par les robots comme une ressource de coordination afin de traverser l'intersection. Ceci est le sujet de la seconde partie de la thèse qui s'intéresse au contrôle de robots à priorités fixées.

**PARTIE II :** Dans la seconde partie de cette thèse, nous supposons disposer d'un système de robots mobiles se déplaçant le long de chemins fixés et dont les priorités telles que définies en première partie sont fixées et supposées faisables. Nous nous proposons alors de mettre au point une loi de contrôle permettant aux robots de traverser l'intersection en respectant les priorités données. La loi de contrôle doit assurer que le chemin parcouru par le système dans l'espace de coordination reste bien dans la classe d'homotopie correspondant aux priorités fixées. Dans un premier temps, les robots sont supposés contrôlés en vitesse et une loi de type Bang-Bang est proposée. Celle-ci consiste pour chaque robot à tenter d'avancer à vitesse maximale et à s'arrêter aussi tard que possible, juste à temps pour laisser passer un robot prioritaire. Cette loi de contrôle garantit le respect des priorités, l'optimalité à priorités fixées de la trajectoire résultante, et assure également que tous les robots traverseront l'intersection en temps fini pourvu que les priorités fixées soient bien faisables. Dans ces premiers résultats, l'inertie des robots est totalement négligée puisqu'ils sont supposés être capables de s'arrêter instantanément. C'est pourquoi une loi de contrôle pour des robots contrôlés en accélération est ensuite proposée, intégrant les distances de freinage, et consistant à anticiper le besoin de freiner pour laisser passer des robots prioritaires. La sécurité et la traversée en temps fini sont à nouveau garanties. D'autre part, cette loi de contrôle présente une

robustesse particulière vis-à-vis de toute une classe d'événements imprévus qui nécessitent le freinage d'un, de plusieurs, ou de tous les robots. A tout moment, un robot peut freiner pour réagir à un événement imprévu (perte de signal des capteurs, obstacle détecté sur la voie, etc.), et les autres robots sont assurés de respecter malgré tout les priorités fixées en appliquant simplement la loi de contrôle. C'est une propriété originale en comparaison avec les méthodes standards de coordination par planification de trajectoires qui nécessitent soit le suivi de la trajectoire planifiée – ce qui n'est pas possible si un robot doit freiner pour réagir à un événement imprévu –, soit une replanification coûteuse en temps de calcul et ne garantissant pas une réponse du système en temps réel. Le dernier chapitre de cette partie s'intéresse enfin aux cas de robots toujours contrôlés en accélération, mais où des incertitudes bornées sont présentes à la fois sur le contrôle des robots et sur les observations. Dans ce cadre non déterministe, la loi de contrôle proposée présente toujours les propriétés recherchées de sécurité, de robustesse et de garantie de traversée en temps fini.

**PARTIE III :** La première partie de cette thèse a suggéré l'utilisation des priorités pour guider les robots au sein de l'intersection et la deuxième partie a fourni des solutions au contrôle de robots à priorités fixées, laissant de coté de nombreux aspects de la mise au point d'un système de coordination multi-robot. Tout d'abord, un tel système est ouvert puisque des robots arrivent dans l'intersection et en sortent au cours du temps. Ainsi, les priorités doivent être affectées de façon dynamique. De plus, l'affectation des priorités et le contrôle des robots à priorités fixées doivent s'exécuter en parallèle. La dernière partie de cette thèse se propose donc de fournir une architecture de contrôle pour un système de coordination de robots mobiles basé sur l'affectation de priorités. Nous adoptons une architecture hybride comprenant trois couches : une couche purement délibérative gérée par une infrastructure centrale, une couche purement réactive distribuée sur chaque robot, et une couche intermédiaire – le séquenceur – également distribuée et jouant un rôle d'interface. Le contrôleur d'intersection est l'infrastructure centrale en charge d'autoriser les robots à pénétrer dans la zone de contrôle et d'affecter les priorités. C'est une tâche purement délibérative puisqu'elle consiste à raisonner sur le futur pour effectuer un choix efficace d'affectation des priorités. La tâche du séquenceur des robots est de communiquer avec le contrôleur de l'intersection afin d'obtenir le droit d'entrer dans la zone de contrôle et d'être affecté une priorité vis-à-vis des autres robots. La couche réactive implémentée par les robots consiste en un ensemble de comportements simples implémentés par chacun d'eux. Tout d'abord, un robot ne doit pas entrer dans la zone de contrôle tant qu'il n'y a pas été accepté par le contrôleur d'intersection. D'autre part, une fois accepté et affecté une priorité par rapport aux autres robots, il doit implémenter une loi de contrôle assurant le respect des priorités fixées telle que celles présentées en deuxième partie. Enfin, le robot peut également implémenter d'autres comportements lui permettant de réagir à son environnement comme la détection d'obstacles ou de piétons sur la voie, l'amenant par exemple à freiner. La robustesse au freinage de la loi de contrôle présentée en deuxième partie apparait alors particulièrement intéressante puisqu'elle autorise un comportement du robot à requérir un freinage – par exemple pour éviter un obstacle sur la voie – tout en conservant la garantie du respect des priorités, et donc l'absence de collisions. De plus, une fois que l'élément requérant le freinage a disparu, le système peut redémarrer en conservant les mêmes priorités et en appliquant la même loi de contrôle sans replanification. Les simulations présentées dans cette dernière partie tentent de montrer les avantages qualitatifs significatifs, en termes de robustesse, du système de coordination proposé par rapport à l'état de l'art.


\end{document}